 \documentclass[mnsc,nonblindrev]{informs3} 

\OneAndAHalfSpacedXI                                            

\usepackage{natbib}
 \bibpunct[, ]{(}{)}{,}{a}{}{,}%
 \def\bibfont{\small}%
 \def\bibsep{\smallskipamount}%
 \def\bibhang{24pt}%
\TheoremsNumberedThrough     
\ECRepeatTheorems

\EquationsNumberedThrough    

\MANUSCRIPTNO{}

\usepackage{bbm,cases,xr}
\usepackage{amsfonts, amssymb, amsmath,comment}
\usepackage{blindtext}
\usepackage[ruled,vlined]{algorithm2e}
\usepackage{graphicx}
\usepackage{caption}
\usepackage{subcaption}
\usepackage[inline]{enumitem}
\usepackage{hyperref}
\usepackage{url}
\usepackage[title]{appendix}
\usepackage{cleveref}
\usepackage{mathrsfs}
\usepackage{xcolor}
\usepackage{thmtools, thm-restate}
\usepackage{amsmath}
\usepackage{amssymb}
\usepackage{mathtools}
\usepackage{minitoc}
\usepackage{tikz}
\usepackage{booktabs}
\usepackage{pifont}   
\newcommand{\cmark}{\ding{51}}
\newcommand{\xmark}{\ding{55}}
\newcommand{\cxmark}{{\ding{51}{\small\kern-0.7em\ding{55}}}}

\usetikzlibrary{arrows.meta, positioning, calc}

\newcommand{\bx} {{\boldsymbol{x}}}
\newcommand{\ba} {{\boldsymbol{a}}}
\newcommand{\bA} {{\boldsymbol{A}}}
\newcommand{\bB} {{\boldsymbol{B}}}

\newcommand{\bv} {{\boldsymbol{v}}}
\newcommand{\bp} {{\boldsymbol{p}}}

\newcommand{\bz} {{\boldsymbol{z}}}
\newcommand{\bzt}{{\boldsymbol{z}_t}}

\newcommand{\btheta} {{\boldsymbol{\theta}}}

\newcommand{\phigh}{p_h}
\newcommand{\plow}{p_l}
\newcommand{\Real}{\mathbb{R}}
\newcommand{\AssortSet}{\mathscr{S}}
\newcommand{\PriceSet}{\mathscr{P}}
\newcommand{\PolicySet}{\mathcal{A}}
\newcommand{\ChoiceProb}{q}
\newcommand{\ChoiceProbt}[2]{\ChoiceProb_{t}(#2;#1)}
\newcommand{\ChoiceProbn}[2]{\ChoiceProb_{n}(#2;#1)}
\newcommand{\PMNL}{\texttt{PMNL}}
\newcommand{\MNL}{\texttt{MNL}}
\newcommand{\FM}{\texttt{FM23}}

\newcommand{\Poi}{\texttt{Poi}}
\newcommand{\BaseArrival}{\Lambda}
\newcommand{\Arrival}{\Lambda}

\newcommand{\Feature}{{\bz}}
\newcommand{\Featuret}{{\bzt}}
\newcommand{\DimFeature}{{d_\bz}}
\newcommand{\DimRank}{{d_\bx}}
\newcommand{\xs}{\bx(S_s, \bp_s)}

\newcommand{\Choice}[2]{{C_{#1}^{(#2)}}}
\newcommand{\Choiceti}{{\Choice{t}{i}}}
\newcommand{\LowerC}[1]{c_{l,#1}}
\newcommand{\revenue}[2][]{%
  \if\relax\detokenize{#2}\relax
    {{r(S, \bp, \bz_t; #2)}}%
  \else
    {r(S_{#1}, \bp_{#1}, \bz_t; #2)}%
  \fi
}
\newcommand{\Revenue}[1][]{%
  \if\relax\detokenize{#1}\relax
    {{R_t(S, \bp)}}%
  \else
    {R_t(S_{#1}, \bp_{#1})}%
  \fi
}
\newcommand{\RevenueUCB}[1][]{%
  \if\relax\detokenize{#1}\relax
    {{\bar{R}_t(S, \bp)}}%
  \else
    {\bar{R}_t(S_{#1}, \bp_{#1})}%
  \fi
}

\newcommand{\hv}[1][]{%
  \if\relax\detokenize{#1}\relax
    \widehat{\bv}
  \else
    \widehat{\bv}_{#1}
  \fi
}

\newcommand{\htheta}[1][]{%
  \if\relax\detokenize{#1}\relax
    \widehat{\btheta}
  \else
    \widehat{\btheta}_{#1}
  \fi
}

\newcommand{\Lik}[1]{\gL_{#1}(\btheta, \bv) }
\newcommand{\LikPoi}[1]{\mathcal{L}^{\Poi}_{#1}(\btheta) }
\newcommand{\likPoi}[1]{\ell^{\Poi}_{#1}(\btheta) }
\newcommand{\LikMNL}[1]{\gL^{\MNL}_{#1}(\bv) }
\newcommand{\likMNL}[1]{\ell^{\MNL}_{#1}(\bv) }

\definecolor{myblue}{RGB}{0, 51, 153}

\crefname{assumption}{assumption}{assumptions}
\Crefname{assumption}{Assumption}{Assumptions}
\newtheorem{instance}{Instance}[section]

\usepackage{amsmath,amsfonts,bm}

\def\eqref#1{(\ref{#1})}

\def\ceil#1{\lceil #1 \rceil}

\def\1{\bm{1}}

\DeclareMathAlphabet{\mathsfit}{\encodingdefault}{\sfdefault}{m}{sl}
\SetMathAlphabet{\mathsfit}{bold}{\encodingdefault}{\sfdefault}{bx}{n}

\def\gE{{\mathcal{E}}}

\def\gL{{\mathcal{L}}}

\newcommand{\E}{\mathbb{E}}

\newcommand{\R}{\mathbb{R}}

\newcommand{\ran}[1]{{\color{red}  #1 [Ran]}}

\newenvironment{carlist}
 {\begin{list}{$\bullet$}
 {\setlength{\topsep}{0in} \setlength{\partopsep}{0in}
  \setlength{\parsep}{0in} \setlength{\itemsep}{\parskip}
  \setlength{\leftmargin}{0.07in} \setlength{\rightmargin}{0.08in}
  \setlength{\listparindent}{0in} \setlength{\labelwidth}{0.08in}
  \setlength{\labelsep}{0.1in} \setlength{\itemindent}{0in}}}
 {\end{list}}

\newcommand{\bcar}{\begin{carlist}}
\newcommand{\ecar}{\end{carlist}}
\makeatletter
\long\def\@makecaption#1#2{
        \vskip 0.8ex
        \setbox\@tempboxa\hbox{\small {\bf #1:} #2}
        \parindent 1.5em  
        \dimen0=\hsize
        \advance\dimen0 by -3em
        \ifdim \wd\@tempboxa >\dimen0
                \hbox to \hsize{
                        \parindent 0em
                        \hfil 
                        \parbox{\dimen0}{\def\baselinestretch{0.96}\small
                                {\bf #1.} #2
                                } 
                        \hfil}
        \else \hbox to \hsize{\hfil \box\@tempboxa \hfil}
        \fi
        }
\makeatother

\newcommand{\Prob}{\ensuremath{\mathbb{P}}}
\newcommand{\Exs}{\ensuremath{\mathbb{E}}}
\newcommand{\Exp}{\ensuremath{\Exs}}

\newcommand{\Regret}[2]{\ensuremath{\mathcal{R}^{#2}(#1)}}

\usepackage{algorithmic}
\newcommand{\matsnorm}[2]{\left|\!\left|\!\left| #1 \right|\!\right|\!\right|_{{#2}}}

\newcommand{\opnorm}[1]{\ensuremath{\matsnorm{#1}{\text{\mbox{op}}}}}

\newcommand{\defn}{\ensuremath{:\,\! =}}

\let\svthefootnote\thefootnote
\newcommand\blfootnote[1]{%
  \let\thefootnote\relax%
  \footnotetext{#1}%
  \let\thefootnote\svthefootnote%
}

\begin{document}



\RUNTITLE{Poisson MNL}

\TITLE{Poisson-MNL Bandit: \\Nearly Optimal Dynamic Joint Assortment and Pricing with Decision-Dependent Customer Arrivals}

\ARTICLEAUTHORS{%
\AUTHOR{Junhui Cai}
\AFF{Department of Information Technology, Analytics, and Operations,
University of Notre Dame,
\EMAIL{jcai2@nd.edu}}
\AUTHOR{Ran Chen}
\AFF{Department of Statistics and Data Science, Washington University in St. Louis,
\EMAIL{ran.c@wustl.edu}}
\AUTHOR{Qitao Huang}
\AFF{Department of  Mathematics, Tsinghua University,
\EMAIL{qitaohuang@tsinghua.edu.cn}}
\AUTHOR{Linda Zhao}
\AFF{Department of Statistics and Data Science,
University of Pennsylvania,
\EMAIL{lzhao@wharton.upenn.edu}}
\AUTHOR{Wu Zhu}
\AFF{Department of Finance, Tsinghua University,
\EMAIL{zhuwu@sem.tsinghua.edu.cn}}
\footnote{Authorship is in alphabetic order.}

} 

\ABSTRACT{ 
We study dynamic joint assortment and pricing where a seller updates decisions at regular accounting/operating intervals to maximize the cumulative per-period revenue over a horizon 
$T$. 
In many settings, assortment and prices affect not only what an arriving customer buys but also how many customers arrive within the period, whereas classical multinomial logit (MNL) models assume arrivals as fixed, potentially leading to suboptimal decisions. 
We propose a Poisson–MNL model that couples a contextual MNL choice model with a Poisson arrival model whose rate depends on the offered assortment and prices.
Building on this model, we develop an efficient algorithm $\PMNL$ based on the idea of upper confidence bound (UCB). We establish its (near) optimality by proving a non-asymptotic regret bound of order $\sqrt{T\log{T}}$ and a matching lower bound (up to $\log{T}$). 
Simulation studies underscore the importance of accounting for the dependency of arrival rates on assortment and pricing: $\PMNL$ effectively learns customer choice and arrival models and provides joint assortment-pricing decisions that outperform others that assume fixed arrival rates.
}


\maketitle
\thispagestyle{empty}
\newpage
\setcounter{page}{1}

\pagebreak
\section{Introduction}
\label{sec:intro}

Assortment (what products to offer) and pricing (at what prices) are among the central decision problems in revenue management. 
In many retail and platform settings, these decisions are made at regular ``accounting'' intervals---daily, weekly, or on other operational cycles \citep{ma2018dynamic, brown2023competition, aparicio2023algorithmic}. Over such a period, the offered assortment and prices influence revenue through two channels: \emph{how many customers arrive} and \emph{what those customers purchase}.
Individual customer purchase behavior is often modeled via discrete choice models, most notably the multinomial logit (MNL) models, focusing on optimizing expected revenue per arriving customer, while the arrival process is typically taken as fixed. 
In practice, however, customer arrivals are affected by assortment and pricing decisions: a more attractive selection or more competitive prices can pull in additional traffic \citep{kahn1995consumer, wang2021consumer}. 
Therefore, models that ignore such decision-dependent arrivals can lead to suboptimal decisions for per-period revenue maximization.

In this paper, we incorporate customer arrivals into the choice-based modeling for the dynamic joint assortment-pricing problem. 
Our goal is to sequentially determine assortment and pricing decisions $\{S_t, \bp_t\}_{t=1}^T$ 
among $N$ products to maximize cumulative expected revenue over a time horizon $T$.
Under a classical MNL model, given an assortment set $S\subseteq [N] = \{1,2,\ldots,N\}$ and prices $\bp = \{p_j\}_{j\in S}$, an arriving customer purchases product $j\in S$ with probability $\ChoiceProb_j(S, \bp) = \frac{\exp(v_j - p_j)}{1+\sum_{k\in S} \exp(v_k - p_k)}$, where $v_j$ denotes the intrinsic value of product $j$ and can be further contextualized by $\DimFeature$-dimensional product features $\Feature$.
Consequently, the classical approach aims at maximizing the expected customer-wise cumulative reward based on the \emph{per-customer} expected reward $\sum_{j\in S} \ChoiceProb_j(S, \bp)p_j$.
However, a key aspect is overlooked: the number of customer arrivals in each period is random and depends on the assortment-pricing decision.
To capture the effect of assortment and pricing on arrivals, we model the number of arrivals in each period as a Poisson count with mean proportional to a decision-dependent arrival rate $\lambda(S, \bp)$, i.e., the arrival rate is a function of the assortment $S$ and prices $\bp$. 
Combining the arrival model and choice model together, given an assortment $S$ and prices $\bp$, the conditional mean of the \emph{per-period} reward for this period is:
\begin{align*}
    \E[R | S, \bp\,] \propto \lambda(S, \bp) \cdot \sum_{j\in S} \ChoiceProb_j(S, \bp)p_j.
\end{align*}

This intuitive form highlights the coupling between per-customer revenue and per-period revenue induced by an assortment-pricing decision: a decision can improve conversion or margins yet reduce customer arrivals enough to lower total period revenue; conversely, one that appears worse on a per-customer revenue can be optimal if it attracts substantially more customers. For example, discounting a ``magnet'' product may reduce sales of other products through substitution, but still increase overall revenue by drawing in more arrivals during the period; similarly, expanding the assortment by adding an item may dilute purchase probabilities, yet increase arrivals by making the offer set more attractive.
These examples underscore that the effect of assortment and pricing on arrivals can take many forms.
The key, then, is to develop a flexible model for the arrival rate that captures the dependency of arrivals on the assortment-pricing decision. 
To this end, instead of confining this dependency to a specific functional form, we parametrize $\lambda(S, \bp)$ through a rich set of basis functions of $(S, \bp)$, allowing the model to be expressive enough to accommodate a wide range of dependence structures. 

Since neither customer arrival nor choice behavior is known a priori, both must be learned from data generated under past decisions while maximizing cumulative reward over time. 
Specifically, in each period, the firm chooses an assortment and prices, observes the realized customer arrivals and purchases, and updates its inference on the model parameters using all available observations.
Such an online learning problem falls under the umbrella of \textit{bandit problems} \citep{lattimore2020bandit}, where a key challenge is to balance ``exploration'' and ``exploitation'' over action space, i.e., feasible assortments and prices, so as to maximize cumulative reward, or equivalently, minimize \textit{cumulative regret} relative to an oracle policy that knows the model parameters and always selects the optimal action. 
We propose a new algorithm, Poisson-MNL (\PMNL), which jointly learns the parameters of the Poisson arrival model and the MNL choice model and provides dynamic assortment and pricing decisions using an upper confidence bound (UCB) approach.
We establish regret bounds showing that $\PMNL$ is optimal up to a logarithmic factor in the time horizon $T$ 
and show in simulation studies that it outperforms benchmarks that assume a fixed arrival rate when customer arrival is influenced by the offered assortment and pricing.

Let us summarize some of our main contributions:
\begin{enumerate}
\item \textbf{A decision-dependent arrival-choice model for joint assortment-pricing.} 
We propose a Poisson-multinomial logit model (Poisson-MNL) that captures customer \emph{arrivals} and customer \emph{choices}.
Specifically, we model customer arrivals within each decision period via a Poisson model with a rate that depends on the offered assortment and corresponding prices, while we model customer purchase behavior using a contextual MNL model incorporating product features and prices. 
This formulation captures how the assortment-pricing decision influence both the volume of customer arrivals and their purchase outcomes. 
Ignoring the dependency of customer arrivals on assortment and pricing, as in classical MNL models, can lead to suboptimal policies when decisions are made per-period rather than per-customer. 

Our model is also flexible and expressive.
For the arrival model, we do not impose a pre-specified relationship between the arrival rate and the assortment-pricing decision. Instead, it suffices to identify a rich enough set of basis functions such that the log arrival rate is linear in these bases.
For the choice model, we allow products to be characterized by a set of observable features, enabling learning at the attribute level and generalization to newly introduced products.
We further allow these features (e.g., customer ratings) to change over time. 
Our framework nests the standard MNL model as a special case when the arrival rate is assumed to be fixed (i.e., the arrival model coefficients are zero).

\item \textbf{A nearly optimal efficient algorithm ($\PMNL$) for online learning.}
We develop an efficient online algorithm that sequentially provides joint assortment-pricing decisions based on observations available up to each period so as to maximize the cumulative expected revenue over a given time horizon $T$. 
The challenge lies in efficiently learning both the Poisson arrival and the MNL choice parameters based on highly dependent observations, while strategically taking actions (joint assortment-pricing) that balance the exploration-exploitation trade-off.

Our algorithm adopts a two-stage design based on maximum likelihood estimation (MLE). The first stage conducts $O(\log T)$ rounds of exploration to obtain sufficiently accurate initial parameter estimates. In the second stage, we take an upper confidence bound (UCB) approach to explore and exploit: in each round, it selects an assortment and prices by maximizing an upper confidence bound on the per-period expected reward.

Constructing both the estimator and the upper bound are substantially more challenging than in standard MNL bandit models, due to three intertwined difficulties: (i) Poisson arrivals introduce unbounded and non-sub-Gaussian random variables that requires new tools beyond those used in standard MNL analyses; (ii) arrival randomness complicates the analysis of the choice model parameter estimation by bringing in extra randomness to the observations of purchase outcomes in addition to the choice randomness; and (iii) the unknown arrival parameters enter MNL error bounds, whereas the algorithm requires confidence bounds to be free from these unknowns. We address these challenges by leveraging concentration inequalities for martingales with increments satisfying Bernstein conditions; carefully designing and analyzing estimators and statistics whose randomness is dominated by that induced by the choice model, and deriving fully data-dependent error bounds. 
We further sharpen the constants in the error bounds with new analytical tools, resulting in significantly improved practical performance.

\item \textbf{Non-asymptotic regret bounds.} 
We establish a non-asymptotic upper bound for the expected regret of $\PMNL$ of order $\sqrt{T\log{T}}$ for all $T$, and a matching non-asymptotic lower bound of the order $\Omega(\sqrt{T})$, implying that our algorithm is nearly minimax optimal. 

To our knowledge, both bounds are the first of their kind for dynamic joint assortment and pricing with contextual information and decision-dependent arrivals. To compare with classical MNL bounds, note that under specific configurations of assumptions for our model parameters, the arrival rate is decision-irrelevant, the choice model is price-irrelevant, and the difference with the classical MNL framework reduces to whether the assortment and pricing decisions can be changed per-customer or per-period. Due to the constant arrival rate, the upper and lower bounds under our setup can be naturally translated to the classical MNL framework while retaining the rate, after reindexing $T$ from periods to customers. Notably, our upper bound improves upon the state-of-the-art results ($O(\sqrt{T}\log{T})$) and our lower bound is free from the stringent conditions required by existing MNL lower bounds. In addition to the non-asymptotic lower bound, we also provide an asymptotic lower bound showing dependence on parameter dimensions.

\end{enumerate}

\subsection{Related Literature}


\paragraph{Dynamic Assortment and Pricing.}

Our paper first contributes to the extensive literature on dynamic assortment and pricing. 
In dynamic assortment, discrete choice model is widely accepted, in particular the multinomial logit (MNL) choice model \citep{rusmevichientong2010dynamic, chen2017note, agrawal2019mnl, chen2020dynamic}. The MNL framework has been extended to contextual settings by incorporating product features \citep{cheung2017thompson, chen2020dynamic, miao2022online, lee2025low}, as well as to adversarial contextual settings \citep{perivier2022dynamic,lee2024nearly}. 
For contextual MNL bandits, \cite{chen2020dynamic} and \cite{oh2021multinomial} establish a regret bound of $O(d_z \sqrt{T}\log T)$ using different proof strategies. 
For lower bounds, \cite{chen2020dynamic} establish an asymptotic lower bound of $\Omega(d_z \sqrt{T}/K)$ and \cite{lee2024nearly} establish $\Omega(d_z \sqrt{T/K})$ for a class of algorithms that select the same product $K$ times in each period.
Our Poisson-MNL model is motivated by practical operational settings where decisions are made per-period instead of per-customer and incorporates 
a Poisson model to account for the random arrival number that depends on the decision.
Under specific configurations of assumptions for our model, our regret bound can be naturally translated to the contextual MNL frameworks with a non-asymptotic upper bound of order $d_z \sqrt{T\log T}$, a non-asymptotic lower bound of order $\sqrt{T}$, and an asymptotic lower bound of order $d_z \sqrt{T}$ ($K$ holds constant) all under weaker assumptions.




Dynamic pricing is another major area in revenue management
\citep{kleinberg2003value, araman2009dynamic, besbes2009dynamic, broder2012dynamic, den2014simultaneously, keskin2014dynamic}. 
Recent research in dynamic pricing further considers customer characteristics \citep{ban2021personalized, chen2021nonparametric, bastani2022meta} and product features \citep{qiang2016dynamic, javanmard2019dynamic, cohen2020feature, miao2022context, fan2022policy}. Much of this literature focuses on single-product settings, while many sellers must price multiple products simultaneously. 
The (multinomial/nested) logit models has been applied to multi-product pricing problems \citep{akccay2010joint, gallego2014multiproduct}
and more recent work incorporates product features into the demand model \citep{javanmard2020multi, ferreira2023demand}.
For example, \cite{javanmard2020multi} establish a regret bound of $O(\log(T d_z) (\sqrt{T} + d_z\log T))$ and a lower bound of $\Omega(\sqrt{T})$. 
Similar to dynamic assortment, our bounds can be translated to such dynamic pricing settings and are better. More importantly, our focus is on dynamic \textit{joint} assortment and pricing decisions made \textit{per-period}, with a random number of customer arrivals whose rate depends on the decision.

While dynamic assortment and pricing problems have been studied
separately and extensively, research on the joint
assortment-pricing problem is relatively sparse.
\cite{chen2022statistical} study this problem in an offline
setting, and \citet{miao2021dynamic} propose a Thompson-sampling based algorithm
using an MNL choice model with product-specific mean utilities and price sensitivities.
In contrast, we incorporate the product feature and, more importantly,  allow customer arrivals to vary depending on the offered assortment and prices.
\Cref{tab:assortment_pricing_comparison} compares our methods with related MNL-based methods.




\begin{table}[t]
\centering
\caption{Comparison of $\PMNL$ and related MNL methods for dynamic assortment or/and pricing problems with $T$ rounds, $N$ products, maximum assortment size $K$, and if any, $d_z$-dimensional product features (contexts) and $d_x$-dimensional set of basis functions on assortment-pricing for the arrival rate model.}
\label{tab:assortment_pricing_comparison}
\scalebox{0.8}{
\begin{tabular}{lcccccc}
\toprule
Method / Paper 
& Assortment
& Pricing 
& Arrival 
& Context
& Upper bound 
& Lower bound \\
\midrule
This paper ($\PMNL$) 
& \cmark & \cmark & \cmark & \cmark & $O((d_z+d_x) \sqrt{T\log T})$ & $\Omega((d_z+\sqrt{d_x})\sqrt{T})$ \\ 

\cite{agrawal2019mnl} & \cmark & \xmark & \xmark & \xmark & $O( \sqrt{NT\log (NT)})$ & $\Omega(\sqrt{NT/K})$  \\

\cite{chen2020dynamic} 
& \cmark & \xmark & \xmark & \cmark & $O(d_z \sqrt{T}\log T)$ & $\Omega(d_z \sqrt{T}/K)$  \\

\cite{oh2021multinomial} 
& \cmark & \xmark & \xmark & \cmark & $O(d_z \sqrt{T}\log T)$ & $\Omega(d_z \sqrt{T}/K)$  \\

\cite{lee2024nearly} & \cmark & \xmark & \xmark & \cmark & 
\begin{tabular}{@{}c@{}}$O(d_z \sqrt{T/K})$ \\(Adversarial) 
\end{tabular}  
& \textcolor{gray}{$\Omega(d_z \sqrt{T/K})$} \\


\cite{javanmard2020multi} & \xmark & \cmark & \xmark & \cmark & $O(\log(T d_z) (\sqrt{T} + d_z\log T))$ & $\Omega(\sqrt{T})$ \\


\cite{miao2021dynamic} &  \cmark & \cmark & \xmark & \xmark & $O(\sqrt{NT}\log (NT))$ & -- \\

\cite{ferreira2023demand} & \xmark & \cmark & \cxmark & \cmark & -- & --  \\

\bottomrule
\end{tabular}
}
\end{table}

\paragraph{Bandits.}

The dynamic assortment and pricing problems are closely related to the bandit problem, which dates back to the seminal work
of~\citet{robbins1952some}.
In each round, a decision-maker chooses
an action (arm) and then observes a reward. The goal is to act strategically to minimize cumulative regret.
There is now an extensive literature
on the bandit problem, including multi-armed bandits, contextual bandits, linear bandits, and generalized linear bandit; we
refer the reader to the comprehensive book
by~\citet{lattimore2020bandit} and references therein for more
background. 
The dynamic assortment problem can be casted as a multi-armed bandit problem: for example, each feasible assortment can be treated as an arm. Such a naive formulation, however, results in $N\choose K$ arms and thus suffers from the curse of dimensionality. 
MNL bandits provide one tractable way to impose structure on this combinatorial action space through an MNL choice model \citep{agrawal2019mnl}
and contextual MNL bandits further account for product features \citep{chen2020dynamic}. While the MNL structure is widely appreciated in the operations management community due to its connection to utility theory, it is highly specified and therefore requires careful, model-specific analysis, as already demonstrated by aforementioned assortment/pricing literature, and even more so for our Poisson–MNL model built upon it.

\paragraph{Customer Arrivals.}
Customer arrival is a key component of any service system and often modeled using Poisson models \citep{poisson1837recherches, kingman1992poisson}. In assortment and pricing optimization, there exists a stream literature that combines the MNL model with Poisson arrival models \citep{vulcano2012estimating, abdallah2021demand, wang2021consumer}. Their focus is on the offline estimation problem instead of the dynamic assortment and pricing decision problem, and they typically impose restrictive parametric forms on the arrival rate. In contrast, we focus on the dynamic decision problem, and our arrival model is designed to be flexible and expressive: we only need to identify a set of basis functions on assortment-pricing such that the log arrival rate is linear in these bases. 



Recently, \cite{ferreira2023demand} propose a demand learning and dynamic pricing algorithm for varying assortments in each accounting period. 
They also use a Poisson model for customer arrivals, assuming the arrival rate is an unknown absolute constant. 
They adopt a learn-then-earn approach: the first learning stage (``price to learn'') learns the parameters in the choice model by offering prices that maximize the expected information gain, i.e., by maximizing the determinant of the Fisher information matrix; and the second stage (``price to earn'') chooses prices in a greedy fashion to maximize the expected revenue. In each period of the second stage, after observing rewards,  both the choice and arrival models are updated using all available data. 
They demonstrate the effectiveness of their algorithm through a controlled field experiment with an industrial partner by benchmarking it against their baseline policies, but do not provide regret bounds. 
Our setting and results differ in several respects. First, we allow the Poisson arrival rate to depend on the assortment-pricing decision, whereas they assume the arrival rate to be fixed. 
Second, our method makes joint assortment-pricing decisions, while theirs optimizes prices for varying assortments that are given at the beginning of each period. 
Finally, we prove that our algorithm is nearly optimal.

\subsection{Notation}
\label{subsec:notation}
We use bold lowercase letters denote vectors (e.g., \(\ba\)) and bold uppercase letters denote matrices (e.g., \(\bA\)).
The Euclidean norm of \(\ba\) is \(\|\ba\|_2\).
For a matrix \(\bA\), its operator (spectral) norm is
\(
\|\bA\|_{\mathrm{op}} := \sup_{\|\bx\|_2 = 1}\,\|\bA \bx\|_2 .
\)
For a symmetric matrix \(\bA\), let \(\lambda_{\min}(\bA)\) and \(\lambda_{\max}(\bA)\) denote its smallest and largest eigenvalues, respectively.
For matrices \(\bA,\bB\) of the same dimension, we write \(\bA \preceq \bB\) if \(\bB-\bA\) is positive semidefinite, and \(\bA \prec \bB\) if \(\bB-\bA\) is positive definite.
For any symmetric positive definite matrix \(\bA\), define the \(\bA\)-weighted norm by
\(
\|\bx\|_{\bA} := \sqrt{\bx^\top \bA \bx}.
\)
For an integer \(N \ge 1\), let \([N] := \{1,2,\ldots,N\}\).
We use \(\mathbf{1}\{\cdot\}\) for the indicator function.
We adopt the standard Landau symbols \(O(\cdot)\) and \(\Omega(\cdot)\) to denote asymptotic upper and exact bounds, respectively.
We use \(\Poi\) for the Poisson arrival model and \(\MNL\) for the multinomial logit (MNL) choice model.

\subsection{Outline}

The remainder of the paper is organized as follows. In \Cref{sec:model}, we formally describe the generalized MNL model with an unknown and time-sensitive customer arrival Poisson process. \Cref{sec:policy_learning} presents our $\PMNL$ algorithm for demand learning and dynamic joint assortment-pricing decision-making. \Cref{sec:main_results} provides the regret bound and a matching lower bound (up to $\log(T)$) for our algorithm. In \Cref{sec:numerical_simulation}, we evaluate the performance of our algorithm and compare it with other existing algorithms. \Cref{sec:conclusion} concludes. Details of the proofs are deferred to the Appendix.

\section{Problem Formulation} \label{sec:model}

\subsection{Choice Model with Poisson Arrival}
\label{Subsec:model}

Consider a retailer selling $N$ available products, indexed by $j \in [N]$.
The retailer makes assortment-pricing decisions over a horizon of $T$ periods, indexed by
$t \in [T]$. The time period can be of any
predetermined \emph{granularity} (e.g., by day, week, month, or by the
arrival of one customer). 

At the start of each time period $t$, the retailer observes a set of $\DimFeature$-dimensional product features 
$\Feature_t = \{\Feature_{jt}\}_{j=1}^N$ where $\Feature_{jt} \in \mathbb{R}^{\DimFeature}$. 
Using the features and past observations up to time $t$, the retailer offers an assortment $S_t \in \AssortSet \subset [N]$ under the cardinality constraint $|S_t| = K$ with prices $\bp_t = (p_{jt})_{j=1}^N \in \PriceSet \subset \mathbb{R}^{N}_{++}$ where \(\mathbb{R}^N_{++}\) represents the set of positive real numbers in \(N\) dimensions.
The retailer then observes $n_t$ customers arriving during the time period where each customer $i \in [n_t]$ either purchases one item from the assortment, i.e., $\Choiceti \in S_t$, or does not purchase, i.e., $\Choiceti = 0$. If product $j\in S_t$  is chosen, the retailer earns revenue $r_{i} = p_{jt}$; otherwise, if $\Choiceti = 0$, then $r_{i} = 0$.
Note that the assortment and prices are fixed within each period and only change across periods, which reflects the practice of many retailers, who prefer to adjust decisions at regular intervals, as discussed previously.
To capture both customer arrivals and their subsequent purchasing behavior, we combine a Poisson arrival model with the choice model described below.

\paragraph{Poisson arrival model.}

In each period $t$, the number of customers arriving $n_t$ follows a Poisson distribution with mean arrival rate $\Arrival_t$:
\begin{align}
\label{eq:pois}
n_t \sim \text{Poisson}(\Arrival_t)
\end{align}
where $\Arrival_t = \BaseArrival \lambda_t$ with $\BaseArrival \in \Real_{++}$ being a known positive base arrival rate, which depends on the predetermined granularity, and \(\lambda_t\) captures the arrival rate per unit time, which depends on the current assortment \(S_t\) and the price vector \(\bp_t\). 
The base arrival rate $\BaseArrival$ acts as a scaling factor that adjusts for the length of time intervals.
For instance, if the retailer wants to change the granularity from day to week, we can simply adjust the base arrival rate by multiplying seven. 

As noted before, most existing literature fails to consider the dependence of customer arrival on the assortment/prices and often assumes a fixed arrival rate, often normalized to one, i.e., $\lambda_t = 1$,  which may lead to less profitable decisions.
To account for the dependency, we model the arrival rate to explicitly depend on the current assortment $S_t$ and prices $\bp_t$.
In particular, we assume that the unit arrival rate $\lambda_t$ takes a log-linear form
\begin{equation}
\label{eq:arrival-rate}
\lambda_t := \lambda(S_t,\bp_t; \btheta^*) = \exp(\btheta^{*\top} \bx(S_t,\bp_t))
\end{equation}
where 
$\bx(S_t,\bp_t)$ is a set of sufficient statistics that fully captures the dependence of arrival rate on the assortment and prices and
$\btheta^* \in \Real^\DimRank$ is the unknown parameter of dimension $\DimRank$.
Without loss of generality, we assume $\text{span}(\{ \bx(S_t, \bp_t) ~|~ S\in \mathscr{S}, \bp \in \mathscr{P} \})$ is full rank with rank $\DimRank$ (see more discussions on the rank in Appendix \ref{app:x.independent}).

Such a log-linear form is simple yet flexible. It is commonly-used in Poisson regression to model the relationships between predictor and count outcomes \citep{brown1986fundamentals, winkelmann2008econometric}.
The set of sufficient statistics $\bx(S_t,\bp_t)$ can be flexibly defined based on the specific context.
For example, $\bx(S_t,\bp_t)$ can include terms that capture various aspects of assortment and pricing structure: the inherent attractiveness of individual products, price effects, pairwise interactions between items, and when available product features. We showcase two specific forms of $\lambda_t$ in the following remark.

\begin{remark}(Two examples of $\lambda_t$)
\label{rmk:lambda}
Our proposed succinct form of the arrival rate is flexible enough to incorporate economic principals in the literature. 
For example, one can consider the following form to account for price sensitivity and product variety:
\begin{equation}
\label{eq:lambda-t-2}
\lambda_t := \prod_{i\in S_t} \left(\frac{p_i}{\phigh}\right)^{-\alpha_i} = \exp\left(-\sum_{i\in S_t} \alpha_i \log  \left(\frac{p_i}{\phigh}\right)\right),
\end{equation}
where $\phigh$ denotes the highest feasible price.
The price sensitivity is modeled through the negative dependence of $\lambda_t$ on $p_{it}$, with lower prices attracting more customers, and the magnitude of this effect governed by item-specific $\alpha_i>0$.
The summation captures assortment variety, as offering more products can attract higher customer arrivals, 
which aligns with the marketing literature on the positive impact of product variety on customer attraction \citep{lancaster1990economics, kahn1995consumer}.

This simple model is effective in capturing the overall impact of assortment and price, but ignores dependencies between products, such as complementarity or substitution. An alternative form that incorporates these dependencies is given by
\begin{align}
\label{eq:pop}
    \lambda_t := \exp\left(\sum_{i \in S_t} \frac{\alpha_i}{p_{it}} + \sum_{\substack{i,j \in S_t \\ i \neq j}}  \beta_{ij} \frac{p_{it}}{p_{jt}} \right)
\end{align}
where the first sum captures the individual effects and the second sum captures the pairwise effects. 

The individual effects account for the inherent utility of each product, represented by $\alpha_i$, and the price sensitivity.
This summation also reflects the product variety effect \citep{kahn1995consumer}, as the arrival rate increases with the inclusion of more products in the assortment.

The second sum captures the effect of relative prices of item pairs. 
When $\beta_{ij}$ is positive, the pairwise effect is complementary, meaning that the presence of both items in the assortment increases the arrival rate, and the effect is more pronounced if the price of product $i$ is high relative to product $j$, i.e., when $\frac{p_{it}}{p_{jt}}$ is large.
This complementary effect enables interesting dynamics \citep{wang2021consumer}.
Including a new product in the assortment may cannibalize the existing products, yet the total sales of the existing products may increase if the arrival rate increases largely enough. 
On the other hand, reducing the price of a product may allow other products to take a ``free-rider'' advantage, benefiting from the higher arrival rate and potentially boosting their total sales. 
The summation also captures a price variety effect: higher price variation within the assortment, as represented by the price ratios, can lead to a higher arrival rate.
When $\beta_{ij}$ is negative, the pairwise effect becomes substitutable, meaning that the presence of both items in the assortment reduces the arrival rate. This substitutive effect can arise from customer perceptions of redundancy and choice overload. Specifically, when similar products have large price differences, customers may become skeptical of the pricing strategy, potentially eroding trust and leading to reduced arrivals.


\end{remark}

\begin{remark} (Exogenous factors for $\lambda_t$)
For simplicity, we focus on the case where \(\lambda_t\) depends only on the assortment and pricing; however, our model can be extended to incorporate other exogenous factors that influences arrivals, such as macroeconomic conditions and seasonal effects.
\end{remark}

\vspace{.1in}
\paragraph{MNL choice model.}

During period $t$, each customer $i \in [n_t]$ either purchases a product or makes no purchase, i.e., $\Choiceti \in S_t \cup \{0\}$, according to an MNL model.
Specifically, the probability of customer \(i\) choosing product \(j\) is given by
\begin{equation}
\label{eq:mnl}
    \ChoiceProb(j,S_t, \bp_t, \Featuret; \bv^*) =
    \mathbb{P}(C^{(i)}_t = j|S_t, \bp_t, \Featuret; \bv^*) 
= \begin{cases}
    \displaystyle \frac{\exp(\bv^{*\top} \Feature_{jt} - p_{jt})}{1 + \sum_{k \in S_t} \exp(\bv^{*\top} \Feature_{kt} - p_{kt})}, & \forall j  \in S_t; \\
    \displaystyle \frac{1}{1 + \sum_{k \in S_t} \exp(\bv^{*\top} \Feature_{kt} - p_{kt})}, &  j = 0,
\end{cases}
\end{equation}
where $\bv^* \in \Real^\DimFeature$ are unknown preference parameters that characterize the impact of product features on the intrinsic value of the products. 
We allow the product features to change over time since product features, such as ratings and popularity scores, are not static and can change over time, thereby allowing the utility of products to evolve over time.
For notation simplicity, when there is no ambiguity, we use $\ChoiceProbt{\bv^*}{j}$ to denote the choice probability of product $j$ for $j\in S_t$ and the non-purchase probability when $j=0$ at time $t$.

Under the choice model, the expected revenue of each customer $i$ is
\begin{align}
\label{eq:r-it}
r(S_t, \bp_t, \Featuret;\bv^*) = \mathbb{E}\left[\sum_{j \in S_t} p_{jt} \mathbf{1}(C^{(i)}_{t} = j | S_t, \bp_t, \Featuret; \bv^*)  \right] = \sum_{j \in S_t} p_{jt} \ChoiceProbt{\bv^*}{j}.
\end{align}
Then the expected revenue at time period \(t\) across all $n_t$ customers  is given by
\begin{align}
\label{eq:r-t}
\Revenue[t] = 
R(S_t, \bp_t, \Featuret; \bv^*, \btheta^*) 
= \E \left[ n_t  r(S_t, \bp_t, \Featuret; \bv^*) \right]
= \Lambda \lambda(S_t, \bp_t; \btheta^*)  \sum_{j \in S_t} p_{jt} \ChoiceProbt{\bv^*}{j},
\end{align}
which is the expected number of customers multiplied by the expected revenue of each customer.

\subsection{Retailer's Objective and Regret}
\label{Subsec:policy}

The objective of the retailer is to design a policy $\pi$ that chooses a sequence of history-dependent actions 
$(S_1, \bp_1, S_2, \bp_2,\ldots, S_T, \bp_T)$ so as to maximize the \emph{expected cumulative revenue} over $T$ periods
$\E_{\pi} \left [\sum_{t=1}^T  \Revenue[t] \right ].$
Formally, a policy is a sequence of (stochastic) functions $\pi = \{ \pi_t \}_{t=1}^T$, where each $\pi_t$ maps a history of actions and observed outcomes up to time $t$ to the assortment and pricing decision at time $t$ in a stochastic sense, i.e., $\pi_t: H_t \rightarrow (S_t, \bp_t)$, where $H_t$ represents the history up to time $t$, and is defined as
\begin{align}\label{df:H_t}
    H_t = \left(  C^{(1)}_{1},  C^{(2)}_{1}, \ldots, C^{(n_1)}_{1}, \ldots, C^{(1)}_{t-1}, \ldots , C^{(n_{t-1})}_{t-1}, S_1, \ldots, S_{t-1}, \bp_1, \ldots, \bp_{t-1}, \Feature_1, \ldots, \Feature_t \right).
\end{align}
Note that $\pi_t$ can be stochastic in the sense that its action output has randomness, i.e., $\pi(H_t)$ is a random variable. Given a policy $\pi$, we use $\mathbb{P}_\pi\{\cdot\}$ and $\mathbb{E}_\pi\{\cdot\}$ to denote the probability measure and expectation if we take actions following policy $\pi$.


If the parameters associated with the arrival model $\btheta^*$ and the choice model $\bv^*$ were known a priori, 
then the retailer could choose an optimal assortment $S_t^* \in \AssortSet$ and prices $\bp_t^* \in \PriceSet$ that
maximizes the expected revenue \eqref{eq:r-t} for each period, i.e., 
\mbox{$(S^*_t, \bp^*_t) \defn \argmax_{S, \bp}  \Revenue $.}
This optimal solution yields an optimal cumulative revenue over time horizon $T$: $\sum_{t=1}^T R_t(S^*_t, \bp^*_t)$, or equivalently, $\sum_{t=1}^T \BaseArrival \lambda(S^*_t, \bp^*_t; \btheta^*) r(S^*_t, \bp^*_t, \Featuret; \bv^*)$. 
This optimal value is not attainable because $(\btheta^*, \bv^*)$ is unknown in practice, 
but it serves as a useful benchmark for performance of any algorithm.

Using this benchmark, we evaluate a policy $\pi$ by \emph{cumulative regret}---that is, the deficit between the expected cumulative revenue over the time horizon $T$ of the optimal solution and that of $\pi$:
\begin{equation}\label{eq:regret_equal_form}
    \mathcal{R}^{\pi}(T) = \left\{\E \left [\sum_{t=1}^T \lambda(S^*_t, \bp^*_t; \btheta^*) r(S^*_t, \bp^*_t, \Featuret; \bv^*)\right] - \E_\pi \left[ \sum_{t=1}^T  \lambda(S_t, \bp_t; \btheta^*) r(S_t, \bp_t,  \Featuret; \bv^*) \right]\right\} \Lambda.
\end{equation}
To estimate the unknown parameters $(\btheta^*, \bv^*)$, we need to design an algorithm that simultaneously learns the Poisson arrival model and the MNL choice model on the fly (exploration) as well as minimizing the cumulative revenue (exploitation).
This exploration–exploitation problem for dynamic assortment and pricing, to which we refer as \emph{Poisson-MNL bandit}, is our focus.

\subsection{Assumptions}
\label{Subsec:assumption}

Before presenting the algorithm, we first state several assumptions. These assumptions are mild and common for MNL models, which we will discuss in detail below. Importantly, our algorithm itself is still executable without these assumptions, and we expect it to behave well even if some assumptions fail to hold. 




\begin{assumption}\label{assump:K}
The feasible assortment set $\mathscr{S}$ consists of all $K$-subsets of $[N]$.
\end{assumption}

\begin{assumption}\label{assump:price}
The feasible price vector $\bp \in \PriceSet$ lies within the range $[\plow, \phigh]^N$ for positive constants $0< \plow < \phigh$. 
\end{assumption}

\begin{assumption}\label{as:v&assortment_new}
The MNL choice model preference parameter $\bv^* \in \Real^\DimFeature$ satisfies $||\bv^*||_2 \leq \bar{v}$.
\end{assumption}

\begin{assumption}\label{assump:theta}
The Poisson arrival model parameter $\btheta^* \in \Real^\DimRank$ satisfies $||\btheta^*||_2 \leq 1$.
\end{assumption}

\begin{assumption}\label{assump:x}
For any feasible assortment $S \in \mathscr{S}$ and any feasible price $\bp \in \PriceSet$,
the sufficient statistic $\bx(S, \bp)$ is scaled such that $||\bx(S, \bp)||_2 \leq \bar{x}$ for some constant $\bar{x}>0$. 
\end{assumption}

\begin{assumption}\label{assump:combined}
The product feature sequence $\left\{\Feature_{t}\right\}_{t=1}^{T}$ is i.i.d. sampled from an unknown distribution with a density $\mu$, where $\|\Feature_{jt}\|_2 \leq 1$ for each $j \in [N]$ and $t \in [T]$. The distribution $\mu$ satisfies the following condition: we can construct a pre-determined sequence of assortments $\{S^{init}_s\}_{s = 1}^{t}$, such that for any $t \geq t_0$ where $t_0 = \max\left\{\ceil{\frac{\log(\DimFeature T)}{\sigma_0 (1 - \log 2)}},\, 2 \DimRank\right\}$,
the following holds with probability at least $1 - T^{-1}$:
\begin{equation}
\label{eq:assump-z}
    \sigma_{\min}\left(\sum_{s = 1}^{t} \sum_{j \in S^{init}_s} \Feature_{js} \Feature^\top_{js}\right) \geq \sigma_0 t,
\end{equation}
where $\sigma_0$ is a positive constant dependent on \(\mu\). Furthermore, there exists a positive constant $\sigma_1$ and a price sequence $\{\bp^{init}_s\}_{s = 1}^{t}$ such that for any $t \geq t_0$
\begin{equation}
\label{eq:assump-x}
    \sigma_{\min} \left(\sum_{s = 1}^{t} \bx(S^{init}_s, \bp^{init}_s) \bx^\top(S^{init}_s, \bp^{init}_s )\right) \geq  \sigma_1 t.
\end{equation}
\end{assumption}



\Cref{assump:K,assump:price} impose standard feasibility conditions on the assortment and price sets as in the literature \citep{chen2020dynamic,chen2021dynamic}.
Compared to \Cref{assump:K}, other related work considers a feasible set of up-to-$K$-product assortment~\citep{agrawal2019mnl, oh2021multinomial}. 
This assumption admits a larger feasible set and thus enlarges the set of legitimate algorithms. Our algorithm also works for this alternative feasible set, with the same theoretical guarantee, and our lower bound applies. 
\Cref{assump:price} defines the feasible price range as $[\plow, \phigh]$: the lower bound typically reflects either the minimal currency unit or product cost, and the upper bound comes from the idea of ``choke price'', beyond which the demand is effectively zero. 

\Cref{as:v&assortment_new} assumes the boundedness of parameter in the MNL choice model \eqref{eq:mnl}, which is standard in the contextual MNL-bandit literature \citep{chen2020dynamic, oh2021multinomial}.
By assuming $\bv^*$ is upper bounded by $\bar{v}$ in $\ell_2$ norm, we encode the common belief that the product features typically have a bounded influence on product utility.

\Cref{assump:theta,assump:x} assume boundedness for the Poisson arrival model \eqref{eq:arrival-rate}: they bound the influence of assortment-pricing decisions on the unit arrival rate $\lambda_t$ and the $\ell_2$-norm of the vector $\bx(S, \bp)$ for all feasible assortment-pricing decisions. These assumptions, together with the base arrival rate $\Lambda$, ensure that  the arrival rate $\Arrival_t$ is bounded, reflecting the reality that the market size is inherently limited.

\Cref{assump:combined} assumes that the product feature vector $\bz_{jt}$ for $j \in [N]$ and $t \in [T]$ are randomly generated from a compactly supported and non-degenerate density, with additional isotropic conditions \eqref{eq:assump-z}-\eqref{eq:assump-x} on $\bz_{jt}$ and $\bx(\cdot, \cdot)$. These isotropic conditions ensure information availability in all directions and are necessary for an algorithm to converge. 
\Cref{assump:combined} can be easily satisfied and is weaker than its counterpart in contextual MNL-bandit literature \citep{chen2021dynamic,oh2021multinomial}, which we state below.

\begin{assumption}\label{assumption::y_exist}
The feature vectors  $\Feature_{j t}$ for $j \in [N]$ and \(t \in [T]\) is i.i.d. across both \(j\)  and  \(t\) from an unknown distribution with density $\mu$ supported on $\left\{\Feature \in \mathbb{R}^\DimFeature : \|\Feature\|_2 \leq 1\right\}$. Additionally, the minimum eigenvalue of the expected covariance matrix $\mathbb{E}_{\mu}(\Feature \Feature^{\top})$  is bounded below by a positive constant $\bar{\sigma}_0$.
\end{assumption}

\Cref{assumption::y_exist} is stronger than the MNL component of \Cref{assump:combined}, but serves the same purpose. We note that the corresponding assumption in \cite{chen2021dynamic} and \cite{oh2021multinomial} is solely for assortment under an MNL model, without pricing and a Poisson arrival model. Under \Cref{assumption::y_exist}, \Cref{assump:combined} holds with $\sigma_0= \frac{K\bar{\sigma}_0}{2}$ and we show the corresponding sequence required by \Cref{assump:combined} in Appendix \ref{appendix:relation}.
\section{Algorithm}
\label{sec:policy_learning}

\begin{algorithm}
\caption{The $\PMNL$ algorithm for the Poisson-MNL bandit problem.}
\label{alg:lambda_dp_mod}
\begin{algorithmic}[2]
	\STATE \textbf{Output}: Assortment and prices $(S_{T_0+1}, \bp_{T_0+1}), \ldots, (S_{T}, \bp_{T})$.
	\STATE \textbf{Input}: Time horizon $T$, feasible assortment set $\AssortSet$, feasible price range $\PriceSet$, parameters \(\bar{x}, \bar{v}, \Lambda, \sigma_0, \sigma_1, t_0\), and dimensions \(\DimFeature, \DimRank\);
    \STATE \textbf{Initialization}: $t = 1$, $t_0 = \max\left\{\ceil{\frac{\log(\DimFeature T)}{\sigma_0 (1 - \log 2)}},\, 2 \DimRank\right\}$,  $T_0 := \max\{t_0 + 1, \lfloor\log T\rfloor\}$, \(c_0\) defined in \Cref{df:c_0}, $\tau_\btheta$ and $\tau_\bv$ in 
    \Cref{eq:tau_theta,eq:tau_v},  $\omega_\btheta$ and $\omega_\bv$ in \Cref{df:w_theta,df:w_bv};
    \STATE \emph{\textbf{Stage 1: Pure Exploration with Global MLE}}
    \WHILE{$t \leq T_0$} 
    \STATE Observe product features $\bz_t = \{\bz_{jt}\}_{j=1}^N$;
    \STATE Set  assortment and pricing as $S_t = S^{init}_{t}$ and $\bp_t = \bp^{init}_t$ as in \Cref{assump:combined};
    \STATE Observe $n_t$ customer arrivals with their decisions $C_t^{(i)} \in S_t \cup \{0\}$ for $i=1,\ldots, n_t$;
    \STATE $t \leftarrow t + 1$;
    \ENDWHILE
    \STATE Compute the pilot estimator using global MLE:
    \STATE \quad $\hat{\btheta} = \argmax_{||\btheta||_2 \leq \bar{v}} \LikPoi{T_0}$;
    $\hat{\bv} = \argmax_{||\bv||_2 \leq 1} \LikMNL{T_0}$. 
	\STATE \emph{\textbf{Stage 2: Local MLE and Upper Confidence Bound (UCB)}}
    \FOR{$t = T_0 + 1$ \TO $T$}
	\STATE Observe product features $\bz_t = \{\bz_{jt}\}_{j=1}^N$;
    \STATE Compute local MLE:
    \STATE \quad $\widehat{\btheta}_{t-1} \in \argmax_{||\btheta - \hat\btheta||_2 \leq \tau_\btheta} \LikPoi{t-1}$;
    $\widehat{\bv}_{t-1} \in \argmax_{||\bv - \hat\bv||_2 \leq \tau_\bv} \LikMNL{t-1}$;
    \STATE For every assortment $S\in\AssortSet$ and price $\bp \in \PriceSet$, compute the upper confidence bound of expected revenue:
    \begin{equation*}
        \begin{aligned}
            &\bar{R}_t(S, \bp) := \Lambda \lambda(S, \bp;\hat{\btheta}_{t - 1})r(S, \bp,\Featuret ;\hat{\bv}_{t - 1}) \\
            & + 
            \begin{multlined}[t]
            \phigh \min \Bigg\{\Lambda (e^{\bar{x}}- e^{-\bar{x}}), 
            \sqrt{
                 \Lambda e^{(2\tau_\btheta + 1) \bar{x}} \omega_\btheta
                \left\|
                    {I}_{t-1}^{\Poi\,-\frac{1}{2}}(\widehat{\btheta}_{t - 1})
                    {M}^\Poi_{t} (\widehat{\btheta}_{t - 1}|S, \bp)
                    {I}_{t-1}^{\Poi\,-\frac{1}{2}}(\widehat{\btheta}_{t - 1})
                \right\|_{\text {op }}
            }
            \Bigg\}
            \end{multlined}\\
            &+  \Lambda e^{\bar{x}} 
            \min 
            \left\{
                \phigh, 
                \sqrt{
                    c_0 \omega_\bv\left\| \widehat{I}_{t-1}^{\MNL \, -\frac{1}{2}}\left(\widehat{\bv}_{t-1}\right)
                    \widehat{M}^\MNL_{t}
                    \left(\widehat{\bv}_{t-1} \mid S, \bp\right) \widehat{I}_{t-1}^{\MNL\,-\frac{1}{2}}\left(\widehat{\bv}_{t-1}\right)
                    \right\|_{\text {op }}
                }
            \right\};
        \end{aligned}
    \end{equation*}
    \STATE Compute and set $(S_t, \bp_t) := \underset{S \in \mathscr{S}, \bp \in \mathcal{P}}\argmax ~ \bar{R}_t(S, \bp)$;
    \STATE 
    Observe $n_t$ customer arrivals with their decisions $ C_t^{(i)} \in S_t \cup \{0\}$ for $i=1,\ldots, n_t$;
    \ENDFOR
\end{algorithmic}
\end{algorithm}

In this section, we describe our $\PMNL$ algorithm for the Poisson-MNL bandit problem. 
Our algorithm involves two stages.
The first stage focuses on pure exploration to obtain a good initial estimator for $(\btheta^*, \bv^*)$, which we refer to as the ``pilot estimators'' $(\hat\btheta, \hat\bv)$, by applying the maximum likelihood estimation (MLE) to data from the first $T_0$ periods.
This stage yields a small confidence region in which the true parameters lie with high probability and is practically very important for the second stage to work well, though it does not affect the regret rate.
The second stage builds on the upper confidence bound (UCB) strategy, also known as optimism in the face of uncertainty, to balance exploration and exploitation. Specifically, in each period we make the assortment-pricing decision that maximizes the upper confidence bound of the expected revenue based on the Fisher information,
and then we update the parameters using a local MLE around the pilot estimators. 
The full procedure is described in \Cref{alg:lambda_dp_mod}.

As noted in the introduction, developing the algorithm is not a routine adaptation of standard MNL bandit algorithms and is technically more challenging. 
The Poisson arrival and MNL choice models are entangled, so the algorithm must jointly account for uncertainty from arrivals and choices. In particular, constructing estimators for $(\btheta^*, \bv^*)$ and, crucially, their corresponding upper bounds for the UCB strategy requires new technical tools. The main difficulties include:
\begin{enumerate}
    \item The Poisson arrival model produces unbounded and non-sub-Gaussian random variables, which makes the mechanisms commonly used for deriving the error bounds in assortment problems inapplicable \citep{oh2021multinomial,abbasi2011improved} and thus requiring new tools. We address this problem by utilizing concentration inequalities for martingales with increments satisfying Bernstein conditions \citep{wainwright2019high}.
    
    \item Arrival randomness affects the number of observed purchases, so the observations used to learn choice model parameters mix both arrival and choice model randomness, 
    making it more difficult to construct an estimator and its error bound for the choice model.
    We tackle this challenge by designing and analyzing statistics whose randomness primarily comes from the choice model rather than the arrival model.

    \item The unknown parameters in the Poisson arrival model inevitably enter the error bounds for the MNL model estimators, whereas the algorithm requires confidence bounds that are free of unknown quantities. We resolve this problem by further bounding the error bound with a purely data-dependent statistic.

    \item As error bounds of the estimators enter UCB type algorithms, their constants also matters for practical performance. Rather than relying on standard analytical techniques, we carry out a refined analysis and obtain better constants.
\end{enumerate}

In what follows, we first introduce the likelihood function and then describe the two stages.

\subsection{Likelihood Function}
\label{Subsec:likelihood}

Both stages of our algorithm rely on MLE to estimate the unknown parameters $(\btheta^*, \bv^*)$.
We therefore first specify the likelihood function.
The likelihood function at period $t$ can be written as
\begin{align*}
\prod_{s=1}^t \left\{
\Prob\left(\Poi(\Lambda \lambda(S_{s}, \bp_s; \btheta)) = n_s; \btheta\right) 
\prod_{i=1}^{n_s} \prod_{j \in S_s\cup\{0\}} \ChoiceProb_s(j;\bv)^{\bm{1}\left\{C^{(i)}_{s} = j\right\} }
\right\}.
\end{align*}
To obtain the maximum likelihood estimator, we usually maximize the log-likelihood function instead.
Specifically, it decomposes into the sum of the log-likelihoods of the Poisson arrival model and the MNL choice model, as follows:
\begin{equation}
\label{eq:likelihood}
\Lik{t} := \LikPoi{t} + \LikMNL{t} = 
\sum_{s=1}^t \likPoi{s} +
\sum_{s=1}^t \likMNL{s}
\end{equation}
where 
\begin{align}
     \likPoi{s} &:= n_s \log \lambda(S_{s}, \bp_s;\btheta) - \Lambda \lambda(S_{s}, \bp_s; \btheta) + n_s\log{\Lambda} - \log n_s!, \\
    \likMNL{s} & := \sum_{i = 1}^{n_{s}} \sum_{j \in S_{s} \cup\{0\}}
\bm{1}\{C^{(i)}_{s} = j\} \log \ChoiceProb_s(j;\bv).
\end{align}

\subsection{Stage 1: Pure Exploration for Pilot Estimators via Global MLE} 
\label{Subsec:exploration}

The first stage of the algorithm is to obtain a pair of the pilot estimators $(\hat\btheta, \hat\bv)$ that are close to the true parameters $(\btheta^*, \bv^*)$ at the end of the first $T_0$ exploration periods.

We begin with setting the length of this initial stage as $T_0 = \max\{t_0 + 1, \lfloor\log T\rfloor\}$, where $t_0$ is defined in \Cref{assump:combined}.
Note that $T_0$ is a very small number, meaning that this stage is very short and thus incurring only a small exploration cost.
Keeping the exploration periods short is critical to firms, particularly for start-ups or businesses in competitive markets, as substantial initial losses can potentially jeopardize survival before they can leverage the insights to generate revenue.

Next, we make assortment-pricing decisions according to the initial sequence $\{S^{init}_t, \, \bp^{init}_t\}_{t=1}^{T_0}$ from \Cref{assump:combined}. 
When the stronger condition \Cref{assumption::y_exist} holds, we can take the initial sequence as in Appendix \ref{appendix:relation}.
In practice, one can let $\bx$ be a vector of product dummies and prices, and construct the initial sequence by uniformly selecting an assortment and prices. There also exist more refined procedures to ensure \Cref{assump:combined} holds.


With an initial sequence $\{S^{init}_t, \bp^{init}_t\}_{t=1}^{T_0}$
and the observations of customer arrivals and their choices,
we proceed to estimate the pilot estimators $(\hat\btheta, \hat\bv)$ by maximizing the log-likelihood of the first stage, $\Lik{T_0}$. Given that the log-likelihood decomposes into $\LikPoi{T_0}$ and $\LikMNL{T_0}$, we can obtain the pilot estimator $(\hat\btheta, \hat\bv)$ by simply maximizing their corresponding log-likelihoods, i.e.,
\begin{align}
	\hat{\btheta} &:= \argmax_{||\btheta||_2 \leq 1} \LikPoi{T_0}, \quad \quad 
	\hat{\bv} := \argmax_{||\bv||_2 \leq \bar{v}} \LikMNL{T_0}.
\end{align}
In \Cref{subsec:bound_para}, we show that the pilot estimators are close to the true parameters $(\btheta^*, \bv^*)$.
In particular, with high probability, the estimation error $\|\hat{\btheta} - \btheta^*\|_2 \leq \tau_\btheta$ and $\left\|\hat{\bv} - \bv^*\right\|_2 \leq \tau_\bv$, where $\tau_\btheta$ and $\tau_\bv$ are defined in \Cref{eq:tau_theta,eq:tau_v},  respectively, and both are of order $O(\frac{\log T}{T_0}) = O(1)$.
These bounds serve as the basis for the local MLE and the confidence bounds in the second stage.
\subsection{Stage 2: Local MLE and Upper Confidence Bound (UCB) } 
\label{Subsec:UCB}

In the second stage, we make assortment-pricing decisions leveraging the idea of UCB to balance exploration and exploitation, and updates the estimates using local MLE, iteratively. 
Given the pilot estimator $(\hat\btheta, \hat\bv)$ with their error bounds  $\tau_\btheta$ and $\tau_\bv$, at each period $t > T_0$, we obtain the local MLE $\widehat{\btheta}_{t-1}$ (respectively, $\widehat{\bv}_{t-1}$) by maximizing the log-likelihood $\LikPoi{t-1}$ (respectively, $\LikMNL{t-1}$) within a ball centered at the pilot estimator $\hat\btheta$ (respectively, $\hat\bv$) with a radius $\tau_\btheta$ (respectively, $\tau_\bv$):
\begin{align}
\label{eq:local-mle}
	\widehat{\btheta}_{t-1} &:= \argmax_{||\btheta - \hat\btheta||_2 \leq \tau_\btheta} \LikPoi{t-1}, \quad \quad 
	\widehat{\bv}_{t-1} := \argmax_{||\bv - \hat\bv||_2 \leq \tau_\bv} \LikMNL{t-1}.
\end{align}
Given $(\widehat{\btheta}_{t-1}, \widehat{\bv}_{t-1})$ and the product features $\bz_t$, we make an assortment-pricing decision ($S_t,\bp_t)$ that maximizes an upper confidence bound of the expected revenue $R_t(S, \bp)$.

The upper confidence bound depends on both the estimation and the estimation error of the Poisson arrival and MNL choice parameters, where the estimation error can be captured by the Fisher information of $\btheta$ and $\bv$, defined as follows.


\begin{definition}\label{def_theta}
For the Poisson arrival model, we define the Fisher information  with respect to $\btheta$ for $t$ periods, $I_t^\Poi(\btheta)$, given history $\{H_s\}_{s=1}^t$ as follows:
\begin{equation}
\begin{aligned}
    I_t^\Poi(\btheta) :=& \sum_{s = 1}^t M_{s}^\Poi(\btheta), \; \mbox{where } \\
    M_{s}^\Poi(\btheta) :=& \E \big[
    -\nabla^2_{\btheta} \likPoi{s} \big|H_{s}\big] = \Lambda\lambda (S_s , \bp_s;\btheta)  \cdot \xs \xs^{\top}.
\end{aligned}  
\end{equation}
\end{definition}

\begin{definition}\label{def_v}
For the MNL choice model, we define the Fisher information with respect to $\bv$ for $t$ periods, $I_{t}^\MNL(\bv)$, given history $\{H_s\}_{s=1}^t$ as follows:
\begin{equation}
\begin{aligned}
    I_{t}^\MNL(\bv) :=& \sum_{s = 1}^t M_{s}^\MNL(\bv), \; \mbox{where } \\
    M_{s}^\MNL(\bv) :=& 
    \E \big[- \nabla^2_{\bv} \likMNL{s} \big|H_{s}\big] \\
    =& \Lambda \lambda(S_s, \bp_s;\btheta^*)
\left (\sum_{j \in S_{s}} \ChoiceProb_s(j; \bv) \Feature_{ j s} \Feature^\top_{j s} - \sum_{j, k\in S_{s}} \ChoiceProb_s(j; \bv)\ChoiceProb_s(k;  \bv) \Feature_{ j s} \Feature^\top_{ k s}\right).
\end{aligned}
\end{equation}
\end{definition}


Note that in \Cref{def_theta,def_v}, the assortment and pricing decisions $S_{s} = \pi(H_s)$ and $\bp_s = \pi(H_s)$, where the policy $\pi$ is an algorithm to be designed to make decisions $S_s$ and $\bp_s$ based on the history $H_s$.
At period $t$, given the history $H_t$ (which includes product features $\bz_t$), we need to decide $S$ and $\bp$. 
To emphasize the dependence of the Fisher information  on different choices of $S$ and $\bp$ at period $t$, we write their summands as functions of $(S,\bp)$ for $ S\in \AssortSet$ and $\bp\in\PriceSet$: 
{\small
\begin{align}
{M}_t^\Poi(\btheta | S, \bp) &:= \Lambda \lambda(S, \bp; \btheta)\bx(S, \bp) \bx^{\top}(S, \bp), \\
M_t^\MNL(\bv| S, \bp) &:= 
\begin{multlined}[t] 
\Lambda\lambda(S, \bp;\btheta^*) \underbrace{
\Bigg( \sum_{j \in S} \ChoiceProb(j, S, \bp, \Feature_t; \bv) \Feature_{jt} \Feature^\top_{jt} 
    - \sum_{j, k\in S} \ChoiceProb(j, S, \bp, \Feature_t; \bv)\ChoiceProb(k, S, \bp,\Feature_t; \bv) \Feature_{jt} \Feature^\top_{kt} \Bigg)}_{\phi(S,\bp, \bz_t; \bv)} .
\end{multlined}\label{eq:df:fisher_mnl}
\end{align}
}
Clearly, $M_{s}^\MNL(\bv) ={M}_s^\MNL(\bv | S_s, \bp_s) $.


To translate the estimation errors of $\widehat{\btheta}_{t-1}$ and $\hv_{t-1}$ into the estimation error of the reward, two quantities are involved:
\mbox{${I}_{t-1}^{\Poi\,-\frac{1}{2}}(\widehat{\btheta}_{t - 1}) {M}^\Poi_{t} (\widehat{\btheta}_{t - 1}|S, \bp){I}_{t-1}^{\Poi\,-\frac{1}{2}}(\widehat{\btheta}_{t - 1})$} and 
\mbox{${I}_{t-1}^{\MNL\,-\frac{1}{2}}(\widehat{\bv}_{t-1}) {M}^\MNL_{t} (\widehat{\bv}_{t-1}|S, \bp){I}_{t-1}^{\MNL\,-\frac{1}{2}}(\widehat{\bv}_{t-1})$}, which change the weighting matrix of the norm in which the error bounds of the estimators can be derived.
Since the latter involves the unknown ground truth $\btheta^*$ as in \Cref{eq:df:fisher_mnl}, we cannot compute it directly. 
Fortunately, we manage to construct a data-dependent bound based on the boundedness \cref{assump:theta,assump:x}, which implies $\exp\{-\bar{x}\} \leq \lambda(S, \bp; \btheta^*) = \exp\{{\btheta^*}^\top \bx(S, \bp)\} \leq \exp\{\bar{x}\}$.



\begin{definition}
\label{def_vhat}
The bounds of 
${I}_{t}^\MNL(\bv)$ and ${M}_t^\MNL(\bv | S, \bp)$
are defined as follows:
\begin{align}
\widehat{I}_{t}^\MNL(\bv) := \sum_{s = 1}^t \Lambda\exp\{-\bar{x}\}  \phi(S_s,\bp_s,\bz_s;\bv),
\quad
\widehat{M}_t^\MNL(\bv|S, \bp) := \Lambda \exp\{ \bar{x}\} \phi(S,\bp,\bz_t;\bv).
\end{align}

\end{definition}

As $\phi(S,\bp,\bz;\bv)$ is positive semidefinite, the following lemma is straightforward.
\begin{lemma}
\label{lemma_estimate_M_I}
The following inequalities hold:
\begin{alignat}{3}
    \widehat{I}_{t}^\MNL(\bv) &\preceq  I_t^\MNL(\bv), \quad 
    \widehat{M}_t^\MNL(\bv|S, \bp) &\succeq {M}_t^\MNL(\bv|S, \bp).
\end{alignat}
\end{lemma}


Now, based on our estimator $\widehat{\btheta}_{t-1}$ and $\hv_{t-1}$, given the newly observed product features $\bz_t$, we construct the following upper confidence bound for expected revenue for every $ S\in \AssortSet$ and $\bp\in\PriceSet$: 
\begin{align}
\label{eq:ucb}
&\bar{R}_t(S, \bp) := \Lambda \lambda(S, \bp;\widehat{\btheta}_{t - 1})r(S, \bp,\Featuret ;\widehat{\bv}_{t - 1}) \\
    & \quad \quad + 
    \phigh
    \min \Bigg\{ \Lambda (e^{\bar{x}}- e^{-\bar{x}}), \,
    \sqrt{
        \Lambda e^{(2\tau_\btheta + 1) \bar{x}} \omega_\btheta 
        \left\|{I}_{t-1}^{\Poi\,-\frac{1}{2}}(\widehat{\btheta}_{t - 1}) {M}^\Poi_{t} (\widehat{\btheta}_{t - 1}|S, \bp){I}_{t-1}^{\Poi\,-\frac{1}{2}}(\widehat{\btheta}_{t - 1})
        \right\|_{\text {op }}
    }
    \Bigg\} \label{eq:ucb-lambda} \\
    & \quad \quad  +  
    \Lambda e^{\bar{x}} 
    \min 
    \left\{ 
    \phigh, \,  
    \sqrt{c_0\omega_\bv  \left\| \widehat{I}_{t-1}^{\MNL \, -\frac{1}{2}}\left(\widehat{\bv}_{t-1}\right) \widehat{M}^\MNL_{t}\left(\widehat{\bv}_{t-1} \mid S, \bp\right) \widehat{I}_{t-1}^{\MNL\,-\frac{1}{2}}\left(\widehat{\bv}_{t-1}\right)\right\|_{\text {op }}}
    \right\}.\label{eq:ucb-R}
\end{align}
The first term corresponds to the estimated revenue and the last two terms \eqref{eq:ucb-lambda} and \eqref{eq:ucb-R} represent the estimation error of the reward induced by the estimation errors of $\widehat{\btheta}_{t-1}$ and $\widehat{\bv}_{t-1}$ respectively.
Specifically, the term \eqref{eq:ucb-lambda} primarily depends on $\omega_\btheta$---the upper bound of $\| \widehat{\btheta}_{t-1} - \btheta^*\|_{I_{t - 1}^\Poi(\htheta_{t-1})}$ defined in \Cref{df:w_theta}---and the upper bound of the operator norm of the matrix that changes the weighted norm $\|\cdot\|_{I_{t - 1}^\Poi(\htheta_{t - 1}) }$ to $\|\cdot\|_{M_t^\Poi(\htheta_{t - 1}|S,\bp) }$.
The term \eqref{eq:ucb-R} primarily depends on $\omega_{\bv}$---the upper bound of $\|\hv_{t-1} - \bv^*\|_{I_{t - 1}^\MNL(\hv_{t-1}) }$ defined in \Cref{df:w_bv}---and the upper bound of the operator norm of the matrix that changes the weighted norm $\|\cdot\|_{I_{t - 1}^\MNL(\hv_{t - 1}) }$ to $\|\cdot\|_{M_t^\MNL(\hv_{t - 1}|S,\bp) }$. 
The constant $c_0$ is defined as: 
\begin{equation}\label{df:c_0}
    c_0 = 
        \frac{
         (\phigh - \plow)^2}
        {\Lambda \exp(-\bar{v})} \left[3(\exp(4\tau_\bv ) - 1)(K \exp(\bar{v} - \plow) + 1) + 1\right] .
\end{equation}

\section{Regret Analysis}\label{sec:main_results}

We now turn to the theoretical analysis of our procedure, beginning in \Cref{Subsec:regret} with our main theorem on the regret bound, followed by \Cref{Subsec:proof} devoted to proofs of the regret bound, and concluding with a matching lower bound in \Cref{Subsec:lower-bound}.


\subsection{Regret  analysis}
\label{Subsec:regret}

We begin by stating a non-asymptotic bound on the expected cumulative regret incurred by \Cref{alg:lambda_dp_mod} of $O(\sqrt{T\log T}) $.


\begin{theorem} 
\label{th:regret_bound} 
For any Poisson-MNL bandit problem under \Crefrange{assump:K}{assump:combined}, 
there are universal positive constants $c_1, c_2, c_3$ such that for all $T \geq \max\left\{4 +\frac{2\log(\DimFeature) + 1}{\sigma_0 (1 - \log 2)} - \frac{2 \log(\sigma_0 (1 - \log 2))}{\sigma_0 (1 - \log 2)}, 2\DimRank + 1 \right\}$,
the expected cumulative regret of \Cref{alg:lambda_dp_mod} is bounded as 
    \begin{align}
        \mathcal{R}^{\pi}(T; \bv, \btheta) \leq &c_1 + c_2 \DimFeature \sqrt{T \log T} + c_3 \DimRank \sqrt{T \log T},
    \end{align}
 where $c_1$  depends only on \( \Lambda, \bar{x}, \phigh\); \(c_2\) only on \(\Lambda, \bar{x}, \bar{v}, \sigma_0, \plow, \phigh\); and \(c_3\) only on \(\Lambda, \bar{x},  \sigma_1, \phigh\).
\end{theorem}

\begin{remark}
\Cref{th:regret_bound} provides a non-asymptotic regret bound that holds uniformly for almost all $T$. The regret bound gives an order of $(\DimFeature + \DimRank)\sqrt{T\log{T}}$, which is nearly optimal, compared with the lower bounds provided in \Cref{Subsec:lower-bound}. Note that the condition on $T$ is very mild---the threshold only depends on the dimensions and $\sigma_0$ (i.e., of the order $\frac{1}{\sigma_0}\log{(\DimFeature/\sigma_0)} + \DimRank$). This condition essentially requires $T\ge T_0$, where $T_0$ is defined in   \Cref{alg:lambda_dp_mod} as $T_0 :=\max\{\ceil{\frac{\log(\DimFeature T)}{\sigma_0 (1 - \log 2)}} + 1,  2 \DimRank + 1,  \lfloor\log T\rfloor\}$. If $T$ is smaller than the threshold, we will be staying in Stage 1 the entire time, and the regret bound will be of the order $T_0 = O(\log{(\DimFeature T)}+\DimRank)$. 
\end{remark}

\begin{remark}
    Note that when the arrival rate is a known constant, it is equivalent to $d_x = 0$, which gives the regret bound $O(d_z\sqrt{T\log{T}})$. If we further require $\plow = \phigh=p>0$, our problem reduces to the dynamic assortment problem, and our rate is faster than the state-of-the-art rate for dynamic assortment $O(d_z\sqrt{T}\log{T})$ \citep{chen2020dynamic,oh2021multinomial}.  
\end{remark}

\subsection{Proof Sketch}
\label{Subsec:proof}

We provide a proof sketch of Theorem \ref{th:regret_bound} in this section. 
For details of the proof, please see Appendix \ref{appendix:proof_of_theta_initial_error} to \ref{appendix::proof_of_sum_bv_error}. 
The proof consists of four major steps:
\begin{itemize}
\item [1.] \textbf{Bounding the parameter estimation errors.} 
In \Cref{subsec:bound_para}, we establish the high-probability bounds on the estimation error for both the MLE pilot estimators ($\hat\btheta$ and $\hat\bv$) in Stage 1 and the local MLE estimators ($\widehat{\btheta}_{t}$ and $\widehat{\bv}_{t}$) in Stage 2.
%
\item [2.] \textbf{Bounding the arrival rate estimation error.} In \Cref{subsec:bound_arrival}, we bound the difference between the estimated arrival rate and the true arrival rate at period $t$ 
with a high probability of $1-O(1/T)$.
\item [3.] \textbf{Bounding the expected per-customer revenue error.} In \Cref{Subsec:bound_return}, we bound difference between the estimation of expected per-customer revenue with truth at period $t$  with a high probability of $1-O(1/T)$.
\item [4.] \textbf{Bounding the expected regret.} Combining the  results above, in \Cref{sec:bound-regret} we show that with high probability, the gap between the reward of the optimal assortment-pricing decision and that of our policy at each time is upper-bounded by the gap between our upper confidence bound and the true reward, uniformly. Bounding the summation of the latter gaps and accounting for the low-probability failure event yields the regret bound.
\end{itemize}


Before going into details, we review two classes of random variables: \emph{sub-Gaussian} random variables (\Cref{def:subgaussian}) and random variables satisfying the \emph{Bernstein condition} (\Cref{def:bernstein}) (cf. \citealp[Chapter 2]{wainwright2019high}). In standard MNL-bandit problems, the log-likelihood or its gradient of each time period is typically sub-Gaussian \citep{abbasi2011improved,chen2020dynamic} followed from the boundedness of the reward, which makes it easier to bound the full log-likelihood (or its gradient) thanks to the thinner tail of the summands. However, in our model, the sub-Gaussian property no longer holds due to the additional Poisson process, thereby requiring a more careful analysis. It turns out that the log-likelihood of each time period in our case instead satisfies the Bernstein condition.


\begin{definition}[Sub-Gaussian]
\label{def:subgaussian}
    A random variable $X$ with mean $\mu = \mathbb{E}[X]$ is sub-Gaussian if there exists a positive constant $\sigma$ such that
    $$
    \mathbb{E}\left[e^{c(X-\mu)}\right] \leq e^{\sigma^2 c^2 / 2} \quad \text{for all } c\in \Real.
    $$
\end{definition}

\begin{definition}[Bernstein Condition]
\label{def:bernstein}
    A random variable $X$ with mean $\mu = \mathbb{E}[X]$ is said to satisfy the Bernstein condition if there exist constants $V > 0$ and $c > 0$ such that for all integers $k \geq 2$:
    $$
    \mathbb{E}\left[ |X - \mu|^k \right] \leq \frac{k!}{2} V c^{k-2}.
    $$
\end{definition}





\subsubsection{Bounding the Parameter Estimation Errors.}\label{subsec:bound_para}

We first establish the high-probability estimation error bounds for the MLE pilot estimator $\htheta$ and $\hv$ from Stage 1 (pure exploration) in \Cref{lm:theta_initial_error,lm:bv_initial_error}, and then provide the high-probability error bounds for the local MLE $\htheta_t$ and $\hv_t$ from Stage 2 in \Cref{lm:theta_error,lm:bv_error}. The proofs and technical details are deferred to Appendix \Cref{appendix:proof_of_theta_initial_error,appendix::proof_of_bv_initial_error} as well as \Cref{appendix::proof_of_theta_error,appendix::proof_of_bv_error}, respectively.

\begin{lemma}\label{lm:theta_initial_error}
With probability at least $1 - 2T^{-1}$, the $\ell_2$ error of the pilot estimator $\htheta$ is bounded as:
\begin{equation*}
\|\htheta - \btheta^*\|_2 \leq \tau_\btheta.
\end{equation*}
Here, \(\tau_\btheta = \min\{1, \tilde{\tau}_\btheta \}\), 
where $\tilde{\tau}_\btheta $ is bounded as:
\begin{multline}
\label{eq:tau_theta}
   \tilde{\tau}_\btheta^2
        \leq \frac{2 \exp(2\bar{x})}{T \Lambda \sigma_1} \left(2 + \frac{4\log(T)}{ \Lambda \exp(\bar{x}) T_0} + \sqrt{\frac{4\log(T)}{ \Lambda \exp(\bar{x}) T_0}} + \exp(-\bar{x})\right)
         \\
         + \frac{8(2 \bar{x} c_4 + 1) \exp(\bar{x})}{T_0\Lambda\sigma_1} \left(\log T + \DimRank \log \left(3\bar{x}(\Lambda T + 1)\right)\right).
\end{multline}
Here, $c_4$ is  defined as 
    \begin{equation}
    \label{eq:df_c_2}
        c_4 = \max\left\{1, \frac{2e^2}{\log^2\left(1 + \frac{3}{\Lambda \exp(\bar{x})}\right)\sqrt{6\pi}\Lambda \exp(-\bar{x})}\right\}  \frac{2e}{\log\left(1 + \frac{3}{\Lambda \exp(\bar{x})}\right)}.
    \end{equation}
\end{lemma}


\begin{lemma}\label{lm:bv_initial_error}
    With probability at least $1-3T^{-1}$, the $\ell_2$ error of the pilot estimator $\hv$ is bounded as:
\begin{align*}
    \left\|\widehat{\bv}-\bv^*\right\|_2 \leq \tau_\bv.
\end{align*}
Here, \(\tau_\bv = \min\{1, \tilde{\tau}_\bv\}\), where $ \tilde{\tau}_\bv$ is bounded as:
\begin{multline}
\label{eq:tau_v}
    \tilde{\tau}_\bv^2 \leq \frac{2\exp(2\bar{x}) }{\kappa T \Lambda \sigma_0} \left(2 + \frac{4\log(T)}{\Lambda \exp(\bar{x}) T_0} + \sqrt{\frac{4\log(T)}{\Lambda \exp(\bar{x}) T_0}}\right) + 
    \frac{8 \exp(\bar{x})}{\kappa T_0 \Lambda \sigma_0} \left((\DimFeature + 1)\log T +  \DimFeature \log(6\Lambda) \right)\\
    +\frac{8 \exp(\bar{x})}{\kappa T_0 \Lambda\sigma_0} \sqrt{T_0 \Lambda \exp(\bar{x})  ((\DimFeature + 1)\log T + \DimFeature\log(6 \Lambda))},
\end{multline}
and 
\begin{equation}
\label{kappa}
\kappa = \frac{\exp( - \bar{v} - \phigh)}{(K\exp(\bar{v} - \plow) + 1)^2}.
\end{equation}
\end{lemma}
\begin{remark}
\label{rmk:tau_theta_bv}
Both $\tau_{\btheta}^2$ in \Cref{lm:theta_initial_error} and $\tau_{\bv}^2$ in \Cref{lm:bv_initial_error} are of the order $O\left(\frac{\log T}{T_0}\right)$. By our initialization $T_0 = \Omega(\log T)$, $O\left(\frac{\log T}{T_0}\right)$ simplifies to $O(1)$. 
Recall that in Stage 2, i.e., for $t=T_0 + 1, \ldots, T$, we obtain the local MLEs $\htheta_t$ and $\hv_t$ within their feasible regions: $||\htheta_t - \htheta||_2 \leq \tau_\btheta$ and $||\hv_t - \hv||_2 \leq \tau_\bv$. While this constant shrinkage of the feasible regions may appear less meaningful, it actually plays a crucial role because the estimation errors of the local MLEs scale exponentially with the radii $\tau_\btheta$ and $\tau_\bv$. Consequently, this shrinkage leads to both a substantial improvement in algorithmic performance and a substantial reduction in the regret.  
\end{remark}
\begin{remark}
    The constant $\kappa$ in \Cref{lm:bv_initial_error} depends only on the number of products in the assortment and the boundedness conditions. In the literature, it is common to introduce an additional abstract assumption that essentially assumes the positive definiteness of the Fisher information via such a constant $\kappa$ (see, e.g., \cite{oh2021multinomial,oh2019thompson}), or, equivalently, $1/\Upsilon$ as used in \cite{cheung2017thompson}. In our case, this positive definiteness follows directly from the more concrete and natural boundedness assumptions, with $\kappa$ explicitly specified in \Cref{kappa}. 
\end{remark}

In Stage 2, we obtain the local MLEs \eqref{eq:local-mle} within the feasible regions. With the radii of the feasible regions shrink from $1$ and $\bar{v}$ to $\tau_{\btheta}$ and $\tau_{\bv}$ respectively, we have the following two uniform error bounds for the local MLEs.
\begin{lemma}\label{lm:theta_error}
    With probability $1 - 4T^{-1}$, the following hold uniformly over all $t = T_0, \ldots, T-1$:
\begin{equation}\label{eq:theta_error_I}
    \begin{aligned}
         \left(\widehat{\btheta}_t - \btheta^*\right)^{\top} I_{t}^\Poi\left(\btheta^*\right)\left(\widehat{\btheta}_t - \btheta^*\right)  \leq \omega_\btheta, \;
         \left(\widehat{\btheta}_t - \btheta^*\right)^{\top}  I_{t}^\Poi\left(\widehat{\btheta}_t\right)\left(\widehat{\btheta}_t - \btheta^*\right)  \leq \omega_\btheta, \;
         \|\widehat{\btheta}_t - \btheta^*\| \leq 2\tau_\theta,
    \end{aligned}
\end{equation}
where
\begin{equation}\label{df:w_theta}
\begin{multlined}
    \omega_\btheta  \leq  
    8 e^{2\tau_\btheta \bar{x}}
    \left(
    \frac{1}{2} + e^{\bar{x}} 
    + \sqrt{\frac{2e^{\bar{x}} \log T}{T\Lambda}} 
    + \frac{4\log T}{T\Lambda} 
    + 2 (\tau_\btheta \bar{x} c_4 + 1)\big(2\log T + d \log (6\tau_\btheta \bar{x}(\Lambda T + 1))\big) 
    \right).
\end{multlined}
\end{equation}

\end{lemma}

\begin{lemma}\label{lm:bv_error}
    With probability $1-4 T^{-1}$, the following hold uniformly over all $t=T_0, \ldots, T-1$:
\begin{equation}\label{eq:v_error_I}
    \begin{aligned}
              \left(\widehat{\bv}_t-\bv^*\right)^{\top} I_{t}^\MNL\left(\bv^*\right)\left(\widehat{\bv}_t-\bv^*\right)  \leq \omega_\bv,\quad
              \left(\widehat{\bv}_t-\bv^*\right)^{\top} I_{t}^\MNL\left(\widehat{\bv}_t\right)\left(\widehat{\bv}_t-\bv^*\right)  \leq \omega_\bv, \quad \|\widehat{\bv}_t - \bv^*\| \leq 2\tau_\bv,
    \end{aligned}
\end{equation}
where
\begin{align}
    \omega_\bv &\leq 
    \begin{multlined}[t]
    8c_8 e^{\bar{x}} + 4c_8\sqrt{\frac{8e^{\bar{x}} \log T}{T\Lambda}} + \frac{32c_8\log T}{T\Lambda} + 
    8(4 \tau_\bv c_4 + c_5)c_8((\DimFeature + 2)\log T + \DimFeature \log (6 \tau_\bv \Lambda)), \label{df:w_bv}
\end{multlined}\\
    c_5 &= \frac{16\tau_\bv^2}{4\tau_\bv+\exp(-4\tau_\bv)  - 1}, \label{eq:df_c_4} \\
    c_8 &= 3(\exp(4\tau_\bv ) - 1)(K \exp(\bar{v} - p_l) + 1) + 1.
\end{align}
\end{lemma}

\begin{remark}
Since both $\tau_{\btheta}$ and $\tau_{\bv}$ are of order $O(1)$, both $\omega_{\btheta}$ and $\omega_{\btheta}$ are of order $O(\log T)$. Note that the Fisher information $I_{t}^\Poi\left(\bv\right)$ and $I_{t}^\MNL\left(\bv\right)$ are the sums of $t$ positive definite rank-$1$ matrices. These lemmas show that the $\htheta_t$ and $\hv_t$ increasingly concentrate around the ground truths $\btheta^*$ and $\bv^*$ as $t$ increases.
\end{remark}
\begin{remark}
    Since $\tau_{\btheta}$ and $\tau_{\bv}$ are of order $O(1)$, both $c_5$ and $c_8$ are also constants. 
    In particular, $c_5$ is small and almost tight, in contrast to existing analyses that typically yield an overly conservative constant of order  $\exp{(8\tau_{\bv})}$.
    Achieving this improvement requires a careful analysis that is often omitted in the literature. 
    Such an improvement also translates into better performance of the algorithm as $c_5$ appears directly in the confidence bound. 
    See \Cref{remark:improve} for details.
\end{remark}

\begin{remark}
    As discussed in \Cref{rmk:tau_theta_bv}, the dependence of $\omega_{\btheta}$ on $\tau_{\btheta}$ is in exponential terms. Hence, the localization (or shrinkage) in Stage 1 plays an important role in the constants, especially when $\bar{x}$ is relatively large. Similarly, the dependence of $\omega_{\bv}$ on $\tau_{\bv}$ is also exponential (in $c_8$). 
\end{remark}



\subsubsection{Bounding the Arrival Rate Error.}
\label{subsec:bound_arrival}

As noted at the end of \Cref{Subsec:UCB}, the estimation error of the expected revenue $R(S, \bp,\Featuret ;\btheta^*, \bv^* )$ of period $t$ depends on the estimation errors of the expected per-customer revenue and the arrival rate, both of which can be analyzed based on the parameter estimation errors in the MNL choice model and the Poisson arrival model discussed above. In this section, we focus on the error of arrival rate; and we will discuss the error of the expected per-customer revenue in \Cref{Subsec:bound_return}.


\begin{lemma}\label{lemma:revenue_estimate}
Suppose \Cref{eq:theta_error_I} in \Cref{lm:theta_error} holds for $t-1$, then the following hold
{\small
\begin{align} 
&\begin{multlined}
    \left|
        \lambda(S, \bp; \widehat{\btheta}_{t - 1}) 
        - \lambda(S, \bp; \btheta^*)
    \right| \\
    \leq  \underbrace{
    \min \left\{
       \Lambda (e^{\bar{x}} - e^{-\bar{x}} ),\,
        \sqrt{\Lambda e^{(2\tau_\btheta + 1) \bar{x}} \omega_\btheta 
        \opnorm{{I}_{t-1}^{\Poi -\frac{1}{2}}(\widehat{\btheta}_{t - 1}) {M}_t^{\Poi} (\widehat{\btheta}_{t - 1}|S, \bp){I}_{t-1}^{\Poi -\frac{1}{2}}(\widehat{\btheta}_{t - 1})}  } 
    \right\} }_{\eta^{\Poi}_{t}}, 
\end{multlined}
\label{bound:arrival_rate_ucb_hat}
\\
&\begin{multlined}
\left|
    \lambda(S, \bp; \widehat{\btheta}_{t - 1}) - \lambda(S, \bp; \btheta^*)\right| \\
            \leq
            \underbrace{
            \min \left\{
            \Lambda (e^{\bar{x}} - e^{-\bar{x}} ),\,
            \sqrt{\Lambda e^{(2\tau_\btheta + 1) \bar{x}} \omega_\btheta
            \opnorm{{I}_{t-1}^{\Poi -\frac{1}{2}}(\btheta^*) {M}_t^{\Poi} (\btheta^*|S, \bp){I}_{t-1}^{\Poi -\frac{1}{2}}(\btheta^*)}  }
            \right\}
            }_{\eta^{*,\Poi}_{t}}. \label{bound:arrival_rate_ucb}
\end{multlined}
\end{align}    
}
Additionally, we have 
{
\small
\begin{equation}
    \opnorm{
        I_{t-1}^{\Poi -\frac{1}{2}}(\widehat{\btheta}_{t - 1})
        {M}_t^{\Poi} (\widehat{\btheta}_{t - 1}|S, \bp)
        I_{t-1}^{\Poi -\frac{1}{2}}(\widehat{\btheta}_{t - 1})
    } \leq
    e^{4\tau_\btheta \bar{x}}
    \opnorm{{I}_{t-1}^{\Poi -\frac{1}{2}}(\btheta^*) {M}_t^{\Poi} (\btheta^*|S, \bp){I}_{t-1}^{\Poi -\frac{1}{2}}(\btheta^*)}.
\end{equation}
}

\end{lemma}

\begin{remark}
    \Cref{lemma:revenue_estimate} holds deterministically. The key idea is to bound the change in the arrival rate given that the change in the parameter (from $\btheta^*$ to $ \widehat{\btheta}_{t - 1}$) is bounded by known quantities. The detailed proof is given in \Cref{appendix:proof_of_sum_lambda_error}.
\end{remark}

\begin{remark}
    Since \Cref{eq:theta_error_I} in \Cref{lm:theta_error} holds uniformly for all $t = T_0, \ldots, T-1$ with probability at least $1-4T^{-1}$, the error bounds in \Cref{lemma:revenue_estimate} also hold uniformly for $t = T_0+1, \ldots, T$ with $1-4T^{-1}$.
\end{remark}

\begin{remark}
    \Cref{bound:arrival_rate_ucb_hat} forms the basis for the arrival-rate-induced confidence bound in \Cref{eq:ucb-lambda}: multiplied by $\phigh$, this data-driven quantity uniformly compensates for the potential deficit in reward estimation due to the estimation error of the arrival rate, with high probability.  
\end{remark}

\begin{remark}
    Note that both  terms $\opnorm{{I}_{t-1}^{\Poi -\frac{1}{2}}(\btheta^*) {M}_t^{\Poi} (\btheta^*|S, \bp){I}_{t-1}^{\Poi -\frac{1}{2} }(\btheta^*)}$ and $ \opnorm{{I}_{t-1}^{\Poi -\frac{1}{2}}(\widehat{\btheta}_{t - 1}) {M}_t^{\Poi} (\widehat{\btheta}_{t - 1}|S, \bp){I}_{t-1}^{\Poi -\frac{1}{2}}(\widehat{\btheta}_{t - 1})} $ vanish as $t$ increases, at rates roughly of the order $1/t$.
\end{remark}

\subsubsection{Bounding the Expected Per-customer Revenue Error.}\label{Subsec:bound_return}

The estimation error bound of the per-customer revenue $\revenue{\bv^*}$ is given in \Cref{lemma:revenue_estimate_2}.

\begin{lemma}\label{lemma:revenue_estimate_2}
Suppose \Cref{eq:v_error_I} in \Cref{lm:bv_error} holds for $t-1$, then the following hold for $t$ with any $S \subseteq [N]$ and $\bp \in (\plow, \phigh)^N$:
{\small
\begin{align}
&\begin{multlined}
        \left|
            r(S, \bp, \Featuret; \widehat{\bv}_{t - 1})-r(S, \bp, \Featuret;\bv^*)
        \right|
         \leq  
        \underbrace{ 
        \min \left\{
        \phigh,\,
        \sqrt{
            \omega_\bv c_0
            \left\|\widehat{I}^{-\frac{1}{2}}{}^\MNL_{t - 1}(\widehat{\bv}_{t-1})\widehat{M}_{t}^\MNL(\widehat{\bv}_{t-1}|S, \bp)\widehat{I}^{-1/2}_{t - 1}{}^\MNL(\widehat{\bv}_{t-1})\right\|_{\mathrm{op}}} \right \} }_{\eta^{\MNL}_{t}},
        \label{bound:per_customer_revenue_ucb_hat}
\end{multlined} \\
& \begin{multlined}
        \left|
            r(S, \bp, \Featuret; \widehat{\bv}_{t - 1})-r(S, \bp, \Featuret;\bv^*)\right| 
         \leq 
        \underbrace{
        \min \left\{
        \phigh,\,
       \sqrt{
            \omega_\bv c_0 \left\|\widehat{I}_{t-1}^{-1/2}{}^\MNL\left(\bv^*\right)\widehat{M}_{t}^\MNL\left(\bv^* \mid S, \bp\right) \widehat{I}_{t-1}^{-1/2}{}^\MNL\left(\bv^*\right)\right\|_{\mathrm{op}}
            } \right\}
        }_{\eta^{*,\MNL}_{t}}.
\end{multlined}
\label{bound:per_customer_revenue_ucb}
\end{align}
}
Additionally, we have 
{\small
\begin{equation}
    \left\|\widehat{I}^{-\frac{1}{2}}{}^\MNL_{t - 1}(\widehat{\bv}_{t-1})\widehat{M}_{t}^\MNL(\widehat{\bv}_{t-1}|S, \bp)\widehat{I}^{-1/2}_{t - 1}{}^\MNL(\widehat{\bv}_{t-1})\right\|_{\mathrm{op}} 
    \leq
    c^2_8 \left\|\widehat{I}_{t-1}^{-1/2}{}^\MNL\left(\bv^*\right)\widehat{M}_{t}^\MNL\left(\bv^* \mid S, \bp\right) \widehat{I}_{t-1}^{-1/2}{}^\MNL\left(\bv^*\right)\right\|_{\mathrm{op}},
\end{equation}
}
where 
\begin{equation}\label{eq:df_c_8}
    c_8 = \big[3(\exp(4\tau_\bv) - 1)(K \exp(\bar{v} - p_l) + 1) + 1\big].
\end{equation}
\end{lemma}
The proof and the implications for \Cref{lemma:revenue_estimate_2} are very similar to those of \Cref{lemma:revenue_estimate} except one subtle difference. We use $\widehat{I}^{-\frac{1}{2}}{}^\MNL_{t - 1}(\cdot), \widehat{M}_{t}^\MNL(\cdot |S, \bp) $ (\Cref{def_vhat}) rather than ${I}^{-\frac{1}{2}}{}^\MNL_{t - 1}(\cdot), {M}_{t}^\MNL(\cdot |S, \bp)$ (\Cref{def_v}) because the latter pair involves unknown parameter $\btheta^*$, which can be bounded by the former pair (see \Cref{lemma_estimate_M_I}). 
The detailed proof is in \Cref{appendix::proof_of_revenue_estimate_2}, and the Inequality \eqref{bound:per_customer_revenue_ucb} forms the basis for the per-customer-revenue-induced confidence bound shown in \Cref{eq:ucb-R}.

\subsubsection{Bounding of the Expected Regret.}
\label{sec:bound-regret}


In this section, we show the upper bound of the expected cumulative regret $\Regret{T}{\pi}$. Let $\gE$ be the high probability event that the inequalities in \Cref{lm:theta_error,lm:bv_error} hold and $\gE^c$ be its complement (i.e., the event that at least one of those inequalities fails to hold). The regret comes from three sources:
\begin{enumerate}
    \item The regret incurred in the first $T_0$ periods in Stage 1, which is upper bounded by $T_0 \phigh \BaseArrival \exp(\bar{x}) = O(\log{T})$ (by the boundedness of prices and the arrival rate).
    \item The regret incurred in Stage 2 under $\gE$, which we will show to be of order $O((d_\bx +d_\bz) \sqrt{T\log{T}})$.
    \item The regret incurred in Stage 2 under $\gE^c$, which is upper bounded by $\frac{8}{T} \cdot T \cdot \phigh\Lambda \exp(\bar{x}) = O(1)$.
\end{enumerate}
Combining the three sources together gives the statement. See below for detailed proof.

\paragraph{Proof of \Cref{th:regret_bound}.}

We first show that under event $\gE$, 
\begin{equation}
    \RevenueUCB > \Revenue \label{eq:UCBlarger}
\end{equation}
for all $t=T_0+1,\ldots, T$ and $ S \subseteq [N], \bp \in (\plow, \phigh)^N$, which directly gives that under event $\gE$,
\begin{equation}
\bar{R}_t(S_t,\bp_t) \ge \bar{R}_t(S^*,\bp^*) \ge R_t(S^*,\bp^*) \ge R_t(S_t,\bp_t)  \label{ineq:individual_rewards}
\end{equation}
for all $t=T_0 + 1,\ldots, T$.
Note that
\begin{align}
\label{eq:RUCB-gap}
     \RevenueUCB - \Revenue &= 
     \lambda(S, \bp; \widehat{\btheta}_{t - 1})  r(S, \bp, \Featuret; \widehat{\bv}_{t - 1}) - \lambda(S, \bp; \btheta^*) r(S, \bp, \Featuret;\bv^*)  + \phigh \eta^{\Poi}_{t} +  \Lambda e^{\bar{x}}  \eta^{\MNL}_{t} \nonumber \\
     &\ge 
     \begin{multlined}[t]
     \underbrace{ - | \lambda(S, \bp; \widehat{\btheta}_{t - 1}) - \lambda(S, \bp; \btheta^*) | r(S, \bp, \Featuret; \widehat{\bv}_{t - 1})  + \phigh \eta^{\Poi}_{t} }_{\zeta_1} \\
     \underbrace{ -  \lambda(S, \bp; \btheta^*)  | r(S, \bp, \Featuret;\bv^*) - r(S, \bp, \Featuret; \widehat{\bv}_{t - 1}) | + \Lambda e^{\bar{x}}  \eta^{\MNL}_{t} }_{\zeta_2}.
    \end{multlined}
\end{align}
By \Cref{lemma:revenue_estimate} and the fact that per-customer revenue is upper bounded by $\phigh$, we have $\zeta_1\ge 0 $. Similarly, \Cref{lemma:revenue_estimate_2} and the boundedness of the arrival rate give $\zeta_2\ge 0$. Taken together, we have Inequality \eqref{eq:UCBlarger}. 

Therefore, under event $\gE$, by Inequality \eqref{ineq:individual_rewards} and similar argument in Inequality \eqref{eq:RUCB-gap}, we have 
\begin{align}
    R_t(S^*,\bp^*) - \Revenue \le & \bar{R}_t(S_t, \bp_t) - \Revenue  \\
    \le &
    \begin{multlined}[t]
    | \lambda(S, \bp; \widehat{\btheta}_{t - 1}) - \lambda(S, \bp; \btheta^*) | r(S, \bp, \Featuret; \widehat{\bv}_{t - 1})  + \phigh \eta^{\Poi}_{t} \\
     + \lambda(S, \bp; \btheta^*)  | r(S, \bp, \Featuret;\bv^*) - r(S, \bp, \Featuret; \widehat{\bv}_{t - 1}) | + \Lambda e^{\bar{x}}  \eta^{\MNL}_{t} 
    \end{multlined} \\
    \le & 2 ( \phigh \eta^{\Poi}_{t}  + \Lambda e^{\bar{x}}  \eta^{\MNL}_{t} )
\end{align}
for $t = T_0+1 , \ldots, T$.
Therefore, the regret incurred by the second source is upper bounded as
\begin{align}
\label{eq:regret_second_source}
    \sum_{t = T_0 + 1}^T \E\left[  (  R_t(S^*,\bp^*) - \Revenue )\mathbbm{1}\{\mathcal{E}\} \right] 
    \le 2\phigh  \sum_{t = T_0 + 1}^T \E( \eta^{\Poi}_{t} ) + 2 \Lambda e^{\bar{x}} \sum_{t = T_0 + 1}^T  \E( \eta^{\MNL}_{t} ).
\end{align}

Next, we bound the two summations in Inequality \eqref{eq:regret_second_source}.

\textbf{(a) Bounding $\sum_{t = T_0 + 1}^T \E( \eta^{\Poi}_{t} )$.}

By \Cref{lemma:revenue_estimate} and Cauchy Schwarz Inequality, we have
\begin{equation}
\label{ineq:poisson_term_second_source}
     \sum_{t = T_0 + 1}^T \E[ \eta^{\Poi}_{t} \mathbbm{1}\{\mathcal{E}\} ]
     \le \E \left[  \sum_{t = T_0 + 1}^T e^{2\tau_{\btheta}\bar{x} }  \eta^{*,\Poi}_{t} \mathbbm{1}\{\mathcal{E}\} \right] 
     \le \sqrt{T} e^{2\tau_{\btheta}\bar{x} }  \sqrt{ \E \left[  \sum_{t = T_0 + 1}^T  ( \eta^{*,\Poi}_{t})^2   \right]},
\end{equation}
where $\eta^{*,\Poi}_{t}$ is defined in \Cref{bound:arrival_rate_ucb}.
To bound the summation on the right-hand side, we have the following lemma (see \Cref{appendix::proof_of_revenue_estimate} for proof).
\begin{lemma} \label{lm:sum_lambda_error} 
 Under \Crefrange{assump:K}{assump:combined}, the following holds 
\begin{equation*}
\begin{aligned}
\sum_{t = T_0 + 1}^T  ( \eta^{*,\Poi}_{t})^2  = & \sum_{t=T_0+1}^T \min \Bigg\{
\Lambda^2 (e^{\bar{x}} - e^{-\bar{x}} )^2, \,
 \omega_\btheta \Lambda 
 e^{(2\tau_\btheta + 1) \bar{x}}
 \opnorm{{I}_{t-1}^{\Poi -\frac{1}{2}}(\btheta^*) {M}_t^{\Poi} (\btheta^*|S_t){I}_{t-1}^{\Poi -\frac{1}{2}}(\btheta^*)}\Bigg\}\\
\leq & c_6  \log \frac{\operatorname{det} {I}_{T}^{\Poi}\left(\btheta^*\right)}{\operatorname{det} {I}_{T_0}^{\Poi}\left(\btheta^*\right)}
\leq  \DimRank c_6 \left( \DimRank \log \frac{\bar{x}^2 T }{\DimRank} +(\DimRank + 1) \bar{x}- \log (T_0 \sigma_1)\right)
=  O(d_\bx^2 \log T),
\end{aligned}
\end{equation*}
where 
\begin{equation}\label{eq:df_c_6}
    c_6 = \dfrac{\Lambda^2(e^{\bar{x}} - e^{-\bar{x}} )^2}{\log\left(1 + \dfrac{\Lambda^2(e^{\bar{x}} - e^{-\bar{x}} )^2}{\omega_\btheta \Lambda  
    e^{(2\tau_\btheta + 1) \bar{x}}
    }\right)}.
\end{equation}
\end{lemma}

Returning back to Inequality \eqref{ineq:poisson_term_second_source} with \Cref{lm:sum_lambda_error}, we have 
\begin{equation}
\label{eq:eta-poi}
    \sum_{t = T_0 + 1}^T \E[ \eta^{\Poi}_{t} \mathbbm{1}\{\mathcal{E}\} ] \le O(d_\bx \sqrt{T \log{T}} ).
\end{equation}

\textbf{(b) Bounding $\sum_{t = T_0 + 1}^T \E( \eta^{\MNL}_{t} )$.}

Similarly, by \Cref{lemma:revenue_estimate_2} and Cauchy Schwarz Inequality we have
\begin{equation}
\label{ineq:MNL_term_second_source}
     \sum_{t = T_0 + 1}^T \E[ \eta^{\MNL}_{t} \mathbbm{1}\{\mathcal{E}\} ]
     \le \E \left[  \sum_{t = T_0 + 1}^T c_8 \eta^{*,\MNL}_{t} \mathbbm{1}\{\mathcal{E}\} \right] 
     \le  c_8 \sqrt{T} \sqrt{ \E \left[  \sum_{t = T_0 + 1}^T  ( \eta^{*,\MNL}_{t})^2  \right]},
\end{equation}
where $ \eta^{*,\MNL}_{t}$ is defined in \Cref{bound:per_customer_revenue_ucb}.
To bound the summation on the right-hand side, we have the following lemma (see \Cref{appendix::proof_of_sum_bv_error} for proof).
\begin{lemma} \label{lm:sum_bv_error}
Under \Crefrange{assump:K}{assump:combined}, the following holds with probability at least $1-\frac{1}{T}$: 
\begin{equation*}
    \begin{aligned}
        \sum_{t=T_0+1}^T  (\eta^{*,\MNL}_{t})^2 
        & = \sum_{t=T_0+1}^T \min \left\{\bar{p}^2_h, \omega_\bv c_0 \left\|\widehat{I}_{t-1}^{\MNL -1 / 2}\left(\bv^*\right)\widehat {M}_t^{\MNL}\left(\bv^* \mid S_t\right) \widehat{I}_{t-1}^{\MNL -1 / 2}\left(\bv^*\right)\right\|_{\mathrm{op}}\right\} \\
        & \leq  \DimFeature c_7 \left(  \log \frac{   4 T   }{\DimFeature} - \log( T_0 \sigma_0)\right)
        = O(d_\bz^2 \log T),
    \end{aligned}
\end{equation*}
where 
\begin{equation}
c_7 = \dfrac{\omega_\bv c_0 \exp(2\bar{x})\bar{p}^2_h}{\log\left(1 + \dfrac{\bar{p}^2_h} {\omega_\bv c_0 \exp(2\bar{x})}\right)}.
\end{equation}
\end{lemma}

Using \Cref{lm:sum_bv_error} to further bound the right-hand side of Inequality \eqref{ineq:MNL_term_second_source}, we have 
\begin{equation}
\label{eq:eta-mnl}
     \sum_{t = T_0 + 1}^T \E[ \eta^{\MNL}_{t} \mathbbm{1}\{\mathcal{E}\} ]
     \le c_8\sqrt{T}\sqrt{ \DimFeature c_7 \left(  \log \frac{   4 T   }{\DimFeature} - \log (T_0 \sigma_0)\right) + \frac{1}{T}\cdot T \phigh^2 }
     = O(d_\bz \sqrt{T \log{T}} ).
\end{equation}

Combining Inequalities \ref{eq:eta-poi} and \ref{eq:eta-mnl} with Inequality \ref{eq:regret_second_source}, we have the regret from the second source upper bounded by
\begin{equation}
     \sum_{t = T_0 + 1}^T \E\left[  (  R_t(S^*,\bp^*) - \Revenue )\mathbbm{1}\{\mathcal{E}\} \right] 
     \le O \left( (d_\bx +d_\bz) \sqrt{T\log{T}}\right).
\end{equation}

Next, we consider the regret from the third source. Since $P(\mathcal{E}^c) \le 8/T$, and   $R_t(S^*,\bp^*) \le \phigh$
\begin{equation}
     \sum_{t = T_0 + 1}^T \E\left[  (  R_t(S^*,\bp^*) - \Revenue )\mathbbm{1}\{\mathcal{E}^c\} \right] \le T \cdot \phigh\Lambda \exp(\bar{x}) \cdot \frac{8}{T} = 8\phigh \Lambda \exp(\bar{x}).
\end{equation}

Combining all three sources gives
\begin{align}
    \Regret{T; \bv, \btheta}{\pi} = &
    \sum_{t=1}^{T_0} \E \left \{ \RevenueUCB[t]  - \Revenue[t] \right \}  +  \sum_{t=T_0 +1}^{T} \E \left \{ \left( \RevenueUCB[t]  - \Revenue[t]  \right) \mathbbm{1}\{\mathcal{E}\}\right \} \\
    & + \sum_{t=T_0 +1}^{T} \E \left \{ \left( \RevenueUCB[t]  - \Revenue[t]  \right) \mathbbm{1}\{\mathcal{E}^c\}\right \} \\
    \le & O(1) + O((d_\bx+d_\bz) \sqrt{T\log{T}} ) + O(1) =  O((d_\bx+d_\bz) \sqrt{T\log{T}} ) + O(1).
\end{align}

\subsection{Lower bound}
\label{Subsec:lower-bound}



To understand the fundamental limitations of any policy in such settings, we now turn to establishing a regret lower bound. This ensures that our upper bound is tight and demonstrates the optimality of our proposed strategy in the worst-case scenario. Specifically, we can show that any policy satisfying the assumptions on model parameters will incur at least the following regret in a worst-case instance.

\begin{theorem}[Non-asymptotic regret lower bound]\label{th:lower_bound}
Let $\PolicySet$ denote the set of policies that output feasible assortment-pricing decisions satisfying \Cref{assump:K,assump:price} for the problem in \Cref{Subsec:model}. 
Suppose at least one of the following conditions is satisfied:
\begin{itemize}
    \item[(i)] 
    \(\min\{\DimFeature - 2,\, N\}\ge K\) and \(\Lambda \ge 1\);
    \item[(ii)]
    \(\Lambda \ge 1\) and 
        $\log\!\frac{N-\DimFeature}{K} \;\ge\; 8\log 2 - \frac{11}{4}\log 3$;
    \item[(iii)] 
    \(\min\{\DimRank - 2,\, N\}\ge K\) and \(\Lambda \ge 1\);
\end{itemize}
then for any policy $\pi \in \PolicySet$, there exists a problem instance satisfying  \Cref{assump:K,as:v&assortment_new,assump:theta,assump:x,assump:combined,assump:price} such that the expected regret of policy $\pi$ under this instance is lower bounded as
\begin{align}
\label{eq:lower_bound}
\mathcal{R}^{\pi}(T) \;\ge\; c_9\,\sqrt{\Lambda T},
\end{align}
where \(c_9>0\) depends only on
\(\DimFeature, \DimRank, K, \plow, \phigh,\) and \(\bar{x}\).
\end{theorem}

\begin{remark}
    This lower bound is a non-asymptotic minimax lower bound primarily focusing on the rate with respect to the time horizon $T$. Comparing with the upper bound $O(\sqrt{T\log{T}})$ shown in \Cref{th:regret_bound}, our algorithm is near minimax optimal (up to log factors). 
\end{remark}

\begin{remark}
    This lower bound is the first lower bound for the joint contextual assortment-pricing problem and is not a direct extension of the lower bound for the contextual assortment problem established in \cite{chen2020dynamic}. The bound in \cite{chen2020dynamic} is asymptotic and requires $N\geq K \cdot 2^\DimFeature$. Our lower bound, however, is non-asymptotic and only requires $N>14K$ (at least one of (i) and (ii) holds). This relaxed assumption significantly enlarges the applicability of the lower bound, especially when the dimension of the product feature is moderately high. 
\end{remark}

\begin{remark}
    We consider algorithms choosing $K$ products for assortment in this lower bound, in alignment with \Cref{assump:K}. However, our proof only requires the assortment's capacity is smaller or equal to $K$ so it also applies to the setting considered in \cite{oh2021multinomial,chen2022nearly,agrawal2019mnl,chen2020dynamic}.
\end{remark}

\Cref{th:lower_bound} establishes a non-asymptotic lower bound of order $\sqrt{\Lambda T}$ under mild conditions. Next, we provide an asymptotic version and further consider the dependence on the dimensions $\DimFeature,\DimRank$.

\begin{theorem}[Asymptotic regret lower bound]\label{th:lower_bound_asymp}
Let $\PolicySet$ denote the set of policies that output feasible assortment-pricing decisions satisfying \Cref{assump:K,assump:price} for the problem in \Cref{Subsec:model}. For any policy $\pi \in \PolicySet$ the following statements hold.

\begin{enumerate}[label=(\roman*)]
\item \label{enum:thm3-1} If $\min\{\dfrac{\DimFeature + 1}{4},N\}\ge K$, then there exists a problem instance satisfying \Cref{assump:K,as:v&assortment_new,assump:theta,assump:x,assump:combined,assump:price} such that the expected regret of policy $\pi$ under this instance is lower bounded by
\[
\liminf_{T\to\infty}\frac{\mathcal{R}^{\pi}(T)}{\sqrt{\Lambda T}}
\ \ge\
\LowerC{1}\sqrt{\DimFeature},
\]
where $\LowerC{1}$ only depends on \(\plow, \phigh, K\). In particular, 
when \(\plow = \Omega (\log K)\) and \(\phigh = \Omega (\log K)\), we have  \(\LowerC{1} = \Omega(\log K)\). When \(\plow = \Omega (1)\) and \(\phigh = \Omega (1)\), we have  \(\LowerC{1} = \Omega(\frac{1}{K})\).

\item \label{enum:thm3-2} If $\Lambda\ge 1$,  $\log\!\frac{N-\DimFeature}{K}\ge 8\log 2-\frac{11}{4}\log 3$, then there exists a problem instance satisfying \Cref{assump:K,as:v&assortment_new,assump:theta,assump:x,assump:combined,assump:price} such that the expected regret of policy $\pi$ under this instance is lower bounded by
\[
\liminf_{T\to\infty}\frac{\mathcal{R}^{\pi}(T)}{\sqrt{\Lambda T}}
\ \ge\
\LowerC{2}
\min \left\{\DimFeature, \log\left((N - \DimFeature)/K\right) \right\}.
\]
where $\LowerC{2}$ only depends on \(\plow, \phigh, K\).  In particular, when \(\plow = \Omega (\log K)\) and \(\phigh = \Omega (\log K)\), we have 
\(\LowerC{2} = \Omega(\log K)\). When \(\plow = \Omega (1)\) and \(\phigh = \Omega (1)\), we have 
\(\LowerC{2} = \Omega(\frac{1}{{K}})\).

\item \label{enum:thm3-3} If $\min\{ \dfrac{\DimRank + 1}{4} ,N\}\ge K$, $\Lambda\ge 1$, then there exists a problem instance satisfying \Cref{assump:K,as:v&assortment_new,assump:theta,assump:x,assump:combined,assump:price} such that the expected regret of policy $\pi$ under this instance is lower bounded by
\[
\liminf_{T\to\infty}\frac{\mathcal{R}^{\pi}(T)}{\sqrt{\Lambda T}}
\ \ge\
\LowerC{3} \sqrt{\DimRank }.
\]
where $\LowerC{3}$ only depends on \(\plow, \phigh, K, \bar{x}\). In particular, when \(\plow = \Omega (\log K)\) and \(\phigh = \Omega (\log K)\), we have 
\(\LowerC{3} = \Omega(\frac{\sqrt{K}\log K}{\exp(\bar{x}) \bar{x}})\). 
When \(\plow = \Omega (1)\) and \(\phigh = \Omega (1)\), we have 
\(\LowerC{3} = \Omega(\frac{\sqrt{K}}{\exp(\bar{x}) \bar{x}})\). 
\end{enumerate}
\end{theorem}


\begin{remark}
    \Cref{th:lower_bound_asymp} establishes an asymptotic lower bound that also highlights the dependence on the dimensions. Note that when the dimensions are relatively large, the asymptotic lower bound is of order $(\DimFeature + \sqrt{\DimRank}) \sqrt{T} $. Compared with the upper bound in \Cref{th:regret_bound}, we establish that our algorithm is optimal with respect to $\DimFeature$, and near optimal with respect to the time horizon $T$ (up to $\sqrt{\log{T}}$). 
    In typical regimes this remaining gap is small: for example $\sqrt{\log{T}} \approx 4$ when $T = 10^7$, and $\sqrt{\DimRank} \le 10$ when $\DimRank \le 100$.
    As a result, further tightening the logarithmic gap may have limited practical impact, although closing it remains an interesting technical question.
\end{remark}

\begin{remark}
    We consider algorithms that choose $K$ products in this lower bound, in alignment with \Cref{assump:K}. However, our proof only requires the capacity of assortments to be  smaller or equal to $K$; therefore it also applies to settings considered in the literature \citep{agrawal2019mnl,chen2020dynamic,oh2021multinomial,chen2022nearly}.
\end{remark}

\begin{remark}
    Note that when $N \ge 18 K$, either the conditions in \cref{enum:thm3-1} or \cref{enum:thm3-2} hold. 
    Compared to prior lower-bound results for contextual MNL bandits, our condition is much more relaxed and reasonable.
    For instance, \cite{chen2020dynamic} require $N \ge 2^\DimFeature K$ and the feature dimension $\DimFeature$ to be divisible by 4.
    Under such conditions, 
    the conditions in \cref{enum:thm3-2} hold, and the lower bound becomes $\LowerC{2} \DimFeature$. 
    In particular, if we let $\plow =\phigh=p$, then the joint assortment and pricing problem becomes the assortment problem, and our asymptotic lower bound in \cref{enum:thm3-2} reduces to $\Omega(\DimFeature\sqrt{T\Lambda}/K)$, matching their lower bound. 
    While there are efforts to sharpen the dependence on $K$ (e.g., \cite{lee2024nearly}), they focus on restricted classes of algorithms that select the same product for $K$ times in each episode.
\end{remark}

\section{Simulations Studies}
\label{sec:numerical_simulation}
We present simulation experiments to evaluate $\PMNL$ in multiple settings, 
especially when customer arrival rates change with the assortment-pricing precision
to highlight the importance of modeling decision-dependent arrivals.
In \Cref{Subsec:sim-price}, we compare the performance of \PMNL~for dynamic pricing with varying assortments with an algorithm proposed by \cite{ferreira2023demand}, which we denote by $\FM$.
\Cref{Subsec:sim-assortment} compares the performance of \PMNL~on the full joint dynamic assortment and pricing with a naive baseline.
We defer additional results to Appendix \ref{app:sim}.

\subsection{Simulation Experiment I: Dynamic Pricing with Varying Assortment 
}
\label{Subsec:sim-price}

To demonstration the importance of modeling customer arrival rates and their dependence on prices and assortments (via product features), we compare $\PMNL$ with a benchmark that considers an unknown constant arrival rate for each period and adjusts prices at the beginning of each period. As discussed in \Cref{sec:intro}, to our knowledge, $\FM$ is the only existing method designed for period-level dynamic pricing.


We follow a simulation setup similar to \cite{ferreira2023demand}. 
Since $\FM$ does not optimize assortments, we benchmark the performance of dynamic pricing under assortments that vary by period but are given at the beginning of each period.
Specifically, we consider a retailer selling $N=K=5$ products over $T=1,000$ periods with feasible prices $\mathcal{P}=[\plow, \phigh]^N=[10, 30]^{N}$.
In each period $t$, product features $\Feature_t$ are drawn independently from a uniform distribution with support \([1,2]\), i.e., each feature $z_{jt,d} \overset{i.i.d.}{\sim} \emph{Uniform}(1,2)$ for $d \in [\DimFeature]$ and $j \in [N]$, where we set $\DimFeature=3$. 
Customer arrivals $n_t$ follow the Poisson arrival model \eqref{eq:pois},
with base arrival rate $\BaseArrival = 20$ and unit arrival rate
\begin{equation}
    \lambda_t := \exp\left(- \sum_{j\in S_t} \alpha \log p_{jt} + \beta \sum_{j \in  S_t} \sum_{d = 1}^{\DimFeature}\log(a z_{jt,d} + b) \right)
    \end{equation}
where we set $\btheta^* = (\alpha, \beta) = (0.2, 0.2)$, $a = 30$, and $b = -15$.
As discussed in \Cref{rmk:lambda}, the first term captures the price sensitivity, with lower prices attract more customers. The second term captures the effect of product features---treating features as desirable attributes, higher feature levels lead to higher arrivals. 
Each customer's purchase decision $\Choiceti$ is drawn independently according to the MNL model \eqref{eq:mnl} with each element of the preference parameter $v^*_{d} \overset{i.i.d.}{\sim} \emph{Uniform}(0,1)$ for $d \in [\DimFeature]$. 

We run 100 Monte Carlo simulations and compare the cumulative regret $\Regret{t}{\pi}$ of $\FM$ against our $\PMNL$ algorithm with $T_0=10$, i.e., the first 10 periods are for Stage 1 of $\PMNL$ and the learning stage for $\FM$.
\Cref{fig:cum-regret-no-assort-main} shows the cumulative regrets of $\PMNL$ and $\FM$. 
The regret of $\PMNL$ converges whereas that of $\FM$ grows linearly,
showing the importance of accounting for the dependency of the arrival rate on the assortment-pricing and corroborating our claims that ignoring this dependency can lead to suboptimal decision-making.

We further examine the pricing decisions to better understand the behavior of the algorithms. 
\Cref{fig:price-no-assort} shows the pricing decisions of the five products under $\FM$ and $\PMNL$.
Both algorithms explore a wide range of prices during the first stage ($t\in[T_0]$). 
In the second stage ($t>T_0$), $\FM$ mostly prices the products around the upper boundary $\phigh=30$, due to the fact that $\FM$ assumes the arrival rate does not change with the price and therefore maximizes per-customer revenue by pricing high.
Under this setting, however, arrivals decrease with higher prices, so persistently setting prices too high reduces cumulative revenue. 
By contrast, \PMNL~prices low, 
attracting more customers while still maintaining sufficient revenue to maximize the cumulative revenue.

We further compare the two algorithms under two additional scenarios. Keep all other settings unchanged, the only difference is the specification of the arrival rate: (i) the arrival rate $\lambda_t$ depends only on prices, i.e., $\beta=0$, and (ii) the arrival rate is constant, i.e., $\lambda_t = 1$. 
In the former case, \Cref{fig:cum-regret-no-assort-price-only} shows that \PMNL~still outperforms \FM.
With constant arrival rate, \Cref{fig:cum-regret-no-assort} shows that \PMNL~performs comparably to $\FM$, suggesting that \PMNL~is robust to the case when the arrival rate is exogenous. 
As shown in \cite{ferreira2023demand}, $\FM$ itself outperforms the M3P algorithm proposed by \cite{javanmard2019dynamic}, which alternates the learning and earning phases in an episodic manner.
Together, these results show that \PMNL~is competitive for dynamic pricing, and outperforms various existing algorithms when arrival rates depend on assortment and pricing.

\begin{figure}[t]
    \centering
    \resizebox{1\textwidth}{!}{%
    \begin{minipage}{\textwidth}
            \centering
    \begin{subfigure}[t]{0.33\textwidth}
        \centering
        \includegraphics[width=\linewidth]{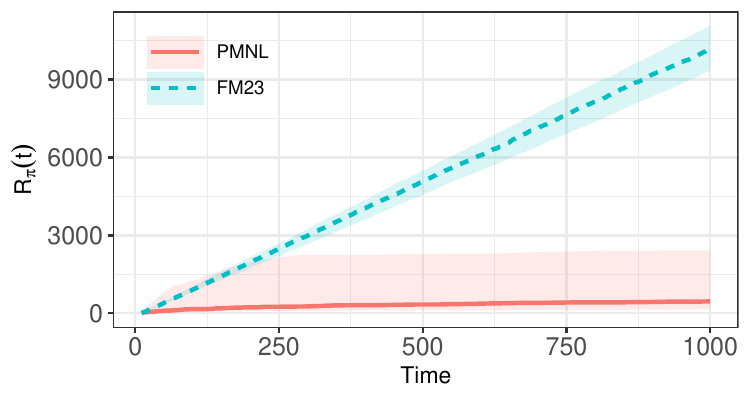}
        \caption{\small Assortment-pricing dependent.}
        \label{fig:cum-regret-no-assort-main}
    \end{subfigure}\hfill
    \begin{subfigure}[t]{0.33\textwidth}
        \centering
        \includegraphics[width=\linewidth]{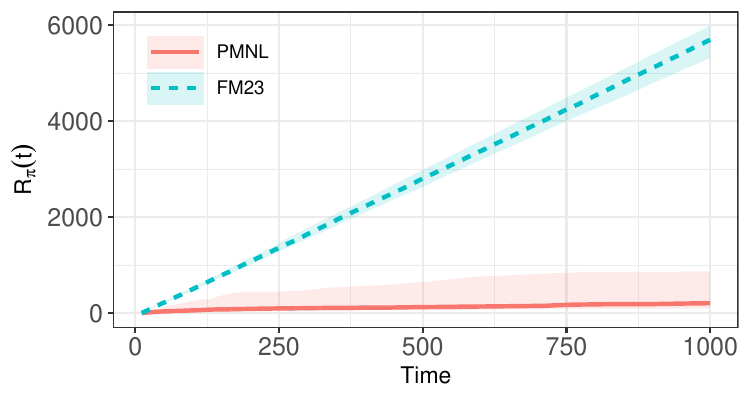}
        \caption{\small Price dependent.}
        \label{fig:cum-regret-no-assort-price-only}
    \end{subfigure}\hfill
    \begin{subfigure}[t]{0.33\textwidth}
        \centering
        \includegraphics[width=\linewidth]{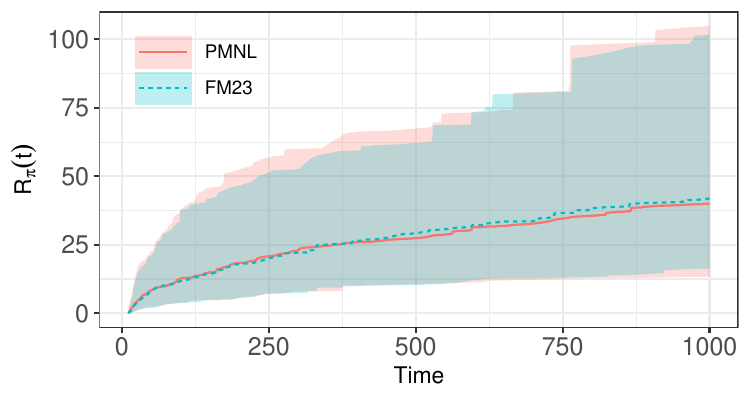}
        \caption{\small Constant arrival rate.}
        \label{fig:cum-regret-no-assort}
    \end{subfigure}
     \end{minipage}%
    }
    \caption{Comparison of cumulative regret of \PMNL\ and $\FM$ \citep{ferreira2023demand} across three settings. In each panel, the solid line shows the sample average cumulative regret over 100 simulations for \PMNL, and the dashed line corresponds to $\FM$. Shaded regions indicate the 10th and 90th percentiles.}
    \label{fig:cum-regret-three-panels}
    \vspace{-.35in}
\end{figure}

\begin{figure}[t]
    \centering
    \includegraphics[width=\textwidth]{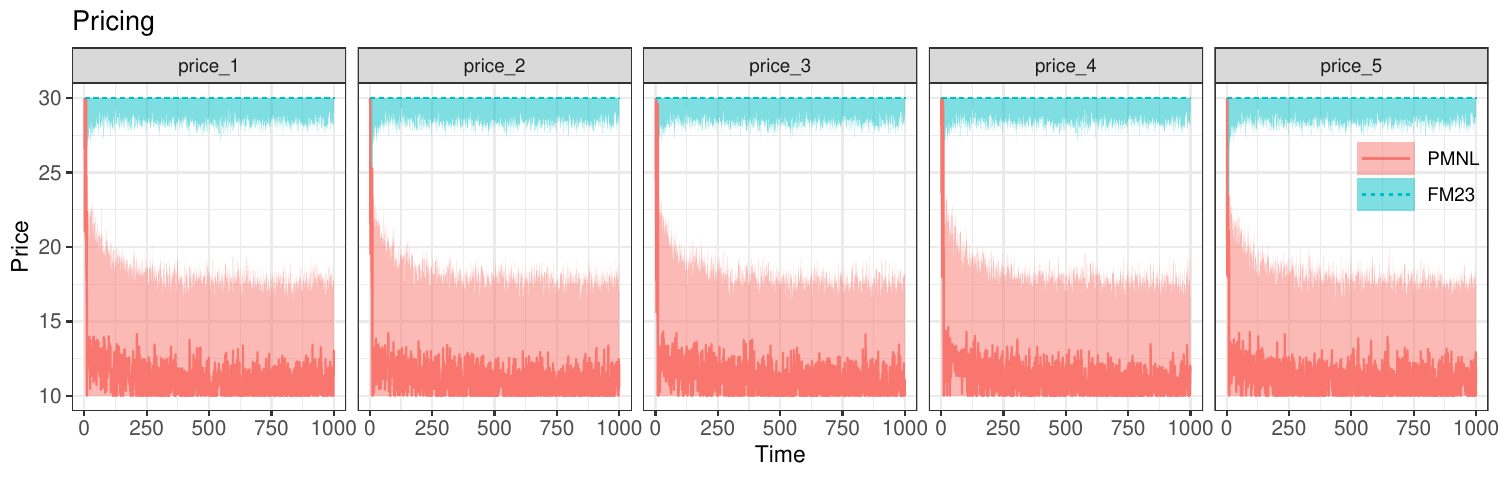}
    
    \caption{Pricing decisions for the products. The solid line shows the median price under \PMNL~across 100 simulations, and the dashed line corresponds to $\FM$. Shaded regions indicate the 10th and 90th percentiles.}
    \label{fig:price-no-assort}
    \vspace{-.25in}
\end{figure}

\subsection{Simulation Experiment II: Joint Dynamic Assortment-Pricing}
\label{Subsec:sim-assortment}

In this section, we consider the joint dynamic assortment and pricing problem. 
Since there exists no competitive algorithm in the literature for our setting, 
we compare \PMNL~with a naive UCB algorithm that assumes a fixed arrival rate and updates the assortment and prices at the beginning of each period using past observations.

We adopt a setup similar to \Cref{Subsec:sim-price} with the following changes.
We set $N=5$ products and an assortment size $K=4$. 
Each product has $\DimFeature=5$ features.
For the Poisson arrival model, we set $\btheta^* = (\alpha, \beta) = (0.1, 0.1)$ and $\BaseArrival = 100$. All other configurations remain unchanged.

\Cref{fig:cum-regret-assort-price} reports cumulative regrets of $\PMNL$ and UCB. 
Similarly, the regret of \PMNL~converges whereas that of UCB grows linearly.
Once again, these results highlight that it is crucial to account for the decision-dependent arrivals: \PMNL~effectively learns both the Poisson arrival and the MNL choice models and can provide better joint assortment-pricing decisions that naive algorithms that assume the arrival rate is fixed.

\begin{figure}[htbp]
\vspace{-.1in}
    \centering 
    
    \includegraphics[width=0.5\textwidth]{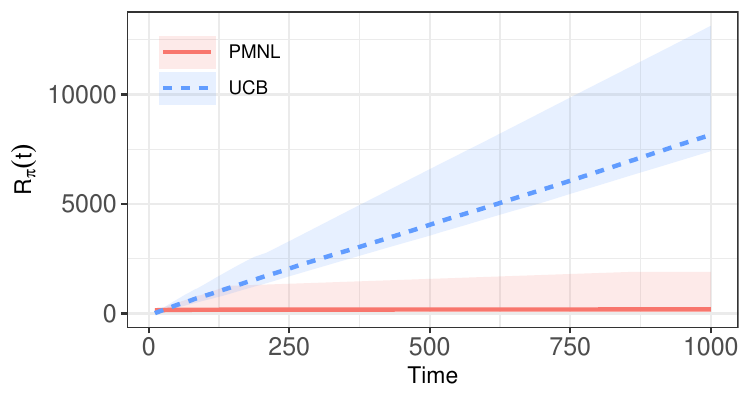}
    
    \caption{Comparison of cumulative regret of \PMNL~and a naive UCB algorithm. The solid line shows the sample average of the cumulative regret across 100 simulations for \PMNL, while the dashed line corresponds to UCB. Shaded regions indicate the 10th and 90th percentiles.}
    \label{fig:cum-regret-assort-price}
    \vspace{-.45in}
\end{figure}

\section{Conclusion}\label{sec:conclusion}

This paper studies dynamic joint assortment and pricing problem when firms update decisions at regular accounting or operating intervals with the goal of maximizing the cumulative per-period revenue over a time horizon $T$. 
In this setting, the reward depends jointly on how many customers arrive during the period and what they purchase conditional on arrival. To capture these two channels, we propose a Poisson–multinomial logit (Poisson–MNL) model that couples a Poisson arrival model with a contextual MNL choice model. 
The key is to allow the arrival rate to depend on the assortment–pricing decision through a rich set of basis functions on the offered assortment and prices, while the choice model leverages (potentially time-varying) product features to enable attribute-level learning and generalization across items. This framework roughly nests the classical MNL models that assume the arrival is fixed as a special case.

Building on this model, we develop $\PMNL$, an efficient online policy leveraging the upper confidence bound (UCB) strategy that jointly learns arrival and choice parameters and selects assortments and prices to maximize cumulative expected reward. We establish near-minimax optimal regret guarantees: an $O(\sqrt{T\log T})$ upper bound with a matching $\Omega(\sqrt{T})$ lower bound (up to $\log{T}$). Simulations show that accounting for decision-dependent arrivals substantially outperforms benchmarks that assume constant arrival rates when arrivals depend on the assortment and pricing. 
An important direction for future work is to close the remaining logarithmic gap between the upper and lower bounds. 
Another direction is to extend the arrival model to capture temporal effects, where arrivals in a given period may depend not only on the current assortment and prices but also on past decisions.

\clearpage
\bibliographystyle{informs2014} 
\bibliography{biblio}

@article{ma2018dynamic,
  title={Dynamic pricing (and assortment) under a static calendar},
  author={Ma, Will and Simchi-Levi, David and Zhao, Jinglong},
  journal={arXiv preprint arXiv:1811.01077},
  year={2018}
}

@article{brown2023competition,
  title={Competition and consumer behavior in online marketplaces},
  author={Brown, Zachary and others},
  journal={Journal of Marketing Research},
  year={2023},
  publisher={SAGE Publications}
}

@article{kahn1995consumer,
  title={Consumer variety-seeking among goods and services: An integrative review},
  author={Kahn, Barbara E},
  journal={Journal of retailing and consumer services},
  volume={2},
  number={3},
  pages={139--148},
  year={1995},
  publisher={Elsevier}
}

@book{lattimore2020bandit,
  title={Bandit algorithms},
  author={Lattimore, Tor and Szepesv{\'a}ri, Csaba},
  year={2020},
  publisher={Cambridge University Press}
}

@article{tropp2012user,
  title={User-friendly tail bounds for sums of random matrices},
  author={Tropp, Joel A},
  journal={Foundations of Computational Mathematics},
  volume={12},
  number={4},
  pages={389--434},
  year={2012},
  publisher={Springer}
}

@article{chen2021dynamic,
  title={Dynamic assortment planning under nested logit models},
  author={Chen, Xi and Shi, Chao and Wang, Yining and Zhou, Yuan},
  journal={Production and Operations Management},
  volume={30},
  number={1},
  pages={85--102},
  year={2021},
  publisher={Wiley Online Library}
}

@article{den2014simultaneously,
  title={Simultaneously learning and optimizing using controlled variance pricing},
  author={den Boer, Arnoud V and Zwart, Bert},
  journal={Management Science},
  volume={60},
  number={3},
  pages={770--783},
  year={2014},
  publisher={Informs}
}

@article{keskin2014dynamic,
  title={Dynamic pricing with an unknown demand model: Asymptotically optimal semi-myopic policies},
  author={Keskin, N Bora and Zeevi, Assaf},
  journal={Operations research},
  volume={62},
  number={5},
  pages={1142--1167},
  year={2014},
  publisher={INFORMS}
}

@article{ban2021personalized,
  title={Personalized dynamic pricing with machine learning: High-dimensional features and heterogeneous elasticity},
  author={Ban, Gah-Yi and Keskin, N Bora},
  journal={Management Science},
  volume={67},
  number={9},
  pages={5549--5568},
  year={2021},
  publisher={INFORMS}
}

@article{chen2021nonparametric,
  title={Nonparametric pricing analytics with customer covariates},
  author={Chen, Ningyuan and Gallego, Guillermo},
  journal={Operations Research},
  volume={69},
  number={3},
  pages={974--984},
  year={2021},
  publisher={INFORMS}
}

@article{qiang2016dynamic,
  title={Dynamic pricing with demand covariates},
  author={Qiang, Sheng and Bayati, Mohsen},
  journal={arXiv preprint arXiv:1604.07463},
  year={2016}
}

@article{agrawal2019mnl,
  title={{MNL}-bandit: A dynamic learning approach to assortment selection},
  author={Agrawal, Shipra and Avadhanula, Vashist and Goyal, Vineet and Zeevi, Assaf},
  journal={Operations Research},
  volume={67},
  number={5},
  pages={1453--1485},
  year={2019},
  publisher={INFORMS}
}

@article{chen2017note,
  title={A note on a tight lower bound for mnl-bandit assortment selection models},
  author={Chen, Xi and Wang, Yining},
  journal={arXiv preprint arXiv:1709.06109},
  year={2017}
}

@article{miao2021dynamic,
  title={Dynamic joint assortment and pricing optimization with demand learning},
  author={Miao, Sentao and Chao, Xiuli},
  journal={Manufacturing \& Service Operations Management},
  volume={23},
  number={2},
  pages={525--545},
  year={2021},
  publisher={INFORMS}
}

@book{wainwright2019high,
  title={High-dimensional statistics: A non-asymptotic viewpoint},
  author={Wainwright, Martin J},
  volume={48},
  year={2019},
  publisher={Cambridge University Press}
}

@article{robbins1952some,
  title={Some aspects of the sequential design of experiments},
  author={Robbins, Herbert},
  journal={Bulletin of the American Mathematical Society},
  volume={58},
  number={5},
  pages={527--535},
  year={1952},
  publisher={American Mathematical Society}
}

@inproceedings{abbasi2011improved,
  title={Improved algorithms for linear stochastic bandits},
  author={Abbasi-Yadkori, Yasin and P{\'a}l, D{\'a}vid and Szepesv{\'a}ri, Csaba},
  booktitle={Advances in Neural Information Processing Systems},
  volume={24},
  year={2011}
}

@article{araman2009dynamic,
  title={Dynamic pricing for nonperishable products with demand learning},
  author={Araman, Victor F and Caldentey, Ren{\'e}},
  journal={Operations Research},
  volume={57},
  number={5},
  pages={1169--1188},
  year={2009},
  publisher={INFORMS}
}

@article{besbes2009dynamic,
  title={Dynamic pricing without knowing the demand function: Risk bounds and near-optimal algorithms},
  author={Besbes, Omar and Zeevi, Assaf},
  journal={Operations Research},
  volume={57},
  number={6},
  pages={1407--1420},
  year={2009},
  publisher={INFORMS}
}

@article{broder2012dynamic,
  title={Dynamic pricing under a general parametric choice model},
  author={Broder, Josef and Rusmevichientong, Paat},
  journal={Operations Research},
  volume={60},
  number={4},
  pages={965--980},
  year={2012},
  publisher={INFORMS}
}

@article{javanmard2019dynamic,
  title={Dynamic pricing in high-dimensions},
  author={Javanmard, Adel and Nazerzadeh, Hamid},
  journal={The Journal of Machine Learning Research},
  volume={20},
  number={1},
  pages={315--363},
  year={2019},
  publisher={JMLR. org}
}

@article{cohen2020feature,
  title={Feature-based dynamic pricing},
  author={Cohen, Maxime C and Lobel, Ilan and Paes Leme, Renato},
  journal={Management Science},
  volume={66},
  number={11},
  pages={4921--4943},
  year={2020},
  publisher={INFORMS}
}

@article{chen2022statistical,
  title={A statistical learning approach to personalization in revenue management},
  author={Chen, Xi and Owen, Zachary and Pixton, Clark and Simchi-Levi, David},
  journal={Management Science},
  volume={68},
  number={3},
  pages={1923--1937},
  year={2022},
  publisher={INFORMS}
}

@article{miao2022online,
  title={Online Personalized Assortment Optimization with High-Dimensional Customer Contextual Data},
  author={Miao, Sentao and Chao, Xiuli},
  journal={Manufacturing \& Service Operations Management},
  volume={24},
  number={5},
  pages={2741--2760},
  year={2022},
  publisher={INFORMS}
}

@article{rusmevichientong2010dynamic,
  title={Dynamic assortment optimization with a multinomial logit choice model and capacity constraint},
  author={Rusmevichientong, Paat and Shen, Zuo-Jun Max and Shmoys, David B},
  journal={Operations Research},
  volume={58},
  number={6},
  pages={1666--1680},
  year={2010},
  publisher={INFORMS}
}

@article{cheung2017thompson,
  title={Thompson sampling for online personalized assortment optimization problems with multinomial logit choice models},
  author={Cheung, Wang Chi and Simchi-Levi, David},
  journal={Available at SSRN 3075658},
  year={2017}
}

@article{chen2020dynamic,
  title={Dynamic assortment optimization with changing contextual information},
  author={Chen, Xi and Wang, Yining and Zhou, Yuan},
  journal={The Journal of Machine Learning Research},
  volume={21},
  number={1},
  pages={8918--8961},
  year={2020},
  publisher={JMLRORG}
}

@inproceedings{kleinberg2003value,
  title={The value of knowing a demand curve: Bounds on regret for online posted-price auctions},
  author={Kleinberg, Robert and Leighton, Tom},
  booktitle={44th Annual IEEE Symposium on Foundations of Computer Science, 2003. Proceedings.},
  pages={594--605},
  year={2003},
  organization={IEEE}
}

@article{bastani2022meta,
  title={Meta dynamic pricing: Transfer learning across experiments},
  author={Bastani, Hamsa and Simchi-Levi, David and Zhu, Ruihao},
  journal={Management Science},
  volume={68},
  number={3},
  pages={1865--1881},
  year={2022},
  publisher={INFORMS}
}

@article{akccay2010joint,
  title={Joint dynamic pricing of multiple perishable products under consumer choice},
  author={Ak{\c{c}}ay, Yal{\c{c}}{\i}n and Natarajan, Harihara Prasad and Xu, Susan H},
  journal={Management Science},
  volume={56},
  number={8},
  pages={1345--1361},
  year={2010},
  publisher={INFORMS}
}

@article{gallego2014multiproduct,
  title={Multiproduct price optimization and competition under the nested logit model with product-differentiated price sensitivities},
  author={Gallego, Guillermo and Wang, Ruxian},
  journal={Operations Research},
  volume={62},
  number={2},
  pages={450--461},
  year={2014},
  publisher={INFORMS}
}

@article{chen2022nearly,
  title={Nearly dimension-independent sparse linear bandit over small action spaces via best subset selection},
  author={Chen, Yi and Wang, Yining and Fang, Ethan X and Wang, Zhaoran and Li, Runze},
  journal={Journal of the American Statistical Association},
  pages={1--13},
  year={2022},
  publisher={Taylor \& Francis}
}

@article{miao2022context,
  title={Context-based dynamic pricing with online clustering},
  author={Miao, Sentao and Chen, Xi and Chao, Xiuli and Liu, Jiaxi and Zhang, Yidong},
  journal={Production and Operations Management},
  volume={31},
  number={9},
  pages={3559--3575},
  year={2022},
  publisher={Wiley Online Library}
}

@article{fan2022policy,
  title={Policy optimization using semiparametric models for dynamic pricing},
  author={Fan, Jianqing and Guo, Yongyi and Yu, Mengxin},
  journal={Journal of the American Statistical Association},
  pages={552--564},
  year={2024},
  publisher={Taylor \& Francis}
}

@inproceedings{oh2021multinomial,
    author={Oh, Min-hwan and Iyengar, Garud},
    title={Multinomial logit contextual bandits: {P}rovable optimality and practicality},
    booktitle={Proceedings of the AAAI conference on artificial intelligence},
  year={2021}
}

@article{de1999general,
  title={A general class of exponential inequalities for martingales and ratios},
  author={de la Pena, Victor H},
  journal={The Annals of Probability},
  volume={27},
  number={1},
  pages={537--564},
  year={1999},
  publisher={Institute of Mathematical Statistics}
}

@book{geer2000empirical,
  title={Empirical Processes in M-estimation},
  author={Sara van de Geer},
  year={2000},
  publisher={Cambridge University Press}
}

@book{poisson1837recherches,
  title={Recherches sur la probabilit{\'e} des jugements en mati{\`e}re criminelle et en mati{\`e}re civile: pr{\'e}c{\'e}d{\'e}es des r{\`e}gles g{\'e}n{\'e}rales du calcul des probabilit{\'e}s},
  author={Poisson, Sim{\'e}on-Denis},
  year={1837},
  publisher={Bachelier}
}

@book{kingman1992poisson,
  title={Poisson processes},
  author={Kingman, John Frank Charles},
  volume={3},
  year={1992},
  publisher={Clarendon Press}
}

@article{aparicio2023algorithmic,
  title={Algorithmic pricing and consumer sensitivity to price variability},
  author={Aparicio, Diego and Eckles, Dean and Kumar, Madhav},
  journal={Available at SSRN 4435831},
  year={2023}
}

@article{ferreira2023demand,
  title={Demand learning and pricing for varying assortments},
  author={Ferreira, Kris Johnson and Mower, Emily},
  journal={Manufacturing \& Service Operations Management},
  volume={25},
  number={4},
  pages={1227--1244},
  year={2023},
  publisher={INFORMS}
}

@article{ahle2022sharp,
  title={Sharp and simple bounds for the raw moments of the binomial and poisson distributions},
  author={Ahle, Thomas D},
  journal={Statistics and Probability Letters},
  volume={182},
  pages={109306},
  year={2022},
  publisher={Elsevier}
}

@article{abdallah2021demand,
  title={Demand estimation under the multinomial logit model from sales transaction data},
  author={Abdallah, Tarek and Vulcano, Gustavo},
  journal={Manufacturing \& Service Operations Management},
  volume={23},
  number={5},
  pages={1196--1216},
  year={2021},
  publisher={Informs}
}

@article{wang2021consumer,
  title={Consumer choice and market expansion: Modeling, optimization, and estimation},
  author={Wang, Ruxian},
  journal={Operations Research},
  volume={69},
  number={4},
  pages={1044--1056},
  year={2021},
  publisher={INFORMS}
}

@article{vulcano2012estimating,
  title={Estimating primary demand for substitutable products from sales transaction data},
  author={Vulcano, Gustavo and Van Ryzin, Garrett and Ratliff, Richard},
  journal={Operations Research},
  volume={60},
  number={2},
  pages={313--334},
  year={2012},
  publisher={INFORMS}
}

@inproceedings{brown1986fundamentals,
  title={Fundamentals of statistical exponential families: with applications in statistical decision theory},
  author={Brown, Lawrence D},
  year={1986},
  organization={Ims}
}

@book{winkelmann2008econometric,
  title={Econometric analysis of count data},
  author={Winkelmann, Rainer},
  year={2008},
  publisher={Springer-Verlag}
}

@article{lancaster1990economics,
  title={The economics of product variety: A survey},
  author={Lancaster, Kelvin},
  journal={Marketing science},
  volume={9},
  number={3},
  pages={189--206},
  year={1990},
  publisher={INFORMS}
}

@misc{Pol15,
    author = {Pollard, David},
    title = {MiniEmpirical},
    howpublished = {\url{http://www.stat.yale.edu/~pollard/Books/Mini/}},
    year = {2015},
    note = {Manuscript (accessed 02-23-2017)}
}

@article{oh2019thompson,
  title={Thompson sampling for multinomial logit contextual bandits},
  author={Oh, Min-hwan and Iyengar, Garud},
  journal={Advances in Neural Information Processing Systems},
  volume={32},
  year={2019}
}

@article{Choi1994,
  author       = {K. P. Choi},
  title        = {On the Medians of Gamma Distributions and an Equation of Ramanujan},
  journal      = {Proceedings of the American Mathematical Society},
  year         = {1994},
  volume       = {121},
  number       = {1},
  pages        = {245--251},
  month        = {May},
  doi          = {10.2307/2160389},
  url          = {https://www.jstor.org/stable/2160389},
  publisher    = {American Mathematical Society}
}

@article{lee2024nearly,
  title={Nearly minimax optimal regret for multinomial logistic bandit},
  author={Lee, Joongkyu and Oh, Min-hwan},
  journal={Advances in Neural Information Processing Systems},
  volume={37},
  pages={109003--109065},
  year={2024}
}

@inproceedings{javanmard2020multi,
  title={Multi-product dynamic pricing in high-dimensions with heterogeneous price sensitivity},
  author={Javanmard, Adel and Nazerzadeh, Hamid and Shao, Simeng},
  booktitle={2020 IEEE International Symposium on Information Theory (ISIT)},
  pages={2652--2657},
  year={2020},
  organization={IEEE}
}

@article{perivier2022dynamic,
  title={Dynamic pricing and assortment under a contextual mnl demand},
  author={Perivier, Noemie and Goyal, Vineet},
  journal={Advances in Neural Information Processing Systems},
  volume={35},
  pages={3461--3474},
  year={2022}
}

@article{lee2025low,
  title={Low-rank online dynamic assortment with dual contextual information},
  author={Lee, Seong Jin and Sun, Will Wei and Liu, Yufeng},
  journal={Journal of the American Statistical Association},
  number={just-accepted},
  pages={1--22},
  year={2025},
  publisher={Taylor \& Francis}
}

\ECSwitch
\SingleSpacedXI
\addtolength{\abovedisplayskip}{-1ex}
\addtolength{\belowdisplayskip}{-1ex}

\vspace{-2em}
\ECHead{Electronic companions}
\small

\setcounter{lemma}{9}

The appendix collects supplementary materials that support the main text and proofs. We first report additional simulation results in \Cref{app:sim}. We then provide a reduction for the feature vector $x(S,\bp)$ in the arrival model in \Cref{app:x.independent}. Next, \Cref{appendix:relation} clarifies the relationship between \Cref{assump:combined,assumption::y_exist} by showing how the stronger condition implies the weaker one and how to construct a sequence required in the latter. Finally, we introduce additional notations and technical preliminaries used throughout the proofs in \Cref{app:notations} and provide the detailed proofs in the rest of the Appendix, which is organized as follows:

\begin{itemize}
  \item Section~\ref{appendix:proof_of_theta_initial_error} proves  \Cref{lm:theta_initial_error} and its related \Cref{lm:ineq,lm:g_1_convergence_rate,lm:poi_k_bound,lm:V_bound,lm:dif_g}.
  Specifically, \Cref{lm:theta_initial_error} depends on the proofs of \Cref{lm:ineq,lm:g_1_convergence_rate}, with the latter relying on \Cref{lm:poi_k_bound,lm:V_bound,lm:dif_g}. 
  
  \item Section~\ref{appendix::proof_of_theta_error} proves \Cref{lm:theta_error}, building on  \Cref{lm:ineq,lm:g_1_convergence_rate}.

  \item Section~\ref{appendix:proof_of_sum_lambda_error} proves \Cref{lemma:revenue_estimate}, which depends on 
  \Cref{lm:theta_error}.

  \item Section~\ref{appendix::proof_of_revenue_estimate} proves \Cref{lm:sum_lambda_error}.

  \item Section~\ref{appendix::proof_of_bv_initial_error} proves \Cref{lm:bv_initial_error} and its auxiliary  \Cref{lm:g_2_convergence_rate,lm:prob_error,lm:init_t_bound_g_bar_to_g,lm:g_k_bound,lm:v_2_bound,lm:expec_g_bar}. Here, \Cref{lm:bv_initial_error} is established via \Cref{lm:g_2_convergence_rate}; \Cref{lm:init_t_bound_g_bar_to_g} relies on \Cref{lm:prob_error,lm:expec_g_bar}, and \Cref{lm:g_k_bound,lm:v_2_bound} both depend on \Cref{lm:prob_error}.

  \item \Cref{appendix::proof_of_bv_error} proves \Cref{lm:bv_error} and \Cref{lm:I_v_close_to_I_v_star}.
  \Cref{lm:bv_error} depends on \Cref{lm:ineq,lm:g_2_convergence_rate,lm:I_v_close_to_I_v_star}.

  \item Section~\ref{appendix::proof_of_revenue_estimate_2} proves \Cref{lemma:revenue_estimate_2} and \Cref{lm:bv_estimate_real_value} where the former depends on the latter and \Cref{lm:bv_error,lm:I_v_close_to_I_v_star}.

  \item Section~\ref{appendix::proof_of_sum_bv_error} proves \Cref{lm:sum_bv_error}.

  \item Section~\ref{appendix::lower_bound} is dedicated to the proof of the lower bound \Cref{th:lower_bound,th:lower_bound_asymp} as well as theirs dependent on \Cref{lm:regret_lower_bound_tilde,lm:bound_kl_div,lm:poi_bound,lem:catalog-feasible-phi,lm:regret_lower_bound_tilde2,lm:bound_kl_div2,lm:lambda_lower_bound_tilde,lm:bound_kl_div_poi}.

\end{itemize}

\begin{figure}[htbp]
  \centering
  \resizebox{0.8\linewidth}{!}{
\definecolor{finalcolor}{RGB}{141,182,205}    
\definecolor{keylemma}{RGB}{255,182,193}        
\definecolor{closely}{RGB}{255,218,185}         
\definecolor{fundamental}{RGB}{152,251,152}     
\definecolor{technical}{RGB}{176,196,222}        

\tikzset{myarrow/.style={<-, line width=1.2pt, >=stealth}}

\tikzset{mybox/.style={draw, rectangle, align=center, minimum width=2cm, minimum height=0.8cm}}

\begin{tikzpicture}[node distance=1.cm and 1.2cm]


\def\xone{-1}
\def\xtwo{2}
\def\xthree{6}
\def\xfour{10}    
\def\xfive{14}   

\node[mybox, fill=finalcolor, text=black] (Th1) at (\xone,0) {\Cref{th:regret_bound}};

\node[mybox, fill=keylemma, text=black] (L2) at (\xtwo,3) {\Cref{lm:theta_initial_error}};
\node[mybox, fill=keylemma, text=black] (L4) at (\xtwo,2) {\Cref{lm:theta_error}};
\node[mybox, fill=keylemma, text=black] (L6) at (\xtwo,0.9) {\Cref{lemma:revenue_estimate}};
\node[mybox, fill=keylemma, text=black] (L3) at (\xtwo,-1) {\Cref{lm:bv_initial_error}};
\node[mybox, fill=keylemma, text=black] (L5) at (\xtwo,-2) {\Cref{lm:bv_error} };
\node[mybox, fill=keylemma, text=black] (L7) at (\xtwo,-3.1) {\Cref{lemma:revenue_estimate_2}};
\node[mybox, fill=keylemma, text=black] (L8) at (\xtwo,-5) {\Cref{lm:sum_lambda_error}};
\node[mybox, fill=keylemma, text=black] (L9) at (\xtwo,-6) {\Cref{lm:sum_bv_error}};

\node[mybox, fill=fundamental, text=black] (EC2) at (\xthree,1) {\Cref{lm:ineq}};
\node[mybox, fill=fundamental, text=black] (EC8) at (\xthree,-1) {\Cref{lm:g_2_convergence_rate}};
\node[mybox, fill=closely, text=black] (EC15) at (\xthree,-2.5) {\Cref{lm:I_v_close_to_I_v_star}}; 
\node[mybox, fill=closely, text=black]  (EC14) at (\xthree,-3.5) {\Cref{lm:bv_estimate_real_value}}; 

\node[mybox, fill=technical, text=black] (EC3) at (\xfour,3) {\Cref{lm:g_1_convergence_rate}};
\node[mybox, fill=technical, text=black] (EC10) at (\xfour,-0.5) {\Cref{lm:init_t_bound_g_bar_to_g}};  
\node[mybox, fill=technical, text=black] (EC11) at (\xfour,-1.5) {\Cref{lm:g_k_bound}};  
\node[mybox, fill=technical, text=black] (EC12) at (\xfour,-3.5) {\Cref{lm:v_2_bound}};  

\node[mybox, fill=technical, text=black] (EC4) at (\xfive,3) {\Cref{lm:poi_k_bound}};
\node[mybox, fill=technical, text=black] (EC5) at (\xfive,2) {\Cref{lm:V_bound}};
\node[mybox, fill=technical, text=black] (EC6) at (\xfive,1) {\Cref{lm:dif_g}};
\node[mybox, fill=technical, text=black] (EC13) at (\xfive,-0.5) {\Cref{lm:expec_g_bar}};
\node[mybox, fill=technical, text=black] (EC9) at (\xfive,-2.5) {\Cref{lm:prob_error}};


\draw[myarrow] (Th1.east) -- (L2.west);
\draw[myarrow] (Th1.east) -- (L4.west);
\draw[myarrow] (Th1.east) -- (L6.west);
\draw[myarrow] (Th1.east) -- (L3.west);
\draw[myarrow] (Th1.east) -- (L5.west);
\draw[myarrow] (Th1.east) -- (L8.west);
\draw[myarrow] (Th1.east) -- (L9.west);

\draw[myarrow] (L2.east) -- (EC2.west);
\draw[myarrow] (L2.east) -- (EC3.west);

\draw[myarrow] (EC3.east) -- (EC4.west);
\draw[myarrow] (EC3.east) -- (EC5.west);
\draw[myarrow] (EC3.east) -- (EC6.west);

\draw[myarrow] (L3.east) -- (EC8.west);

\draw[myarrow] (EC8.east) -- (EC10.west);
\draw[myarrow] (EC8.east) -- (EC11.west);
\draw[myarrow] (EC8.east) -- (EC12.west);
\draw[myarrow] (EC8.east) -- (7.8,-2.5) -- (EC9.west);

\draw[myarrow] (EC10.east) -- (EC13.west);

\draw[myarrow] (EC11.east) -- (EC9.north);
\draw[myarrow] (EC12.east) -- (EC9.south);

\draw[myarrow] (L4.east) -- (EC2.west);
\draw[myarrow] (L4.east) -- (EC3.west);

\draw[myarrow] (L5.east) -- (EC2.west);
\draw[myarrow] (L5.east) -- (EC8.west);
\draw[myarrow] (EC14.west) -- (L5.east);



\draw[myarrow] (L6.north) -- (L4.south); 
\draw[myarrow] (L7.north) -- (L5.south);
\draw[myarrow] (L7.east) -- (EC14.west);


\draw[myarrow] (L5.east) -- (EC15.west);
\draw[myarrow] (L7.east) -- (EC15.west);
\draw[myarrow] (EC14.north) -- (EC15.south);

\node[anchor=west, inner sep=5pt] at ([xshift=8mm]current bounding box.east) {
  \begin{tabular}{ll}
    \textcolor{finalcolor}{\rule{1cm}{0.3cm}} & Theorem (Final Conclusion) \\[4mm]
    \textcolor{keylemma}{\rule{1cm}{0.3cm}} & Lemma (Key Steps) \\[4mm]
    \textcolor{closely}{\rule{1cm}{0.3cm}} & Useful conclusion \\[4mm]
    \textcolor{fundamental}{\rule{1cm}{0.3cm}} & Fundamental Lemma \\[4mm]
    \textcolor{technical}{\rule{1cm}{0.3cm}} & Technical Lemma \\[4mm]
  \end{tabular}
};

\end{tikzpicture}} 
  \caption{Proof Structure of \Cref{th:regret_bound}.}
  \label{fig:proof-diagram}
  \vspace{-2em}
\end{figure}

\section{Additional Simulation Results}
\label{app:sim}



\Cref{fig:error-param} shows that estimation error of the parameters for the simulation in \Cref{Subsec:sim-price}. 
Specifically, \Cref{fig:error-v} shows the estimation error of the MNL parameters $\|\hv - \bv^*\|_2$: $\PMNL$ converges faster than $\FM$; 
and \Cref{fig:error-alpha} shows the the estimation error of the arrival parameters $\|\htheta - \btheta^*\|_2$: $\PMNL$ learns the arrival model efficiently and the estimation error converges.

\begin{figure}[ht]
\begin{center}
    \begin{subfigure}[b]{0.45\textwidth}
        \centering
        \includegraphics[width=\textwidth, trim=0 0 0 30pt, clip]{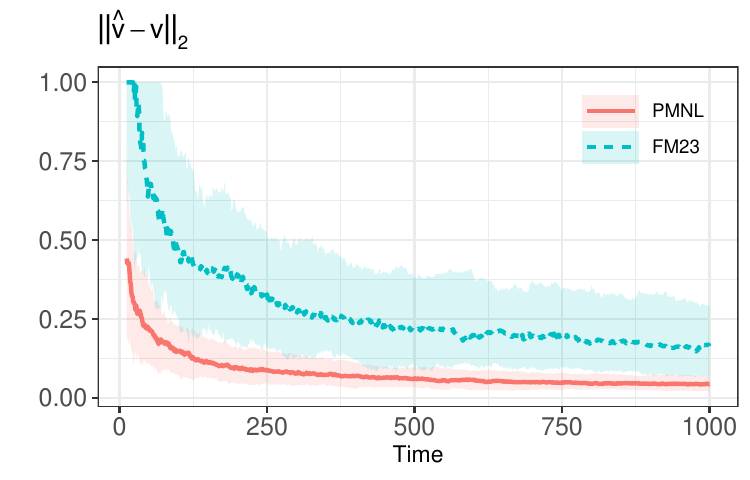}
        \caption{Estimation Error of $\hat\bv$: \(||\hat \bv - \bv||^2\).}
        \label{fig:error-v}
    \end{subfigure}
    \hspace{0.05\textwidth} 
    \begin{subfigure}[b]{0.45\textwidth}
        \centering
        \includegraphics[width=\textwidth, trim=0 0 0 30pt, clip]{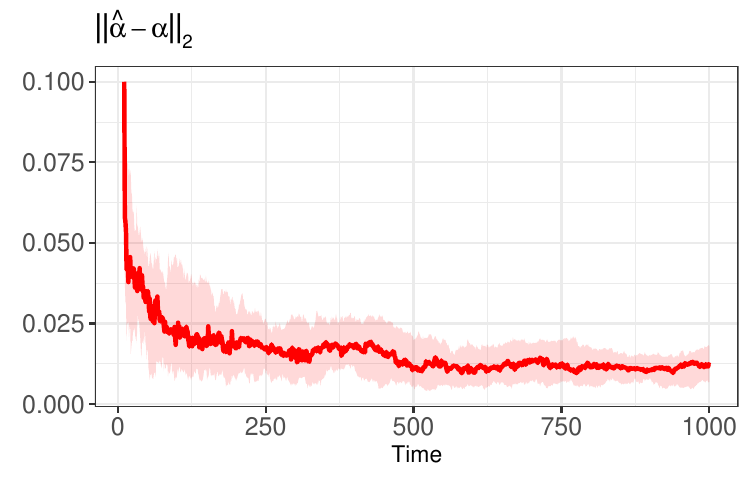}
        \caption{Estimation Error of $\hat\btheta$: \(||\hat\btheta - \btheta||^2 \).}
        \label{fig:error-alpha}
    \end{subfigure}
    \caption{Estimation Error of Unknown Parameters by $\PMNL$ and $\FM$. The left panel shows the estimation error for the customer preference for both algorithms, while the right panel shows the estimation error for the arrival parameters, which only appear in $\PMNL$.}
    \label{fig:error-param}
\end{center}
\end{figure}



\section{Reduction of \texorpdfstring{$\bx(S_t, \bp_t)$}{x(St, pt)}}
\label{app:x.independent}

This section shows that, without loss of generality, we can work with a reduced representation of the arrival feature vector \(\bx(S_t, \bp_t)\) and describe how to construct such a reduction by removing coordinates that are redundant over the feasible decision set.

Consider the set of all attainable feature vectors \(\{\bx(S, \bp) \mid S \in \mathscr{S}, \bp \in \mathscr{P}\}\). We can assume that the maximal linearly independent group within the set \(\{\bx(S, \bp) \mid S \in \mathscr{S}, \bp \in \mathscr{P}\}\) comprises \(\DimRank\) vectors. If the number of vectors is less than \(\DimRank\), then the matrix formed by this maximal linearly independent group is rank-deficient. 

This implies that some components of \(\bx\) in the maximal linearly independent group can be expressed in terms of other components. However, the maximal linearly independent group can represent all the vectors in \(\{\bx(S, \bp) \mid S \in \mathscr{S}, \bp \in \mathscr{P}\}\), which means that a certain component of all the vectors in \(\{\bx(S, \bp) \mid S \in \mathscr{S}, \bp \in \mathscr{P}\}\) can be expressed by other components. Consequently, we can eliminate this component. 

By repeating this process, we can deduce that the number of vectors in the maximal linearly independent groups is equal to the dimension of the variables. This condition ensures that \( \text{span}(\{\bx(S_t, \bp_t) \mid S \in \mathscr{S}, \bp \in \mathscr{P}\}) \) has rank \(\DimRank\), which significantly simplifies our analysis.

\section{Relationship Between \Cref{assump:combined} and Prior Work}
\label{appendix:relation}
\begin{lemma}\label{lemma:combine_assumptions}
Suppose that \Cref{assumption::y_exist} holds. There exist constants
\(\sigma_0=\tfrac12 K\bar{\sigma}_0\) and \(\sigma_1>0\) such that we can construct
assortments \(\{S_s\}_{s=1}^{t}\) and price vectors \(\{\bp_s\}_{s=1}^{t}\) for which,
for any
\[
t \ge t_0:=\max\left\{\frac{\log(\DimFeature T)}{\sigma_0(1-\log 2)},\,2\DimRank\right\},
\]
the following hold with probability at least \(1-T^{-1}\):
\begin{enumerate}
\item
\[
\sigma_{\min}\!\left(\sum_{s=1}^{t}\sum_{j\in S_s}\Feature_{js}\Feature_{js}^\top\right)
\ge \sigma_0 t.
\]
\item There exists a price sequence \(\{\bp_s\}_{s=1}^{t}\) such that
\[
\sigma_{\min}\!\left(\sum_{s=1}^{t}\bx(S_s,\bp_s)\bx^\top(S_s,\bp_s)\right)
\ge \sigma_1 t.
\]
\end{enumerate}
\end{lemma}

We prove each part of \Cref{lemma:combine_assumptions} below.

\textbf{Proof of Part 1 of \Cref{lemma:combine_assumptions}.}
By \Cref{assumption::y_exist}, \(\|\Feature_{js}\|_2\le 1\), hence
\(\sigma_{\max}(\Feature_{js}\Feature_{js}^\top)\le 1\).
Applying the matrix Chernoff bound \cite[Thm.~1.1]{tropp2012user}, for any
\(\delta\in[0,1)\),
\begin{equation}\label{eq:relationship_ineq1}
\mathbb{P}\!\left\{
\sigma_{\min}\!\left(\sum_{s=1}^{t}\sum_{j\in S_s}\Feature_{js}\Feature_{js}^\top\right)
\le (1-\delta)\,tK\bar{\sigma}_0
\right\}
\le
\DimFeature\left[\frac{e^{-\delta}}{(1-\delta)^{1-\delta}}\right]^{tK\bar{\sigma}_0}.
\end{equation}
Setting \(\delta=\tfrac12\) gives
\[
\mathbb{P}\!\left\{
\sigma_{\min}\!\left(\sum_{s=1}^{t}\sum_{j\in S_s}\Feature_{js}\Feature_{js}^\top\right)
\le \tfrac12\,tK\bar{\sigma}_0
\right\}
\le
\DimFeature\exp\!\left(-\tfrac12 tK\bar{\sigma}_0(1-\log 2)\right).
\]
Thus the right-hand side is at most \(T^{-1}\) provided
\(t \ge \frac{2\log(\DimFeature T)}{K\bar{\sigma}_0(1-\log 2)}
= \frac{\log(\DimFeature T)}{\sigma_0(1-\log 2)}\).
For all \(t\ge t_0\), this implies
\[
\sigma_{\min}\!\left(\sum_{s=1}^{t}\sum_{j\in S_s}\Feature_{js}\Feature_{js}^\top\right)
\ge \sigma_0 t
\quad\text{with probability at least }1-T^{-1}. \hfill \qed
\]

\textbf{Proof of Part 2 of \Cref{lemma:combine_assumptions}.}
Since \(\{\bx(S,\bp)\}\) spans \(\R^{\DimFeature}\), we can choose
\(\DimRank\) linearly independent vectors \(\{\bx(S_s,\bp_s)\}_{s=1}^{\DimRank}\).
Hence there exists \(\lambda_{\min}>0\) such that
\[
\sigma_{\min}\!\left(\sum_{s=1}^{\DimRank}\bx(S_s,\bp_s)\bx^\top(S_s,\bp_s)\right)
\ge \lambda_{\min}.
\]
For any \(t\ge 2\DimRank\), repeat this length-\(\DimRank\) block (and truncate the
last block if needed) to form \(\{(S_s,\bp_s)\}_{s=1}^t\). Since each term
\(\bx(S_s,\bp_s)\bx^\top(S_s,\bp_s)\) is positive semidefinite,
\[
\sigma_{\min}\!\left(\sum_{s=1}^{t}\bx(S_s,\bp_s)\bx^\top(S_s,\bp_s)\right)
\ge \Big\lfloor \frac{t}{\DimRank}\Big\rfloor \lambda_{\min}
\ge \left(\frac{t}{\DimRank}-1\right)\lambda_{\min}
\ge \frac{t\,\lambda_{\min}}{2\DimRank}.
\]
Let \(\sigma_1:=\lambda_{\min}/(2\DimRank)>0\). Then
\[
\sigma_{\min}\!\left(\sum_{s=1}^{t}\bx(S_s,\bp_s)\bx^\top(S_s,\bp_s)\right)
\ge \sigma_1 t,\qquad \forall\, t\ge 2\DimRank. \hfill \qed
\]

\section{Notations}
\label{app:notations}
This section provides a detailed proof of the results presented in the main body of the paper. To facilitate the proof, we begin by introducing several definitions.

For each period \(t\), we use Maximum Likelihood Estimation (MLE) to estimate the unknown parameters \((\btheta, \bv)\). The log-likelihood function for these parameters, up to period \(t\), can be expressed as:
\begin{equation}
\Lik{t}= \LikPoi{t} + \LikMNL{t}.
\end{equation}
Here, \(\LikPoi{t}\) and \(\LikMNL{t}\) are defined as in Section \ref{Subsec:likelihood}, where:
\begin{align*}
    \LikPoi{t} &= \sum_{s=1}^t \left\{ - \Lambda \lambda(S_{s}, \bp_s; \btheta)  + n_s \log \lambda(S_{s}, \bp_s;\btheta) + \left(n_s\log{\Lambda} - \log n_s!\right)\right\}, \\
    \LikMNL{t} &= \sum_{s=1}^t \sum_{i = 1}^{n_{s}} \sum_{j \in S_{s} \cup\{0\}}
\bm{1}\{C^{(i)}_{s} = j\} \log \ChoiceProb(j, S_{s}, \bp_{s},  \Feature_s; \bv).
\end{align*}
The term \(\Lik{t}\) represents the logarithm of the joint likelihood for the observed customer arrivals and purchase choices up to time \(t\). 
Its gradients with respect to the parameters \((\btheta, \bv)\) are given by:
\begin{align*}
\nabla_{\btheta} \LikPoi{t} &= -\sum_{s=1}^t \big(\Lambda \lambda(S_{s}, \bp_s;\btheta) - n_s\big)\xs, \\
\nabla_\bv \LikMNL{t} &= -\sum_{s=1}^t \sum_{i = 1}^{n_{s}} \sum_{j \in S_s}\left(\ChoiceProb(j, S_{s}, \bp_{s},  \Feature_{js}; \bv) - \bm{1}\{C^{(i)}_{s} = j\}\right) \Feature_{js}.
\end{align*}

We denote the ground truth of \(\btheta\) and \(\bv\) as \(\btheta^*\) and \(\bv^*\), respectively, and their corresponding estimates at the end of period \(t\), which maximize the log-likelihood, as \(\hat{\btheta}_t\) and \(\hat{\bv}_t\). 
The calculation of the first and second derivatives of the log-likelihood ratio with respect to the unknown parameters is cumbersome and tedious. To facilitate the proofs, we define several terms for \(\forall \btheta \in \Real^\DimRank\) and \(\forall \bv \in \Real^\DimFeature\) for these unknown parameters.

\begin{definition}\label{def_theta_app} 
For the Poisson arrival process, we define the following terms:
\begin{align*}
    \bar{G}_{t}^{\Poi}(\btheta) :=& \sum_{s = 1}^{t} \bar{g}_{s}^{\Poi}(\btheta),  \mbox{where} \quad \\
    & \begin{aligned}
    \bar{g}_{s}^{\Poi} (\btheta) :=& \log \Prob\bigg(\text{Poisson}\big(\Lambda \lambda(S_{s}, \bp_s;\btheta) \big) = n_s\bigg) - \log \Prob\bigg(\text{Poisson}\big(\Lambda \lambda(S_{s}, \bp_s;\btheta^*) \big) = n_s\bigg) \\
    =&\Lambda \big(\lambda(S_{s}, \bp_s;\btheta^*) - \lambda(S_{s}, \bp_s;\btheta) \big)  - n_s \big(\log \lambda(S_{s}, \bp_s;\btheta^*) -\log \lambda(S_{s}, \bp_s;\btheta)\big),
    \end{aligned} \\
    G_{t}^{\Poi}(\btheta) :=& \sum_{s=1}^t g_{s}^{\Poi}(\btheta), \mbox{where}\\
    & \begin{aligned}
    g_{s}^{\Poi}(\btheta) :=& \E \big[\bar{g}_{s}^{\Poi}(\btheta) \big| H_{s} \big] \\
    =& \Lambda \big(\lambda(S_{s}, \bp_s;\btheta^*) - \lambda(S_{s}, \bp_s;\btheta) \big) - \Lambda \lambda(S_{s}, \bp_s;\btheta^*) \big(\log \lambda(S_{s}, \bp_s;\btheta^*) - \log \lambda(S_{s}, \bp_s;\btheta)\big),
    \end{aligned} \\
    \bar{I}_{t}^{\Poi}(\btheta) :=& \sum_{s = 1}^t \bar{M}_{s}^{\Poi}(\btheta), \mbox{where}\\
    & \bar{M}_{s}^{\Poi}(\btheta) := - \nabla^2_{\btheta} \log \Prob\bigg(\text{Poisson}\big(\Lambda \lambda(S_{s}, \bp_s;\btheta)\big) = n_s\bigg) = \Lambda \lambda(S_{s}, \bp_s;\btheta)  \cdot \xs \xs^{\top}.
\end{align*}
\end{definition}

To verify the above definitions and equations, recall that the parameter governing customer arrivals is expressed as 
\(\lambda(S_t, \bp_t; \btheta) = \exp(\btheta^\top \bx(S_t, \bp_t))\), where \(\bx(S_t, \bp_t)\) represents a set of sufficient statistics encapsulating the impact of both the assortment \(S_t\) and prices \(\bp_t\) on customer arrivals. 

Here, \(\bar{g}_{s}^{\Poi}(\btheta)\) represents the difference in the logarithm of the Poisson probability between \(\btheta\) and its ground truth \(\btheta^*\), while \(\bar{G}_{t}^{\Poi}(\btheta)\) denotes the cumulative sum of these differences up to period \(t\).

The term \(g_{s}^{\Poi}(\btheta)\) is the conditional expectation of \(\bar{g}_{s}^{\Poi}(\btheta)\) given the history \(H_s\), and \(G_{t}^{\Poi}(\btheta)\) is its cumulative counterpart. Notably, \(-G_{t}^{\Poi}(\btheta)\) aligns with the Kullback-Leibler divergence between the conditional distributions parameterized by \(\btheta\) and \(\btheta^*\) since the data generation process is governed by the true parameters, making \(G_{t}^{\Poi}(\btheta)\) inherently non-positive. 

Furthermore, \(\bar{G}_{t}^{\Poi}(\btheta)\) serves as the empirical counterpart to \(G_{t}^{\Poi}(\btheta)\), and their difference constitutes a martingale: 
\[
\mathbb{E}[\bar{G}_{t}^{\Poi}(\btheta) - G_{t}^{\Poi}(\btheta) \mid H_t] = \bar{G}_{t-1}^{\Poi}(\btheta) - G_{t-1}^{\Poi}(\btheta).
\]

Lastly, \(\bar{I}_{t}^{\Poi}(\btheta)\) is the Fisher Information matrix related to the maximum likelihood estimation (MLE) of the sequence \(\{n_1, n_2, \ldots, n_t\}\). This matrix is symmetrical and positively definite.

Given that customers have arrived, their product choices are entirely characterized by the parameter \(\bv\). For any \(\bv \in \Real^N\), we define a set of analogous functions and terms: \(\bar{g}_{s}^{\MNL}(\bv)\), \(\bar{G}_{t}^{\MNL}(\bv)\), \(g_{s}^{\MNL}(\bv)\), and \(G_{t}^{\MNL}(\bv)\). These functions mirror the role of their counterparts in customer arrivals, offering a comprehensive framework for analyzing and understanding the dynamics of consumer behavior in the context of product choice.

\begin{definition}\label{def_v_app} 
We define log-likelihood ratio and the second order derivative of the log-likelihood ratio with respect to the parameter \(\bv\),

\begin{align*}
    \bar{G}_{t}^{\MNL}(\bv) :=& \sum_{s = 1}^{t} \bar{g}_{s}^{\MNL}(\bv), \quad \mbox{where} \quad 
    \bar{g}_{s}^{\MNL}(\bv) := \sum_{i = 1}^{n_{s}} \sum_{j \in S_{s} \cup\{0\}}\bm{1}\{C^{(i)}_{s} = j\} \log \frac{\ChoiceProb(j, S_s, \bp_s, \Featuret; \bv)}{\ChoiceProb(j, S_s, \bp_s, \Featuret; \bv^*)},\\
    G_{t}^{\MNL}(\bv) :=& \sum_{s=1}^t g_{s}^{\MNL}(\bv), \quad \mbox{where} \quad
    g_{s}^{\MNL}(\bv) := 
     \Lambda \lambda(S_s, \bp_s ;\btheta^*) \sum_{j \in S_{s} \cup\{0\}} \ChoiceProb(j, S_s, \bp_s, \Featuret;\bv^*) \log \frac{\ChoiceProb(j, S_s, \bp_s, \Featuret;\bv)}{\ChoiceProb(j, S_s, \bp_s, \Featuret;\bv^*)},\\
    \bar{I}_{t}^{\MNL}(\bv) :=& \sum_{s = 1}^t \bar{M}_{s}^{\MNL}(\bv) , \quad \mbox{where} \quad \\
    &
    \begin{aligned}
        \bar{M}_{s}^{\MNL}(\bv) :=& - \nabla^2_{\bv} \sum_{i = 1}^{n_s}\sum_{j \in S_{s} }\bm{1}\{C^{(i)}_{s} = j\} \log \ChoiceProb(j, S_s, \bp_s, \Feature_s; \bv) \\
    =& n_s \sum_{j \in S_{s} } \ChoiceProb(j, S_s, \bp_s, \Feature_s;  \bv) \Feature_{j s} \Feature'_{j s} - n_s \sum_{j, k\in S_{s} } \ChoiceProb(j, S_s, \bp_s, \Feature_s;  \bv)\ChoiceProb(k, S_s, \bp_s, \Feature_s;  \bv) \Feature_{j s} \Feature'_{k s}.
    \end{aligned}
\end{align*}
\end{definition}

Here, \(\bar{g}_{s}^{\MNL}(\bv)\) captures the difference in the log-likelihood ratio between \(\bv\) and \(\bv^*\), while \(\bar{G}_{t}^{\MNL}(\bv)\) denotes its cumulative sum up to period \(t\). The term \(g_{s}^{\MNL}(\bv)\) is the conditional expectation of \(\bar{g}_{s}^{\MNL}(\bv)\) given the history \(H_s\), and \(G_{t}^{\MNL}(\bv)\) is its cumulative counterpart. Similar to the Poisson case, \(-G_{t}^{\MNL}(\bv)\) aligns with the Kullback-Leibler divergence between the conditional distributions parameterized by \(\bv\) and \(\bv^*\), making \(G_{t}^{\MNL}(\bv)\) inherently non-positive. Additionally, \(\bar{G}_{t}^{\MNL}(\bv)\) serves as the empirical counterpart to \(G_{t}^{\MNL}(\bv)\), and their difference constitutes a martingale: 
\[
\mathbb{E}[\bar{G}_{t}^{\MNL}(\bv) - G_{t}^{\MNL}(\bv) \mid H_t] = \bar{G}_{t-1}^{\MNL}(\bv) - G_{t-1}^{\MNL}(\bv).
\]

Lastly, \(\bar{I}_{t}^{\MNL}(\bv)\) is the Fisher Information matrix related to the maximum likelihood estimation (MLE) of the choice sequence \(\{C^{(i)}_1, C^{(i)}_2, \ldots, C^{(i)}_t\}\). This matrix is symmetrical and positively definite.

\section{Proof of Lemma \ref{lm:theta_initial_error}}\label{appendix:proof_of_theta_initial_error}

This section provides a detailed proof for \Cref{lm:theta_initial_error}, which focuses on narrowing down the initial estimation for \(\btheta\). Before diving into the proof of the lemma, we first introduce two auxiliary lemmas, \Cref{lm:ineq} and \Cref{lm:g_1_convergence_rate}, which will be instrumental in the main proof. The structure of this section is as follows: we begin by presenting the two lemmas, followed by the proof of \Cref{lm:theta_initial_error}. Finally, we provide the proofs for the auxiliary lemmas, \Cref{lm:ineq} and \Cref{lm:g_1_convergence_rate}.

\begin{lemma}\label{lm:ineq}
    Suppose \(y\) satisfies the inequality
    \(
        y \leq a + \sqrt{b(y + c)},
    \)
    then it follows that 
  \(
        y \leq b + 2a + c.
\)
\end{lemma}

\begin{lemma}\label{lm:g_1_convergence_rate}
Suppose \(\|\btheta - \btheta^*\| \leq 2 \tau\), where \(\tau > 0 \). Let $c_4$ be defined in \Cref{eq:df_c_2}. Then, the following holds:

\begin{enumerate}
    \item \textbf{For any fixed \(t\)}:
    {\footnotesize
    \begin{equation}\label{eq:g_1_convergence_rate_1}
    \begin{aligned}
              \bar{G}_{t}^\Poi(\btheta) - G_{t}^\Poi(\btheta) 
               \leq & \frac{t\exp(\bar{x})}{T} \left(2 + \frac{4\log(T)}{ \Lambda \exp(\bar{x}) t} + \sqrt{\frac{4\log(T)}{ \Lambda \exp(\bar{x}) t}}\right) + 4\tau \bar{x} c_4 \big(\log T + \DimRank \log (6\tau \bar{x}(\Lambda T + 1))\big)\\
               &+\sqrt{\left(\left|G_{t}^\Poi(\btheta)\right| + \frac{2t}{T}\right)8 \big(\log T + \DimRank \log (6\tau \bar{x}(\Lambda T + 1))\big)} \quad \forall ~ \|\btheta - \btheta^*\| \leq 2 \tau,
    \end{aligned}
    \end{equation}
    }
    with probability \(1 - 2T^{-1}\).     

    \item \textbf{For \(T_0 < t \leq T\)}:
    {\footnotesize
    \begin{equation}\label{eq:g_1_convergence_rate_2}
    \begin{aligned}
                \bar{G}_{t}^\Poi(\btheta) - G_{t}^\Poi(\btheta)
                \leq & 2\exp(\bar{x}) + \sqrt{\frac{8\exp(\bar{x}) \log T}{T\Lambda}} + \frac{8\log T}{T\Lambda }
                + 4\tau \bar{x} c_4 \big(2\log T + \DimRank \log (6\tau \bar{x}(\Lambda T + 1))\big)\\
                &+\sqrt{\left(\left|G_{t}^\Poi(\btheta)\right| + 2\right) 8 \big(2\log T + \DimRank \log (6\tau \bar{x}(\Lambda T + 1))\big)},
    \end{aligned}
    \end{equation}
    }
    holds uniformly with probability \(1 - 2T^{-1}\), for all \(\|\btheta - \btheta^*\| \leq 2 \tau\).
\end{enumerate}
\end{lemma}

Now, let us detail the proof of Lemma \ref{lm:theta_initial_error}.

\textbf{Proof.}
It is easy to verify that:
\begin{align*}
    \nabla_{\btheta} G_{T_0}^\Poi (\btheta) &= \Lambda \sum_{s=1}^{T_0} \big(\lambda (S_s , \bp_s ; \btheta^*) - \lambda (S_s , \bp_s ; \btheta)\big) \bx_s, \quad 
    \nabla^2_{\btheta} G_{T_0}^\Poi (\btheta) = -\Lambda \sum_{s=1}^{T_0} \lambda (S_s , \bp_s ; \btheta) \bx_s \bx_s^{\top} = -I_{T_0}^\Poi(\btheta).
\end{align*}

Note that:
\[
G_{T_0}^\Poi (\btheta^*) = 0, \quad \nabla_{\btheta} G_{T_0}^\Poi (\btheta^*) = 0, \quad \text{and} \quad \nabla^2_{\btheta} G_{T_0}^\Poi (\btheta^*) = -I_{T_0}^\Poi(\btheta^*).
\]

Using the Taylor expansion with Lagrange remainder, there exists some \(\bar{\btheta} = \alpha \btheta^* + (1-\alpha) \widehat{\btheta}\) for \(\alpha \in (0,1)\) such that:
\begin{equation}
G_{T_0}^\Poi (\widehat{\btheta}) = -\frac{1}{2}\left(\widehat{\btheta} - \btheta^*\right)^{\top} I_{T_0}^\Poi\left(\bar{\btheta}\right)\left(\widehat{\btheta} - \btheta^*\right).
\end{equation}
From \Cref{eq:assump-x} in \Cref{assump:combined}, we have 
\[
I_{T_0}^\Poi (\bar \btheta) \succeq T_0 \Lambda \exp(-\bar{x}) \sigma_1  I_{\DimRank \times \DimRank},
\]
where \(I_{\DimRank \times \DimRank}\) denotes the identity matrix of dimension \(\DimRank\).

Thus, we have:
\begin{equation}\label{eq:2-T0}
    -G_{T_0}^\Poi(\widehat{\btheta}) \geq \frac{1}{2}T_0 \Lambda \exp(-\bar{x}) \sigma_1 \cdot \|\widehat{\btheta} - \btheta^*\|^2_2.
\end{equation}

Next, we provide an upper bound for \(-G_{T_0}^\Poi(\widehat{\btheta})\), which will then be used in conjunction with inequality \eqref{eq:2-T0} to bound \(\|\widehat{\btheta} - \btheta^*\|^2_2\).

Using the facts that \(G_{T_0}^\Poi\left(\widehat{\btheta}\right) \leq 0\) and \(\bar{G}_{T_0}^\Poi\left(\widehat{\btheta}\right) \geq 0\), along with inequality \eqref{eq:g_1_convergence_rate_1} from \Cref{lm:g_1_convergence_rate}, we derive:
\begin{multline*}
       \left|G_{T_0}^\Poi\left(\widehat{\btheta}\right)\right| \leq \bar{G}_{T_0}^\Poi\left(\widehat{\btheta}\right) - G_{T_0}^\Poi\left(\widehat{\btheta}\right) 
   \leq \frac{T_0 \exp(\bar{x})}{T} \left(2 + \frac{4\log(T)}{\Lambda \exp(\bar{x}) T_0} + \sqrt{\frac{4\log(T)}{\Lambda \exp(\bar{x}) T_0}}\right) \\
    + 4 \bar{x} c_4 \big(\log T + \DimRank \log (3\bar{x}(\Lambda T + 1))\big) 
     + \sqrt{\left(\left|G_{T_0}^\Poi(\widehat{\btheta})\right| + \frac{2 T_0}{T}\right)  4 \big(\log T + \DimRank \log (3\bar{x}(\Lambda T + 1))\big)},
\end{multline*}
which holds with probability \(1 - 2/T\). 

In the second inequality, we use the result from Lemma \ref{lm:g_1_convergence_rate} with \(\tau = \frac{1}{2}\), which guarantees that \(\|\widehat{\btheta} - \btheta^*\| \leq 1 = 2 \tau\), thereby ensuring the validity of the bound under this setting.

Combining the above inequality and Lemma \ref{lm:ineq}, we obtain an upper bound for
\begin{equation*}
    \begin{aligned}
        \left|G_{T_0}^\Poi\left(\widehat{\btheta}_t\right)\right| &\leq \frac{2 T_0\exp(\bar{x})}{T} \left(2 + \frac{4\log(T)}{\Lambda \exp(\bar{x}) T_0} + \sqrt{\frac{4\log(T)}{\Lambda \exp(\bar{x}) T_0}}\right) 
         + \frac{2 T_0}{T}\\
        &\quad + (8 \bar{x} c_4 + 4) \big(\log T + \DimRank \log (3\bar{x}(\Lambda T + 1))\big),
    \end{aligned}
\end{equation*}
with a probability \( 1 - 2/T \).

Given that \(G_{T_0}^\Poi\left(\widehat{\btheta}_t\right) \leq 0\) and by using \Cref{eq:2-T0}, we have, with a probability \(1-2T^{-1}\),
\begin{multline*}
        \frac{1}{2}T_0 \Lambda \exp(-\bar{x})\sigma_1||\widehat{\btheta}-\btheta^*||^2_2 
        \leq - G_{T_0}^\Poi \left(\widehat{\btheta}_t\right) 
        \leq  \frac{2 T_0\exp(\bar{x})}{T} \left(2 + \frac{4\log(T)}{\Lambda \exp(\bar{x}) T_0} + \sqrt{\frac{4\log(T)}{\Lambda \exp(\bar{x}) T_0}}\right)  +  \frac{2 T_0}{T} \\
        \quad + (8  \bar{x} c_4 + 4)  \big(\log T + \DimRank \log (3\bar{x}(\Lambda T + 1))\big).
\end{multline*}
Thus, we conclude:
\begin{multline*}
       \|\widehat{\btheta}-\btheta^*\|^2_2
        \leq \frac{4 \exp(2\bar{x})}{T \Lambda \sigma_1} \left(2 + \frac{4\log(T)}{\Lambda \exp(\bar{x}) T_0} + \sqrt{\frac{4\log(T)}{\Lambda \exp(\bar{x}) T_0}}\right)  +   \frac{4 \exp(\bar{x})}{T \Lambda \sigma_1} \\
        \quad + \frac{2(8 \bar{x} c_4 + 4) \exp(\bar{x})}{T_0\Lambda\sigma_1} \big(\log T + \DimRank \log (3\bar{x}(\Lambda T + 1))\big),
\end{multline*}
with probability \(1 - 2 T^{-1}\). \qed

Now, we turn to the proof for  \Cref{lm:ineq,lm:g_1_convergence_rate}.

\subsection{Proof of Lemma \ref{lm:ineq}}
Lemma \ref{lm:ineq} follows directly using the inequality
\(
\sqrt{b(y+c)} \leq \frac{b + y + c}{2}.
\)
Substituting this into the given condition, we obtain
\(
y \leq a + \frac{b + y + c}{2}.
\)
Rearranging terms gives
\(
y \leq b + 2a + c.
\) This completes the proof of Lemma \ref{lm:ineq}.
\qed

\subsection{Proof of Lemma \ref{lm:g_1_convergence_rate}}\label{appendix:proof_of_g_1_convergence_rate}

This subsection details the proof of Lemma \ref{lm:g_1_convergence_rate}. Before diving into the details, we first introduce three auxiliary lemmas:  \Cref{lm:poi_k_bound}, \Cref{lm:V_bound} and \Cref{lm:dif_g}. In this subsection, we present the proof of Lemma \ref{lm:g_1_convergence_rate} assuming the validity of \Cref{lm:poi_k_bound,lm:V_bound,lm:dif_g} The proofs of \Cref{lm:poi_k_bound,lm:V_bound,lm:dif_g} are deferred to the end of this subsection.

For \(\|\btheta - \btheta^*\|_2 \leq 2\tau\), let us first define:
\begin{equation}\label{eq:d_f_v_s}
\begin{aligned}
        \mathcal{V}_s^{Poi} := & \operatorname{Var}\bigg[(n_s - \Lambda \lambda(S_s, \bp_s; \btheta^*))\log \frac{\lambda(S_s, \bp_s; \btheta^*)}{\lambda(S_s, \bp_s; \btheta)} \,\bigg| H_s\bigg] 
        = \Lambda \lambda(S_s, \bp_s; \btheta^*)\log^2 \frac{\lambda(S_s, \bp_s; \btheta^*)}{\lambda(S_s, \bp_s; \btheta)},\\
    \mathcal{SV}^{Poi}_t := & \sum_{s = 1}^t \mathcal{V}^{Poi}_s.
\end{aligned}
\end{equation}

\begin{lemma}\label{lm:poi_k_bound}
    For all \(k > 2\), we have the following inequality:
    \begin{equation}\label{eq:key_control_ineq}
        \E\bigg[\big| (n_s - \Lambda \lambda(S_s, \bp_s; \btheta^*))\log\frac{\lambda(S_s, \bp_s; \btheta^*)}{\lambda(S_s, \bp_s; \btheta)} \big|^k \,\bigg| H_s\bigg]
        \;\leq\; \frac{k!}{2}\,\mathcal{V}^{Poi}_s (\tau \bar{x} c_4)^{k-2}.
    \end{equation}
    Here, $c_4$ is defined \Cref{eq:df_c_2}.
\end{lemma}

\begin{lemma}\label{lm:V_bound}
    \(\mathcal{SV}^{Poi}_t \leq 2\left|G_{t}^{Poi} (\btheta)\right|.\)
\end{lemma}
\begin{lemma}\label{lm:dif_g}
    If \(\left\|\btheta - \btheta^{\prime}\right\|_2 \leq \epsilon\), then:
    \begin{equation}
        \left|G_{t}^\Poi(\btheta)\right| \leq \left|G_{t}^\Poi(\btheta')\right| + \Lambda \exp(\bar{x}) \big( \exp(\epsilon \bar{x}) - 1 + \epsilon \bar{x} \big) t.
    \end{equation}
\end{lemma}

Using Lemma \ref{lm:poi_k_bound} (i.e., conditional Bernstein condition), it's easy to check that the conditions of Theorem 1.2A in \cite{de1999general} hold. Therefore, we have for all \(x, y > 0\):
\begin{equation}\label{eq:G_1_convergence_1}
\begin{aligned}
     \Prob\left(\bar{G}_{t}^\Poi(\btheta) - G_{t}^\Poi(\btheta) \geq x, \mathcal{SV}^{Poi}_t \leq y \right) 
     &\leq \exp \left\{-\frac{x^2}{2(y+\tau \bar{x}c_4 x)}\right\},
\end{aligned}
\end{equation}

Using \Cref{lm:V_bound} and setting \(y = 2\left|G_{t}^\Poi(\btheta)\right|\), the inequality \Cref{eq:G_1_convergence_1} becomes:
\begin{equation*}
     \Prob\left(\bar{G}_{t}^\Poi(\btheta) - G_{t}^\Poi(\btheta) \geq x\right) 
     \leq \exp \left\{-\frac{x^2}{2(2\left|G_{t}^\Poi(\btheta)\right| + \tau \bar{x} c_4 x)}\right\}.
\end{equation*}

Setting the right-hand side to be \(\delta\), we derive the following inequality for each \(\btheta \in \{\btheta: \|\btheta - \btheta^*\|_2 \leq 2 \tau\}\):
\begin{equation}\label{eq:bound_g_bar_to_g}
         \Prob\left(\bar{G}_{t}^\Poi(\btheta) - G_{t}^\Poi(\btheta) \geq \sqrt{8 \left|G_{t}^\Poi(\btheta)\right|\log\frac{1}{\delta}} + 4\tau \bar{x} c_4 \log\frac{1}{\delta}\right) 
         \leq \delta.
\end{equation}

For any \(\epsilon > 0\), let \(\mathcal{H}(\epsilon)\) be a finite covering of \(\left\{\btheta \in \mathbb{R}^\DimRank : \|\btheta - \btheta^*\|_2 \leq 2 \tau\right\}\) in \(\|\cdot\|_2\) up to precision \(\epsilon\). That is, for \(\forall \btheta' \in \{\btheta: \|\btheta - \btheta^*\|_2 \leq 2 \tau\}\), there exists \(\btheta \in \mathcal{H}(\epsilon)\) such that \(\|\btheta - \btheta'\|_2 \leq \epsilon\). By standard covering number arguments in \cite{geer2000empirical}, such a finite covering set \(\mathcal{H}(\epsilon)\) exists, and its size can be upper bounded as:
\[
|\mathcal{H}(\epsilon)| \leq \left( \frac{6\tau}{\epsilon} \right)^{\DimRank}.
\]

For each \(\btheta \in \mathcal{H}(\epsilon)\), the set satisfying the probability condition in (\ref{eq:bound_g_bar_to_g}) has probability \(2\delta\). Thus, we obtain:
\begin{equation}\label{eq:bound_g_bar_to_g_combined}
         \Prob\left(\bar{G}_{t}^\Poi(\btheta) - G_{t}^\Poi(\btheta) \geq \sqrt{8 \left|G_{t}^\Poi(\btheta)\right|\log\frac{1}{\delta}} + 4\tau \bar{x} c_4 \log\frac{1}{\delta}, \forall \btheta \in \mathcal{H}(\epsilon)\right) 
         \leq \delta |\mathcal{H}(\epsilon)|.
\end{equation}

\textbf{Proof of Part 1 in Lemma \ref{lm:g_1_convergence_rate}:}

Set \(\delta = \frac{1}{T}\left(\frac{\epsilon}{6\tau}\right)^\DimRank\). Then, with probability \(1 - T^{-1}\), we have:
\begin{equation}\label{eq:bound_g_bar_to_g_1}
\bar{G}_{t}^\Poi(\btheta) - G_{t}^\Poi(\btheta) \leq 4\tau \bar{x} c_4\big(\log T + \DimRank \log (6\tau / \epsilon)\big) + \sqrt{\left|G_{t}^\Poi(\btheta)\right| 8 \big(\log T + \DimRank \log (6\tau / \epsilon)\big)}, \quad \forall \btheta \in \mathcal{H}(\epsilon).
\end{equation}

Given that the inequality holds for a finite set of \(\btheta \in \mathcal{H}(\epsilon)\), we now extend it to the general case where \(|\bar{G}_{t}^\Poi(\btheta') - G_{t}^\Poi(\btheta')|\) holds for all \(\btheta'\) and is well-bounded from above with a probability \(1 - 2T^{-1}\) (see the final part, inequality \ref{eq:g_1_convergence_rate_1}, of the proof).

For \(\btheta'\), we can find a \(\btheta \in \mathcal{H}(\epsilon)\) such that \(\|\btheta - \btheta'\|_2 \leq \epsilon\). Using this, we derive the following inequality:
\begin{equation}\label{ineq:bound_g_bar_to_g}
    \begin{aligned}
           \bar{G}_{t}^\Poi(\btheta') - G_{t}^\Poi(\btheta') 
           \leq  \bar{G}_{t}^\Poi(\btheta) - G_{t}^\Poi(\btheta) 
           + \sum_{s = 1}^t \left| n_s \log\left(\frac{\lambda\left(S_s, \boldsymbol{p}_s ; \btheta'\right)}{\lambda\left(S_s, \boldsymbol{p}_s ; \btheta\right)}\right) \right| 
            + \Lambda \sum_{s = 1}^t \left|\lambda\left(S_s, \boldsymbol{p}_s ; \btheta^*\right) \log\left(\frac{\lambda\left(S_s, \boldsymbol{p}_s ; \btheta'\right)}{\lambda\left(S_s, \boldsymbol{p}_s ; \btheta\right)}\right)\right|.
    \end{aligned}
\end{equation}

The inequality \ref{ineq:bound_g_bar_to_g} can be derived by observing that:
\[
\bar{g}_{s}^\Poi(\btheta) - g_{s}^\Poi(\btheta) = (\Lambda \lambda(S_s, \bp_s; \btheta^*) - n_s)\bigl(\log \lambda(S_s, \bp_s; \btheta^*) - \log \lambda(S_s, \bp_s; \btheta)\bigr),
\]
which holds for both \(\btheta\) and \(\btheta'\). Consequently, we have:
\begin{equation}\label{eq:one_more_step_final}
\bigl(\bar{G}_{t}^\Poi(\btheta') - G_{t}^\Poi(\btheta')\bigr) - \bigl(\bar{G}_{t}^\Poi(\btheta) - G_{t}^\Poi(\btheta)\bigr) 
= \sum_{s=1}^t (\Lambda \lambda(S_s, \bp_s; \btheta^*) - n_s)\bigl(\log \lambda(S_s, \bp_s; \btheta) - \log \lambda(S_s, \bp_s; \btheta')\bigr).
\end{equation}

By applying the triangle inequality directly to \eqref{eq:one_more_step_final}, we can derive the inequality \ref{ineq:bound_g_bar_to_g}.

Now, to bound \(\bar{G}_{t}^\Poi(\btheta') - G_{t}^\Poi(\btheta')\) for all \(\btheta'\), we only need to bound each term on the right-hand side of inequality \ref{ineq:bound_g_bar_to_g}. Specifically, we focus on bounding the following terms:

1. \(\bar{G}_{t}^\Poi(\btheta) - G_{t}^\Poi(\btheta)\) for \(\btheta \in \mathcal{H}(\epsilon)\), which has already been controlled by \eqref{eq:bound_g_bar_to_g_1}.

2. The second term:
\[
\sum_{s = 1}^t \left| n_s \log\left(\frac{\lambda\left(S_s, \bp_s ; \btheta'\right)}{\lambda\left(S_s, \bp_s ; \btheta\right)}\right) \right|.
\]

3. The third term:
\[
\Lambda \sum_{s = 1}^t \left|\lambda\left(S_s, \bp_s ; \btheta^*\right) \log\left(\frac{\lambda\left(S_s, \bp_s ; \btheta'\right)}{\lambda\left(S_s, \bp_s ; \btheta\right)}\right)\right|.
\]

Bounding these terms ensures that \(\bar{G}_{t}^\Poi(\btheta') - G_{t}^\Poi(\btheta')\) is well controlled for all \(\btheta'\), completing the necessary analysis.

\begin{itemize}

\item \textbf{For the first part of the inequality (\ref{ineq:bound_g_bar_to_g})}, we use inequality \ref{eq:bound_g_bar_to_g_1} and substitute \(\left|G_{t}^\Poi(\btheta)\right|\) with \(\left|G_{t}^\Poi(\btheta')\right|\). 
It is worth mentioning that we need to bound the first part using \(G_{t}^\Poi(\btheta')\) rather than \(G_{t}^\Poi(\btheta)\). To achieve this, we use \Cref{lm:dif_g}, which will be proved in \Cref{appendix:proof_dif_g}.


\item \textbf{For the second part of inequality (\ref{ineq:bound_g_bar_to_g})}, we have:
\begin{align*}
    \sum_{s = 1}^t\left| n_s \log\left(\frac{\lambda \left(S_s, \boldsymbol{p}_s ; \btheta'\right)}{\lambda\left(S_s, \boldsymbol{p}_s ; \btheta\right)} \right)\right| 
    \leq \bar{x} \epsilon \sum_{s = 1}^t n_s,
\end{align*}
where a random realization of \( \sum_{s=1}^t n_s \) needs to be bounded in probability.

Using the result from Chapter 3.5 of \cite{Pol15} and noting that \(\lambda(S_s, \bp_s ; \btheta) \leq \exp(\bar{x})\) for all \(\|\btheta^*\| \leq 1\), we have:
\begin{equation}\label{eq:concerntration_bd_Poi}
\begin{aligned}
    \Prob \left(\sum_{s = 1}^t {n}_s \geq \Lambda \exp(\bar{x}) t (1 + \delta) \right) 
    \leq \exp\left(-\frac{\delta^2 \Lambda \exp(\bar{x}) t}{2(1 + \delta)}\right).
\end{aligned}
\end{equation}

Set \(\delta = \frac{4\log(T)}{ \Lambda \exp(\bar{x}) t} + \sqrt{\frac{4\log(T)}{ \Lambda \exp(\bar{x}) t}}\). Substituting this into the bound, we obtain:
\begin{equation}\label{eq:concerntration_bd_Poi_delta}
\begin{aligned}
    \Prob \left(\sum_{s = 1}^t {n}_s \geq \Lambda \exp(\bar{x}) t \left(1 + \frac{4\log(T)}{ \Lambda \exp(\bar{x}) t} + \sqrt{\frac{4\log(T)}{ \Lambda \exp(\bar{x}) t}}\right)\right) \leq \frac{1}{T}.
\end{aligned}
\end{equation}

Thus, with probability \(1-T^{-1}\), we have:
\begin{equation}\label{eq:g_bar_error_uniform}
         \sum_{s = 1}^t\left| n_s \log\left( \frac{\lambda \left(S_s, \boldsymbol{p}_s ; \btheta'\right)}{\lambda\left(S_s, \boldsymbol{p}_s ; \btheta\right)}\right) \right|  
         \leq  \Lambda \bar{x}\exp(\bar{x}) t \left(1 + \frac{4\log(T)}{ \Lambda \exp(\bar{x}) t} + \sqrt{\frac{4\log(T)}{ \Lambda \exp(\bar{x}) t}}\right) \epsilon.
\end{equation}

\item \textbf{For the third part of inequality (\ref{ineq:bound_g_bar_to_g})}, it is straightforward to calculate:
\begin{equation}\label{eq:g_error_uniform}
    \Lambda  \sum_{s = 1}^t\left |\lambda \left(S_s, \boldsymbol{p}_s ; \btheta^*\right) \log\left(\frac{\lambda \left(S_s, \boldsymbol{p}_s ; \btheta'\right)}{\lambda\left(S_s, \boldsymbol{p}_s ; \btheta\right)}\right) \right| 
    \leq \Lambda \bar{x} \epsilon \exp\{\bar{x}\} t.
\end{equation}

Note that the upper bounds for the first, second, and third parts (see (\ref{eq:bound_g_bar_to_g_1}), (\ref{eq:g_bar_error_uniform}), and (\ref{eq:g_error_uniform})) all include \(\epsilon\). Carefully setting \(\epsilon = \log(1 + 1/\Lambda T)/\bar{x}\), we have \( \Lambda \bar{x} \epsilon t \leq t/T \) and \(\exp(\epsilon \bar{x}) - 1 = \frac{1}{\Lambda T}\). 

For \(\left\|\btheta' - \btheta^*\right\|_2 \leq 2 \tau\), let \(\btheta\) be the closest point in \(\mathcal{H}(\epsilon)\). By the definition of \(\mathcal{H}(\epsilon)\), we have \(\|\btheta - \btheta'\|_2 \leq \epsilon\). Combining inequalities (\ref{eq:bound_g_bar_to_g}), (\ref{eq:g_bar_error_uniform}), and (\ref{eq:g_error_uniform}) with Lemma \ref{lm:dif_g}, we obtain the following with probability \(1-2T^{-1}\):
\begin{align*}
    &\bar{G}_{t}^\Poi(\btheta') - G_{t}^\Poi(\btheta') \\
    \leq & \bar{G}_{t}^\Poi(\btheta) - G_{t}^\Poi(\btheta)
    + \sum_{s = 1}^t n_s\log\left( \frac{\lambda \left(S_s, \boldsymbol{p}_s ; \btheta'\right)}{\lambda\left(S_s, \boldsymbol{p}_s ; \btheta\right)} \right)  
    + \Lambda \sum_{s = 1}^t\left |\lambda \left(S_s, \boldsymbol{p}_s ; \btheta^*\right) \log\left(\frac{\lambda \left(S_s, \boldsymbol{p}_s ; \btheta'\right)}{\lambda\left(S_s, \boldsymbol{p}_s ; \btheta\right)}\right) \right| \\
    \leq & \frac{t\exp(\bar{x})}{T} \left(2 + \frac{4\log(T)}{ \Lambda \exp(\bar{x}) t} + \sqrt{\frac{4\log(T)}{ \Lambda \exp(\bar{x}) t}}\right)
    + 4\tau \bar{x} c_4 \big(\log T + \DimRank \log (6\tau \bar{x}(\Lambda T + 1))\big)\\
    &+\sqrt{\left(\left|G_{t}^\Poi(\btheta')\right| + \frac{2t}{T}\right)8 \big(\log T + \DimRank \log (6\tau \bar{x}(\Lambda T + 1))\big)}, 
    \quad \forall \btheta' \in \{\btheta : \|\btheta - \btheta^*\|_2 \leq 2 \tau\}.
\end{align*}

Thus, we complete the proof for the first part of Lemma \ref{lm:g_1_convergence_rate}.
\qed

\end{itemize}
\bigskip

\textbf{Proof of the Second Part of the Lemma \ref{lm:g_1_convergence_rate}}

Similar to the proof of the first part of Lemma \ref{lm:g_1_convergence_rate}, we set \(\delta = \frac{1}{T^2}\left(\frac{\epsilon}{6\tau}\right)^{\DimRank}\) in \Cref{eq:bound_g_bar_to_g_combined}. For \(\forall T_0 < t \leq T\) and \(\forall \btheta \in \mathcal{H}(\epsilon)\), we obtain:
\begin{equation}\label{eq:bound_g_bar_to_g_2}
\bar{G}_{t}^\Poi(\btheta) - G_{t}^\Poi(\btheta)
\leq 4\tau \bar{x} c_4\big(2\log T + \DimRank \log (6\tau / \epsilon)\big) 
+ \sqrt{\left|G_{t}^\Poi(\btheta)\right| 8 \big(2\log T + \DimRank \log (6\tau / \epsilon)\big)},
\end{equation}
with probability \(1 - T^{-1}\).

Following a similar procedure to the first part of the proof for Lemma \ref{lm:g_1_convergence_rate}, we extend this result to the general case \(\btheta' \in \{\btheta : \|\btheta - \btheta^*\|_2 \leq 2\tau\}\). Specifically:
\begin{equation}\label{eq:bound_g_bar_to_g_1_all_2}
\bar{G}_{t}^\Poi(\btheta') - G_{t}^\Poi(\btheta') 
\leq 4\tau \bar{x} c_4\big(2\log T + \DimRank \log (6\tau / \epsilon)\big) 
+ \sqrt{\left|G_{t}^\Poi(\btheta')\right| 8 \big(2\log T + \DimRank \log (6\tau / \epsilon)\big)},
\end{equation}
holds uniformly for \(\forall \btheta' \in \{\btheta : \|\btheta - \btheta^*\|_2 \leq 2 \tau\}\) and \(\forall T_0 < t \leq T\), with probability \(1 - T^{-1}\).

To achieve these uniform results for all \(t\), we use the same triangular inequality (\ref{ineq:bound_g_bar_to_g}) and take an upper bound for each term.

\begin{itemize}
    \item \textbf{For the first part of inequality (\ref{ineq:bound_g_bar_to_g})}, we directly use inequality (\ref{eq:bound_g_bar_to_g_1_all_2}) along with Lemma \ref{lm:dif_g} to bound this term.

    \item \textbf{For the second part of inequality (\ref{ineq:bound_g_bar_to_g})}, note that for arbitrary \(\|\btheta' - \btheta\|_2 \leq \epsilon\):
    \begin{align}
        \sum_{s = 1}^t\left | n_s \log\left(\frac{\lambda \left(S_s, \boldsymbol{p}_s ; \btheta'\right)}{\lambda\left(S_s, \boldsymbol{p}_s ; \btheta\right)} \right)\right| \leq \bar{x} \epsilon \sum_{s = 1}^t n_s.\label{eq:g_bar_theta_to_theta'}
    \end{align}

    Similar to \Cref{eq:concerntration_bd_Poi}, we obtain the following concentration inequality:
    \begin{equation}\label{eq:concerntration_bd_Poi_uni}
        \begin{aligned}
            \Prob \left( \sum_{s = 1}^t n_s \geq \Lambda \exp(\bar{x}) t \left( 1 + \frac{8 \log(T)}{ \Lambda \exp(\bar{x}) t} + \sqrt{\frac{8 \log(T)}{ \Lambda \exp(\bar{x}) t}} \right) \right) \leq \frac{1}{T^2}.
        \end{aligned}
    \end{equation}
    This bound holds for each \(t\) such that \( t \in \left\{ T_0 + 1, \dots, T \right\} \).

Thus, we have the following inequality,
\begin{equation}\label{eq:bound_n_t_2}
    \Prob \left( \sum_{s = 1}^t n_s \geq \Lambda \exp(\bar{x}) t \left( 1 + \frac{8 \log(T)}{ \Lambda \exp(\bar{x}) t} + \sqrt{\frac{8 \log(T)}{ \Lambda \exp(\bar{x}) t}}\right) , \forall T_0 < t \leq T\right) \leq \frac{1}{T}.
\end{equation}

\item \textbf{For the third part of the inequality (\ref{ineq:bound_g_bar_to_g})}. Note that, have the inequality,
\begin{equation} \label{eq:g_theta_to_theta'}
    \Lambda \sum_{s = 1}^t\left |\lambda \left(S_s, \boldsymbol{p}_s ; \theta^*\right) \log \frac{\lambda \left(S_s, \boldsymbol{p}_s ; \btheta'\right)}{\lambda\left(S_s, \boldsymbol{p}_s ; \btheta\right)} \right| \leq \Lambda \bar{x} \epsilon \exp\{\bar{x}\} t.
\end{equation}

\end{itemize}

Similarly, note that the upper bound of the first, second, and third parts (see (\ref{eq:bound_g_bar_to_g_1_all_2}), (\ref{eq:g_theta_to_theta'}), and (\ref{eq:bound_n_t_2})) includes \(\epsilon\). We carefully set \(\epsilon = \log(1 + 1/\Lambda T)/\bar{x}\), and combining inequalities \ref{eq:bound_g_bar_to_g_1_all_2}, (\ref{eq:g_theta_to_theta'}), and (\ref{eq:bound_n_t_2}), we have, with probability \(1 - 2T^{-1}\), that:
\begin{align*}
    &\bar{G}_{t}^\Poi(\btheta') - G_{t}^\Poi(\btheta')\\
    \leq& \bar{G}_{t}^\Poi(\btheta) - G_{t}^\Poi(\btheta) + \sum_{s = 1}^t\left| n_s \log\left(\frac{\lambda \left(S_s, \boldsymbol{p}_s ; \btheta'\right)}{\lambda\left(S_s, \boldsymbol{p}_s ; \btheta\right)}\right) \right| + \Lambda \sum_{s = 1}^t\left|\lambda \left(S_s, \boldsymbol{p}_s ; \btheta^*\right) \log\left(\frac{\lambda \left(S_s, \boldsymbol{p}_s ; \btheta'\right)}{\lambda\left(S_s, \boldsymbol{p}_s ; \btheta\right)}\right) \right| \\
   \leq& \frac{\exp(\bar{x})t}{T}\left(2 + \frac{8\log(T)}{ \Lambda \exp(\bar{x}) t} + \sqrt{\frac{8\log(T)}{ \Lambda \exp(\bar{x}) t}}\right) 
   +4\tau \bar{x} c_4 \big(2\log T + \DimRank \log (3\tau \bar{x}(\Lambda T + 1))\big)\\
    &+\sqrt{\left(\left|G_{t}^\Poi(\btheta')\right| + \frac{2t}{T}\right) 8 \big(2\log T + \DimRank \log (3\tau \bar{x}(\Lambda T + 1))\big)}\\
    \leq& 2\exp(\bar{x}) + \sqrt{\frac{8\exp(\bar{x}) \log T}{T\Lambda}} + \frac{8\log T}{T\Lambda} + 4\tau \bar{x} c_4 \big(2\log T + \DimRank \log (6\tau \bar{x}(\Lambda T + 1))\big)\\
    &+\sqrt{\left(\left|G_{t}^\Poi(\btheta')\right| + 2 \right) 8 \big(2\log T + \DimRank \log (6\tau \bar{x}(\Lambda T + 1))\big)}, \quad \forall T_0 < t \leq T.
\end{align*}

Thus, we conclude the proof for part 2 of Lemma \ref{lm:g_1_convergence_rate}.
\qed

\subsection{Proofs for Supporting Lemmas \ref{lm:poi_k_bound}, \ref{lm:V_bound} and \ref{lm:dif_g}}\label{appendix:proof_of_supporting_lemma_1}
We now provide the proofs for the lemmas:\Cref{lm:poi_k_bound,lm:V_bound,lm:dif_g}.

\subsection{Proof of Lemma \ref{lm:poi_k_bound}}\label{appendix:proof_of_poi_k_bound}

In this subsection, we provide the proof for Lemma \ref{lm:poi_k_bound} under the case \(k > 2\), which is the primary focus. From \cite{ahle2022sharp}, the following bound holds:
\[
\E[n_s^k \mid H_s] \leq \left(\frac{k}{\log(1 + k/(\Lambda \lambda(S_s, \bp_s; \btheta^*)))}\right)^k.
\]
Since \(n_s \geq 0\), this implies:
\[
\E\big[|n_s - \Lambda \lambda(S_s, \bp_s; \btheta^*)|^k \mid H_s\big] \leq \E[n_s^k \mid H_s]  
\leq \left(\frac{k}{\log(1 + k / (\Lambda \lambda(S_s, \bp_s; \btheta^*)))}\right)^k.
\]
As \(\lambda(S_s, \bp_s; \btheta^*) \leq \exp(\bar{x})\), we obtain:
\begin{equation}
    \E\big[|n_s - \Lambda \lambda(S_s, \bp_s; \btheta^*)|^k \mid H_s\big] 
\leq \left(\frac{k}{\log(1 + k / (\Lambda \exp(\bar{x})))}\right)^k.
\end{equation}

Using the definition of the constant \(c_4\) (see \eqref{eq:df_c_2}), we derive the following inequality for all \(k > 2\):
\begin{equation}\label{eq:ineq2_mod}
\left(\frac{k}{\log (1+ k / (\Lambda \exp(\bar{x})))}\right)^k \leq   \frac{\sqrt{2\pi k}}{2} \left(\frac{k}{e}\right)^k \Lambda \lambda(S_s, \bp_s; \btheta^*) \left(\frac{c_4}{2}\right)^{k - 2}.
\end{equation}

Using Stirling’s approximation, we further refine the bound:
\begin{equation}
  \frac{\sqrt{2\pi k}}{2} \left(\frac{k}{e}\right)^k \Lambda \lambda(S_s, \bp_s; \btheta^*) \left(\frac{c_4}{2}\right)^{k - 2} \leq \frac{k!}{2} \Lambda \lambda(S_s, \bp_s; \btheta^*) \left(\frac{c_4}{2}\right)^{k - 2}.
\end{equation}

Combining these results, we conclude:
\[
\E\big[|n_s - \Lambda \lambda(S_s, \bp_s; \btheta^*)|^k \mid H_s\big] 
\leq \frac{k!}{2} \Lambda \lambda(S_s, \bp_s; \btheta^*) \left(\frac{c_4}{2}\right)^{k - 2}.
\]

Since \(\|\btheta - \btheta^*\|_2 \leq 2\tau\), it follows that
\(
\left|\log\frac{\lambda(S_s, \bp_s; \btheta^*)}{\lambda(S_s, \bp_s; \btheta)}\right|^{k-2} \leq (2\tau \bar{x})^{k-2}.
\)
Thus, we have:
\[
\E\bigg[\left| (n_s - \Lambda \lambda(S_s, \bp_s; \btheta^*))\log\frac{\lambda(S_s, \bp_s; \btheta^*)}{\lambda(S_s, \bp_s; \btheta)} \right|^k \,\bigg| H_s\bigg]
\leq \frac{k!}{2} \Lambda \lambda(S_s, \bp_s; \btheta^*) \log^2\frac{\lambda(S_s, \bp_s; \btheta^*)}{\lambda(S_s, \bp_s; \btheta)} \left(\frac{c_4}{2}\right)^{k - 2} (2\tau \bar{x})^{k - 2}.
\]

By substituting the definition of \(\mathcal{V}_s^{\Poi}\) from \eqref{eq:d_f_v_s}, the inequality above matches \eqref{eq:key_control_ineq}:
\[
\E\bigg[\left| (n_s - \Lambda \lambda(S_s, \bp_s; \btheta^*))\log\frac{\lambda(S_s, \bp_s; \btheta^*)}{\lambda(S_s, \bp_s; \btheta)} \right|^k \,\bigg| H_s\bigg]
\leq \frac{k!}{2}\,\mathcal{V}_s^{\Poi} (\tau \bar{x} c_4)^{k-2}.
\]
\qed

\subsection{Proof of Lemma \ref{lm:V_bound}}


We begin by considering the expression for \(\mathcal{V}_s^{\Poi}\). Starting from its definition, we have:
\begin{align*}
    \mathcal{V}_s^{\Poi} = &  \mathbb{E}\bigg[n_s^2\big(\log \lambda(S_s, \bp_s;\btheta) - \log \lambda(S_s, \bp_s;\btheta^*)\big)^2 \mid H_s\bigg] - \mathbb{E} \bigg[ n_s\big(\log \lambda(S_s, \bp_s;\btheta) - \log \lambda(S_s, \bp_s;\btheta^*)\big) \mid H_s\bigg]^2.
\end{align*}

Using the definition of \( \lambda(S_s, \bp_s;\btheta) \) and substituting the terms, we get:
\begin{align*}
    \mathcal{V}_s^{\Poi} = & (\Lambda^2 \lambda(S_s, \bp_s;\btheta^*)^2 + \Lambda \lambda(S_s, \bp_s;\btheta^*)) \log^2 y - \Lambda^2 \lambda(S_s, \bp_s;\btheta^*)^2 \log^2 y \\
    = & \Lambda \lambda(S_s, \bp_s;\btheta^*) \log^2 y.
\end{align*}

Now, let's consider \( g_{s}^\Poi(\btheta) \), which can be expressed as:
\begin{align*}
    g_{s}^\Poi(\btheta) &= \Lambda \lambda(S_s, \bp_s; \btheta^*) - \Lambda \lambda(S_s, \bp_s; \btheta) - \Lambda \lambda(S_s, \bp_s; \btheta^*) \big(\log \lambda(S_s, \bp_s; \btheta^*) - \log \lambda(S_s, \bp_s; \btheta)\big) \\
    &= -\Lambda \lambda(S_s, \bp_s; \btheta^*) (y - 1 - \log y),
\end{align*}
where \(y = \dfrac{\lambda(S_s, \bp_s; \btheta)}{\lambda(S_s, \bp_s; \btheta^*)}\).

Next, we use the inequality \(
    \log^2 y \leq 2(y - 1 - \log y),
\)
which directly leads to:
\begin{equation}
\begin{aligned}
    \mathcal{V}_s^{\Poi} &\leq \Lambda \lambda(S_s, \bp_s; \btheta^*) \log^2 y \leq 2\big(-g_{s}^\Poi(\btheta)\big).
\end{aligned}
\end{equation}

Finally, summing over all periods, we conclude
\(
    \mathcal{SV}_t^{\Poi} \leq 2\big(-G_{t}^\Poi(\btheta)\big).
\)
\qed

\subsection{Proof of Lemma \ref{lm:dif_g}}\label{appendix:proof_dif_g}

We begin by considering the difference between \( g_{s}^\Poi(\btheta) \) and \( g_{s}^\Poi(\btheta') \):
\begin{equation*}
    \begin{aligned}
        \left| g_{s}^\Poi(\btheta) - g_{s}^\Poi(\btheta') \right| 
        &= \Lambda \left| \lambda(S_s, \boldsymbol{p}_s; \btheta') - \lambda(S_s, \boldsymbol{p}_s; \btheta) + \lambda(S_s, \boldsymbol{p}_s; \btheta^*) \log \frac{\lambda(S_s, \boldsymbol{p}_s; \btheta)}{\lambda(S_s, \boldsymbol{p}_s; \btheta')} \right| \\
        &\leq \Lambda \left| \lambda(S_s, \boldsymbol{p}_s; \btheta') - \lambda(S_s, \boldsymbol{p}_s; \btheta) \right| + \Lambda \left| \lambda(S_s, \boldsymbol{p}_s; \btheta^*) \log \frac{\lambda(S_s, \boldsymbol{p}_s; \btheta)}{\lambda(S_s, \boldsymbol{p}_s; \btheta')} \right| \\
        &\leq \Lambda \exp(\bar{x}) (\exp(\epsilon \bar{x}) - 1) + \Lambda \exp(\bar{x}) \epsilon \bar{x} = \Lambda \exp(\bar{x}) \left( \exp(\epsilon \bar{x}) - 1 + \epsilon \bar{x} \right).
    \end{aligned}
\end{equation*}

Summing the inequalities, we have
\(
        \left| G_{t}^\Poi(\btheta) - G_{t}^\Poi(\btheta') \right| 
        \leq \Lambda \exp(\bar{x}) \left( \exp(\epsilon \bar{x}) - 1 + \epsilon \bar{x} \right) t.\)
\qed

\section{Proof of Lemma \ref{lm:theta_error}}\label{appendix::proof_of_theta_error}

It is easy to calculate that 
\begin{align*}
    \nabla_{\btheta} G_{t}^\Poi (\btheta) &= \sum_{s=1}^t \big(\lambda(S_s, \bp_s; \btheta^*) - \lambda(S_s, \bp_s; \btheta)\big) \bx_s, \quad
    \nabla^2_{\btheta} G_{t}^\Poi (\btheta) = \sum_{s=1}^t \lambda(S_s, \bp_s; \btheta) \bx_s \bx_s' = -I_{t}^\Poi(\btheta).
\end{align*}
Note that \(G_{t}^\Poi (\btheta^*) = 0\), \(\nabla_{\btheta} G_{t}^\Poi (\btheta^*) = 0\), and \(\nabla^2_{\btheta} G_{t}^\Poi (\btheta^*) = -I_{t}^\Poi(\btheta^*)\). Using Taylor expansion at the point \(\btheta^*\) with Lagrangian remainder, there exists \(\bar{\btheta}_t = \alpha \btheta^* + (1-\alpha) \widehat{\btheta}_t\) for some \(\alpha \in (0,1)\) such that
\begin{equation}\label{eq:G_equal_theta_I_theta}
G_{t}^\Poi (\widehat{\btheta}_t) = -\frac{1}{2}\left(\widehat{\btheta}_t - \btheta^*\right)^{\top} I_{t}^\Poi\left(\bar{\btheta}_t\right)\left(\widehat{\btheta}_t - \btheta^*\right).
\end{equation}
From \(\left\|\bar{\btheta}_t - \btheta^*\right\|_2 \leq 2 \tau_\btheta\) and \(\|\bx'(\bar{\btheta}_t - \btheta^*)\|_2 \leq \|\bx\|_2 \|\bar{\btheta}_t - \btheta^*\|_2 \leq 2\tau_\btheta \bar{x}\), we obtain \( -M_{t}^\Poi(\bar{\btheta}_t) \preceq -\exp\{-2\tau_\btheta \bar{x}\}M_{t}^\Poi(\btheta^*)\) which indicates\( -I_{t}^\Poi(\bar{\btheta}_t) \preceq -\exp\{-2\tau_\btheta \bar{x}\}I_{t}^\Poi(\btheta^*)\) using the definition of \(I_{t}^\Poi(\cdot) \). So we have
\begin{equation}\label{eq:G_leq_theta_I_theta}
G_{t}^\Poi (\widehat{\btheta}_t) \leq -\frac{1}{2}\exp\{-2\tau_\btheta \bar{x}\}\left(\widehat{\btheta}_t - \btheta^*\right)^{\top} I_{t}^\Poi\left(\btheta^*\right)\left(\widehat{\btheta}_t - \btheta^*\right).
\end{equation}

Using inequality (\ref{eq:g_1_convergence_rate_1}) of \Cref{lm:g_1_convergence_rate}, we have an uniform bound on the difference between \(\bar{G}_{t}^\Poi(\btheta)\) and \(G_{t}^\Poi(\btheta)\) for all \(\left\|\btheta - \btheta^*\right\| \leq 2 \tau_\btheta\) as follows:
\begin{equation}\label{eq:g_1_convergence_rate_3}
\begin{aligned}
    \bar{G}_{t}^\Poi(\widehat{\btheta}_t) - G_{t}^\Poi(\widehat{\btheta}_t)
    \leq & 2\exp(\bar{x}) + \sqrt{\frac{8\exp(\bar{x}) \log T}{T\Lambda}} + \frac{8\log T}{T\Lambda} + 4\tau_\btheta \bar{x} c_4 \big(2\log T + d \log (6\tau_\btheta \bar{x}(\Lambda T + 1))\big)\\
    &+\sqrt{\left(\left|G_{t}^\Poi(\widehat{\btheta}_t)\right| + 2 \right) 8 \big(2\log T + d \log (6\tau_\btheta \bar{x}(\Lambda T + 1))\big)}, \quad \forall T_0 < t \leq T.
\end{aligned}
\end{equation}
This holds uniformly with probability \(1 - 2 T^{-1}\) for \(t \in \left\{T_0+1, \ldots, T\right\}\).

By \Cref{eq:g_1_convergence_rate_3} and the fact that \(G_{t}^\Poi\left(\widehat{\btheta}_t\right) \leq 0 \leq \bar{G}_{t}^\Poi\left(\widehat{\btheta}_t\right)\), we have
\begin{equation}
    \begin{aligned}
        \left|G_{t}^\Poi\left(\widehat{\btheta}_t\right)\right| \leq & 2\exp(\bar{x}) + \sqrt{\frac{8\exp(\bar{x}) \log T}{T\Lambda}} + \frac{8\log T}{T\Lambda} + 4\tau_\btheta \bar{x} c_4 \big(2\log T + d \log (6\tau_\btheta \bar{x}(\Lambda T + 1))\big) \\
    &+\sqrt{\left(\left|G_{t}^\Poi(\widehat{\btheta}_t)\right| + 2 \right) 8 \big(2\log T + d \log (6\tau_\btheta \bar{x}(\Lambda T + 1))\big)}, \quad \forall T_0 < t \leq T.
    \end{aligned}
\end{equation}

Using \Cref{lm:ineq}, we have 
\begin{equation*}
    \begin{aligned}
        \left|G_{t}^\Poi\left(\widehat{\btheta}_t\right)\right| \leq & 4\exp(\bar{x}) + 2\sqrt{\frac{8\exp(\bar{x}) \log T}{T\Lambda}} + \frac{16\log T}{T\Lambda} + 8\tau_\btheta \bar{x} c_4 \big(2\log T + d \log (6\tau_\btheta \bar{x}(\Lambda T + 1))\big) + 2 \\
    &+ 8 \big(2\log T + d \log (6\tau_\btheta \bar{x}(\Lambda T + 1))\big), \quad \forall T_0 < t \leq T.
    \end{aligned}
\end{equation*}

Notice that \(G_{t}^\Poi\left(\widehat{\btheta}_t\right) \leq 0\) and combining \Cref{eq:G_leq_theta_I_theta}, we have
\begin{equation}
    \begin{aligned}
       \left(\widehat{\btheta}_t - \btheta^*\right)^{\top} I_{t}^\Poi\left(\btheta^*\right)\left(\widehat{\btheta}_t - \btheta^*\right) 
       \leq & 2\exp(2\tau_\btheta \bar{x})\left|G_{t}^\Poi\left(\widehat{\btheta}_t\right)\right| \\
       \leq & 8\exp(2\tau_\btheta \bar{x})\left[\exp(\bar{x}) + \sqrt{\frac{2\exp(\bar{x}) \log T}{T\Lambda}} + \frac{4\log T}{T\Lambda} + \frac{1}{2}\right] \\
            & + 16 \exp(2\tau_\btheta \bar{x}) (\tau_\btheta \bar{x} c_4 + 1)\big(2\log T + d \log (6\tau_\btheta \bar{x}(\Lambda T + 1))\big).
    \end{aligned}
\end{equation}

Similarly, 
\begin{multline*}
      \left(
        \widehat{\btheta}_t - \btheta^*\right)^{\top} I_{t}^\Poi\left(\widehat{\btheta}_t\right)\left(\widehat{\btheta}_t - \btheta^*
      \right) 
       \leq 
       8\exp(2\tau_\btheta \bar{x})\left[\exp(\bar{x}) + \sqrt{\frac{2\exp(\bar{x}) \log T}{T\Lambda}} + \frac{4\log T}{T\Lambda} + \frac{1}{2}\right] \\
             + 16 \exp(2\tau_\btheta \bar{x}) (\tau_\btheta \bar{x} c_4 + 1)\big(2\log T + d \log (6\tau_\btheta \bar{x}(\Lambda T + 1))\big).
\end{multline*}

\section{Proof of Lemma \ref{lemma:revenue_estimate}}\label{appendix:proof_of_sum_lambda_error}

Now, let us turn to the proof for the Lemma \ref{lemma:revenue_estimate} which states the estimation error on the Poisson arrival rate. 
Unless stated otherwise, all statements are conditioned on the success event in \Cref{lm:theta_error}.
On this event, the following inequalities hold uniformly for all \(t \in \{T_0,\ldots,T-1\}\):
\begin{equation}
    \begin{aligned}
         \left(\widehat{\btheta}_t - \btheta^*\right)^{\top} I_{t}^\Poi\left(\btheta^*\right)\left(\widehat{\btheta}_t - \btheta^*\right)  \leq \omega_\btheta, \;
         \left(\widehat{\btheta}_t - \btheta^*\right)^{\top}  I_{t}^\Poi\left(\widehat{\btheta}_t\right)\left(\widehat{\btheta}_t - \btheta^*\right)  \leq \omega_\btheta, \;
         \|\widehat{\btheta}_t - \btheta^*\| \leq 2\tau_\theta,
    \end{aligned}
\end{equation}
We start by noting that the gradient of \(\lambda(S, \bp; \btheta)\) with respect to \(\btheta\) is given by:
\begin{equation}\label{eq:nabla_lambda}
    \nabla_{\btheta} \lambda(S, \bp; \btheta) = \lambda(S, \bp; \btheta) \bx^\top(S, \bp).
\end{equation}

Next, by applying the mean value theorem, we know that there exists some \(\widetilde{\btheta}_{t-1} = \btheta^* + \xi\left(\widehat{\btheta}_{t-1} - \btheta^*\right)\) for some \(\xi \in (0,1)\), such that:
\begin{equation}\label{eq:lambda_deta}
  \begin{aligned}
\left|\lambda(S, \bp; \widehat{\btheta}_{t - 1})-\lambda(S, \bp;\btheta^*)\right| & =\left|\left\langle\nabla_{\btheta} \lambda(S, \bp; \widetilde{\btheta}), \widehat{\btheta}_{t-1}-\btheta^*\right\rangle\right| \\
& =\sqrt{\left(\widehat{\btheta}_{t-1}-\btheta^*\right)^{\top}\left[\nabla_{\btheta} \lambda(S, \bp; \widetilde{\btheta})\nabla_{\btheta}\lambda(S, \bp; \widetilde{\btheta})^{\top}\right]\left(\widehat{\btheta}_{t-1}-\btheta^*\right)} \\
&= \sqrt{\frac{\lambda(S, \bp; \widetilde{\btheta}) }{\Lambda} \left(\widehat{\btheta}_{t-1}-\btheta^*\right)^{\top}M_{t}^\Poi(\widetilde{\btheta} \mid S, \bp)\left(\widehat{\btheta}_{t-1}-\btheta^*\right)}.
\end{aligned}  
\end{equation}

Notice that \(M_{t}^\Poi(\tilde{\btheta} \mid S, \bp) = \Lambda \exp(\bx^\top \tilde{\btheta}) \bx(S, \bp) \bx^\top(S, \bp) \preceq \exp\{2\tau_\btheta \bar{x}\}M_{t}^\Poi(\btheta^* \mid S, \bp)\) and \( \lambda(S, \bp; \widetilde{\btheta}) \leq \exp(\bar{x})\). Combining with \Cref{eq:lambda_deta} and \Cref{lm:theta_error}, we obtain:
\begin{align*}
&\Lambda \left|\lambda(S, \bp; \widehat{\btheta}_{t - 1})-\lambda(S, \bp;\btheta^*)\right| \\
&= 
\sqrt{
    \Lambda 
    \lambda(S, \bp; \widetilde{\btheta}) 
    \left(
    \widehat{\btheta}_{t-1}-\btheta^*
    \right)^{\top}
    I_{t-1}^{\Poi \, \frac{1}{2}}\left(\btheta^*\right)
    I_{t-1}^{\Poi \, -\frac{1}{2}}
    \left(\btheta^*\right)
    M_{t}^\Poi(\widetilde{\btheta} \mid S, \bp)
    I_{t-1}^{\Poi \, -\frac{1}{2}}\left(\btheta^*\right)
    I_{t-1}^{\Poi \, \frac{1}{2}}\left(\btheta^*\right)
    \left(\widehat{\btheta}_{t-1}-\btheta^*\right)
} \\
&\leq \sqrt{\omega_\btheta \Lambda  \lambda(S, \bp; \widetilde{\btheta}) \left\|I_{t-1}^\Poi\left(\btheta^*\right)^{-1 / 2} M_{t}^\Poi\left(\widetilde{\btheta}_{t-1} \mid S, \bp\right) I_{t-1}^\Poi\left(\btheta^*\right)^{-1 / 2}\right\|_{\mathrm{op}}}\\
&\leq   \sqrt{\omega_\btheta  \Lambda \exp ((2\tau_\btheta +1) \bar{x}) \left\|I_{t-1}^\Poi\left(\btheta^*\right)^{-1 / 2} M_{t}^\Poi\left(\btheta^*\mid S, \bp\right) I_{t-1}^\Poi\left(\btheta^*\right)^{-1 / 2}\right\|_{\mathrm{op}}}.
\end{align*}

Similarly, we have 
{\scriptsize
\begin{equation}
\begin{aligned}
&\Lambda \left|\lambda(S, \bp; \widehat{\btheta}_{t - 1})-\lambda(S, \bp;\btheta^*)\right| \leq  
\sqrt{
    \omega_\btheta
    \Lambda
    \exp ((2\tau_\btheta +1) \bar{x}) 
    \left\|
        I_{t-1}^{\Poi \, -\frac{1}{2}}
        \left(\widehat{\btheta}_{t - 1}\right)
        M_{t}^\Poi
        \left(\widehat{\btheta}_{t - 1}\mid S, \bp\right)
        I_{t-1}^{\Poi \, -\frac{1}{2}}
        \left(\widehat{\btheta}_{t - 1}\right)]
    \right\|_{\mathrm{op}}
}.
\end{aligned}
\end{equation}
}

Notice that \(M_{t}^\Poi(\widehat{\btheta} \mid S, \bp)  \preceq \exp\{2\tau_\btheta \bar{x}\}M_{t}^\Poi(\btheta^* \mid S, \bp)\) and  \( I_{t}^\Poi(\btheta^*) \preceq  \exp\{2\tau_\btheta \bar{x}\} I_{t}^\Poi(\widehat{\btheta} ) \),
\begin{equation*}
    \opnorm{
        I_{t-1}^{\Poi -\frac{1}{2}}(\widehat{\btheta}_{t - 1})
        {M}_t^{\Poi} (\widehat{\btheta}_{t - 1}|S, \bp)
        I_{t-1}^{\Poi -\frac{1}{2}}(\widehat{\btheta}_{t - 1})
    } \leq
    \exp\{4\tau_\btheta \bar{x}\}
    \opnorm{{I}_{t-1}^{\Poi -\frac{1}{2}}(\btheta^*) 
    M_t^{\Poi} (\btheta^*|S, \bp)
    I_{t-1}^{\Poi -\frac{1}{2}}(\btheta^*)}.
\end{equation*}

Thus, we proved the \Cref{lemma:revenue_estimate}.
\qed


\section{Proof of Lemma \ref{lm:sum_lambda_error}}\label{appendix::proof_of_revenue_estimate}
This subsection provides a detailed proof of the Lemma \ref{lm:sum_lambda_error}.
\begin{itemize}
    \item \textbf{Proof of the first inequality of Lemma \ref{lm:sum_lambda_error}}.
    Denote \(\widehat{A}_t := {I}_{t-1}^{\Poi -1 / 2}(\btheta^*) {M}_{t}^\Poi(\btheta^* \mid S_t) {I}_{t-1}^{\Poi -1 / 2}(\btheta^*)\) as a \(d\)-dimensional positive semi-definite matrix with eigenvalues sorted as \(\sigma_1\left(\widehat{A}_t\right) \geq \ldots \geq \sigma_d\left(\widehat{A}_t\right) \geq 0\). By applying spectral properties, we have:
    \begin{equation*}
    \begin{aligned}
        &\sum_{t=T_0+1}^T \min \left\{\frac{\Lambda^2(\exp(\bar{x}) - \exp(-\bar{x}))^2}{\omega_\btheta \Lambda  \exp ((2\tau_\btheta +1) \bar{x})}, \left\|{I}_{t-1}^{\Poi -1 / 2}(\btheta^*) M_{t}^\Poi(\btheta^* \mid S_t) {I}_{t-1}^{\Poi -1 / 2}(\btheta^*)\right\|_{\text{op}}\right\} \\
        =&\sum_{t=T_0+1}^T \min \left\{\frac{\Lambda^2(\exp(\bar{x}) - \exp(-\bar{x}))^2}{\omega_\btheta \Lambda  \exp ((2\tau_\btheta +1) \bar{x})}, \sigma_1\left(\widehat{A}_t\right)\right\} \\
         \overset{\text{(a)}}{\leq} &\sum_{t=T_0+1}^T \frac{c_6}{\omega_\btheta \Lambda  \exp ((2\tau_\btheta +1) \bar{x})} \log \left(1+\min \left\{\frac{\Lambda^2(\exp(\bar{x}) - \exp(-\bar{x}))^2}{\omega_\btheta \Lambda  \exp ((2\tau_\btheta +1) \bar{x})}, \sigma_1\left(\widehat{A}_t\right)\right\}\right)\\
        \leq &\sum_{t=T_0+1}^T \frac{c_6}{\omega_\btheta \Lambda  \exp ((2\tau_\btheta +1) \bar{x})} \log \left(1+\sigma_1\left(\widehat{A}_t\right)\right)
    \end{aligned}
    \end{equation*}
    where \(c_6 = \frac{\Lambda^2(\exp(\bar{x}) - \exp(-\bar{x}))^2}{\log(1 + \frac{\Lambda^2(\exp(\bar{x}) - \exp(-\bar{x}))^2}{\omega_\btheta \Lambda  \exp ((2\tau_\btheta +1) \bar{x}) })},\)
    which is defined in \Cref{eq:df_c_6} and independent of time t. Here, the  inequality  \((a)\) can be obtained using an inequality that \(y \leq \dfrac{\log(1 + y)c}{\log(1 + c)}\) for \(0 \leq y \leq c\).
    
    Next, observe that:
    \begin{equation*}
    {I}_{t}^\Poi\left(\btheta^*\right) = {I}_{t-1}^\Poi\left(\btheta^*\right) + {M}_{t}^\Poi\left(\btheta^* \mid S_t\right) = {I}_{t-1}^\Poi\left(\btheta^*\right)^{1 / 2} \left[{I}_{d \times d} + \widehat{A}_t\right] {I}_{t-1}^\Poi\left(\btheta^*\right)^{1 / 2},
    \end{equation*}
    and therefore \({I}_{t}^\Poi\left(\btheta^*\right) = \log \operatorname{det} {I}_{t-1}^\Poi\left(\btheta^*\right) + \sum_{j=1}^d \log \left(1 + \sigma_j\left(\widehat{A}_t\right)\right).
\)

    By comparing the last two inequalities, we obtain:
    {\scriptsize
    \begin{equation*}
    \begin{aligned}
    &\sum_{t=T_0+1}^T \min \left\{\frac{\Lambda^2(\exp(\bar{x}) - \exp(-\bar{x}))^2}{\omega_\btheta \exp(2\tau_\btheta \bar{x})}, 
    \left\|{I}_{t-1}^{\Poi -1 / 2}\left(\btheta^*\right) 
    {M}_{t}^\Poi\left(\btheta^* \mid S_t\right) 
    {I}_{t-1}^{\Poi -1 / 2}
    \left(\btheta^*\right)
    \right\|_{\text{op}}\right\} 
    \leq \frac{c_6}{\omega_\btheta \Lambda  \exp ((2\tau_\btheta +1) \bar{x})}  \log \frac{\operatorname{det} {I}_{T}^\Poi\left(\btheta^*\right)}{\operatorname{det} {I}_{T_0}^\Poi\left(\btheta^*\right)},
    \end{aligned}
    \end{equation*}
    }
    which completes the proof of the first inequality.

    \item \textbf{Proof of the second inequality of Lemma \ref{lm:sum_lambda_error}}

Note that
    $\sigma_{\min}\left(\widehat{I}_{T_0}^\Poi\left(\btheta^*\right)\right) \geq  \Lambda \exp(-\bar{x}) T_0 \sigma_1$
and
\(
        \operatorname{tr}({I}_{T}^\Poi\left(\btheta^*\right)) \leq \sum_{s=1}^T \Lambda \exp(\bar{x}) \operatorname{tr}(\bx_s \bx_s^{\top}) 
        \leq \Lambda \exp(\bar{x}) \bar{x}^2 T.
\)
Therefore, we can bound
\begin{equation*}
    \begin{aligned}
        \log \left(\operatorname{det} {I}_{T}^\Poi\left(\btheta^*\right)\right) \leq \DimRank \log \frac{\operatorname{tr}({I}_{T}^\Poi\left(\btheta^*\right))}{\DimRank} \leq \DimRank \log \frac{\Lambda \exp(\bar{x}) \bar{x}^2 T}{\DimRank}.
    \end{aligned}
\end{equation*}

As a result, we obtain:
\begin{equation}
     \log 
     \frac{\operatorname{det} {I}_{T}^\Poi\left(\btheta^*\right)}{\operatorname{det} {I}_{T_0}^\Poi\left(\btheta^*\right)}
     \leq 
     \DimRank 
     \log \frac{\Lambda \exp(\bar{x}) \bar{x}^2 T}{\DimRank} + \bar{x} 
     - \log (\Lambda T_0 \sigma_1)
     =
     \DimRank \log \frac{ \bar{x}^2 T}{\DimRank} + (\DimRank + 1) \bar{x} - \log (T_0 \sigma_1),
\end{equation}
which completes the proof of the second inequality.

\end{itemize}
\qed

\section{Proof of Lemma \ref{lm:bv_initial_error}}\label{appendix::proof_of_bv_initial_error}
This section provides proof of \Cref{lm:bv_initial_error}. To facilitate the proof, we first present \Cref{lm:g_2_convergence_rate} and then provide the proof for \Cref{lm:bv_initial_error}. We leave the proof for the \Cref{lm:g_2_convergence_rate} to the end of this subsection. 

\begin{lemma}\label{lm:g_2_convergence_rate}
\begin{enumerate}
    \item With probability as least \(1 - 2 T^{-1}\), the following  holds for all \(\bv\) satisfying \(\left\|\bv - \bv^*\right\|_2 \leq 2 \tau\):
    \begin{equation}\label{eq:g_2_convergence_rate_1}
\begin{aligned}
    \bar{G}_{T_0}^\MNL(\bv) - G_{T_0}^\MNL(\bv)
    \leq & \frac{\exp(\bar{x}) T_0}{T} \left(2 + \frac{4\log(T)}{\Lambda \exp(\bar{x}) T_0} + \sqrt{\frac{4\log(T)}{\Lambda \exp(\bar{x}) T_0}}\right) + 4 \tau ((\DimFeature + 1)\log T +  \DimFeature \log(6\Lambda\tau))\\
    &+4\tau \sqrt{T_0 \Lambda \exp(\bar{x}) ((\DimFeature + 1)\log T + \DimFeature\log(6\tau \Lambda))}.
\end{aligned}
    \end{equation}
    \item With probability at least \(1 - 2 T^{-1}\), the following holds for all \( t \in \{T_0+1, \ldots, T\}\) and all  \(\bv\) satisfying \(\left\|\bv - \bv^*\right\|_2 \leq 2 \tau\):
    \begin{equation}\label{eq:g_2_convergence_rate_2}
        \begin{aligned}
        \bar{G}_{t}^\MNL(\bv) - G_{t}^\MNL(\bv)\leq & \exp(\bar{x}) + \sqrt{\frac{8 \exp(\bar{x}) \log T}{T\Lambda}} + \frac{8\log T}{T\Lambda}  
        + 4 \tau c_4((\DimFeature + 2)\log T +\DimFeature  \log (6 \tau \Lambda)) \\
    &+\sqrt{(\left|G_{t}^\MNL(\bv)\right| + 2\exp(\bar{x})) 4c_5((\DimFeature  + 2)\log T + \DimFeature  \log (6 \tau \Lambda))}.
    \end{aligned}
    \end{equation}
\end{enumerate}
\end{lemma}

Using the Taylor expansion with Lagrange remainder and similar argument as in \Cref{lm:theta_initial_error}, there exists some \(\bar{\bv} = \alpha \bv^* + (1-\alpha) \widehat{\bv}_{T_0}\) for \(\alpha \in (0,1)\) such that:
\[
G_{T_0}^\MNL (\widehat{\bv}_{T_0}) = -\frac{1}{2}\left(\widehat{\bv}_{T_0} - \bv^*\right)^{\top} I_{T_0}^\MNL\left(\bar{\bv}\right)\left(\widehat{\bv}_{T_0} - \bv^*\right).
\]

Note that 
\begin{align*}
\frac{M_{T_0}^\MNL\left(\widehat{\bv}_{T_0}\right)}{\Lambda \lambda(S_{T_0}, \bp_{T_0};\btheta^*)}  
    &=\sum_{i \in S_{T_0}} \ChoiceProb(i;\widehat{\bv}_{T_0}) \bz_{i} \bz_{i}^{\top}-\sum_{i \in S_{T_0}} \sum_{j \in S_{T_0}}\ChoiceProb(i;\widehat{\bv}_{T_0}) \ChoiceProb(j;\widehat{\bv}_{T_0}) \bz_{i} \bz_{j}^{\top} \\
    &=\sum_{i \in S_{T_0}} \ChoiceProb(i;\widehat{\bv}_{T_0}) \bz_{i} \bz_{i}^{\top}-\frac{1}{2} \sum_{i \in S_{T_0}} \sum_{j \in S_{T_0}} \ChoiceProb(i;\widehat{\bv}_{T_0}) \ChoiceProb(j;\widehat{\bv}_{T_0})\left(\bz_{i} \bz_{j}^{\top}+\bz_{j} \bz_{i}^{\top}\right) \\
    &\succeq \sum_{i \in S_{T_0}} \ChoiceProb(i;\widehat{\bv}_{T_0}) \bz_{i} \bz_{i}^{\top}-\frac{1}{2} \sum_{i \in S_{T_0}} \sum_{j \in S_{T_0}} \ChoiceProb(i;\widehat{\bv}_{T_0})\ChoiceProb(j;\widehat{\bv}_{T_0})\left(\bz_{i} \bz_{i}^{\top}+\bz_{j} \bz_{j}^{\top}\right) \\
    &=\sum_{i \in S_{T_0}} \ChoiceProb(i;\widehat{\bv}_{T_0}) \bz_{i} \bz_{i}^{\top}-\sum_{i \in S_{T_0}} \sum_{j \in S_{T_0}} \ChoiceProb(i;\widehat{\bv}_{T_0}) \ChoiceProb(j;\widehat{\bv}_{T_0})\bz_{i} \bz_{i}^{\top} \\
    &=\sum_{i \in S_{T_0}} \ChoiceProb(i;\widehat{\bv}_{T_0})\left(1-\sum_{j \in S_{T_0}} \ChoiceProb(j;\widehat{\bv}_{T_0})\right) \bz_{i} \bz_{i}^{\top} \\
    &= \sum_{i \in S_{T_0}} \ChoiceProb(i;\widehat{\bv}_{T_0}) \ChoiceProb(0;\widehat{\bv}_{T_0}) \bz_{i} \bz_{i}^{\top}.
\end{align*}

By \Cref{assump:combined}, we have 
\[
\sum_{i \in S_{T_0}} \ChoiceProb(i;\widehat{\bv}_{T_0}) \ChoiceProb(0;\widehat{\bv}_{T_0}) \bz_{i} \bz_{i}^{\top} \ge  \underbrace{\frac{\exp( - \bar{v} - \phigh)}{(K\exp(\bar{v} - \plow) + 1)^2}}_{\kappa}T_0\sigma_0,
\]
which further gives
\(
\sigma_{\min}\left(I_{T_0}^\MNL(\widehat{\bv}_{T_0})\right) \geq \Lambda \exp(-\bar{x}) \kappa T_0\sigma_0
\)
with probability \(1 - T^{-1}\). 

Therefore,
\begin{equation}\label{eq:3}
    -G_{T_0}^\MNL(\widehat{\bv}_{T_0}) \geq \frac{1}{2} \Lambda  \exp(-\bar{x})\kappa T_0\sigma_0\|\widehat{\bv}_{T_0} - \bv^*\|^2.
\end{equation}

Combining \Cref{eq:3}, \Cref{eq:g_2_convergence_rate_1} of \Cref{lm:g_2_convergence_rate} with $\tau =1$ (by $\|\hv_{T_0} -\hv^* \|\le 2$), and the straightforward fact \({G}_{T_0}^\MNL\left(\widehat{\bv}_{T_0}\right) \leq 0 \leq \bar{G}_{T_0}^\MNL\left(\widehat{\bv}_{T_0}\right)\), we have with probability \(1 - 3T^{-1}\) the following holds:
\begin{align*}
    \frac{1}{2}\Lambda  \exp(-\bar{x})\kappa\left\|\widehat{\bv}_{T_0} - \bv^*\right\|^2_2 
    \leq &\frac{\exp(\bar{x}) }{T\sigma_0} \left(2 + \frac{4\log(T)}{\Lambda \exp(\bar{x}) T_0} + \sqrt{\frac{4\log(T)}{\Lambda \exp(\bar{x}) T_0}}\right) + \frac{4 }{T_0 \sigma_0} ((\DimFeature + 1)\log T +  \DimFeature \log(6\Lambda))\\
    &+\frac{4}{T_0 \sigma_0} \sqrt{\sum_{s = 1}^{T_0} \Lambda \lambda(S_s, \bp_s; \btheta^*)  ((\DimFeature + 1)\log T + \DimFeature\log(6 \Lambda))}\\
    \leq & \frac{\exp(\bar{x}) }{T\sigma_0} \left(2 + \frac{4\log(T)}{\Lambda \exp(\bar{x}) T_0} + \sqrt{\frac{4\log(T)}{\Lambda \exp(\bar{x}) T_0}}\right) + \frac{4 }{T_0 \sigma_0} ((\DimFeature + 1)\log T +  \DimFeature \log(6\Lambda))\\
    &+\frac{4}{T_0 \sigma_0} \sqrt{T_0 \Lambda \exp(\bar{x})  ((\DimFeature + 1)\log T + \DimFeature\log(6 \Lambda))}.
\end{align*} \qed

\subsection{Proof of Lemma \ref{lm:g_2_convergence_rate}}
This subsection details the proof of Lemma \ref{lm:g_2_convergence_rate}. Before diving into the details, we first present four lemmas: Lemma \ref{lm:prob_error}, \ref{lm:init_t_bound_g_bar_to_g}, \ref{lm:g_k_bound}, and \ref{lm:v_2_bound}. We first present the proof of Lemma \ref{lm:g_2_convergence_rate} given  \Cref{lm:prob_error,lm:init_t_bound_g_bar_to_g,lm:g_k_bound,lm:v_2_bound} hold. We leave the proofs of \Cref{lm:prob_error,lm:init_t_bound_g_bar_to_g,lm:g_k_bound,lm:v_2_bound} to the end of this subsection.

For a given \(\bv\), we define the following variance terms:
\begin{equation*}
\begin{aligned}
        \mathcal{V}^{MNL}_{ s} :=& \text{Var} \left[ \sum_{i = 1}^{n_s} \sum_{j \in S_s} \bm{1}\{ C^{(i)}_s = j \} \log \frac{q(j; \bv)}{q(j; \bv^*)} - \Lambda \lambda(S_s, \bp_s; \btheta^*) \sum_{j \in S_s} q(j; \bv^*) \log \frac{q(j; \bv)}{q(j; \bv^*)} \Bigg| H_{s} \right],\\
    =&  \text{Var}\left[\bar{g}_{T_0}-g_{T_0}\right],
\end{aligned}
\end{equation*}
and
\(
\mathcal{SV}^{MNL}_{t} := \sum_{s = 1}^{t} \mathcal{V}^{MNL}_{ s},
\)
where \(q(j; \bv)\) is shorthand for \(\Prob(C^{(i)}_s = j | S_s, \bp_s, \bv, \Feature_s)\).

We need the following lemmas:

\begin{lemma}\label{lm:prob_error}
Suppose \( \|\bv^{(1)} - \bv^{(2)}\|_2 \leq 2\tau \). Then for any \( j \in S_t \), the following inequality holds:
\[
\exp\{-4\tau\} \leq \frac{\Prob(C^{(i)}_{s} = j \mid S_{s}, \bp_{s}, \bv^{(1)}, \Feature_s)}{\Prob(C^{(i)}_{s} = j \mid S_{s}, \bp_{s}, \bv^{(2)}, \Feature_s)} \leq \exp\{4\tau\}.
\]
\end{lemma}

\begin{lemma}\label{lm:init_t_bound_g_bar_to_g}
Suppose for any \( j \in S_t \)
\[
\exp\{-4\tau\} \leq \frac{\Prob(C^{(i)}_{s} = j \mid S_{s}, \bp_{s}, \bv, \Feature_s)}{\Prob(C^{(i)}_{s} = j \mid S_{s}, \bp_{s}, \bv^*, \Feature_s)} \leq \exp\{4\tau\}.
\]
We have the tail bound for \(\bar{G}_{T_0}^\MNL(\bv) - {G}_{T_0}^\MNL(\bv)\) as follows:
\begin{equation}\label{ineq:init_t_bound_g_bar_to_g}
    \Prob  (\bar{G}_{T_0}^\MNL(\bv) - {G}_{T_0}^\MNL(\bv) \geq y) \leq \exp\left(\frac{-y^2}{2\tau  [y + 4\tau \sum_{s = 1}^{T_0}\Lambda  \lambda(S_s, \bp_s, \btheta^*)]} \right).
\end{equation}
\end{lemma}

\begin{lemma}\label{lm:g_k_bound}
For all \(k > 2\), we have the following inequality:
\begin{equation}\label{ineq:g_g_bar_k}
    \mathbb{E} \big|g_{s}^{MNL}(\bv) - \bar{g}_{s}^{MNL}(\bv)\big|^k \leq \frac{k!}{2} \mathcal{V}_{s}^{MNL} (2 \tau c_4)^{k - 2}.
\end{equation}
Here, \(\bv \in \mathbb{R}^d\) satisfies \(\|\bv - \bv^*\|_2 \leq 2\tau\), and \(c_4\) is defined in \Cref{eq:df_c_2}.
\end{lemma}

\begin{lemma}\label{lm:v_2_bound}
If \(\|\bv - \bv^*\| \leq 2\tau_\bv\), then 
\(
    \mathcal{SV}_{t}^{MNL} \leq   c_5
 \big|G_{t}^{MNL}(\bv)\big|,
\)
where $c_5$ is defined in \Cref{eq:df_c_4} as
\(
    c_5 = \frac{16\tau_\bv^2}{4\tau_\bv+\exp(-4\tau_\bv)  - 1}.
\)
\end{lemma}

\textbf{Proof of the First Part of Lemma  \ref{lm:g_2_convergence_rate}}

Using \Cref{lm:prob_error}, we know the condition in \Cref{lm:init_t_bound_g_bar_to_g} holds, thus we have an inequality (\ref{ineq:init_t_bound_g_bar_to_g}). Solving \(y\) such that the right hand side of inequality (\ref{ineq:init_t_bound_g_bar_to_g}) is upper bounded by \(\delta\), we have the following holds for each \(\bv \in \{\bv | \left\|\bv - \bv^*\right\|_2 \leq 2 \tau\}\):
\begin{equation}
           \Prob  \left(\bar{G}_{T_0}^\MNL(\bv) - {G}_{T_0}^\MNL(\bv) \geq 4 \tau \sqrt{\sum_{s = t}^t \Lambda \lambda(S_s, \bp_s; \btheta^*)  \log \frac{1}{\delta}} + 4 \tau \log \frac{1}{\delta}\right) \leq \delta.
\end{equation}

Similar to the proof of \Cref{lm:g_1_convergence_rate}, define a finite covering \(\mathcal{H}_2(\epsilon)\) of \(\left\{\bv \in \mathbb{R}^\DimFeature : \left\|\bv - \bv^*\right\|_2 \leq 2 \tau \right\}\). That is, for \(\forall \bv \in \{\bv: \|\bv - \bv^*\|_2 \leq 2 \tau\}\), there exists \(\bv' \in \mathcal{H}_2(\epsilon)\) such that \(\|\bv - \bv'\|_2 \leq \epsilon\). By standard covering number arguments in \cite{geer2000empirical}, such a finite covering set \(\mathcal{H}_2(\epsilon)\) exists whose size can be upper bounded by \(|\mathcal{H}_2(\epsilon)| \leq ( (6\tau / \epsilon))^{\DimFeature} \).
Setting \(\delta = \frac{1}{T}\left(\frac{\epsilon}{6\tau}\right)^\DimFeature\), we have with probability \(1 - T^{-1}\) that:
\begin{equation}\label{ineq:t_bound_g_bar_to_g_bv}
\bar{G}_{T_0}^\MNL(\bv) - {G}_{T_0}^\MNL(\bv) \leq 4 \tau \sqrt{\sum_{s = 1}^t \Lambda \lambda(S_s, \bp_s; \btheta^*)  (\log T + \DimFeature \log(6\tau / \epsilon))} + 4 \tau (\log T +  \DimFeature \log(6\tau / \epsilon)), \forall \bv \in \mathcal{H}(\epsilon).
\end{equation}
Given the inequality holds for a finite set $\bv \in \mathcal{H}_2(\epsilon)$, now we can extend to the general case, that $\bar{G}_{T_0}^{\mathrm{MNL}}\left(\bv^{\prime}\right)-G_{T_0}^{\mathrm{MNL}}\left(\bv^{\prime}\right)$ holds for all $\bv^{\prime}$ and well bounded from above with a probability $1-T^{-1}$ (see the final part, the inequality (\ref{eq:g_2_convergence_rate_1}), of the proof).

Note that, for \(\bv'\), we can find a \(\bv \in \mathcal{H}_2(\epsilon) \) such that \(||\bv' - \bv|| \leq \epsilon\), we have the following inequality
\begin{equation}\label{ineq:bound_g_bar_to_g_bv}
    \bar{G}_{T_0}^\MNL(\bv') - {G}_{T_0}^\MNL(\bv') \leq \bar{G}_{T_0}^\MNL(\bv) - {G}_{T_0}^\MNL(\bv) + 
    \bigg|\bar{G}_{T_0}^\MNL(\bv) -\bar{G}_{T_0}^\MNL(\bv')\bigg| + \bigg|G_{T_0}^\MNL(\bv) -G_{T_0}^\MNL(\bv')\bigg|.
\end{equation}

Now to bound \( \bar{G}_{T_0}^\MNL(\bv') - {G}_{T_0}^\MNL(\bv')\) for all \(\bv\), we only need to bound all terms on the right hand side of the inequality (\ref{ineq:bound_g_bar_to_g_bv}).
\begin{itemize}
    \item \textbf{For the first part of the inequality (\ref{ineq:bound_g_bar_to_g_bv})}, we directly use the inequality (\ref{ineq:t_bound_g_bar_to_g_bv}) to bound it.

    \item \textbf{For the second part of the inequality (\ref{ineq:bound_g_bar_to_g_bv})}, 
We note that for arbitrary \(\|\bv' - \bv\|_2 \leq \epsilon\):
\begin{align*}
    \bigg|\bar{G}_{T_0}^\MNL(\bv) - \bar{G}_{T_0}^\MNL(\bv')\bigg| &\leq \sum_{s=1}^{T_0} \sum_{i=1}^{n_{s}} \sum_{j \in S_{s}}
\bm{1}\{C^{(i)}_{s} = j\} \left|\log \frac{q(j, S_s, \bp_s, \bv)}{q(j, S_s, \bp_s, \bv')}\right| \\
&\leq \sum_{s=1}^{T_0} \sum_{i=1}^{n_{s}} \sum_{j \in S_{s}}
\bm{1}\{C^{(i)}_{s} = j\} 2\|\bv' - \bv\|_2 
\leq  2\epsilon \sum_{s=1}^{T_0} n_s.
\end{align*}

We use the same concentration bound as in \Cref{eq:concerntration_bd_Poi}:
\begin{equation*}
\begin{aligned}
    \Prob \left(\sum_{s=1}^{T_0} {n}_s \geq 2\Lambda \exp(\bar{x}) T_0 \left(1 + \frac{4\log(T)}{\Lambda \exp(\bar{x}) T_0} + \sqrt{\frac{4\log(T)}{\Lambda \exp(\bar{x}) T_0}}\right)\right) \leq \frac{1}{T}.
\end{aligned}
\end{equation*}

Therefore, we have with probability \(1 - T^{-1}\) that
\begin{equation}\label{eq:g_bar_error_uniform_2}
     \bigg|\bar{G}_{T_0}^\MNL(\bv) - \bar{G}_{T_0}^\MNL(\bv')\bigg| \leq 2\epsilon \Lambda \exp(\bar{x}) T_0 \left(1 + \frac{4\log(T)}{\Lambda \exp(\bar{x}) T_0} + \sqrt{\frac{4\log(T)}{\Lambda \exp(\bar{x}) T_0}}\right).
\end{equation}

\item \textbf{For the third part of the inequality (\ref{ineq:bound_g_bar_to_g_bv})}, it is easy to calculate that 
\begin{equation}\label{eq:g_error_uniform_2}
     \bigg|G_{T_0}^\MNL(\bv) - G_{T_0}^\MNL(\bv')\bigg| \leq 2\epsilon \Lambda \exp(\bar{x}) T_0.
\end{equation}
\end{itemize}

Note that the upper bound of the first, second, and third part (see  (\ref{ineq:t_bound_g_bar_to_g_bv}),(\ref{eq:g_bar_error_uniform_2}), and (\ref{eq:g_error_uniform_2})) include \(\epsilon\), we carefully set \(\epsilon = 1/(T \Lambda)\), and combining inequalities (\ref{ineq:t_bound_g_bar_to_g_bv}),(\ref{eq:g_bar_error_uniform_2}) and (\ref{eq:g_error_uniform_2}), we have with probability $1-2T^{-1}$ that
\begin{align*}
       \bar{G}_{T_0}^\MNL(\bv') - G_{T_0}^\MNL(\bv')\leq & \bar{G}_{T_0}^\MNL(\bv) - \bar{G}_{T_0}^\MNL(\bv') + \left|\bar{G}_{T_0}^\MNL(\bv) - G_{T_0}^\MNL(\bv)\right| + \left|G_{T_0}^\MNL(\bv) - G_{T_0}^\MNL(\bv')\right| \\
    \leq & \frac{\exp(\bar{x}) T_0}{T} \left(4 + \frac{8\log(T)}{\Lambda \exp(\bar{x}) T_0} + \sqrt{\frac{16\log(T)}{\Lambda \exp(\bar{x}) T_0}}\right) \\
    &+4\tau \sqrt{\sum_{s = 1}^{T_0} \Lambda \lambda(S_s, \bp_s; \btheta^*)  ((\DimFeature + 1)\log T + \DimFeature\log(6\tau \Lambda))}\\
    &+ 8 \tau ((\DimFeature + 1)\log T +  \DimFeature \log(6\Lambda\tau)).
\end{align*}

\qed

\textbf{Proof of the Second Part of Lemma \ref{lm:g_2_convergence_rate}.}

Using \Cref{lm:g_k_bound} and noticing that the condition of Theorem 1.2A in \cite{de1999general} holds, we know that, for all \(x, y > 0\),
\begin{equation}\label{eq:G_2_convergence_1}
    \Prob\left(\bar{G}_{t}^\MNL(\bv)-G_{t}^\MNL(\bv) \geq x, \mathcal{SV}_{t}^\MNL \leq y\right) \leq \exp \left\{-\frac{x^2}{2(y+ 2 \tau c_4 x)}\right\}.
\end{equation}

Using \Cref{lm:v_2_bound}, set \(y = c_5 \left|G_{t}^\MNL(\bv)\right|\), and \Cref{eq:G_2_convergence_1} becomes:
\begin{equation*}
    \Prob\left(\bar{G}_{t}^\MNL(\bv)-G_{t}^\MNL(\bv) \geq x\right) \leq \exp \left\{-\frac{x^2}{2(c_5 \left|G_{t}^\MNL(\bv)\right|+ 2 \tau c_4 x)}\right\}.
\end{equation*}

Setting the right-hand side to be \(\delta\), we have the following inequality for each \(\bv \):
\begin{equation*}
    \Prob\left(|\bar{G}_{t}^\MNL(\bv)-G_{t}^\MNL(\bv)| \geq \sqrt{4 c_5 \left|G_{t}^\MNL(\bv)\right| \log\frac{1}{\delta}} + 8\tau c_4 \log\frac{1}{\delta}\right) \leq \delta.
\end{equation*}

Similar to the proof of \Cref{lm:g_1_convergence_rate}, define a finite covering \(\mathcal{H}_2(\epsilon)\) of \(\left\{\bv \in \mathbb{R}^N:\left\|\bv-\bv^*\right\|_2 \leq 2 \tau\right\}\). That is, for \(\forall \bv \in \{\bv: \|\bv - \bv^*\|_2 \leq 2 \tau\}\), there exists \(\bv' \in \mathcal{H}_2(\epsilon)\) such that \(\|\bv - \bv'\|_2 \leq \epsilon\). By standard covering number arguments in \cite{geer2000empirical}, such a finite covering set \(\mathcal{H}_2(\epsilon)\) exists whose size can be upper bounded by \(|\mathcal{H}_2(\epsilon)| \leq ( 6\tau / \epsilon)^{\DimFeature} \).
For each \(\bv \in \mathcal{H}_2(\epsilon)\), the set satisfying the probability condition in the inequality above has probability \(\delta\). Thus, we have $\forall \bv \in \mathcal{H}_2(\epsilon)$,
\begin{equation}\label{ineq:g_2_convergence_rate}
    \Prob\left(\bar{G}_{t}^\MNL(\bv)-G_{t}^\MNL(\bv) \geq \sqrt{4 c_5 \left|G_{t}^\MNL(\bv)\right| \log\frac{1}{\delta}} + 8\tau c_4 \log\frac{1}{\delta} \right) \leq \delta |\mathcal{H}_2(\epsilon) |.
\end{equation}
Setting \(\delta = \frac{1}{T^2}\left(\frac{ \epsilon}{6\tau }\right)^\DimFeature\) in \Cref{ineq:g_2_convergence_rate}, with probability $1-T^{-1}$, \(\forall T_0<t \leq T, \bv \in \mathcal{H}_2(\epsilon)\), we have 
\begin{equation}\label{eq:bound_g_bar_to_g_2_v}
\bar{G}_{t}^\MNL(\bv)-G_{t}^\MNL(\bv) \leq 8\tau c_4(2 \log T+\DimFeature    \log (6\tau / \epsilon))+\sqrt{\left|G_{t}^\MNL(\bv)\right| 4c_5(2\log T+ \DimFeature \log (6 \tau / \epsilon))}.
\end{equation}

Given the inequality holds for a finite set $\bv \in \mathcal{H}_2(\epsilon)$, now we can extend to the general case, that $\bar{G}_{t}^{\MNL}\left(\bv^{\prime}\right)-G_{t}^{\MNL}\left(\bv^{\prime}\right)$ holds for all $\bv^{\prime}$ and well bounded from above with a probability $1-2 T^{-1}$ (see the final part, the inequality (\ref{eq:g_2_convergence_rate_2}), of the proof).

Note that, for \(\bv'\), we can find a \(\bv \in \mathcal{H}_2(\epsilon) \) such that \(||\bv' - \bv|| \leq \epsilon\), we have the following inequality
\begin{equation}\label{ineq:bound_g_bar_to_g_bv_t}
    \bar{G}_{t}^\MNL(\bv') - {G}_{t}^\MNL(\bv') \leq \bar{G}_{t}^\MNL(\bv) - {G}_{t}^\MNL(\bv) + 
    \bigg|\bar{G}_{t}^\MNL(\bv) -\bar{G}_{t}^\MNL(\bv')\bigg| + \bigg|{G}_{t}^\MNL(\bv) -G_{t}^\MNL(\bv')\bigg|
\end{equation}

Now to bound \( \bar{G}_{t}^\MNL(\bv') - {G}_{t}^\MNL(\bv')\) for all \(\bv\), we only need to bound all terms on the right hand side of the inequality (\ref{ineq:bound_g_bar_to_g_bv_t}).
\begin{itemize}
    \item \textbf{For the first part of the inequality (\ref{ineq:bound_g_bar_to_g_bv_t})}, we directly use the inequality (\ref{eq:bound_g_bar_to_g_2_v}) to bound it.
    \item \textbf{For the second part of the inequality (\ref{ineq:bound_g_bar_to_g_bv_t})}, we have:
    \begin{align}
            \bigg|\bar{G}_{t}^\MNL(\bv) - \bar{G}_{t}^\MNL(\bv')\bigg| 
    \leq&  \sum_{s = 1}^{t} 2 \epsilon n_s.
    \end{align}
    Combining with inequality (\ref{eq:bound_n_t_2}), we have 
    \begin{equation}\label{eq:g_bar_bv_to_bv'}
        \bigg|\bar{G}_{t}^\MNL(\bv) - \bar{G}_{t}^\MNL(\bv')\bigg| \leq 2\epsilon \Lambda \exp(\bar{x}) t \left( 1 + \frac{8 \log(T)}{ \Lambda \exp(\bar{x}) t} + \sqrt{\frac{8 \log(T)}{ \Lambda \exp(\bar{x}) t}}\right) , \forall T_0 < t \leq T,
    \end{equation}
    holds with probability \(1 - T^{-1}\).    
    \item \textbf{For the third part of the inequality (\ref{ineq:bound_g_bar_to_g_bv_t})}, we have:
    \begin{align}
        \bigg|G_{t}^\MNL(\bv) - G_{t}^\MNL(\bv')\bigg| \leq & 2 \epsilon \Lambda \exp(\bar{x}) t. \label{eq:g_bv_to_bv'}
    \end{align}
\end{itemize}

Similarly, Note that the upper bound of the first, second, and third part (see  (\ref{eq:bound_g_bar_to_g_2_v}), (\ref{eq:g_bar_bv_to_bv'}) and (\ref{eq:g_bv_to_bv'}) include \(\epsilon\), we carefully set  \(\epsilon = \frac{1}{T\Lambda}\), and combining inequality (\ref{eq:bound_g_bar_to_g_2_v}), (\ref{eq:g_bar_bv_to_bv'}) and (\ref{eq:g_bv_to_bv'}), With probability at least \(1 - 2 T^{-1}\), the following holds for all \( t \in \{T_0+1, \ldots, T\}\) and all  \(\bv\) satisfying \(\left\|\bv - \bv^*\right\|_2 \leq 2 \tau\):
\begin{equation*}
    \begin{aligned}
        \bar{G}_{t}^\MNL(\bv')-G_{t}^\MNL(\bv') \leq &\frac{4\exp(\bar{x}) t}{T} + \frac{1}{T\Lambda}\sqrt{32t\Lambda \exp(\bar{x}) \log T} + \frac{16\log T}{T\Lambda}  
        +8 \tau c_4((\DimFeature  + 2)\log T + \DimFeature \log (6 \tau \Lambda)) \\
         &+\sqrt{\big(\left|G_{t}^\MNL(\bv')\right|+2\exp(\bar{x})\big) 4c_5((\DimFeature + 2)\log T + \DimFeature \log (6 \tau \Lambda))}.
    \end{aligned}
\end{equation*}

Thus, we have:
\begin{equation*}
    \begin{aligned}
       \bar{G}_{t}^\MNL(\bv')-G_{t}^\MNL(\bv') \leq & 4\exp(\bar{x}) + \sqrt{\frac{32 \exp(\bar{x}) \log T}{T\Lambda}} + \frac{16\log T}{T\Lambda}
       +8 \tau c_4((\DimFeature + 2)\log T + \DimFeature \log (6 \tau \Lambda)) \\
         &+\sqrt{\big(\left|G_{t}^\MNL(\bv')\right|+2\exp(\bar{x})\big) 4c_5((\DimFeature + 2)\log T +\DimFeature\log (6 \tau \Lambda))}.
    \end{aligned}
\end{equation*}

\qed

\subsection{Proof of Lemma \ref{lm:prob_error}}
For simplicity, denote \(\Prob(C^{(i)}_{s} = j | S_{s}, \bp_{s}, \bv^{(1)}, \Feature_s)\) as \(\ChoiceProb(j;\bv^{(1)})\). We begin by considering the derivative of \( \ChoiceProb(j;\bv^{(2)} + \alpha(\bv^{(1)} - \bv^{(2)})) \) with respect to \( \alpha \). By applying the chain rule, we have:
\begin{align*}
    \frac{\mathrm{d} \ChoiceProb(j;\bv^{(2)} + \alpha(\bv^{(1)} - \bv^{(2)}))}{\mathrm{d} \alpha} 
    &= \left\langle \nabla_{\bv^{(1)}} \ChoiceProb(j;\widetilde{\bv}_\alpha), \bv^{(1)} - \bv^{(2)} \right\rangle 
    = \ChoiceProb(j;\widetilde{\bv}_\alpha) \left\langle \Feature_{jt} - \sum_{k \in S_t} \ChoiceProb(k;\widetilde{\bv}_\alpha) \Feature_{kt}, \bv^{(1)} - \bv^{(2)} \right\rangle
\end{align*}
where \(\widetilde{\bv}_\alpha = \bv^{(2)} + \alpha(\bv^{(1)} - \bv^{(2)})\).

Since \( \|\Feature_{kt}\|_2 \leq 1 \) for all \( k \in [N] \) and \( t \in [T] \), we can bound the derivative as follows:
\[
\frac{\mathrm{d} \ChoiceProb(j;\bv^{(2)} + \alpha(\bv^{(1)} - \bv^{(2)}))}{\mathrm{d} \alpha} \leq \ChoiceProb(j;\widetilde{\bv}_\alpha) 2 \|\bv^{(1)} - \bv^{(2)}\|_2 \leq \ChoiceProb(j;\widetilde{\bv}_\alpha) 4\tau,
\]
for \(\forall \bv^{(1)} \) such that \( ||\bv^{(1)} - \bv^{(2)}||_2 \leq 2\tau \). Now, by the Gr\"{o}nwall's inequality, we obtain:
\[
\ChoiceProb(j;\bv^{(2)} + \alpha(\bv^{(1)} - \bv^{(2)})) \leq \exp\{4\tau \alpha\} \ChoiceProb(j;\bv^{(2)}).
\]

Set \( \alpha = 1 \), we conclude 
\(
\exp\{-4\tau\} \leq \frac{\ChoiceProb(j;\bv^{(1)})}{\ChoiceProb(j;\bv^{(2)})} \leq \exp\{4\tau\}.
\)
Thus, the desired inequality holds.
\qed

\subsection{Proof of Lemma \ref{lm:init_t_bound_g_bar_to_g}}\label{appendix:proof_of_g_bar_to_g}
In this subsection, we present the proof of Lemma \ref{lm:init_t_bound_g_bar_to_g}. To streamline notation, let us define
\[
  \Lambda^* \;:=\; \sum_{s = 1}^{T_0} \Lambda \,\lambda\bigl(S_s, \bp_s, \btheta^*\bigr).
\]

\begin{lemma}\label{lm:expec_g_bar}
For any \(u > 0\), the following equality holds:
{\scriptsize
\begin{equation}\label{eq:expe_g_bar_to_g}
\begin{aligned}
    &\mathbb{E} \bigg[\exp\left[u\left(\bar{G}_{T_0}^\MNL(\bv) - {G}_{T_0}^\MNL(\bv) \right)\right]\bigg]
     =\exp\left(\Lambda^*  \left(\sum_{j \in S_s \cup \{0\}}\frac{q^{u}(j;\bv)}{q^{u - 1}(j;\bv^*)} - 1 - u  \sum_{j \in S_s\cup\{0\}} q(j;\bv^*) \log\left(\frac{q(j; \bv)}{q(j; \bv^*)}\right)\right)\right) .
\end{aligned}
\end{equation}
}
\end{lemma}

Using Lemma \ref{lm:expec_g_bar}, we can bound the probability as follows:
\begin{equation}\label{ineq:g_bar_to_g_t_0}
\begin{aligned}
    &\Prob \bigg[\bar{G}_{T_0}^\MNL(\bv) - {G}_{T_0}^\MNL(\bv) \geq y\bigg] \\
    \leq & \mathbb{E} \bigg[\exp\big[u(\bar{G}_{T_0}^\MNL(\bv) - {G}_{T_0}^\MNL(\bv) - y)\big] \bigg]\\
    \leq & \exp\left(\Lambda^*  \left(\sum_{j \in S_s\cup \{0\}}\frac{q^{u}(j;\bv)}{q^{u-1}(j;\bv^*)} - 1\right) - u \Lambda^*  \sum_{j \in S_s\cup \{0\}}q(j;\bv^*) \log\left(\frac{q(j; \bv)}{q(j; \bv^*)}\right) - uy\right),\\
    = &\exp\Bigg(\Lambda^*  q(j;\bv^*) \underbrace{\left(\sum_{j \in S_s\cup \{0\}}\frac{q^{u}(j;\bv)}{q^{u}(j;\bv^*)} - u \log\left(\frac{q(j; \bv)}{q(j; \bv^*)}\right)\right) }_{h\left(\frac{q(j; \bv)}{q(j; \bv^*)}\right)}-uy - \Lambda^* \Bigg).
\end{aligned}
\end{equation}

We now analyze the function \(h(\xi) = \xi^{u} - u\log(\xi)\). The function \(h(\xi)\) decreases when \(\xi \leq 1\) and increases when \(\xi \geq 1\). From Lemma \ref{lm:prob_error}, we know that:
\[
\exp\{-4\tau\} \leq \frac{\ChoiceProb(j;\bv)}{\ChoiceProb(j;\bv^*)} \leq \exp\{4\tau\}.
\]
Using this bound and the property of \(h(\xi)\), we deduce:
\begin{equation*}
\begin{aligned}
    &\frac{q^{u}(j;\bv)}{q^{u}(j;\bv^*)} - u \log\left(\frac{q(j; \bv)}{q(j; \bv^*)}\right) 
    \leq  \max\bigg\{\exp(4\tau u) -  4\tau u , \exp(-4\tau u) +  4\tau u \bigg\}
    \leq \exp(4\tau u) -  4\tau u.
\end{aligned}
\end{equation*}
This allows us to bound the summation term:
\[
\sum_{j \in S_s \cup \{0\}}\frac{q^{u}(j;\bv)}{q^{u-1}(j;\bv^*)} - u \sum_{j \in S_s\cup \{0\}}q(j;\bv^*) \log\left(\frac{q(j; \bv)}{q(j; \bv^*)}\right) \leq \sum_{j \in S_s\cup \{0\}}q(j;\bv^*) \left[\exp(4\tau u) - 4\tau u\right].
\]

Substituting this bound into inequality \eqref{ineq:g_bar_to_g_t_0}, we obtain:
\begin{equation*}
\begin{aligned}
    \Prob \bigg[\bar{G}_{T_0}^\MNL(\bv) - {G}_{T_0}^\MNL(\bv) \geq y\bigg]
    \leq  \exp\left(\Lambda^* \left(\exp(4\tau u) - 1 - 4 \tau u\right) - uy\right).
\end{aligned}
\end{equation*}

Following the derivation for the tail bound of a Poisson random variable, set
\(
 u^* = \frac{1}{4\tau} \log \left(1 + \dfrac{y}{4\tau \Lambda^* }\right).
\)

Substituting \(u^*\), we have:
\begin{equation*}
\begin{aligned}
    \Prob \bigg[\bar{G}_{T_0}^\MNL(\bv) - {G}_{T_0}^\MNL(\bv) \geq y\bigg]
    = & \exp\left(\Lambda^* \left(\dfrac{y}{4\tau\Lambda^* } - \log \left(1 + \dfrac{y}{4\tau\Lambda^* }\right)\right) -\frac{y}{4\tau} \log \left(1 + \dfrac{y}{4\tau \Lambda^* }\right)\right)\\
    \leq &\exp\left(- \frac{ y^2}{2\tau \left[y + 4\tau   \Lambda^* \right]}\right).
\end{aligned}
\end{equation*}

This completes the proof.
\qed

\subsubsection{Proof of Lemma \ref{lm:expec_g_bar}}
\begin{equation}
\begin{aligned}
    \mathbb{E} \Bigl[\exp\bigl(y\,\bar{G}_{T_0}^\MNL(\bv)\bigr)\Bigr]
    \;=\;&
    \mathbb{E} \Bigl[
      \exp\Bigl(
        y\sum_{s = 1}^{T_0}
          \Bigl(\sum_{i = 1}^{n_s}\sum_{j \in S_s \cup \{0\}}
            \bm{1}\bigl\{C^{(i)}_s = j\bigr\}
            \log\!\Bigl(\frac{q\bigl(j;\bv\bigr)}{q\bigl(j;\bv^*\bigr)}\Bigr)
          \Bigr)
      \Bigr)
    \Bigr]
    \\
    \overset{\text{(a)}}{=}&
    \prod_{s = 1}^{T_0}
    \mathbb{E} \Bigl[
      \exp\Bigl(
        y\Bigl(\sum_{i = 1}^{n_s}\sum_{j \in S_s \cup \{0\}}
          \bm{1}\bigl\{C^{(i)}_s = j\bigr\}
          \log\!\Bigl(\frac{q\bigl(j;\bv\bigr)}{q\bigl(j;\bv^*\bigr)}\Bigr)
        \Bigr)
      \Bigr)
    \Bigr]
    \\
    \overset{\text{(b)}}{=}&
    \prod_{s = 1}^{T_0}
    \sum_{n = 1}^{\infty} \Prob\!\bigl(n_s = n\bigr)\,
    \mathbb{E} \Bigl[
      \exp\Bigl(
        y\Bigl(\sum_{i = 1}^{n_s}\sum_{j \in S_s \cup \{0\}}
          \bm{1}\bigl\{C^{(i)}_s = j\bigr\}
          \log\!\Bigl(\frac{q\bigl(j;\bv\bigr)}{q\bigl(j;\bv^*\bigr)}\Bigr)
        \Bigr)
      \Bigr)
      \Bigm| n_s = n
    \Bigr]
    \\
    \overset{\text{(c)}}{=}&
    \prod_{s = 1}^{T_0}
    \sum_{n = 1}^{\infty} \Prob\!\bigl(n_s = n\bigr)\,
    \prod_{i = 1}^{n}
      \mathbb{E} \Bigl[
        \exp\Bigl(
          y\Bigl(\sum_{j \in S_s \cup \{0\}}
            \bm{1}\bigl\{C^{(i)}_s = j\bigr\}
            \log\!\Bigl(\frac{q\bigl(j;\bv\bigr)}{q\bigl(j;\bv^*\bigr)}\Bigr)
          \Bigr)
        \Bigr)
      \Bigr]
    \\
    \overset{\text{(d)}}{=}&
    \prod_{s = 1}^{T_0}
    \sum_{n = 1}^{\infty}
      \Prob\!\bigl(n_s = n\bigr)\,
      \prod_{i = 1}^{n}
        \sum_{j \in S_s \cup \{0\}}
          q\bigl(j;\bv^*\bigr)\,
          \Bigl(\frac{q\bigl(j;\bv\bigr)}{q\bigl(j;\bv^*\bigr)}\Bigr)^{y}
    \\
    \overset{\text{(e)}}{=}&
    \prod_{s = 1}^{T_0}
    \sum_{n = 1}^{\infty}
      \Prob\!\bigl(n_s = n\bigr)\,
      \Bigl(
        \sum_{j \in S_s \cup \{0\}}
          q\bigl(j;\bv^*\bigr)\,
          \Bigl(\frac{q\bigl(j;\bv\bigr)}{q\bigl(j;\bv^*\bigr)}\Bigr)^{y}
      \Bigr)^n
    \\
    \overset{\text{(f)}}{=}&
    \prod_{s = 1}^{T_0}
      \exp\Bigl(
        \Lambda\,\lambda\bigl(S_s,\bp_s,\btheta^*\bigr)
        \Bigl(
           -\,1
           \;+\;
           \sum_{j \in S_s \cup \{0\}}
             q\bigl(j;\bv^*\bigr)\,
             \Bigl(\frac{q\bigl(j;\bv\bigr)}{q\bigl(j;\bv^*\bigr)}\Bigr)^{y}
        \Bigr)
      \Bigr)
    \\
    =\;&
    \exp\Bigl(
      \sum_{s = 1}^{T_0}
        \Lambda\,\lambda\bigl(S_s,\bp_s,\btheta^*\bigr)\,
        \Bigl(
          \sum_{j \in S_s \cup \{0\}}
            \frac{q^{\,y}(j;\bv)}{\,q^{\,y-1}(j;\bv^*)}
          \;-\;
          1
        \Bigr)
    \Bigr).
\end{aligned}
\end{equation}
where Equality \((a)\) follows from the independence of each period; Equality \((b)\) separates the probability space into events that \(n_s\) takes different values;  
    Equalities \((c), (d), (e)\) simplify the conditional expectations; Equality \((f)\) follows from substituting Poisson distribution into \(n_s\) and then simplifying the equation.
 Multiplying both sides by \( \exp[-y G^{\MNL}_{T_0}(\bv)] \), we can obtain \Cref{eq:expe_g_bar_to_g}.
\qed


\subsection{Proof of Lemma \ref{lm:g_k_bound}}

We first simplify \(\mathcal{V}_{s}^{\MNL}\). The variance term can be written as:
\[
\mathcal{V}_{s}^{\MNL}
=
\Lambda \,\lambda\bigl(S_s , \bp_s; \btheta^*\bigr)
\sum_{j \in S_s \cup \{0\}}
q\bigl(j; \bv^*\bigr)\,\log^2 \!\Bigl(\frac{q\bigl(j; \bv\bigr)}{q\bigl(j; \bv^*\bigr)}\Bigr).
\]
This reduces the variance analysis to the sum of squared logarithmic terms.

Next, we bound the expectation of the error term \(g_{s}^{\MNL}(\bv) - \bar{g}_{s}^{\MNL}(\bv)\). We have:
\[
\mathbb{E}\!\Bigl[\bigl|g_{s}^{\MNL}(\bv) - \bar{g}_{s}^{\MNL}(\bv)\bigr|^k\Bigr]
\;\le\;
2^{\,k - 1}\,\mathbb{E}\!\Bigl[\bigl|\bar{g}_{s}^{\MNL}(\bv)\bigr|^k\Bigr].
\]
Expanding \(\bar{g}_{s}^{\MNL}(\bv)\), we get:
\begin{equation*}
\begin{aligned}
    2^{k - 1}\mathbb{E}\!\Bigl[\bigl(\bar{g}_{s}^{\MNL}(\bv)\bigr)^k\Bigr]
    \;=\;&
    2^{k - 1}\mathbb{E}\!\Bigl[
      \sum_{i = 1}^{n_s} 
      \sum_{j \in S_s \cup \{0\}}
        \bm{1}\{C^{(i)}_s = j\}\,\log\!\Bigl(\frac{q(j;\bv)}{q(j;\bv^*)}\Bigr)
    \Bigr]^k
    \\
    =\;&
    2^{\,k - 1}
    \sum_{n = 1}^{\infty} 
      \mathbb{P}\bigl[n_s = n\bigr]\,
      \mathbb{E}\!\Bigl[
        \Bigl[
          \sum_{i = 1}^{n_s}
          \sum_{j \in S_s \cup \{0\}}
            \bm{1}\{C^{(i)}_s = j\}\,\log\!\Bigl(\frac{q(j;\bv)}{q(j;\bv^*)}\Bigr)
        \Bigr]^k
        \,\Bigm|\,
        n_s = n
      \Bigr]
    \\
    \overset{(a)}{\le}\;&
    2^{\,k - 1}
    \sum_{n = 1}^{\infty} 
      \mathbb{P}\bigl[n_s = n\bigr]\,
      \mathbb{E}\!\Bigl[
        n^{\,k - 1}
        \sum_{i = 1}^{n}
          \biggl(
            \sum_{j \in S_s \cup \{0\}}
              \bm{1}\{C^{(i)}_s = j\}\,\log\!\Bigl(\frac{q(j;\bv)}{q(j;\bv^*)}\Bigr)
          \biggr)^k
      \Bigr]
    \\
    =\;&
    2^{\,k - 1}\,\mathbb{E}\!\Bigl[
      n_s^k\,\biggl(
        \sum_{j \in S_s \cup \{0\}}
          \bm{1}\{C^{(1)}_s = j\}\,\log\!\Bigl(\frac{q(j;\bv)}{q(j;\bv^*)}\Bigr)
      \biggr)^k
    \Bigr]
    \\
    =\;&
    2^{\,k - 1}
    \sum_{j \in S_s \cup \{0\}}
      q\bigl(j;\bv^*\bigr)\,
      \biggl[\log\!\Bigl(\frac{q(j;\bv)}{q(j;\bv^*)}\Bigr)\biggr]^k
    \,\mathbb{E}\bigl[n_s^k\bigr],
\end{aligned}
\end{equation*}
where Equation \((a)\) uses H\"older’s inequality for \(k > 2\). Combining the last two inequalities, we obtain:
\begin{equation*}
    \mathbb{E}\!\Bigl[\bigl|g_{s}^{\MNL}(\bv) - \bar{g}_{s}^{\MNL}(\bv)\bigr|^k\Bigr]
    \;\le\;
    2^{\,k - 1}
    \sum_{j \in S_s \cup \{0\}}
      q\bigl(j;\bv^*\bigr)\,
      \biggl[\log\!\Bigl(\frac{q(j;\bv)}{q(j;\bv^*)}\Bigr)\biggr]^k
    \,\mathbb{E}\bigl[n_s^k\bigr].
\end{equation*}

Finally, using \Cref{lm:prob_error}, we obtain:
\[
\mathbb{E}\!\Bigl[\bigl|g_{s}^{\MNL}(\bv) - \bar{g}_{s}^{\MNL}(\bv)\bigr|^k\Bigr]
\;\le\;
2^{\,k - 1}\,4^{\,k-2}\,\tau^{\,k-2}
\sum_{j \in S_s \cup \{0\}}
  q\bigl(j;\bv^*\bigr)\,
  \log^2\!\Bigl(\frac{q(j;\bv)}{q(j;\bv^*)}\Bigr)
\,\mathbb{E}\bigl[n_s^k\bigr].
\]
Using the similar argument in \Cref{eq:ineq2_mod}, we have
\(
\mathbb{E}\!\Bigl[\bigl|g_{s}^{\MNL}(\bv) - \bar{g}_{s}^{\MNL}(\bv)\bigr|^k\Bigr]
\;\le\;
\dfrac{k!}{2}\,\mathcal{V}_{s}^{\MNL}\,\bigl(4\,\tau\,c_4\bigr)^{\,k - 2}.
\)
\qed

\subsection{Proof of Lemma \ref{lm:v_2_bound}}\label{appendix:proof_of_v_2_bound}

We simplify the expression for the probability \(\Prob(C^{(i)}_{s} = j | S_{s}, \bp_{s}, \bv, \Feature_s)\) as \(\ChoiceProb(j;\bv)\). Thus, the expression for \(\mathcal{V}_{s}^\MNL{}\) becomes:
\begin{equation}
    \mathcal{V}_{s}^\MNL{} = \Lambda \lambda(S_s, \bp_s; \btheta^*) \sum_{j \in S_{s} \cup\{0\}}
    \ChoiceProb(j;\bv^*) \log^2 \frac{\ChoiceProb(j;\bv)}{\ChoiceProb(j;\bv^*)}.
\end{equation}

Note that
\(
    \sum_{j \in S_{s} \cup\{0\}} \ChoiceProb(j;\bv^*)\left(\frac{\ChoiceProb(j;\bv)}{\ChoiceProb(j;\bv^*)} -1\right) = 0.
\) We have:
\begin{equation}
    \begin{aligned}        
        -g_{s}^\MNL(\bv) & = \Lambda \lambda(S_s, \bp_s; \btheta^*) \sum_{j \in S_{s} \cup\{0\}} \ChoiceProb(j;\bv^*) \left(-\log\left(\frac{\ChoiceProb(j;\bv)}{\ChoiceProb(j;\bv^*)}\right)+\frac{\ChoiceProb(j;\bv)}{\ChoiceProb(j;\bv^*)} -1\right).
    \end{aligned}
\end{equation}

Thus, to prove this lemma, we only need to show that the inequality holds for each component:
\begin{equation}\label{ineq:log_bound}
    -\log\left(\frac{\ChoiceProb(j;\bv)}{\ChoiceProb(j;\bv^*)}\right) + \frac{\ChoiceProb(j;\bv)}{\ChoiceProb(j;\bv^*)} - 1 
    \geq \frac{4\tau_\bv + \exp(-4\tau_\bv) - 1}{16\tau_\bv^2} \log^2 \frac{\ChoiceProb(j;\bv)}{\ChoiceProb(j;\bv^*)},
\end{equation}
which can be proved using \Cref{lm:prob_error}:
\[
\exp\{-4\tau_\bv\} \leq \frac{\ChoiceProb(j;\bv)}{\ChoiceProb(j;\bv^*)} \leq \exp\{4\tau_\bv\}.
\]
\qed 

\begin{remark}\label{remark:improve}
    Our proof simplifies the approach presented in \cite{chen2020dynamic} and improves upon their result. Specifically, we refine the bound on \(\mathcal{SV}_{t}^{\MNL}\) from a second-order bound:
    \begin{equation*}
        \mathcal{SV}_{t}^{\MNL} \leq  2\left(1+\max _j\left(\frac{\left|\ChoiceProb(j;\bv)-\ChoiceProb(j;\bv^*)\right|}{\ChoiceProb(j;\bv^*)}\right)\right)^2 \big|G_{t}^{\MNL}(\bv)\big|,
    \end{equation*}
    to a logarithmic-order bound:
    \begin{equation*}
        \mathcal{SV}_{t}^{\MNL} \leq  \frac{\log^2 \min _j\left(\frac{\ChoiceProb(j;\bv)}{\ChoiceProb(j;\bv^*)}\right)}{-\log \min _j\left(\frac{\ChoiceProb(j;\bv)}{\ChoiceProb(j;\bv^*)}\right) + \min _j\left(\frac{\ChoiceProb(j;\bv)}{\ChoiceProb(j;\bv^*)}\right) - 1} \big|G_{t}^{\MNL}(\bv)\big|.
    \end{equation*}
    This enhancement demonstrates a significant improvement in the asymptotic characterization of the sensitivity of the value function.
\end{remark}

\section{Proof of Lemma \ref{lm:bv_error}}\label{appendix::proof_of_bv_error}
This section provides a proof for the Lemma \ref{lm:bv_error}. Using a similar argument as in the proof of \Cref{lm:theta_initial_error}, we see that there exists \(\bar{\bv}_t = \xi \bv^* + (1-\xi) \widehat{\bv}_t\) for some \(\xi \in (0,1)\) such that
\begin{equation}\label{eq:G_2_equal_theta_I_theta}
G_{t}^\MNL (\widehat{\bv}_t) = -\frac{1}{2}\left(\widehat{\bv}_t - \bv^*\right)^{\top} I_{t}^\MNL\left(\bar{\bv}_t\right)\left(\widehat{\bv}_t - \bv^*\right).
\end{equation}

If \(\bar{\bv}_t\) is close to \(\bv^*\), then \(I_{t}^\MNL\left(\bar{\bv}_t\right)\) is also close to \(I_{t}^\MNL\left(\bv^*\right)\) due to continuity. Before diving into the details for the proof, we first present a lemma, which will be proved in the next subsection, to facilitate the proof.

\begin{lemma}\label{lm:I_v_close_to_I_v_star}
For \(\forall \bv^{(1)},\bv^{(2)} \) such that \(\|\bv^{(2)} - \bv^{(1)}\| \leq 2 \tau_\bv\),  
\begin{equation}
    c_8 I_{t}^\MNL\left(\bv^{(2)}\right) \succeq I_{t}^\MNL\left(\bv^{(1)}\right)
\end{equation}
where \(c_8\) is defined in \Cref{eq:df_c_8}. 
. Additionally, the same result holds when replacing \(I_{t}^\MNL\) with \(M_{t}^\MNL\):
\begin{equation}
    c_8 M_{t}^\MNL\left(\bv^{(2)}\right) \succeq M_{t}^\MNL\left(\bv^{(1)}\right).
\end{equation}
\end{lemma}

Using \Cref{lm:I_v_close_to_I_v_star}, we can see
\begin{equation}\label{eq:G_2_leq_theta_I_theta}
-G_{t}^\MNL (\bv) \geq \frac{1}{2c_8} \left(\widehat{\bv}_t - \bv^*\right)^{\top} I_{t}^\MNL\left(\bv^*\right)\left(\widehat{\bv}_t - \bv^*\right).
\end{equation}

Using \Cref{eq:g_2_convergence_rate_2} of \Cref{lm:g_2_convergence_rate} and the fact that \(G_{t}^\MNL\left(\widehat{\bv}_t\right) \leq 0 \leq \bar{G}_{t}^\MNL\left(\widehat{\bv}_t\right)\), we have:
\begin{multline*}
        \left|G_{t}^\MNL\left(\widehat{\bv}_t\right)\right| \leq  \bar{G}_{t}^\MNL\left(\widehat{\bv}_t\right) - G_{t}^\MNL\left(\widehat{\bv}_t\right) 
        \leq  \exp(\bar{x}) + \sqrt{\frac{8\exp(\bar{x}) \log T}{T\Lambda}} + \frac{8\log T}{T\Lambda} 
        + 8 \tau_\bv c_4((\DimFeature + 2)\log T + \DimFeature  \log (6 \tau_\bv \Lambda)) \\
         +\sqrt{\big(\left|G_{t}^\MNL(\widehat{\bv}_t)\right| + 2\exp(\bar{x})\big) 4c_5((\DimFeature + 2)\log T + \DimFeature \log (6 \tau_\bv \Lambda))}
\end{multline*}
for all \( T_0 \leq t \leq T\) with probability \(1- 2 T^{-1}\).
Using \Cref{lm:ineq}, we have 
\begin{equation*}
    \begin{aligned}
        \left|G_{t}^\MNL\left(\widehat{\bv}_t\right)\right| \leq  4\exp(\bar{x}) + 2\sqrt{\frac{8\exp(\bar{x}) \log T}{T\Lambda}} + \frac{16\log T}{T\Lambda}
         + (16 \tau_\bv c_4 + 4c_5)((\DimFeature + 2)\log T + \DimFeature \log (6 \tau_\bv \Lambda)).
    \end{aligned}
\end{equation*}

Notice that \(G_{t}^\MNL\left(\widehat{\bv}_t\right) \leq 0\). Combining \Cref{eq:G_2_leq_theta_I_theta} and the last inequality, we have with probability \(1- 2 T^{-1}\),
\begin{multline}
        \left(\widehat{\bv}_t - \bv^*\right)^{\top} I_{t}^\MNL\left(\bv^*\right)\left(\widehat{\bv}_t - \bv^*\right) 
        \leq 2c_8\left|G_{t}^\MNL\left(\widehat{\bv}_t\right)\right| 
        \leq  8c_8\exp(\bar{x}) + 4c_8\sqrt{\frac{8\exp(\bar{x}) \log T}{T\Lambda}} \\ + \frac{32c_8\log T}{T\Lambda} 
         + 8(4 \tau_\bv c_4 + c_5)c_8((\DimFeature + 2)\log T + \DimFeature \log (6 \tau_\bv \Lambda)).
\end{multline}

Similarly, we have 
\begin{multline}
        \left(\widehat{\bv}_t - \bv^*\right)^{\top} I_{t}^\MNL\left(\widehat{\bv}_t\right)\left(\widehat{\bv}_t - \bv^*\right) 
        \leq 2c_8\left|G_{t}^\MNL\left(\widehat{\bv}_t\right)\right| 
        \leq  8c_8\exp(\bar{x}) + 4c_8\sqrt{\frac{8\exp(\bar{x}) \log T}{T\Lambda}} + \frac{32c_8\log T}{T\Lambda} \\
         + 8(4 \tau_\bv c_4 + c_5)c_8((\DimFeature + 2)\log T + \DimFeature \log (6 \tau_\bv \Lambda)).
\end{multline}
\qed

\subsection{Proof of Lemma \ref{lm:I_v_close_to_I_v_star}}

To prove this lemma, we only need to show
\begin{equation*}
    \big[3(\exp(4\tau_\bv ) - 1)(K \exp(\bar{v} - p_l) + 1) + 1\big]M_{s}^\MNL\left(\bv^{(2)}\right) \succeq  M_{s}^\MNL\left(\bv^{(1)}\right),
\end{equation*}
for all \(s\), which is equivalent to 
\begin{equation*}
    M_{s}^\MNL\left(\bv^{(1)}\right) - M_{s}^\MNL\left(\bv^{(2)}\right) \preceq 3(\exp(4\tau_\bv ) - 1)(K \exp(\bar{v} - p_l) + 1)   M_{s}^\MNL\left(\bv^{(2)}\right).
\end{equation*}
Here, for simplicity, we will write \(\ChoiceProb(C^{(i)}_{s} = j \mid S_{s}, \bp_{s}, \bv, \Feature_s)\) as \(\ChoiceProb(j;\bv)\), \(\Feature_{is}\) as \(\bz_i\), and \(\ChoiceProb(i;\bv^{(2)}) - \ChoiceProb(i;\bv^{(1)})\) as \(\delta \ChoiceProb(i;\bv)\).

Note that
\[
\left(\bz_i-\bz_j\right)\left(\bz_i-\bz_j\right)^{\top}=\bz_i \bz_i^{\top}+\bz_j \bz_j^{\top}-\bz_i \bz_j^{\top}-\bz_j \bz_i^{\top} \succeq 0,
\]
which implies \(\bz_i \bz_i^{\top}+\bz_j \bz_j^{\top} \succeq \bz_i \bz_j^{\top}+\bz_j \bz_i^{\top}\). Therefore, let's express the matrix \(M_{s}^\MNL(\bv^{(2)})\) in terms of the arrival rate and probabilities:

\begin{align*}
\frac{M_{s}^\MNL\left(\bv^{(2)}\right)}{\Lambda \lambda(S_s, \bp_s;\btheta^*)}  
    &=\sum_{i \in S_t} \ChoiceProb(i;\bv^{(2)}) \bz_{i} \bz_{i}^{\top}-\sum_{i \in S_t} \sum_{j \in S_t}\ChoiceProb(i;\bv^{(2)}) \ChoiceProb(j;\bv^{(2)}) \bz_{i} \bz_{j}^{\top} \\
    &=\sum_{i \in S_t} \ChoiceProb(i;\bv^{(2)}) \bz_{i} \bz_{i}^{\top}-\frac{1}{2} \sum_{i \in S_t} \sum_{j \in S_t} \ChoiceProb(i;\bv^{(2)}) \ChoiceProb(j;\bv^{(2)})\left(\bz_{i} \bz_{j}^{\top}+\bz_{j} \bz_{i}^{\top}\right) \\
    &\succeq \sum_{i \in S_t} \ChoiceProb(i;\bv^{(2)}) \bz_{i} \bz_{i}^{\top}-\frac{1}{2} \sum_{i \in S_t} \sum_{j \in S_t} \ChoiceProb(i;\bv^{(2)})\ChoiceProb(j;\bv^{(2)})\left(\bz_{i} \bz_{i}^{\top}+\bz_{j} \bz_{j}^{\top}\right) \\
    &=\sum_{i \in S_t} \ChoiceProb(i;\bv^{(2)}) \bz_{i} \bz_{i}^{\top}-\sum_{i \in S_t} \sum_{j \in S_t} \ChoiceProb(i;\bv^{(2)}) \ChoiceProb(j;\bv^{(2)})\bz_{i} \bz_{i}^{\top} \\
    &=\sum_{i \in S_t} \ChoiceProb(i;\bv^{(2)})\left(1-\sum_{j \in S_t} \ChoiceProb(j;\bv^{(2)})\right) \bz_{i} \bz_{i}^{\top} \\
    &= \sum_{i \in S_t} \ChoiceProb(i;\bv^{(2)}) \ChoiceProb(0;\bv^{(2)}) \bz_{i} \bz_{i}^{\top}.
\end{align*}

Next, we consider the difference between \(M_{s}^\MNL({\bv^{(1)}})\) and \(M_{s}^\MNL(\bv^{(2)})\):
\begin{align*}
    \frac{M_{s}^\MNL\left({\bv^{(1)}}\right) - M_{s}^\MNL\left(\bv^{(2)}\right)}{\Lambda \lambda(S_s, \bp_s;\btheta^*)}
    &= -\sum_{i \in S_t} \delta \ChoiceProb(i;\bv) \bz_{i} \bz_{i}^{\top} + \sum_{i \in S_t} \sum_{j \in S_t}\big(\ChoiceProb(i;\bv^{(1)}) \ChoiceProb(j;\bv^{(1)}) -  \ChoiceProb(i;\bv^{(2)}) \ChoiceProb(j;\bv^{(2)})\big) \bz_{i} \bz_{j}^{\top} \\
    \preceq &-\sum_{i \in S_t} \delta \ChoiceProb(i;\bv) \bz_{i} \bz_{i}^{\top} + \sum_{i \in S_t} \sum_{j \in S_t}\big|\ChoiceProb(i;\bv^{(1)}) \ChoiceProb(j;\bv^{(1)}) -  \ChoiceProb(i;\bv^{(2)}) \ChoiceProb(j;\bv^{(2)})\big| \bz_{i} \bz_{i}^{\top}.
\end{align*}

Finally, by comparing the coefficients of \(\bz_i \bz_i^{\top}\), we only need to prove:
\begin{equation*}
    \sum_{j \in S_t}\big|\ChoiceProb(i;\bv^{(1)}) \ChoiceProb(j;\bv^{(1)}) -  \ChoiceProb(i;\bv^{(2)}) \ChoiceProb(j;\bv^{(2)})\big| - \delta \ChoiceProb(i;\bv) \leq 3(\exp(4\tau_\bv ) - 1)(K \exp(\bar{v} - p_l) + 1)  \ChoiceProb(i;\bv^{(2)}) \ChoiceProb(0;\bv^{(2)}),
\end{equation*}
which can be derived as follows:
\begin{align*}
& \sum_{j \in S_t}
\big|
    \ChoiceProb(i;\bv^{(1)}) \ChoiceProb(j;\bv^{(1)}) -  \ChoiceProb(i;\bv^{(2)}) \ChoiceProb(j;\bv^{(2)})
\big|
- \delta \ChoiceProb(i;\bv^{(1)}) \\
\leq & \sum_{j \in S_t}
\left[
    \ChoiceProb(i;\bv^{(2)}) \big|  \delta \ChoiceProb(j;\bv)\big| +  \ChoiceProb(j;\bv^{(1)})\big|\delta \ChoiceProb(i;\bv)\big|
\right]
- \delta \ChoiceProb(i;\bv)  \\
\leq & \ChoiceProb(i;\bv^{(2)}) 
\left[\sum_{j \in S_t}\big| \delta \ChoiceProb(j;\bv)\big|\right] 
+ \big|\delta \ChoiceProb(i;\bv)\big|- \delta \ChoiceProb(i;\bv) \\
\leq & \ChoiceProb(i;\bv^{(2)})  
\left[\sum_{j \in S_t}\big| \delta \ChoiceProb(j;\bv)\big|\right]
+ 2\big|\delta \ChoiceProb(i;\bv)\big|\\
\leq &  \ChoiceProb(i;\bv^{(2)})
\left[\sum_{j \in S_t} (\exp(4\tau_\bv ) - 1) \ChoiceProb(j;\bv^{(2)})\right]
+ 2(\exp(4\tau_\bv ) - 1) \ChoiceProb(i;\bv^{(2)}) \\
\leq & 3(\exp(4\tau_\bv ) - 1) \ChoiceProb(i;\bv^{(2)})\\
\leq & 3(\exp(4\tau_\bv ) - 1)\ChoiceProb(i;\bv^{(2)}) (K \exp(\bar{v} - p_l) + 1)\ChoiceProb(0;\bv^{(2)}).
\end{align*}

Here, we see that the left-hand side of the inequality, involving the probabilities under \(\bv^{(1)}\) and \(\bv^{(2)}\), is bounded by the right-hand side, confirming the required condition.
\qed



\section{Proof of Lemma \ref{lemma:revenue_estimate_2}}\label{appendix::proof_of_revenue_estimate_2}
Now, we turn to prove the Lemma \ref{lemma:revenue_estimate_2} to bound the regret for each period and \Cref{lm:sum_bv_error} to bound the cumulative regret across all periods.
Similar to the proof of \Cref{lemma:revenue_estimate}, unless stated otherwise, all statements are conditioned on the success event in \Cref{lm:bv_error}.
On this event, the following inequalities hold uniformly for all \(t \in \{T_0,\ldots,T-1\}\):
\begin{equation}
    \begin{aligned}
              \left(\widehat{\bv}_t-\bv^*\right)^{\top} I_{t}^\MNL\left(\bv^*\right)\left(\widehat{\bv}_t-\bv^*\right)  \leq \omega_\bv,\quad
              \left(\widehat{\bv}_t-\bv^*\right)^{\top} I_{t}^\MNL\left(\widehat{\bv}_t\right)\left(\widehat{\bv}_t-\bv^*\right)  \leq \omega_\bv, \quad \|\widehat{\bv}_t - \bv^*\| \leq 2\tau_\bv,
    \end{aligned}
\end{equation}

we first present a Lemma \ref{lm:bv_estimate_real_value} to facilitate the proof.
\begin{lemma}\label{lm:bv_estimate_real_value}
    Suppose \Cref{eq:v_error_I} holds uniformly over all $t = T_0, \ldots, T-1$. For all $t>T_0$ and $S \subseteq[N],|S| \leq K$, it holds that      
\begin{equation}
     \begin{aligned}
        \left|r(S, \bp, \Featuret; \widehat{\bv}_{t - 1}) - r(S, \bp, \Featuret;\bv^*)\right|   
        \leq  & 
        \sqrt{c_0\omega_\bv \left\|I_{t-1}^\MNL\left(\widehat{\bv}_{t-1}\right)^{-1 / 2} M_{t}^\MNL\left(\widehat{\bv}_{t-1} \mid S, \bp\right) I_{t-1}^\MNL\left(\widehat{\bv}_{t-1}\right)^{-1 / 2}\right\|_{\mathrm{op}}},\\
        \left|
            r(S, \bp, \Featuret; \widehat{\bv}_{t - 1}) - r(S, \bp, \Featuret;\bv^*)
        \right| 
        \leq &
        \sqrt{ c_0\omega_\bv
        \left\|
            I_{t-1}^\MNL
            \left(
                \bv^*\right)^{-1 / 2} M_{t}^\MNL\left(\bv^* \mid S, \bp
            \right)
            I_{t-1}^\MNL
            \left(\bv^*\right)^{-1 / 2}
        \right\|_{\mathrm{op}}}.
     \end{aligned}
\end{equation}
\end{lemma}
The proof of the lemma is left to the next subsection \ref{lm:bv_estimate_real_value}.

Note $M_{t}^\MNL(\bv^*|S, \bp) \preceq \widehat{M}_{t}^\MNL(\bv| S , \bp)$ and $ \widehat{I}_{t}^\MNL(\bv^*) \preceq I_{t}^\MNL(\bv^*)$, from \Cref{lm:bv_estimate_real_value}, we see
\begin{align*}
    \left|
        r(S, \bp, \Featuret; \widehat{\bv}_{t - 1})-r(S, \bp, \Featuret;\bv^*)
    \right|  
    \leq  
    \sqrt{
         c_0 \omega_\bv
        \left\|
            \widehat{I}^{-1/2}_{t - 1}{}^\MNL(\widehat{\bv}_{t-1})
            \widehat{M}_{t}^\MNL(\widehat{\bv}_{t-1}|S, \bp)
            \widehat{I}^{-1/2}_{t - 1}{}^\MNL(\widehat{\bv}_{t-1})
        \right\|_{\mathrm{op}}
    }.
\end{align*}
Similarly, we have 
\begin{equation*}
    \begin{aligned}
         \left|
            r(S, \bp, \Featuret; \widehat{\bv}_{t - 1})-r(S, \bp, \Featuret;\bv^*)
         \right|     
         \leq 
         \sqrt{c_0 \omega_\bv\left\|\widehat{I}_{t-1}^{-1/2}{}^\MNL\left(\bv^*\right)\widehat{M}_{t}^\MNL\left(\bv^* \mid S, \bp\right) \widehat{I}_{t-1}^{-1/2}{}^\MNL\left(\bv^*\right)\right\|_{\mathrm{op}}}.
    \end{aligned}
\end{equation*}
Using \Cref{lm:I_v_close_to_I_v_star}, we have \(
        \widehat{M}_{t}^\MNL(\widehat{\bv}_{t-1}|S, \bp) \preceq   
        c_8 \widehat{M}_{t}^\MNL\left(\bv^* \mid S, \bp\right)\) and \(\widehat{I}_{t-1}{}^\MNL\left(\bv^*\right) \preceq  c_8 \widehat{I}{}^\MNL_{t - 1}(\widehat{\bv}_{t-1})\). Therefore,
\begin{equation*}
    \left\|
        \widehat{I}^{-\frac{1}{2}}{}^\MNL_{t - 1}(\widehat{\bv}_{t-1})
        \widehat{M}_{t}^\MNL(\widehat{\bv}_{t-1}|S, \bp)
        \widehat{I}^{-1/2}_{t - 1}{}^\MNL(\widehat{\bv}_{t-1})
    \right\|_{\mathrm{op}} 
    \leq
    c^2_8 
    \left\|
        \widehat{I}_{t-1}^{-1/2}{}^\MNL\left(\bv^*\right)
        \widehat{M}_{t}^\MNL\left(\bv^* \mid S, \bp\right) 
        \widehat{I}_{t-1}^{-1/2}{}^\MNL\left(\bv^*\right)
    \right\|_{\mathrm{op}}
\end{equation*}

Thus, we proved the \Cref{lemma:revenue_estimate_2}.
\qed

\subsection{Proof of Lemma \ref{lm:bv_estimate_real_value}}

Note that,
\begin{equation}\label{eq:nabla_bv}
    \begin{aligned}
        \nabla_{\bv} r({S, \bp, \Featuret ; \bv}) = &\sum_{j \in S_t} \ChoiceProb(j;\bv) p_{j t} \Feature_{j t} - \sum_{j \in S_t} \ChoiceProb(j;\bv) p_{j t} \sum_{j \in S_t} \ChoiceProb(j;\bv) \Feature_{j t} \\
        = & \sum_{j \in S_t \cup \{0\} } \ChoiceProb(j;\bv)(p_{j t} - \plow)\left(\Feature_{j t} - \sum_{k \in S_t} \ChoiceProb(k;\bv) \Feature_{k t}\right) .
    \end{aligned}
\end{equation}
For simplicity, we denote 
\begin{equation*}
    \begin{aligned}
        \bar{\Feature}_{j t} = & \left(\Feature_{j t} - \sum_{k \in S_t} \ChoiceProb(k;\bv) \Feature_{k t}\right),\quad
        a_j = p_{j t} - \plow \geq 0.
    \end{aligned}
\end{equation*}

We have 
\begin{equation}\label{eq:partial_r}
    \begin{aligned}
        \nabla r({S, \bp, \Featuret; \bv}) \nabla r({S, \bp, \Featuret;\bv})^{\top} 
        =& \left[\sum_{j \in S_t \cup\{0\}} \ChoiceProb(j;\bv)a_j\bar{\Feature}_{j t} \right]
        \left[\sum_{j \in S_t \cup\{0\}} \ChoiceProb(j;\bv)a_j\bar{\Feature}_{j t} \right]^\top\\
        =&  \sum_{j \in S_t \cup\{0\}, k\in S_t \cup\{0\}} \ChoiceProb(j;\bv)\ChoiceProb(k;\bv)a_j\bar{\Feature}_{j t} a_k\bar{\Feature}^{\top}_{k t}\\
        \preceq & \frac{1}{2}\sum_{j \in S_t \cup\{0\}, k\in S_t \cup\{0\}} \ChoiceProb(j;\bv)\ChoiceProb(k;\bv)
        \left(a^2_j\bar{\Feature}_{j t}\bar{\Feature}_{j t}^\top + a^2_k\bar{\Feature}_{k t}\bar{\Feature}_{k t}^\top\right)\\
        =& \sum_{j \in S_t \cup\{0\}} \ChoiceProb(j;\bv)(1 - \ChoiceProb(0;\bv))
        a^2_j\bar{\Feature}_{j t}\bar{\Feature}_{j t}^\top .
    \end{aligned}
\end{equation}

From the definition of \(M_{t}^\MNL(\bv|S, \bp)\), we have 
\begin{equation}\label{ineq:partial_r_4}
    \begin{aligned}
       \sum_{j \in S_t \cup \{0\}} \ChoiceProb(j;\bv)\bar{\Feature}_{j t}\bar{\Feature}_{j t}^\top
        =&\frac{
         M_{t}^\MNL(\bv|S, \bp) }
        {\Lambda \lambda(S_s, \bp_s; \btheta^*)} .
    \end{aligned}
\end{equation}
Combining \Cref{eq:partial_r,ineq:partial_r_4}, we have

\begin{multline}
        \nabla r({S, \bp, \Featuret; \bv}) \nabla r({S, \bp, \Featuret;\bv})^{\top} 
        \preceq  \sum_{j \in S_t} \ChoiceProb(j;\bv)(1 - \ChoiceProb(0;\bv))
        a^2_j\bar{\Feature}_{j t}\bar{\Feature}_{j t}^\top
        \preceq   \sum_{j \in S_t} \ChoiceProb(j;\bv)
        (\phigh - \plow)^2\bar{\Feature}_{j t}\bar{\Feature}_{j t}^\top \\
        \preceq 
        \frac{
         (\phigh - \plow)^2 }
        {\Lambda \lambda(S_s, \bp_s; \btheta^*)}  M_{t}^\MNL(\bv|S, \bp)
        \preceq 
        \sqrt{
            \frac{(\phigh - \plow)^2}
            {\Lambda \exp(-\bar{v})}  
        }
        M_{t}^\MNL(\bv|S, \bp).
\end{multline}
By the inequality above and the mean value
theorem, there exists \(\widetilde{\bv}_{t-1}=\bv^*+\xi\left(\widehat{\bv}_{t-1}-\bv^*\right)\) for some \(\xi \in(0,1)\) such that
\begin{equation}\label{eq:bv_deta}
  \begin{aligned}
        \left|
        r({S, \bp , \Featuret; \widehat{\bv}_{t-1}}) - r({S, \bp, \Featuret; \bv^*}) \right| 
         =&
        \left|
        \left\langle\nabla r({S, \bp, \Featuret; \widetilde{\bv}_{t-1}}), \widehat{\bv}_{t-1}-\bv^*\right\rangle
        \right| \\
         =&\sqrt{\left( \widehat{\bv}_{t-1}-\bv^*\right)^{\top}\left[\nabla r({S, \bp, \Featuret; \widetilde{\bv}_{t-1}}) \nabla r({S, \bp, \Featuret;\widetilde{\bv}_{t-1}})^{\top}\right]\left(\widehat{\bv}_{t-1}-\bv^*\right)} \\
        \leq &
        \sqrt{
            \frac{(\phigh - \plow)^2}
            {\Lambda \exp(-\bar{v})}
        }
        \sqrt{\left(\widehat{\bv}_{t-1}-\bv^*\right)^{\top}M_{t}^\MNL(\widetilde{\bv}_{t-1}|S, \bp)\left(\widehat{\bv}_{t-1}-\bv^*\right)}.
    \end{aligned}
\end{equation}
Combining inequality (\ref{eq:bv_deta}), \Cref{lm:bv_error}, and \Cref{lm:I_v_close_to_I_v_star}, we can see that 
\begin{align*}
    &\left|
        r({S_t, \bp, \Featuret; \bv_t})-r({S, \bp, \Featuret;\bv^*})
        \right| \\
    \leq&\sqrt{ 
            \frac{(\phigh - \plow)^2 }
            {\Lambda \exp(-\bar{v})}  
         } \times  \\
    &\sqrt{
        \left(
            \widehat{\bv}_{t-1}-\bv^*
        \right)^{\top}
        I_{t-1}^{\MNL \, \frac{1}{2}}
        \left(
            \widehat{\bv}_{t-1}
        \right) 
        I_{t-1}^{\MNL \, -\frac{1}{2}}
        \left(
            \widehat{\bv}_{t-1}
        \right)
        M_{t}^\MNL(\widetilde{\bv}_{t-1}|S, \bp)
        I_{t-1}^{\MNL \, -\frac{1}{2}}
        \left(
            \widehat{\bv}_{t-1}
        \right)
        I_{t-1}^{\MNL \, \frac{1}{2}}
        \left(
            \widehat{\bv}_{t-1}
        \right)
        \left(\widehat{\bv}_{t-1}-\bv^*\right)
    } \\
    \leq &
        \sqrt{
        \frac{
         (\phigh - \plow)^2 }
        {\Lambda \exp(-\bar{v})}  
        }
    \sqrt{
        \omega_\bv
        \left\|
        I_{t-1}^{\MNL \, -\frac{1}{2}}\left(\widehat{\bv}_{t-1}\right)
        M_{t}^\MNL\left(\widetilde{\bv}_{t-1} \mid S, \bp\right) 
        I_{t-1}^{\MNL \, -\frac{1}{2}}\left(\widehat{\bv}_{t-1}\right)^{-1 / 2}\right\|_{\mathrm{op}}
        }\\
    \leq & 
        \sqrt{
        \frac{
         (\phigh - \plow)^2}
        {\Lambda \exp(-\bar{v})}  
        }
    \sqrt{c_8 \omega_\bv 
    \left\|
        I_{t-1}^{\MNL \, -\frac{1}{2}}
        \left(\widehat{\bv}_{t-1}\right) M_{t}^\MNL\left(\widehat{\bv}_{t-1} \mid S, \bp\right) I_{t-1}^{\MNL \, -\frac{1}{2}}\left(\widehat{\bv}_{t-1}\right)
    \right\|_{\mathrm{op}}},
\end{align*}
and
\begin{align*}
    &\left|r({S_t, \bp, \Featuret; \bv_t})-r({S, \bp, \Featuret;\bv^*})\right| 
    \leq  
        \sqrt{
        \frac{
         (\phigh - \plow)^2 }
        {\Lambda \exp(-\bar{v})}  
        }
    \sqrt{
        c_8 \omega_\bv 
        \left\|
            I_{t-1}^{\MNL \, -\frac{1}{2}}\left(\bv^*\right)
            M_{t}^\MNL\left(\bv^* \mid S, \bp\right)
        I_{t-1}^{\MNL \, -\frac{1}{2}}
        \left(\bv^*\right)
        \right\|_{\mathrm{op}}
    }.
\end{align*}

\section{Proof of  Lemma \ref{lm:sum_bv_error}}\label{appendix::proof_of_sum_bv_error}

We define
\begin{align} 
\widetilde{M}_t^\MNL(\bv|S, \bp) &:=
\begin{multlined}[t]
    \Lambda \Bigg( \sum_{j \in S} \ChoiceProb(j, S, \bp, \Feature_t; \bv) \Feature_{jt} \Feature^\top_{jt} 
    - \sum_{j, k\in S} \ChoiceProb(j, S, \bp, \Feature_t; \bv)\ChoiceProb(k, S, \bp,\Feature_t; \bv) \Feature_{jt} \Feature^\top_{kt} \Bigg), 
\end{multlined}\\
\widetilde{I}_{t}^\MNL(\bv) &:= \sum_{s = 1}^t \Lambda
    \left(\sum_{j \in S_s}  \ChoiceProb_s(j;\bv) \Feature_{js} \Feature^\top_{js} 
    - \sum_{j, k\in S_s} \ChoiceProb_s(j;\bv) \ChoiceProb_s(k;\bv) \Feature_{js} \Feature^\top_{ks}  \right).
\end{align}

Similar to the proof of \Cref{lm:sum_lambda_error}, we have:
\begin{equation*}
    \begin{aligned}
        &\sum_{t=T_0+1}^T \min \left\{\bar{p}^2_h, \omega_\bv c_0 \left\|\widehat{I}_{t-1}^{-1 / 2}{}^\MNL\left(\bv^*\right)\widehat{M}_{t}^\MNL\left(\bv^* \mid S_t\right) \widehat{I}_{t-1}^{-1 / 2}{}^\MNL\left(\bv^*\right)\right\|_{\mathrm{op}}\right\}\\
        \leq &\sum_{t=T_0+1}^T \min \left\{\bar{p}^2_h, 
        \exp(2\bar{x})\omega_\bv c_0 
        \left\|\widetilde{I}_{t-1}^{-1 / 2}{}^\MNL\left(\bv^*\right)
        \widetilde{M}_{t}^\MNL\left(\bv^* \mid S_t\right) 
        \widetilde{I}_{t-1}^{-1 / 2}{}^\MNL
        \left(\bv^*\right)
        \right\|_{\mathrm{op}}\right\}
        \leq   c_7 \log \frac{\operatorname{det} \widetilde{I}_{T}^\MNL\left(\bv^*\right)}{\operatorname{det} \widetilde{I}_{T_0}^\MNL\left(\bv^*\right)}.
    \end{aligned}
\end{equation*}

Note that \(\sigma_{\min}\left(\widetilde{I}_{T_0}^\MNL\left(\bv^*\right)\right) \geq  \Lambda T_0 \sigma_0.
\)
This implies 
\(
    \log \left(\operatorname{det} \widetilde{I}_{T_0}^\MNL\left(\bv^*\right)\right) \geq \DimFeature \log(\Lambda T_0 \sigma_0).
\)
On the other hand, we have:
\begin{equation*}
    \begin{aligned}
        &\operatorname{tr}(\widetilde{I}_{T}^\MNL\left(\bv^*\right)) 
        = \sum_{s = 1}^T \Lambda  \sum_{j \in S_{s}} \ChoiceProb(j; \bv) \left( \Feature^{\top}_{j s} - \sum_{ k\in S_{s}}\ChoiceProb(k;  \bv) \Feature^{\top}_{ k s}\right) 
        \left( \Feature_{j s} - \sum_{k\in S_{s}}\ChoiceProb(k;  \bv) \Feature_{ k s}\right)
   \leq  4 T \Lambda .
    \end{aligned}
\end{equation*}
Therefore, we can bound 
\(
        \log \left(\operatorname{det}\widetilde{I}_{T}^\MNL\left(\bv^*\right)\right) \leq \DimFeature \log \frac{ \operatorname{tr}(\widetilde{I}_{T}^\MNL\left(\bv^*\right)) }{\DimFeature} \leq \DimFeature \log \frac{   4 T \Lambda }{\DimFeature}.
\)

As a result, we obtain \(
     \log \frac{\operatorname{det} \widetilde{I}_{T}^\MNL\left(\bv^*\right)}{\operatorname{det} \widetilde{I}_{T_0}^\MNL\left(\bv^*\right)} \leq \DimFeature \log \frac{4 T}{\DimFeature} - \DimFeature \log( T_0 \sigma_0),
\)
which completes the proof of the  inequality.

\section{Proof of Theorem \ref{th:lower_bound} and Theorem \ref{th:lower_bound_asymp}}
\label{appendix::lower_bound}

Note that the regret of policy $\pi$ at time $T$ depends on $\bv,\btheta$ and the data generation mechanism of $\bz_{i,t}$. Therefore, we can denote the regret as $\mathcal{R}^{\pi}(T; \bv, \btheta, \bz) $ to emphasize the dependence on  $ \bv, \btheta, \bz$.

Suppose $P_{\bv, \btheta, \Feature}$ is a distribution over $(\bv, \btheta, \bz)$. Clearly, 
\begin{equation}
    \inf_{\pi}\sup_{\bv,\btheta,\bz}
    \mathcal{R}^{\pi}(T;\bv,\btheta,\bz)
    \geq  \inf_{\pi}\, \E_{\bv ,\Feature \sim P_{\bv, 
    \btheta, \Feature,} } 
    [\mathcal{R}^{\pi}(T;\bv,\btheta,\bz)]
\end{equation}
for any such distribution. When we take a supremum over such distributions we arrive at a lower bound for the minimax optimal rate of the regret.

In the following, we consider the following two cases:
\begin{itemize}
    \item Case I. In Section \ref{sec:case-i}, we set $\btheta = 0$ and consider $P_{\bv, \btheta, \Feature} = P_{\bv, \Feature}$. 
    \item Case II. In Section \ref{sec:case-ii}, we set $\bv = 0$ and consider $P_{\bv, \btheta, \Feature} = P_{\btheta}$. 
\end{itemize}
For both cases, we fix the features over time: $\bz_{i,t} = \bz_i$ for all $i\in [N], t\ge 1$.

\subsection{Case I}
\label{sec:case-i}
When $\btheta = 0$, the problem of interest reduces to the following.

\begin{instance}
\label{inst:const-poisson}
The seller has $N$ distinct products, and  each product $i \in [N]$ is associated with a fixed feature vector $\Feature_i \in \mathbb{R}^d$.  In each period $t\in [T]$, the seller selects an assortment $S_t \subseteq [N]$ with $|S_t| = K$ and sets prices of products in the assortment. For simplicity, we use a price vector 
$\bp_t = (p_{1t}, p_{2t}, \ldots, p_{Nt}) \in [\plow, \phigh]^{N}$ to represent the pricing decision (and set prices corresponding to the out-of-assortment products to be $\phigh$).
 The number of arrived customers for each period follows Poisson distribution with a known positive arrival rate $\Lambda$.
Under the multinomial logit (MNL) model, each arriving customer selects product $i \in S_t$ with probability 
\(
  \frac{\exp\!\bigl(\Feature_i^\top \bv - p_{it}\bigr)}
       {1 \;+\; \sum_{j \in S_t} \exp\!\bigl(\Feature_j^\top \bv - p_{jt}\bigr)},
\)
and makes no purchase with the remaining probability. Here, $\bv \in \mathbb{R}^d$ is an unknown parameter 
that influences the purchase likelihoods.
\end{instance}

Note that the feature vectors are fixed, thus the optimal strategy is time-invariant. We denote the  optimal strategy as  \((S^*, \bp^*)\) and the strategy seller adopted as \((S_t, \bp_t) = \pi(H_t)\) where \(H_t\) is defined in \Cref{df:H_t}. 

Substituting $\Lambda_t = \Lambda$ into \Cref{eq:regret_equal_form}, we have
 \begin{equation*}
        \mathcal{R}^{\pi}(T; \bv, 0, \Feature)  = \sum_{t = 1}^{T} \Lambda \Big\{ r(S^*, \bp^*) - \E_\pi \big[ r(S_t, \bp_t) \big]\Big\} .
 \end{equation*}

Next, we establish the connection between the regret of our-true-time-horizon problem and the regret of the customer-horizon problem.

Denote $M_t = \sum_{s = 1}^{t} n_s$ as the number of arriving customers in the first $t$ periods. 
The $n$-th customer is uniquely indexed by a pair $(t(n), i(n))$, where $t(n)$ indicates the period in which the customer arrives and $i(n)$ indicates the order of the arrival within period $t(n)$. The assortment-pricing offered to the $n$-th customer is thus \((S_{t(n)}, \bp_{t(n)})\). 
Then the regret can be rewritten as 
 \begin{equation}\label{eq:regret_time_form}
        \mathcal{R}^{\pi}(T; \bv, 0, \Feature)  = \E \left [\sum_{n = 1}^{M_T} \Big\{ r(S^*, \bp^*) - \E_\pi \big[ r(S_{t(n)}, \bp_{t(n)}) \big]\Big\}\right] .
    \end{equation}
 It is straightforward that $M_T$ follows a Poisson distribution with parameter $\Lambda T$.

To relate our period-based problem to the customer-level formulation in \cite{agrawal2019mnl}, let
\(P_{\bv,\Feature}\) be a prior on \((\bv,\Feature)\). Taking expectation with respect to this prior yields
\begin{equation}\label{eq:exp_regret}
\E_{\bv ,\Feature \sim P_{\bv,\Feature} } \mathcal{R}^{\pi}(T; \bv, 0, \Feature) = \E_{\bv,\Feature} \left(\E \left [\sum_{n = 1}^{M_T} \Big\{ r(S^*(\bv,\Feature), \bp^*(\bv,\Feature)) - \E_\pi \big[ r(S_{t(n)}, \bp_{t(n)}) \big]\Big\}\right]\right), 
\end{equation}
where $(S^*(\bv,\Feature), \bp^*(\bv,\Feature))$ denotes the optimal assortment–price pair 
under parameters $(\bv,\Feature)$.

To bound this expectation of regret, we consider the following  dynamic assortment-pricing problem that has similar form but can adjust the assortment and pricing per-customer.

\begin{instance}\footnote{The problem studied in \cite{agrawal2019mnl} is a special case of this instance.}
\label{inst:customer_level}
The seller has $N$ distinct products,  
each product associated with a fixed feature vector $\bz_i \in \mathbb{R}^d$ over $\widetilde{T}$ customer arrivals. 
Using similar notations in \Cref{inst:const-poisson}, 
for each customer $n\in[\widetilde{T}]$, 
the seller selects an assortment $S_n \subseteq [N]$ with $|S_n| = K$ and sets the price vector 
$\bp_n = (p_{1n}, p_{2n}, \ldots, p_{Nn})  \in [\plow, \phigh]^N$. Under the multinomial logit (MNL) model, 
customer $n$ purchases product $i \in S_n$ with probability 
  $\frac{\exp\!\bigl(\Feature_{i}^\top \bv - p_{in}\bigr)}
       {1 \;+\; \sum_{j \in S_n} \exp\!\bigl(\Feature_{j}^\top \bv - p_{jn}\bigr)}$,
and makes no purchase with the remaining probability.

We define the history before customer $n$ as
\begin{equation}\label{df:Hprime_n}
H'_{n}
\;:=\;
\bigl(S_1,\bp_1,j_1,\;
      S_2,\bp_2,j_2,\;
      \ldots,\;
      S_{n-1},\bp_{\,n-1},j_{\,n-1}\bigr),
\end{equation}
with the convention $H'_1=\varnothing$, where $j_i$ denotes the purchase outcome of customer $i$.
A policy at the customer level is a sequence $\pi'=\{\pi'_n\}_{n=1}^{\widetilde{T}}$ of (possibly stochastic) functions
\(
\pi'_n:\ H'_n \mapsto (S_n,\bp_n).
\)
 We define the collection of legitimate policy with $\tilde{T}$ customer arrivals as \(\mathcal{A}'(\widetilde{T})\).



\end{instance}

Next, we define a subset of policies that restricts the changes of assortment-pricing to limited points $(m_1, m_1 + m_2,\cdots, m_1 + \cdots + m_T)$ as 
\begin{align}\label{df:strategy_pi''}
    \mathcal{A}''(\widetilde{T}, m_1, m_2, \ldots, m_T) := \Bigg\{  (\pi_1,\cdots,\pi_{\tilde{T}}) \Big| &  \pi_{j} (H_j') = \pi_{i} (H_i') , \\ & \text{ if exists an } s \text{ such that } \sum_{t=1}^s m_t \le i,j < \sum_{t=1}^{s+1} m_t   \Bigg\},
\end{align} 
where $m_1,m_2,\cdots, m_T,\tilde{T}$ satisfy $\sum_{j=1}^T m_j = \tilde{T}$. 

Clearly, when we know the number of customer arrivals in the problem described in \ref{inst:const-poisson} and condition on these numbers, the original policy can be considered as a policy in this restricted class for the customer-level problem, which is in a smaller set than the entire legitimate policy space.

Then, by taking the conditional expectation with respect to customer arrival ($n_1, \cdots, n_T$) first, we can simplify \Cref{eq:exp_regret} as 
\begin{equation}
\begin{aligned}
    &\E_{\bv ,\Feature \sim P_{\bv,\Feature} } \mathcal{R}^{\pi}(T; \bv, 0, \Feature) \\
    = &  \E \left\{\E \left[   \E_{\bv,\Feature} \left (\sum_{n = 1}^{M_T} \Big\{ r(S^*(\bv,\Feature), \bp^*(\bv,\Feature)) - \E_\pi \big[ r(S_{t(n)}, \bp_{t(n)}) \big]\Big\}\right)  \Bigg| M_T , n_1, n_2, \ldots, n_T\right]\right\} \\
    \ge&  \E \left\{\E \left[  \inf_{\pi''\in  \mathcal{A}''(M_T, n_1, n_2, \ldots, n_T)  } \E_{\bv,\Feature} \left (\sum_{n = 1}^{M_T} \Big\{ r(S^*(\bv,\Feature), \bp^*(\bv,\Feature)) - \E_{\pi''} \big[ r(S_{n}, \bp_{n}) \big]\Big\}\right)  \Bigg| M_T , n_1, n_2, \ldots, n_T\right]\right\} \\
    \ge & \E \left\{\E \left[  \inf_{\pi'\in  \mathcal{A}'(M_T)  } \E_{\bv,\Feature} \left (\sum_{n = 1}^{M_T} \Big\{ r(S^*(\bv,\Feature), \bp^*(\bv,\Feature)) - \E_{\pi'} \big[ r(S_{n}, \bp_{n}) \big]\Big\}\right)  \Bigg| M_T , n_1, n_2, \ldots, n_T\right]\right\} \\
    = & \E \left\{\E \left[  \inf_{\pi'\in  \mathcal{A}'(M_T)  } \E_{\bv,\Feature} \left (\sum_{n = 1}^{M_T} \Big\{ r(S^*(\bv,\Feature), \bp^*(\bv,\Feature)) - \E_{\pi'} \big[ r(S_{n}, \bp_{n}) \big]\Big\}\right)  \Bigg| M_T \right]\right\}
\end{aligned}
\end{equation}

For simplicity, for some $\pi' \in \mathcal{A}' (m) $,  we denote 
\begin{equation*}
    \widehat{\mathcal{R}}^*(m, \pi', \bv, \Feature) = \sum_{n = 1}^{m} \Big\{ r(S^*(\bv,\Feature), \bp^*(\bv,\Feature)) - \E_{\pi'} \big[ r(S_{n}, \bp_{n}) \big]\Big\}.
\end{equation*}

By monotonicity (each summand is nonnegative as \(S^*(\bv,\Feature), \bp^*(\bv,\Feature)\) are optimal assortment-pricing), if \(M_T \geq \lceil \Lambda T\rceil\), we have 
\begin{equation*}
    \widehat{\mathcal{R}}^*(M_T, \pi', \bv, \Feature) \geq \widehat{\mathcal{R}}^*(\lceil \Lambda T\rceil, \pi', \bv, \Feature).
\end{equation*}
Therefore,
\begin{equation}\label{ineq:r_to_r_hat}
    \begin{aligned}
    \inf_{\pi} \sup_{\btheta, \bv,\bz}\mathcal{R}^{\pi}(T; \bv, \btheta, \Feature)  \geq & \inf_{\pi} \sup_{ \bv,\bz}\mathcal{R}^{\pi}(T; \bv, 0, \Feature)\\
    \geq & \E \left\{\E \left[ \inf_{\pi'\in \mathcal{A}'(M_T)}  \E_{\bv,\Feature} \left (\sum_{n = 1}^{M_T} \Big\{ r(S^*(\bv,\Feature), \bp^*(\bv,\Feature)) - \E_{\pi'} \big[ r(S_{n}, \bp_{n}) \big]\Big\}\right)  \bigg| M_T \right]\right\}\\
    = & \E \left\{\E \left[ \inf_{\pi'\in \mathcal{A}'(M_T)}  \E_{\bv,\Feature}  \widehat{\mathcal{R}}^*(M_T, \pi', \bv, \Feature) \bigg| M_T \right]\right\}\\
    \geq & 
    \E \left\{
    \E \left[ 
    \inf_{\pi'\in \mathcal{A}'(M_T)} 
    \E_{\bv,\Feature}
    \widehat{\mathcal{R}}^*(\lceil \Lambda T\rceil, \pi', \bv, \Feature) \mathbf{1}_{\{M_T \geq \lceil \Lambda T\rceil\}} \bigg| M_T \right]\right\}\\
    \geq & 
    \inf_{\pi' \in \mathcal{A}'(\lceil \Lambda T \rceil)} \E_{\bv,\Feature} 
    \widehat{\mathcal{R}}^*(\lceil \Lambda T\rceil, \pi', \bv, \Feature) \Prob[M_T \geq\lceil \Lambda T\rceil]
    \end{aligned}
\end{equation}

\subsubsection{Adversarial construction I.}
Suppose $\min\{\DimFeature - 2, N \}\ge K, \Lambda \geq 1 $.
\begin{instance}(Worst-case instance I)\label{df:worst_case_example}
First, we let \(d, \bar{K}\) be positive integers satisfying
\[
\bar{K} \;=\; \min\!\Bigl\{\Bigl\lfloor\frac{\DimFeature - K + 1}{3}\Bigr\rfloor,\; K\Bigr\},
\qquad
d \;=\; \DimFeature - K + \bar{K}.
\]
Then clearly, we have 
\(
    d\le \DimFeature, \quad d- \DimFeature = K- \bar{K}, \quad  d \ge 4\bar{K} -1 .
\)
Next, we fix \(\epsilon \in \bigl(0,\, \min\{\bar{v}/\sqrt{\DimFeature}, 1\}\bigr]\) and define a set of $\DimFeature$-dimension vectors, $\mathbf{V}$, as follows.

For each subset \(W\subseteq [d]\) with \(|W|=\bar{K}\), define the corresponding parameter vector \(\bv_W\in\mathbb{R}^{\DimFeature}\) as
\[
\bv_W(i) \;=\;
\begin{cases}
\epsilon, & i\in W,\\[2pt]
0, & i\in [d]\setminus W,\\[2pt]
\epsilon, & i\in\{d+1,\ldots,\DimFeature\}.
\end{cases}
\qquad i\in[\DimFeature].
\]
Clearly, $\|\bv_W\|_2 \le \bar{v}$, satisfying the boundedness required in \Cref{as:v&assortment_new}.

Collecting all such vectors gives the parameter set $\mathbf{V}$:
\[
\mathbf{V}\;=\;\{\bv_W: W\subseteq [d],\, |W|=\bar{K}\}\;=\;\{\bv_W: W\in \mathcal{W}_{\bar{K}}\},
\]
where, for simplicity of notation, \(\mathcal{W}_{\bar{K}}\) denotes the class of all subsets of  \([d]\) whose size is $\bar{K}$. 

Let the (time-invariant) product features $\{\Feature_i\}_{1\le i \le N}$ be the standard basis in \(\mathbb{R}^{\DimFeature}\) , i.e.,
the $j$-th element of $\Feature_i$ is defined as
\[
\Feature_i(j) \;=\; \begin{cases} 1, & j=i,\\ 0, & j\ne i . \end{cases}
\]
It is straightforward that $\|\Feature_i\|_2 \le 1$. So this construction of features satisfies the boundedness requirement in \Cref{assumption::y_exist}.
\end{instance}

Clearly, for parameter $\bv_{W}$ and aforementioned features $\{ \Feature_i \}_{1 \le i\le N}$, the optimal assortment $S^*(\bv_{W}, \Feature)$ is $W \cup   \{d + 1, d + 2, \ldots, \DimFeature\} $.

\paragraph{Regret of each customer.}

Next, we derive an explicit lower bound  on the per-customer regret. 

We use \(\mathbb{E}_W\) and \(\mathbb{P}_W\) to denote the expectation and probability measure under the model parameterized by \(\bv_W\) and the 
policy \(\pi^\prime\), respectively. The following lemma establishes a lower bound for 
\( r\bigl(S^*(\bv_{W}, \Feature), \bp^*(\bv_{W}, \Feature)\bigr) - r\bigl(S_n, \bp_n\bigr)\) by comparing \(S^*(\bv_{W})\) with \(S_n\).

\begin{lemma}\label{lm:regret_lower_bound_tilde}
    Suppose \(\epsilon \leq 1\).
    Then, for any customer \(n\),
    \[
      r\bigl(S^*(\bv_{W}, \Feature), \bp^*(\bv_{W}, \Feature)\bigr) - r\bigl(S_n, \bp_n\bigr)
      \;\geq\;         
      c_{10} \epsilon
     \left(\bar{K} - \bigl|S_n \cap W\bigr|\right).
    \]
    where \(c_{10}\) is defined in \Cref{df_c_10}.
\end{lemma}
\begin{remark}\label{remark:c_10}
    When \(\plow = \Omega (\log K)\) and \(\phigh = \Omega (\log K)\), we have 
    \(c_{10} = \Omega (\frac{\log K}{K})\).

    When \(\plow = \Omega (1)\) and \(\phigh = \Omega (1)\), we have 
    \(c_{10} = \Omega (\frac{1}{K^2})\).
\end{remark}
Next, we establish a lower bound on the cumulative regret. Define \(\widetilde{N}_i := \sum_{n=1}^{\lceil \Lambda T \rceil} \mathbf{1}\{i \in S_n\}\). By \Cref{lm:regret_lower_bound_tilde}, it follows that
\begin{equation*}
    \begin{aligned} 
        &\widehat{\mathcal{R}}^*(\ceil{\Lambda T}, \pi', \bv_{W}, \Feature)
        =
        \mathbb{E}_W \sum_{n=1}^{\ceil{\Lambda T}} 
        \Bigl(
            r\bigl(S^*(\bv_{W}, \Feature), \bp^*(\bv_{W}, \Feature)\bigr)- r\bigl(S_n, \bp_n
            \bigr)
        \Bigr) \\
        \geq& \mathbb{E}_W \sum_{n=1}^{\ceil{\Lambda T}}
        c_{10} \epsilon
        \left(
            \bar{K} - \bigl|S_n \cap W\bigr|
        \right)
        = \;        
        c_{10} \epsilon
        \Bigl(\bar{K} \ceil{\Lambda T} - \sum_{i \in W} \mathbb{E}_W[\widetilde{N}_i]\Bigr),
        \quad 
        \forall\,W \in \mathcal{W}_{\bar{K}}.
    \end{aligned}
\end{equation*}

Denote \(\mathcal W_{\bar K}^{(i)} := \{W\in\mathcal W_{\bar K}: i\in W\}\) and
\(\mathcal W_{\bar K-1} := \{W\subseteq[d]: \lvert W\rvert=\bar K-1\}\).
We take the prior \(P_{\bv,\Feature}\) to be the uniform distribution over
\(\{(\bv_W,\Feature):\, W\in\mathcal W_{\bar K}\}\); that is, sample \(W\sim{\rm Unif}(\mathcal W_{\bar K})\) and set \((\bv,\Feature)=(\bv_W,\Feature)\).
Therefore, using the previous Inequality \ref{ineq:r_to_r_hat} on minimax rates
\begin{equation*}
    \begin{aligned}
    &\E_{\bv ,\Feature \sim P_{\bv,\Feature} } \mathcal{R}^{\pi}(T; \bv, 0, \Feature) /  \Prob[M_T \geq \ceil{\Lambda T}] \\
    \geq & \inf_{\pi' \in \mathcal{A}'(\lceil \Lambda T \rceil)}
    \E_{\bv,\Feature}  
    \widehat{\mathcal{R}}^*(\lceil \Lambda T\rceil, \pi', \bv, \Feature) 
    \\
        \geq &
        \inf_{\pi' \in \mathcal{A}'(\lceil \Lambda T \rceil)}
        c_{10} \epsilon
        \biggl(
            \bar{K} \ceil{\Lambda T} - 
            \frac{1}{|\mathcal{W}_{\bar{K}}|} 
            \sum_{W \in \mathcal{W}_{\bar{K}}}\sum_{i\in W}
            \mathbb{E}_W[\widetilde{N}_i]
        \biggr)\\
        = &
        \inf_{\pi' \in \mathcal{A}'(\lceil \Lambda T \rceil)}
        c_{10} \epsilon
        \biggl(
            \bar{K} \ceil{\Lambda T}- 
                \frac{1}{|\mathcal{W}_{\bar{K}}|} 
                \sum_{i=1}^d \sum_{W \in \mathcal{W}_{\bar{K}}^{(i)}} \mathbb{E}_W[\widetilde{N}_i]
        \biggr)\\
        = &
        \inf_{\pi' \in \mathcal{A}'(\lceil \Lambda T \rceil)}
        c_{10} \epsilon
        \biggl(
            \bar{K} \ceil{\Lambda T}-
                \frac{1}{|\mathcal{W}_{\bar{K}}|}
                \sum_{W \in \mathcal{W}_{\bar{K}-1}} \sum_{i \notin W} \mathbb{E}_{W \cup \{i\}}[\widetilde{N}_i]
        \biggr) \\
        \geq&
        \inf_{\pi' \in \mathcal{A}'(\lceil \Lambda T \rceil)}
        c_{10} \epsilon
        \biggl(
            \bar{K} \ceil{\Lambda T}-
                \frac{|\mathcal{W}_{\bar{K}-1}|}{|\mathcal{W}_{\bar{K}}|} 
                \max_{W \in \mathcal{W}_{\bar{K}-1}} 
                \sum_{i \notin W} \mathbb{E}_{W \cup \{i\}}[\widetilde{N}_i]
        \biggr)  
        \\
        =& 
        \inf_{\pi' \in \mathcal{A}'(\lceil \Lambda T \rceil)}
        c_{10} \epsilon
        \biggl(
            \bar{K} \ceil{\Lambda T} -
                \frac{|\mathcal{W}_{\bar{K}-1}|}{|\mathcal{W}_{\bar{K}}|} 
                \max_{W \in \mathcal{W}_{\bar{K}-1}} 
                \sum_{i \notin W}
                \left[
                \mathbb{E}_{W \cup \{i\}}[\widetilde{N}_i]
                - \mathbb{E}_W[\widetilde{N}_i] 
                + \mathbb{E}_W[\widetilde{N}_i] 
                \right]
        \biggr).
    \end{aligned}
\end{equation*}

Since \(\sum_{i \notin W} \mathbb{E}_{W}[\widetilde{N}_i] \leq \sum_{i=1}^d \mathbb{E}_W[\widetilde{N}_i] \leq  \bar{K} \ceil{\Lambda T}\) and \(\frac{|\mathcal{W}_{\bar{K}-1}|}{|\mathcal{W}_{\bar{K}}|} =\frac{\binom{d}{\bar{K} - 1}}{\binom{d}{\bar{K}}} =\frac{\bar{K}}{(d - \bar{K} + 1)} \leq 1/3,\) we obtain
\begin{equation*}
    \begin{aligned}
        \frac{\E_{\bv ,\Feature \sim P_{\bv,\Feature} } \mathcal{R}^{\pi}(T; \bv, 0, \Feature) }{  \Prob[M_T \geq \ceil{\Lambda T}]}
        \geq &
        \inf_{\pi' \in \mathcal{A}'(\lceil \Lambda T \rceil)}
        c_{10} \epsilon
        \Bigl(
            \frac{2}{3}\bar{K} \ceil{\Lambda T} - 
            \frac{\bar{K}}{d - \bar{K} + 1}
            \max_{W \in \mathcal{W}_{\bar{K}-1}}
            \sum_{i \notin W}
            \bigl|
                \mathbb{E}_{W \cup \{i\}}[\widetilde{N}_i] - \mathbb{E}_W[\widetilde{N}_i]
            \bigr|
        \Bigr).
    \end{aligned}
\end{equation*}

\paragraph{Pinsker's inequality.}
Finally, we derive an upper bound on 
\(\bigl|\mathbb{E}_{W \cup\{i\}}\bigl[\widetilde{N}_i\bigr]-\mathbb{E}_W\bigl[\widetilde{N}_i\bigr]\bigr|\) 
for any \(W \in \mathcal{W}_{\bar{K}-1}\):
\begin{equation*}
    \begin{aligned}
        \Bigl|
            \mathbb{E}_W
            \bigl[
                \widetilde{N}_i
            \bigr] 
            - 
            \mathbb{E}_{W \cup \{i\}}
            \bigl[
                \widetilde{N}_i
            \bigr]
        \Bigr| 
        \leq&
        \sum_{j=0}^{\ceil{\Lambda T}} 
        j \,\Bigl|
            \mathbb{P}_W
            \bigl[
                \widetilde{N}_i=j
            \bigr]
            -
            \mathbb{P}_{W \cup \{i\}}
            \bigl[
                \widetilde{N}_i=j
            \bigr]
        \Bigr|\\
        \leq &
        \ceil{\Lambda T} \cdot
        \sum_{j=0}^{\ceil{\Lambda T}}
        \Bigl|
            \mathbb{P}_W
            \bigl[
                \widetilde{N}_i=j
            \bigr]
            -
            \mathbb{P}_{W \cup \{i\}}
            \bigl[
                \widetilde{N}_i=j
            \bigr]
        \Bigr|\\
        \leq& 2 \ceil{\Lambda T} \cdot 
        \bigl\|
            \mathbb{P}_W - \mathbb{P}_{W \cup \{i\}}
        \bigr\|_{\mathrm{TV}} 
        \;\leq\;
        \ceil{\Lambda T} \cdot 
        \sqrt{
            2 \mathrm{KL}(\mathbb{P}_W \,\|\, \mathbb{P}_{W \cup \{i\}})
        }.
    \end{aligned}
\end{equation*}
Here \(\|\mathbb{P}_W - \mathbb{P}_{W \cup \{i\}}\|_{\mathrm{TV}}\) and \(\mathrm{KL}(\mathbb{P}_W \| \mathbb{P}_{W \cup \{i\}})\) denote the total variation distance and the Kullback–Leibler divergence, respectively; the last inequality uses Pinsker’s inequality.

The following lemma provides an upper bound on the Kullback–Leibler divergence.

\begin{lemma}\label{lm:bound_kl_div}
Suppose\(\epsilon \leq 1\). For any \(W \in \mathcal{W}_{\bar{K}-1}\) and \(i \in[d]\),
\[
\mathrm{KL}(\mathbb{P}_W \,\|\, \mathbb{P}_{W \cup \{i\}}) \;\leq\; c_{11} \,\cdot\, \mathbb{E}_W[\widetilde{N}_i] \,\cdot\, \epsilon^2,
\]
where \(c_{11} > 0\) is defined in \Cref{df:c_11}.
\end{lemma}
\begin{remark}\label{remark:c_11}
    When \(\plow = \Omega (\log K)\) and \(\phigh = \Omega (\log K)\), we have 
    \(c_{11} = \Omega(\frac{1}{K})\).

    When \(\plow = \Omega (1)\) and \(\phigh = \Omega (1)\), we have 
    \(c_{11} = \Omega (\frac{1}{K})\).
\end{remark}

Combining \Cref{lm:bound_kl_div} with the  bound above yields
\begin{equation*}
    \begin{aligned}
    \frac{\E_{\bv ,\Feature \sim P_{\bv,\Feature} } \mathcal{R}^{\pi}(T; \bv, 0, \Feature) }{  \Prob[M_T \geq \ceil{\Lambda T}] }
    \;\geq\; &
        \inf_{\pi' \in \mathcal{A}'(\lceil \Lambda T \rceil)}
        c_{10} \epsilon
    \Bigl(
        \frac{2\bar{K} \ceil{\Lambda T}}{3}
        \;-\;
        \frac{\bar{K}\ceil{\Lambda T}}{d - \bar{K} + 1} \sum_{i=1}^d \sqrt{2 c_{11}\,\mathbb{E}_W[\widetilde{N}_i]\,\epsilon^2}
    \Bigr).        
    \end{aligned}
\end{equation*}
Applying the Cauchy-Schwarz inequality, and since \(\sum_{i=1}^d \mathbb{E}_W[\widetilde{N}_i] \leq \bar{K}\,\ceil{\Lambda T}\)
\[
\sum_{i=1}^d 
\sqrt{2c_{11}\,\mathbb{E}_W[\widetilde{N}_i]\,\epsilon^2} 
\;\;\leq\;\;
\sqrt{d}\,\cdot\,\sqrt{\sum_{i=1}^d 2c_{11}\,\mathbb{E}_W[\widetilde{N}_i]\,\epsilon^2}
\;\leq\;  \sqrt{\,2c_{11}\,d\,\bar{K}\,\ceil{\Lambda T}\,\epsilon^2}.
\]
Thus, we obtain
\begin{equation}
\label{eq:case_1_1_final}
\begin{aligned}
    \E_{\bv ,\Feature \sim P_{\bv,\Feature} } \mathcal{R}^{\pi}(T; \bv, 0, \Feature) 
    \;\geq &\;
    \sup_{W \in \mathcal{W}_{\bar{K}}} \left [ \inf_{ \pi' \in \mathcal{A}'(\lceil \Lambda T \rceil)} \widehat{\mathcal{R}}^*(\ceil{\Lambda T}, \pi', \bv_{W}, \Feature)\right] \Prob[M_T \geq \ceil{\Lambda T}]\\
    \;\geq &\; 
        c_{10} \epsilon
    \Bigl(
        \frac{2\bar{K}\,\ceil{\Lambda T}}{3}\;-\;
        \frac{ \bar{K}\ceil{ \Lambda T}}{d - \bar{K} + 1}\,
        \sqrt{
            2 c_{11}\,d\, \bar{K} \, \ceil{\Lambda T}\, \epsilon^2
        }
    \Bigr) \Prob[M_T \geq \ceil{\Lambda T}].
\end{aligned}
\end{equation}

To further simplify Inequality \eqref{eq:case_1_1_final}, we first have the following lower bound for the last probability term:
\begin{lemma}\label{lm:poi_bound}
Let $M_T \sim \mathrm{Poi}(\mu)$ with $\mu=\Lambda T \ge 1$. Then
\[
\mathbb{P}\!\bigl[M_T \ge \lceil \mu\rceil\bigr]
\;\ge\;
\frac{1}{2}\;-\;\frac{1}{\sqrt{2\pi}}
\;>\; 0.1.
\]
\end{lemma}

Next, we will bound the middle term by discussing two scenarios of $\Lambda T$ and setting $\epsilon$ correspondingly.
\begin{enumerate}
    \item If
\[
\Lambda T 
\;\ge\;
\frac{\DimFeature^{2}}{18\,c_{11}\,(\DimFeature - K + 1)}\,
\max\!\left\{1,\ \frac{\DimFeature}{\bar v^{2}}\right\},
\]
set \(\epsilon = \sqrt{\dfrac{(d - \bar{K} + 1)^2}{18\,c_{11}\,d\,\bar{K}\,\lceil \Lambda T \rceil}}\),
then \(\epsilon \in (0,\,\min\{\bar{v}/\sqrt{\DimFeature}, 1\}]\).  
Substituting this choice of \(\epsilon\) into Inequality \eqref{eq:case_1_1_final} yields

\begin{equation*}
\begin{aligned}
    \E_{\bv ,\Feature \sim P_{\bv,\Feature} } \mathcal{R}^{\pi}(T; \bv, 0, \Feature) 
    \geq &
    \frac{c_{10}(d - \bar{K} + 1) \sqrt{\bar{K}}}{90\sqrt{2 c_{11} d}} 
    \sqrt{
        \ceil{\Lambda T}
    } \geq 
    \frac{c_{10} \sqrt{d \bar{K}}}{120\sqrt{2 c_{11}}} 
    \sqrt{
        \ceil{\Lambda T}
    }\\
    \geq &
    \frac{c_{10}}{120\sqrt{2c_{11}}}
\min\!\bigl\{
\frac{\DimFeature - K - 1}{\sqrt{3}},\;
\sqrt{\DimFeature K}
\bigr\}
\sqrt{\lceil \Lambda T \rceil}.
\end{aligned}
\end{equation*}

\item If 
\[
\Lambda T 
\;<\;
\frac{\DimFeature^{2}}{18\,c_{11}\,(\DimFeature - K + 1)}\,
\max\!\left\{1,\ \frac{\DimFeature}{\bar v^{2}}\right\},
\]
we set \(\epsilon = \min\{\bar{v}/\sqrt{\DimFeature}, 1\}\). Substituting this choice of \(\epsilon\) into Inequality \eqref{eq:case_1_1_final} yields

\begin{equation*}
    \begin{aligned}
    \E_{\bv ,\Feature \sim P_{\bv,\Feature} } \mathcal{R}^{\pi}(T; \bv, 0, \Feature) 
    \geq & c_{10} \epsilon \frac{\bar{K}\lceil \Lambda T \rceil}{30}
    =   c_{10} \min\{\bar{v}/\sqrt{\DimFeature}, 1\} \frac{\min\!\Bigl\{\frac{\DimFeature - K - 1}{3},\; K\Bigr\}\lceil \Lambda T \rceil}{30}.
    \end{aligned}
\end{equation*}
\end{enumerate}

Thus, we claim that 
\begin{equation*}
    \begin{aligned}
    \E_{\bv ,\Feature \sim P_{\bv,\Feature} }  \mathcal{R}^{\pi}(T; \bv, 0, \Feature) \geq c_{9, 4} \min\!\left\{1,\ \sqrt{\frac{\bar v^{2}}{\DimFeature}}\right\} \sqrt{\Lambda T},
    \end{aligned}
\end{equation*}

for some positive constant $c_{9,4}$ only depends on $ K, \plow, \phigh, \bar{v}, \DimFeature$.


Besides, if \(K \leq \dfrac{\DimFeature + 1}{4}\), we have \(d = \DimFeature, \bar{K} = K \) and
\begin{equation}
    \begin{aligned}
        \liminf_{T \to \infty}\dfrac{\E_{\bv ,\Feature \sim P_{\bv,\Feature}}     \mathcal{R}^{\pi}(T; \bv, 0, \Feature)}{\sqrt{\DimFeature  \Lambda T}} \geq 
        \frac{c_{10}\sqrt{K}}{120\sqrt{2c_{11}}} := \LowerC{1} .
    \end{aligned}
\end{equation}

\begin{remark}\label{remark:c_9_1}
    When \(\plow = \Omega (\log K)\) and \(\phigh = \Omega (\log K)\), we have 
    \(\LowerC{1} = \Omega(\log K)\).

    When \(\plow = \Omega (1)\) and \(\phigh = \Omega (1)\), we have 
    \(\LowerC{1} = \Omega (\frac{1}{K})\).
\end{remark}

\subsubsection{Proof of Lemma \ref{lm:regret_lower_bound_tilde}}

For simplicity, we define 
\begin{equation*}
    \begin{aligned}
        W_1 = & S_n \cap (W \cup \{d + 1 , d + 2, \ldots, N\}), \quad
        W_2 =  S_n \setminus (W \cup \{d + 1 , d + 2, \ldots, N\}).
    \end{aligned}
\end{equation*}
Note that 
\begin{equation*}
\begin{aligned}
     r\bigl(S_n, \bp_{n}\bigr) 
     =& 
    \frac{\sum_{j \in W_1} \bp_{jn} \exp(\epsilon - \bp_{jn}) + 
    \sum_{j \in W_2} \bp_{jn} \exp(- \bp_{jn})}
    {1 + \sum_{j \in W_1} \exp(\epsilon  - \bp_{jn}) + 
    \sum_{j \in W_2}  \exp(- \bp_{jn})}\\
    =& \sum_{j \in W_1} \bp_{jn} \ChoiceProbn{\bv_W}{j}
    = \sum_{j \in W_1} 
    \log \left(
        \frac{\ChoiceProbn{\bv_W}{0}}{\ChoiceProbn{\bv_W}{j}}
        + \epsilon
    \right)\ChoiceProbn{\bv_W}{j}
    + \sum_{j \in W_2} 
    \log \left(
        \frac{\ChoiceProbn{\bv_W}{0}}{\ChoiceProbn{\bv_W}{j}}
    \right)\ChoiceProbn{\bv_W}{j}
\end{aligned}
\end{equation*}
Thus, \( r\bigl(S_n, \bp_{n}\bigr) \) is concave with respect to \((\ChoiceProbn{\bv_W}{0}, \ChoiceProbn{\bv_W}{1}, \ldots, \ChoiceProbn{\bv_W}{n})\).
Since all feasible \((\ChoiceProbn{\bv_W}{0}, \ChoiceProbn{\bv_W}{1}, \ldots, \ChoiceProbn{\bv_W}{n})\) form a convex set, the supremum of \(r\bigl(S_n, \bp_{n}\bigr)\) is attained when the values \(\ChoiceProbn{\bv_W}{j}\) are equal within \(W_1\) and within \(W_2\).
Consequently, the optimal prices are constant within each group: all products in \(W_1\) share the same optimal price and all products in \(W_2\) share the same optimal price. We denote these prices by \(p_1\) (for \(W_1\)) and \(p_2\) (for \(W_2\)).

We first show that the optimal prices satisfy \(p_1 = p_2\).
Define
\[
D \;=\; 1 + |W_1|\,\exp(\epsilon - p_1) + |W_2|\,\exp(-p_2), 
\qquad 
N \;=\; |W_1|\,p_1 \exp(\epsilon - p_1) + |W_2|\,p_2 \exp(-p_2),
\]
so the period-\(t\) revenue is \(r(p_1,p_2)=N/D\).
If \(W_1\) or \(W_2\) is empty, then  the problem reduces to a single price decision, and the conclusion trivially holds. When both \(W_1\) and \(W_2\) are not empty, a direct calculation yields
\[
\frac{\partial r}{\partial p_1}
= \frac{|W_1|\,\exp(\epsilon - p_1)}{D^2}\,\bigl[(1-p_1)D + N\bigr], 
\qquad
\frac{\partial r}{\partial p_2}
= \frac{|W_2|\,\exp(-p_2)}{D^2}\,\bigl[(1-p_2)D + N\bigr],
\]
and the identity
\[
\bigl[(1-p_1)D+N\bigr] - \bigl[(1-p_2)D+N\bigr] \;=\; (p_2 - p_1)\,D .
\]

Suppose \(p_1\neq p_2\).
Consider the ascent direction
\[
\delta p \;=\; (\delta p_1,\delta p_2)
\;:=\;
\Bigl(\tfrac{D^2(p_2-p_1)}{|W_1|\,\exp(\epsilon-p_1)},\ \tfrac{D^2(p_1-p_2)}{|W_2|\,\exp(-p_2)}\Bigr),
\qquad
\widehat{\delta p}\;:=\;\frac{\delta p}{\|\delta p\|_\infty}.
\]
The directional derivative at \((p_1,p_2)\) along the  direction \(\widehat{\delta p}\) is
\[
\bigl\langle \nabla r(p_1,p_2),\,\widehat{\delta p}\bigr\rangle
= \frac{D}{\|\delta p\|_\infty}\,(p_2-p_1)^2 \;>\; 0 .
\]

Because \(p_1,p_2\in[\bar p_l,\bar p_h]\), for any \(\eta\in(0,|p_2 - p_1|]\) the perturbed point \((p_1,p_2)+\eta\,\widehat{\delta p}\) remains feasible. Therefore, \(\widehat{\delta p}\) is a feasible direction, and first-order necessary optimality requires every feasible directional derivative to be nonpositive at an optimal point. This contradiction shows that \(p_1\neq p_2\) cannot occur. Hence every optimal solution must satisfy \(p_1=p_2\).

Since $p_1=p_2$, we have 
\begin{equation*}
    \begin{aligned}
      r\bigl(S^*(\bv_{W}, \Feature), \bp^*(\bv_{W}, \Feature)\bigr)
      - r\bigl(S_n, \bp_n\bigr)
      \geq\ &
      \frac{K p_1 \exp(\epsilon - p_1)}
           {1 + K \exp(\epsilon - p_1)}
      -
      \frac{|W_1| p_1 \exp(\epsilon - p_1) + 
            |W_2| p_1 \exp(-p_1)}
           {1 + |W_1| \exp(\epsilon - p_1) + 
                 |W_2| \exp(-p_1)}\\[6pt]
      =\ &
      \frac{|W_2| p_1 \bigl(\exp(\epsilon - p_1) - \exp(-p_1)\bigr)}
           {\bigl(1 + K \exp(\epsilon - p_1)\bigr)
            \bigl(1 + |W_1| \exp(\epsilon - p_1) + |W_2| \exp(-p_1)\bigr)}.
    \end{aligned}
\end{equation*}

Note that for small $\epsilon>0$,
$\exp(\epsilon - p_1) - \exp(-p_1)
= \exp(-p_1)\bigl(\exp(\epsilon) - 1\bigr)
\ge \exp(-p_1)\,\epsilon$,
we further obtain
\begin{equation*}
    r\bigl(S^*(\bv_{W}, \Feature), \bp^*(\bv_{W}, \Feature)\bigr)
    - r\bigl(S_n, \bp_n\bigr)
    \ge
    \frac{|W_2|\, \epsilon\,  p_1 \exp(-p_1)}
         {\bigl(1 + K \exp(1 - p_1)\bigr)^2}
    \geq
    \frac{|W_2|\, \epsilon\,  \plow \exp(-\phigh)}
         {\bigl(1 + K \exp(1 - \plow)\bigr)^2} .
\end{equation*}
For simplicity, we denote
\begin{equation}\label{df_c_10}
    c_{10}
    = \frac{\plow \exp(-\phigh)}
           {\bigl(1 + K \exp(1 - \plow)\bigr)^2},
\end{equation}
which gives the final statement.

\qed

\subsubsection{Proof of Lemma \ref{lm:bound_kl_div}}\label{app:proof_bound_kl_div}

\begin{equation*}
    \begin{aligned}
        \mathrm{KL}
        \left(
            \mathbb{P}_W 
            \left(
                \cdot \mid S_n, \bp_n
            \right) 
            \|  
            \mathbb{P}_{W \cup \{i\}}
            \left(
            \cdot \mid S_n, \bp_n
            \right)
        \right)
        =&\sum_{j \in S_n \cup\{0\}}
        q(j, S_n, \bp_n, \bz_n; \bv_{W} )
        \log 
            \frac{q(j, S_n, \bp_n, \bz_n; \bv_{W} )}{q(j, S_n, \bp_n, \bz_n; \bv_{W \cup \{i\}} ) }
             \\
        \leq & 
        \sum_j q(j, S_n, \bp_n, \bz_n; \bv_{W} )
        \frac{q(j, S_n, \bp_n, \bz_n; \bv_{W} )-q(j, S_n, \bp_n, \bz_n; \bv_{W \cup \{i\}} )}
        {q(j, S_n, \bp_n, \bz_n; \bv_{W \cup \{i\}} )} \\
        =& \sum_j 
        \frac{\left|q(j, S_n, \bp_n, \bz_n; \bv_{W} )-q(j, S_n, \bp_n, \bz_n; \bv_{W \cup \{i\}})\right|^2}
        {q(j, S_n, \bp_n, \bz_n; \bv_{W \cup \{i\}})}.
    \end{aligned}
\end{equation*}
where the inequality holds due to $\log (1+y) \leq y$ for all $y>-1$. Because $q(j, S_n, \bp_n, \bz_n; \bv_{W \cup \{i\}}) \geq 
        \exp(-\phigh)/(1 + K \exp(\epsilon-\plow))$ for all $j \in S_n \cup\{0\}$. Thus, the inequality above is reduced to
\begin{equation}\label{eq:kl_bound}
    \begin{aligned}
        &\mathrm{KL}
            \left(
                \mathbb{P}_W 
                \left(
                    \cdot \mid S_n, \bp_n
                \right) 
                \|  
                \mathbb{P}_{W \cup \{i\}}
                \left(
                \cdot \mid S_n, \bp_n
                \right)
            \right)\\
        \leq&
        \exp(\phigh) (K \exp(\epsilon-\plow) + 1)
        \cdot \sum_{j \in S_n \cup\{0\}}
        \left|q(j, S_n, \bp_n, \bz_n; \bv_{W} )-q(j, S_n, \bp_n, \bz_n; \bv_{W \cup \{i\}})\right|^2.
    \end{aligned}
\end{equation}
We next upper bound $\left|q(j, S_n, \bp_n, \bz_n; \bv_{W} )-q(j, S_n, \bp_n, \bz_n; \bv_{W \cup \{i\}})\right|$ in several scenarios separately.

\paragraph{ \textbf{Scenario 1}: $j=0$. }
We have
\begin{equation*}
    \begin{aligned}
        &\left|
            q(j, S_n, \bp_n, \bz_n; \bv_{W} )-q(j, S_n, \bp_n, \bz_n; \bv_{W \cup \{i\}})
        \right| \\
        =& 
        \left|
            \frac{1}{1+\sum_{k \in S_n} \exp \left(\bz_{kn}^{\top} \bv_W - \bp_k\right)}
            -\frac{1}{1+\sum_{k \in S_n} \exp \left(\bz_{kn}^{\top} \bv_{W \cup\{i\}} - \bp_k \right)}
        \right| \\
        = &       
        \left|
            \frac{
                \sum_{k \in S_n } \left[\exp\left(\bz_{kn}^{\top} \bv_{W \cup\{i\}} - \bp_k \right) -  \exp \left(\bz_{kn}^{\top} \bv_W - \bp_k\right)\right]
                }
                {\left(
                    1+\sum_{k \in S_n} \exp \left(\bz_{kn}^{\top} \bv_W - \bp_k\right)
                \right) 
                \left(
                    1+\sum_{k \in S_n} \exp \left(\bz_{kn}^{\top} \bv_{W \cup\{i\}} - \bp_k \right)
                \right)
                }
        \right| 
        \\
        \leq& 
        \frac{1}{(1+K\exp(-\phigh))^2} \cdot \exp(\epsilon-\plow)
        \sum_{k \in S_n}
        \left|
            \bz_{kn}^{\top}\left(\bv_W-\bv_{W \cup\{i\}}\right)
        \right| 
        \leq 
        \frac{ \exp(\epsilon-\plow) \epsilon}
        {(1+K\exp(-\phigh))^2} .
\end{aligned}
\end{equation*}
Here, we use \(\exp(-\phigh) \leq \exp \left(\bz_{kn}^{\top} \bv_W - \bp_k\right) \leq \exp(\epsilon - \plow) \), \( \exp(-\phigh) \leq \exp \left(\bz_{kn}^{\top} \bv_{W \cup\{i\}} - \bp_k \right) \leq \exp(\epsilon - \plow)\), and \(|\exp(a) -  \exp (b)|\leq \exp(\max\{a, b\}) |a - b|\).


\paragraph{ \textbf{Scenario 2}: $j>0$ and $i \neq j$.}
We have
\begin{equation*}
    \begin{aligned}
        &\left|
            q(j, S_n, \bp_n, \bz_n; \bv_{W} )-q(j, S_n, \bp_n, \bz_n; \bv_{W \cup \{i\}})
        \right| \\
        =& 
        \left|
            \frac{\exp \left(\bz_{jn}^{\top} \bv_W - \bp_j\right)}
            {1+\sum_{k \in S_n} \exp \left(\bz_{kn}^{\top} \bv_W - \bp_k\right)}
            -\frac{\exp \left(\bz_{jn}^{\top} \bv_{W \cup\{i\}} - \bp_j \right)}
            {1+\sum_{k \in S_n} \exp \left(\bz_{kn}^{\top} \bv_{W \cup\{i\}} - \bp_k \right)}
        \right| \\
        =&  
        \exp \left(\bz_{jn}^{\top} \bv_W - \bp_j\right)
        \left|
            \frac{1}{1+\sum_{k \in S_n} \exp \left(\bz_{kn}^{\top} \bv_W - \bp_k\right)}
            -\frac{1}{1+\sum_{k \in S_n} \exp \left(\bz_{kn}^{\top} \bv_{W \cup\{i\}} - \bp_k \right)}
        \right|\\
        \leq&  
        \exp(\epsilon-\plow)
        \frac{\exp(\epsilon-\plow) \epsilon}
        {(1+K\exp(-\phigh))^2}
        = 
        \frac{\exp(2\epsilon-2\plow) \epsilon}
        {(1+K\exp(-\phigh))^2}.
\end{aligned}
\end{equation*}

The last two (in)equalities holds because $\exp \left(\bz_{jn}^{\top} \bv_W - \bp_j\right)=\exp \left(\bz_{jn}^{\top} \bv_{W\cup\{i\}} - \bp_j\right) \leq \exp(\epsilon-\plow)$ for $i \neq j$.

\paragraph{\textbf{Scenario 3}: $j = i$.}
We have
\begin{equation*}
    \begin{aligned}
        &\left|
            q(j, S_n, \bp_n, \bz_n; \bv_{W} )-q(j, S_n, \bp_n, \bz_n; \bv_{W \cup \{i\}})
        \right| \\
        =& 
        \left|
            \frac{\exp \left(\bz_{jn}^{\top} \bv_W - \bp_j\right)}
            {1+\sum_{k \in S_n} \exp \left(\bz_{kn}^{\top} \bv_W - \bp_k\right)}
            -\frac{\exp \left(\bz_{jn}^{\top} \bv_{W \cup\{i\}} - \bp_j \right)}
            {1+\sum_{k \in S_n} \exp \left(\bz_{kn}^{\top} \bv_{W \cup\{i\}} - \bp_k \right)}
        \right| \\
        \leq &  
        \exp \left(\bz_{jn}^{\top} \bv_W - \bp_j\right)
        \left|
            \frac{1}{1+\sum_{k \in S_n} \exp \left(\bz_{kn}^{\top} \bv_W - \bp_k\right)}
            -\frac{1}{1+\sum_{k \in S_n} \exp \left(\bz_{kn}^{\top} \bv_{W \cup\{i\}} - \bp_k \right)}
        \right|\\
        &+
        \left|
            \exp \left(\bz_{jn}^{\top} \bv_W - \bp_j\right) 
            -\exp \left(\bz_{jn}^{\top} \bv_{W \cup \{i\}} - \bp_j\right) 
        \right|
        \frac{1}{1+\sum_{k \in S_n} \exp \left(\bz_{kn}^{\top} \bv_{W \cup\{i\}} - \bp_k \right)}
        \\
        \leq&  
        \frac{\exp(2\epsilon-2\plow) \epsilon}
        {(1+K\exp(-\phigh))^2}
        + \exp(\epsilon-\plow) \epsilon \cdot \frac{1}{1+K \exp(-\phigh)}\\
        \leq & 
        \frac{\exp(2\epsilon-2\plow) \epsilon}
        {(1+K\exp(-\phigh))^2}
        + 
        \frac{\exp(\epsilon-\plow) \epsilon}
        {1+K \exp(-\phigh)}.
\end{aligned}
\end{equation*}

Combining all upper bounds on $\left|
            q(j, S_n, \bp_n, \bz_n; \bv_{W} )-q(j, S_n, \bp_n, \bz_n; \bv_{W \cup \{i\}})
        \right|$ and \Cref{eq:kl_bound}, we have
\begin{equation*}
    \begin{aligned}
    &\mathrm{KL}
    \left(
        \mathbb{P}_W 
        \left(
            \cdot \mid S_n, \bp_n
        \right) 
        \|  
        \mathbb{P}_{W \cup \{i\}}
        \left(
        \cdot \mid S_n, \bp_n
        \right)
    \right) \\
     \leq  &
    \exp(\phigh)(1 + K \exp(\epsilon-\plow))
    \cdot
    \left[
        \frac{ \exp(4\epsilon-4\plow) \epsilon^2}
        {(1+K\exp(-\phigh))^4} (K + 1) 
        +
        \frac{ \exp(2\epsilon-2\plow) \epsilon^2}
        {(1+K\exp(-\phigh))^4} 
        +\frac{2\exp(2\epsilon-2\plow) \epsilon^2}
        {(1+K \exp(-\phigh))^2}
    \right] \\
    \leq &
    \exp(\phigh)(1 + K \exp(\epsilon-\plow))
    \cdot
    \left[
        \frac{ \exp(4-4\plow) \epsilon^2}
        {(1+K\exp(-\phigh))^4} (K + 1) 
        +
        \frac{ \exp(2-2\plow) \epsilon^2}
        {(1+K\exp(-\phigh))^4} 
        +\frac{2\exp(2-2\plow) \epsilon^2}
        {(1+K \exp(-\phigh))^2}
    \right] \\
     = & c_{11} \epsilon^2,
    \end{aligned}
\end{equation*}
where 
\begin{equation}\label{df:c_11}
    c_{11} =  \exp(\phigh)(1 + K \exp(1-\plow))
    \left[
        \frac{ \exp(4-4\plow) }{(1+K\exp(-\phigh))^4} (K + 1) 
        +
        \frac{\exp(2-2\plow)}{(1+K\exp(-\phigh))^4} 
        +\frac{2\exp(2-2\plow)}{(1+K \exp(-\phigh))^2}
    \right].
\end{equation}

The above upper bound holds for arbitrary $S_n$. If \(i \notin S_n\), we can calculate the KL divergence exactly: \(\mathrm{KL}\!\left(
    \mathbb{P}_W (\cdot \mid S_n, \bp_n)
    \,\big\|\,  
    \mathbb{P}_{W \cup \{i\}} (\cdot \mid S_n, \bp_n)
  \right) = 0\). Thus, by the chain rule for KL divergence for sequential observations,
\begin{equation*}
\begin{aligned}
\mathrm{KL}\!\bigl(\mathbb{P}_{W}\,\|\,\mathbb{P}_{W\cup\{i\}}\bigr)
& = \mathbb{E}_{W}\!\left[
\log \frac{\prod_{n=1}^{\lceil \Lambda T \rceil}\mathbb{P}_{W}\!\bigl(j_n \mid H'_{n-1}\bigr)}
               {\prod_{n=1}^{\lceil \Lambda T \rceil}\mathbb{P}_{W\cup\{i\}}\!\bigl(j_n \mid H'_{n-1}\bigr)}
\right] 
=\mathbb{E}_{W}\!\left[
\sum_{n=1}^{\lceil \Lambda T \rceil}
\log \frac{\mathbb{P}_{W}\!\bigl(j_n \mid H'_{n-1}\bigr)}
          {\mathbb{P}_{W\cup\{i\}}\!\bigl(j_n \mid H'_{n-1}\bigr)}
\right] \\[4pt]
& = \sum_{n=1}^{\lceil \Lambda T \rceil}
\mathbb{E}_{W}\!\left[
\log \frac{\mathbb{P}_{W}\!\bigl(j_n \mid H'_{n-1}\bigr)}
          {\mathbb{P}_{W\cup\{i\}}\!\bigl(j_n \mid H'_{n-1}\bigr)}
\right] = \sum_{n=1}^{\lceil \Lambda T \rceil}
\mathbb{E}_{W}\!\left[
\mathrm{KL}\!\left(
\mathbb{P}_{W}(j_n  \mid H'_{n-1})
\,\big\|\,
\mathbb{P}_{W\cup\{i\}}(j_n  \mid H'_{n-1})
\right)
\right] \\[4pt]
&= \mathbb{E}_{W}\!\left[
\sum_{n=1}^{\lceil \Lambda T \rceil}
\mathrm{KL}\!\left(
\mathbb{P}_{W}(j_n  \mid S_n,\bp_n)
\,\big\|\,
\mathbb{P}_{W\cup\{i\}}(j_n  \mid S_n,\bp_n)
\right)
\right] 
\le \mathbb{E}_{W}\!\left[
\sum_{n=1}^{T} c_{11}\,\epsilon^2\,\mathbf{1}\{\,i\in S_n\,\}
\right]
\;=\;
c_{11}\,\epsilon^2\,\mathbb{E}_{W}\!\bigl[\widetilde{N}_i\bigr].
\end{aligned}
\end{equation*}


which proves the lemma.\qed

\subsubsection{Proof of Lemma \ref{lm:poi_bound}}
Let $m:=\operatorname{Med}(M_T)$ be the median,     which is defined to be the least integer  such that $P\left(M_T \leq m\right) \geq \frac{1}{2}$. Since $\mu\ge 1$, we have  $m\ge 1$ (otherwise $m=0$ would force $\mathbb{P}(M_T=0)\ge 1/2$, i.e., $\mu\le \log 2<1$, a contradiction arises). By the classical bounds for Poisson medians
(see \citet{Choi1994}), we have
\[
-\log 2 \;\le\; m-\mu \;<\; \frac{1}{3}.
\]
We have 
\[
\lceil \mu\rceil-1 \;\le\; \lceil \mu-\ln 2\rceil \;\le\; m \;\le\; \lfloor \mu+\tfrac13\rfloor \;\le\; \lceil \mu\rceil.
\]

Hence
\[
\lceil \mu\rceil - 1 \;\le\; m \;\le\; \lceil \mu\rceil,
\]
so $m \in \{\lceil \mu\rceil - 1,\ \lceil \mu\rceil\}$. Consider two cases.

\textbf{Case 1:} $m=\lceil \mu\rceil$. By the definition of the median, $\mathbb{P}(M_T \leq m - 1) \leq 1/2$. Therefore, 
$\mathbb{P}(M_T \ge m) \ge 1/2$, and thus
$\mathbb{P}(M_T \ge \lceil \mu\rceil) \ge 1/2$.

\textbf{Case 2:} $m=\lceil \mu\rceil - 1$. Then
\[
\mathbb{P}(M_T \ge \lceil \mu\rceil)
= \mathbb{P}(M_T \ge m+1)
= \mathbb{P}(M_T \ge m) - \mathbb{P}(M_T = m)
\;\ge\; \frac{1}{2} - \mathbb{P}(M_T = m).
\]
To bound $\mathbb{P}(M_T = m)$ uniformly in $\mu$, note that for fixed integer $m\ge 0$,
the function $g(\mu)=\exp(-\mu)\mu^m$ is maximized at $\mu=m$; hence
\[
\mathbb{P}(M_T = m) 
= \exp(-\mu)\frac{\mu^m}{m!}
\;\le\; \exp(-m)\frac{m^m}{m!}
\;\le\; \frac{1}{\sqrt{2\pi m}},
\]
where the last step uses Stirling’s lower bound $m!\ge \sqrt{2\pi m}\,(m/e)^m$. 
Note that \(m \geq 1\) , we further get
$\mathbb{P}(M_T=m)\le 1/\sqrt{2\pi}$.
Therefore,
\[
\mathbb{P}(M_T \ge \lceil \mu\rceil)
\;\ge\;
\frac{1}{2} - \frac{1}{\sqrt{2\pi}}
\;>\; 0.1.
\]

Combining the two cases completes the proof.

\qed
\subsubsection{Adversarial construction II.} Suppose $\min\{\DimFeature - 2, N \}\ge K, \DimFeature \geq 4, \Lambda \geq 1 $.
\begin{instance}(Worst-case instance II)\label{df:worst_case_example_2}
First, we let
\[
H(p)\;:=\;-p\log p-(1-p)\log(1-p),
\]
and define
\begin{equation}\label{eq:explicit-d-kbar}
d \;:=\; \min\left\{\left\lfloor \frac{\log\!\big((N-\DimFeature)/K\big)-\tfrac14\log 3}{\,H(1/4)\,}\right\rfloor , \DimFeature\right\},
\qquad
\bar K \;:=\; \Bigl\lfloor \frac{d+1}{4}\Bigr\rfloor .
\end{equation}
Suppose 
\begin{equation}\label{eq:phi-condition}
\log\!\frac{N-\DimFeature}{K} \;\ge\; \tfrac14\log 3 \;+\; 4\,H(1/4),
\end{equation}
which in particular guarantees $d\ge 4$ (hence $d>3$) and \(d \leq \DimFeature\) under the choice \eqref{eq:explicit-d-kbar}.

The following lemma verifies that the explicit choice of \((d,\bar K)\) in \eqref{eq:explicit-d-kbar} satisfies several properties required in subsequent analysis.

\begin{lemma}\label{lem:catalog-feasible-phi}
Let $d$ and $\bar K$ be chosen as in \eqref{eq:explicit-d-kbar}.
If  condition \eqref{eq:phi-condition} holds, then
\[
K\binom{d}{\bar K}+\DimFeature-d \;\le\; N.
\]
\end{lemma}


Next, we fix \(\epsilon\in(0,\, \min\{\bar{v}/\sqrt{\bar{K}}, 1\}]\), and define a set of $\DimFeature$-dimension vectors, $\mathbf{V}$, as follows.

    For each subset \(W\subseteq[d]\) with \(|W|=\bar{K}\), define the corresponding parameter vector
\(\bv_W\in\mathbb{R}^{\DimFeature}\) as
\[
\bv_W(i)=
\begin{cases}
\epsilon,& i\in W,\\
0,& i\notin W,
\end{cases}
\qquad i\in[\DimFeature].
\]
Since $\|\bv_W\|_2 \le \bar{v}$, this construction satisfies the boundedness condition in \Cref{as:v&assortment_new}.

Collecting these vectors gives the parameter set
\[
\mathbf V=\{\bv_W: W\subseteq[d],\,|W|=\bar{K}\}
=\{\bv_W: W\in\mathcal W_{\bar K}\},
\]
where, for simplicity of notation, \(\mathcal{W}_{\bar{K}}\) denotes the class of all subsets of  \([d]\) whose size is $\bar{K}$.

\begin{definition}(Feature construction) \label{df_feature}
The (time-invariant) product features are constructed as follows.
For each \(U\in\mathcal W_{\bar K}\), define \(\Feature_U\in\mathbb{R}^{\DimFeature}\) by
\[
\Feature_U(i)=
\begin{cases}
1/\sqrt{\bar K},& i\in U,\\
0,& i\notin U,
\end{cases}
\qquad i\in[\DimFeature],
\]
and for each \(j\in\{ d+1,\ldots,\DimFeature\}\), let \(\Feature_{\{j\}}\) be the \(j\)-th standard basis vector scaled by $\epsilon$, i.e.,
\end{definition}
\[ [\Feature_{\{j\}}]_i \;=\; \begin{cases} \epsilon, & i=j,\\ 0, & i\ne j.\end{cases} \]
The catalog contains at least \(K\) products with feature vector \(\Feature_U\) for every \(U\in\mathcal W_{\bar K}\),
and one product with feature vector \(\Feature_{\{j\}}\) for each \(j\in\{d+1,\ldots,\DimFeature\}\);
\Cref{lem:catalog-feasible-phi} shows that  the total number of products does not exceed \(N\). Besides the aforementioned products, the rest of the products all have feature $\Feature_{\{d_z\}}$.

It is straightforward that \(\|\Feature_{U}\|_2 \leq 1\) and $\|\Feature_{\{i\}}\|_2 \le 1$. So this construction of features satisfies the boundedness requirement in \Cref{assumption::y_exist}.

Clearly for any $\Feature$ such that $\|\Feature\|\le 1$, $\Feature^\top \bv_W \le \|\bv_W\| = \epsilon \sqrt{\bar{K}}$, which is attained by $\Feature = \Feature_W$. Hence, selecting the $K$ products with feature $\Feature_W$ gives the optimal assortment of size $K$.


\end{instance}

\paragraph{Regret of each customer.}

Next, we derive an explicit lower bound on the per-customer regret. 


By the construction of product features, we can suppose assortment $S_n$ contains $K$ products with features $\Feature_{U_{1,n}}, \Feature_{U_{2,n}},\cdots, \Feature_{U_{K,n}}$, where the set $U_{i,n}$ is either in $\mathcal{W}_{\bar{K}}$ or $\{\{d_1\}, \cdots, \{d_z\}\}$. We single out the feature of the product in the assortment $S_n$ with the highest intrinsic value and denote its corresponding set to be $$\tilde{U}_n = \argmax_{ 1\le k \le K} \langle  \Feature_{U_{k,n}}, \bv_W \rangle .$$


We use \(\mathbb{E}_W\) and \(\mathbb{P}_W\) to denote the expectation and probability measure under the model parameterized by \(\bv_W\) and the 
policy \(\pi^\prime\), respectively. The following lemma establishes a lower bound for 
\( r\bigl(S^*(\bv_{W}, \Feature), \bp^*(\bv_{W}, \Feature)\bigr) - r\bigl(S_t, \bp_t\bigr)\) by comparing \(\widetilde{U}_n\) with \(W\).

\begin{lemma}\label{lm:regret_lower_bound_tilde2}
    Suppose \(\epsilon\in(0,\,\bar{v}/\sqrt{\bar{K}}]\).
    Then,
    \[
      r\bigl(S^*(\bv_{W}, \Feature), \bp^*(\bv_{W}, \Feature)\bigr) - r\bigl(S_n, \bp_n\bigr)
      \;\geq\;         
      \frac{c_{12} \epsilon}{\sqrt{\bar{K}}}
     \left(\bar{K} - \bigl|\widetilde U_n \cap W\bigr|\right),
    \]
    where 
        $c_{12}\;:=\;\frac{K\,\bar p_l\,\exp(-\bar p_h)}
        {\bigl(K\,\exp(\bar v-\bar p_l)+1\bigr)^2}$.
   
\end{lemma}

\begin{remark}\label{remark:c_12}
    When \(\plow = \Omega (\log K)\) and \(\phigh = \Omega (\log K)\), we have 
    \(c_{12} = \Omega(\log K)\).

    When \(\plow = \Omega (1)\) and \(\phigh = \Omega (1)\), we have 
    \(c_{12} = \Omega(\frac{1}{K})\). 
\end{remark}

Next, we establish a lower bound on the cumulative regret. Define \(\widetilde{N}_i := \sum_{n=1}^{\ceil{\Lambda T}} \mathbf{1}\{i \in \widetilde{U}_n\}\). By \Cref{lm:regret_lower_bound_tilde2}, it follows that
\begin{multline*}
        \widehat{\mathcal{R}}^*(\ceil{\Lambda T}, \pi', \bv_{W}, \Feature)
        =
        \mathbb{E}_W \sum_{n=1}^{\ceil{\Lambda T}} 
        \Bigl(
            r\bigl(
                S^*(\bv_{W}, \Feature), \bp^*(\bv_{W}, \Feature)\bigr) - r\bigl(S_n, \bp_n
            \bigr)
        \Bigr) \\
        \geq
        \mathbb{E}_W \sum_{n=1}^{\ceil{\Lambda T}}
        \frac{c_{12} \epsilon}{\sqrt{\bar{K}}}
        \left(
            \bar{K} - \bigl|\widetilde U_n\cap W\bigr|
        \right)
        = \;    
        \frac{c_{12} \epsilon}{\sqrt{\bar{K}}}
        \Bigl(\bar{K} \ceil{\Lambda T} - \sum_{i \in W} \mathbb{E}_W[\widetilde{N}_i]\Bigr),
        \quad 
        \forall\,W \in \mathcal{W}_{\bar K}.
\end{multline*}

Denote \(\mathcal{W}_{\bar{K}}^{(i)} := \{W \in \mathcal{W}_{\bar{K}} : i \in W\}\) and \(\mathcal{W}_{\bar{K}-1} := \{W \subseteq[d] : |W|=\bar{K} - 1\}\). 
We give $\bv$ a uniform distribution over $\mathcal{W_{\bar{K}}}$, and fix the feature construction to be \Cref{df_feature}. We use $P_{\bv,\Feature}$ to denote this joint distribution. Using Inequality \eqref{ineq:r_to_r_hat} on minimax rate gives
\begin{equation*}
    \begin{aligned}
   &\E_{\bv ,\Feature \sim P_{\bv,\Feature} } \mathcal{R}^{\pi}(T; \bv, 0, \Feature) /  \Prob[M_T \geq \ceil{\Lambda T}] \\
    \geq & \inf_{\pi' \in \mathcal{A}'(\lceil \Lambda T \rceil)} 
    \E_{\bv,\Feature}  
    \widehat{\mathcal{R}}^*(\lceil \Lambda T\rceil, \pi', \bv, \Feature) 
    \\
        \geq &
        \inf_{\pi' \in \mathcal{A}'(\lceil \Lambda T \rceil)}        
        \frac{c_{12} \epsilon}{\sqrt{\bar{K}}}
        \biggl(
            \bar{K} \ceil{\Lambda T} - 
            \frac{1}{|\mathcal{W}_{\bar{K}}|} 
            \sum_{W \in \mathcal{W}_{\bar{K}}}\sum_{i\in W}
            \mathbb{E}_W[\widetilde{N}_i]
        \biggr)\\
        = &
        \inf_{\pi' \in \mathcal{A}'(\lceil \Lambda T \rceil)}
        \frac{c_{12} \epsilon}{\sqrt{\bar{K}}}
        \biggl(
            \bar{K} \ceil{\Lambda T}- 
                \frac{1}{|\mathcal{W}_{\bar{K}}|} 
                \sum_{i=1}^d \sum_{W \in \mathcal{W}_{\bar{K}}^{(i)}} \mathbb{E}_W[\widetilde{N}_i]
        \biggr)\\
        = &
        \inf_{\pi' \in \mathcal{A}'(\lceil \Lambda T \rceil)}
        \frac{c_{12} \epsilon}{\sqrt{\bar{K}}}
        \biggl(
            \bar{K} \ceil{\Lambda T}-
                \frac{1}{|\mathcal{W}_{\bar{K}}|}
                \sum_{W \in \mathcal{W}_{\bar{K}-1}} \sum_{i \notin W} \mathbb{E}_{W \cup \{i\}}[\widetilde{N}_i]
        \biggr) \\
        \geq&
        \inf_{\pi' \in \mathcal{A}'(\lceil \Lambda T \rceil)}
        \frac{c_{12} \epsilon}{\sqrt{\bar{K}}}
        \biggl(
            \bar{K} \ceil{\Lambda T}-
                \frac{|\mathcal{W}_{\bar{K}-1}|}{|\mathcal{W}_{\bar{K}}|} 
                \max_{W \in \mathcal{W}_{\bar{K}-1}} 
                \sum_{i \notin W} \mathbb{E}_{W \cup \{i\}}[\widetilde{N}_i]
        \biggr)\\
        =& 
        \inf_{\pi' \in \mathcal{A}'(\lceil \Lambda T \rceil)}
        \frac{c_{12} \epsilon}{\sqrt{\bar{K}}}
        \biggl(
            \bar{K} \ceil{\Lambda T} -
                \frac{|\mathcal{W}_{\bar{K}-1}|}{|\mathcal{W}_{\bar{K}}|} 
                \max_{W \in \mathcal{W}_{\bar{K}-1}} 
                \sum_{i \notin W}
                \left[
                \mathbb{E}_{W \cup \{i\}}[\widetilde{N}_i]
                - \mathbb{E}_W[\widetilde{N}_i] 
                + \mathbb{E}_W[\widetilde{N}_i] 
                \right]
        \biggr).
    \end{aligned}
\end{equation*}

Since \(\sum_{i \notin W} \mathbb{E}_W[\widetilde{N}_i] \leq \sum_{i=1}^d \mathbb{E}_W[\widetilde{N}_i] \leq  \bar{K} \ceil{\Lambda T}\) and \(\frac{|\mathcal{W}_{\bar{K}-1}|}{|\mathcal{W}_{\bar{K}}|} =\frac{\binom{d}{\bar{K} - 1}}{\binom{d}{\bar{K}}} =\frac{\bar{K}}{(d - \bar{K} + 1)} \leq 1/3,\) we obtain
\begin{equation*}
    \begin{aligned}
        &
        \frac{\E_{\bv ,\Feature \sim P_{\bv,\Feature} } \mathcal{R}^{\pi}(T; \bv, 0, \Feature) }{ \Prob[M_T \geq \ceil{\Lambda T}] }
        \geq 
        \inf_{\pi' \in \mathcal{A}'(\lceil \Lambda T \rceil)}
        \frac{c_{12} \epsilon}{\sqrt{\bar{K}}}
        \Bigl(
            \frac{2}{3}\bar{K} \ceil{\Lambda T} - 
            \frac{\bar{K}}{d - \bar{K} + 1}
            \max_{W \in \mathcal{W}_{\bar{K}-1}}
            \sum_{i \notin W}
            \bigl|\,
                \mathbb{E}_{W \cup \{i\}}[\widetilde{N}_i] - \mathbb{E}_W[\widetilde{N}_i]
            \,\bigr|
        \Bigr).
    \end{aligned}
\end{equation*}

\paragraph{Pinsker's inequality.}
Finally, we derive an upper bound on 
\(\bigl|\mathbb{E}_{W \cup\{i\}}\bigl[\widetilde{N}_i\bigr]-\mathbb{E}_W\bigl[\widetilde{N}_i\bigr]\bigr|\) 
for any \(W \in \mathcal{W}_{\bar{K}-1}\):
\begin{equation*}
    \begin{aligned}
        \Bigl|
            \mathbb{E}_W
            \bigl[
                \widetilde{N}_i
            \bigr] 
            - 
            \mathbb{E}_{W \cup \{i\}}
            \bigl[
                \widetilde{N}_i
            \bigr]
        \Bigr| 
        \leq&
        \sum_{j=0}^{\ceil{\Lambda T}} 
        j \,\Bigl|
            \mathbb{P}_W
            \bigl[
                \widetilde{N}_i=j
            \bigr]
            -
            \mathbb{P}_{W \cup \{i\}}
            \bigl[
                \widetilde{N}_i=j
            \bigr]
        \Bigr|\\
        \leq &
        \ceil{\Lambda T} \cdot
        \sum_{j=0}^{\ceil{\Lambda T}}
        \Bigl|
            \mathbb{P}_W
            \bigl[
                \widetilde{N}_i=j
            \bigr]
            -
            \mathbb{P}_{W \cup \{i\}}
            \bigl[
                \widetilde{N}_i=j
            \bigr]
        \Bigr|\\
        \leq& 2\,\ceil{\Lambda T} \cdot 
        \bigl\|
            \mathbb{P}_W - \mathbb{P}_{W \cup \{i\}}
        \bigr\|_{\mathrm{TV}} 
        \;\leq\;
        \ceil{\Lambda T} \cdot 
        \sqrt{
            2\, \mathrm{KL}(\mathbb{P}_W \,\|\, \mathbb{P}_{W \cup \{i\}})
        }.
    \end{aligned}
\end{equation*}
Here \(\|\mathbb{P}_W - \mathbb{P}_{W \cup \{i\}}\|_{\mathrm{TV}}\) and \(\mathrm{KL}(\mathbb{P}_W \| \mathbb{P}_{W \cup \{i\}})\) denote the total variation distance and the Kullback–Leibler divergence, respectively; the last inequality uses Pinsker’s inequality.

For every $i \in[d]$ define random variables $N_i:=\sum_{t=1}^{\ceil{\Lambda T}}  \sum_{\Feature_U \in S_t} \mathbf{1}\{i \in U\}$.  The following lemma provides an upper bound on the Kullback–Leibler divergence.

\begin{lemma}\label{lm:bound_kl_div2}
Suppose \(\epsilon < 1\). For any \(W \in \mathcal{W}_{\bar{K}-1}\) and \(i \in[d]\),
\[
\mathrm{KL}(\mathbb{P}_W \,\|\, \mathbb{P}_{W \cup \{i\}}) \;\leq\; c_{13} \,\cdot\, \mathbb{E}_W[N_i] \,\cdot\, \frac{\epsilon^2}{\bar{K}},
\]
where \(c_{13}\) is defined in \Cref{df:c_13}.
\end{lemma}
\begin{remark}\label{remark:c_13}
    When \(\plow = \Omega (\log K)\) and \(\phigh = \Omega (\log K)\), we have 
    \(c_{13} = \Omega(\frac{1}{K})\).

    When \(\plow = \Omega (1)\) and \(\phigh = \Omega (1)\), we have 
    \(c_{13} = \Omega(\frac{1}{K})\). 
\end{remark}

Using \Cref{lm:bound_kl_div2} with the  bound above yields
\[
    \E_{\bv ,\Feature \sim P_{\bv,\Feature} } \mathcal{R}^{\pi}(T; \bv, 0, \Feature) /  \Prob[M_T \geq \ceil{\Lambda T}] 
    \;\geq\; 
        \inf_{\pi' \in \mathcal{A}'(\lceil \Lambda T \rceil)}
        \frac{c_{12} \epsilon}{\sqrt{\bar{K}}}
    \Bigl(
        \frac{2\bar{K} \ceil{\Lambda T}}{3}
        \;-\;
        \frac{\bar{K}\ceil{\Lambda T}}{d - \bar{K} + 1} \sum_{i=1}^d \sqrt{2 c_{13}\,\mathbb{E}_W[N_i]\,\frac{\epsilon^2}{\bar{K}}}
    \Bigr).
\]
Applying the Cauchy-Schwarz inequality,
\[
\sum_{i=1}^d 
\sqrt{2c_{13}\,\mathbb{E}_W[N_i]\,\frac{\epsilon^2}{ \bar{K}}} 
\;\;\leq\;\;
\sqrt{d}\,\cdot\,\sqrt{\sum_{i=1}^d 2c_{13}\,\mathbb{E}_W[N_i]\,\frac{\epsilon^2}{\bar{K}}}.
\]
Since \(\sum_{i=1}^d \mathbb{E}_W[N_i] \leq K \bar{K}\,\ceil{\Lambda T}\), we further have
\[
\sum_{i=1}^d \sqrt{2c_{13} \,\mathbb{E}_W[N_i]\,\frac{\epsilon^2}{\bar{K}}} 
\;\leq\;  \sqrt{\,2c_{13} \,d \, K \,\,\ceil{\Lambda T}\,\epsilon^2}.
\]
Thus, we obtain
\begin{equation}\label{eq:case_1_2_final}
\begin{aligned}
    \E_{\bv ,\Feature \sim P_{\bv,\Feature} } \mathcal{R}^{\pi}(T; \bv, 0, \Feature)   
    \;\geq &\;
    \sup_{W \in \mathcal{W}_{\bar{K}}} 
    \left [ 
    \inf_{\pi' \in \mathcal{A}'(\lceil \Lambda T\rceil)} \widehat{\mathcal{R}}^*(\ceil{\Lambda T}, \pi', \bv_{W}, \Feature)
    \right] 
    \Prob[M_T \geq \ceil{\Lambda T}]\\
    \;\geq &\; 
        \frac{c_{12} \epsilon}{\sqrt{\bar{K}}}
    \Bigl(
        \frac{2\bar{K}\,\ceil{\Lambda T}}{3}\;-\;
        \frac{ \bar{K}\ceil{ \Lambda T}}{d - \bar{K} + 1}\,
        \sqrt{
            2c_{13} \,d \, K \, \ceil{\Lambda T}\, \epsilon^2
        }
    \Bigr) \Prob[M_T \geq \ceil{\Lambda T}].
\end{aligned}
\end{equation}

Note that by Lemma \ref{lm:poi_bound} , $\Prob(M_T\ge  \ceil{\Lambda T}) \ge 0.1 $. Next, we will bound the middle term by discussing two scenarios of $\Lambda T$ and setting $\epsilon$ correspondingly.

\begin{itemize}
    \item If 
    \[\Lambda T\ \ge\
\frac{
\dfrac{\log\!\big((N-\DimFeature)/K\big)-\tfrac14\log 3}{H(1/4)} + 5
}{
32\,c_{13} K 
}
\;
\max\!\left\{
1,\,
\frac{
\dfrac{\log\!\big((N-\DimFeature)/K\big)-\tfrac14\log 3}{H(1/4)} + 1
}{
4\,\bar v^{2}
}
\right\},\]
by setting \(\epsilon = \sqrt{\dfrac{(d - \bar{K} + 1)^2}{18\,c_{13}\,d\,K\,\ceil{\Lambda T}}}\), we have \(\epsilon\in(0,\,\min\{\bar{v}/\sqrt{\bar{K}}, 1\})\). Substituting this choice of \(\epsilon\) into Inequality \eqref{eq:case_1_2_final} yields
\begin{equation*}
\begin{aligned}
    &\E_{\bv ,\Feature \sim P_{\bv,\Feature} } \mathcal{R}^{\pi}(T; \bv, 0, \Feature)  \geq 
    \frac{c_{12} (d - \bar{K} + 1)\sqrt{\bar{K}}}{90\sqrt{2c_{13}  d K}} 
    \sqrt{
        \ceil{\Lambda T}
    } \geq \frac{c_{12}  \sqrt{d\bar{K}}}{120\sqrt{2c_{13} K}} 
    \sqrt{
        \ceil{\Lambda T}
    }\\
    \ge\ &
\frac{c_{12}}{240\sqrt{2c_{13} K }}\,
\min\left\{\dfrac{\log((N-\DimFeature)/K)-\tfrac14\log 3}{H(1/4)}, \DimFeature\right\}
\sqrt{\lceil\Lambda T\rceil}.
\end{aligned}
\end{equation*}
\item If 
\[\Lambda T\ <\
\frac{
\dfrac{\log\!\big((N-\DimFeature)/K\big)-\tfrac14\log 3}{H(1/4)} + 5
}{
32\,c_{13} K
}
\;
\max\!\left\{
1,\,
\frac{
\dfrac{\log\!\big((N-\DimFeature)/K\big)-\tfrac14\log 3}{H(1/4)} + 1
}{
4\,\bar v^{2}
}
\right\},\]
we set \(\epsilon = \min\{\bar{v}/\sqrt{\bar{K}}, 1\}\). Substituting this choice of \(\epsilon\) into Inequality \eqref{eq:case_1_2_final} yields
\begin{equation*}
    \begin{aligned}
        &\E_{\bv ,\Feature \sim P_{\bv,\Feature} } \mathcal{R}^{\pi}(T; \bv, 0, \Feature)  \geq  \frac{c_{12}  \min\{\bar{v}/\sqrt{\bar{K}}, 1\}}{\sqrt{\bar{K}}} \frac{\bar{K} \lceil  \Lambda T \rceil}{30} \\
        \geq &\frac{c_{12}\, \min\{\bar{v}/\sqrt{\DimFeature}, 1\}}{30}\,\sqrt{
        \Biggl\lfloor \frac{
        \Bigl\lfloor \dfrac{\log((N-\DimFeature)/K)-\tfrac14\log 3}{H(1/4)} \Bigr\rfloor +1
        }{4} \Biggr\rfloor
        }\;
        \lceil \Lambda T \rceil.
    \end{aligned}
\end{equation*}
\end{itemize}

Combining two cases, we have 
\begin{equation*}
    \E_{\bv ,\Feature \sim P_{\bv,\Feature} } \mathcal{R}^{\pi}(T; \bv, 0, \Feature)  \geq c_{9, 5} \sqrt{\log\!\big((N-\DimFeature)/K\big)} \min\{\bar{v}/\sqrt{\DimFeature}, 1\}\sqrt{\Lambda T},
\end{equation*}
for some positive constant $c_{9,2}$ only depends on $ K, \plow, \phigh, \bar{v}, \DimFeature$.

Besides, if $\DimFeature \leq \left\lfloor \frac{\log\!\big((N-\DimFeature)/K\big)-\tfrac14\log 3}{\,H(1/4)\,}\right\rfloor$,
we have \(d = \DimFeature\) and 
\begin{equation}
    \begin{aligned}
        \liminf_{T \to \infty}\dfrac{\E_{\bv ,\Feature \sim P_{\bv,\Feature}}     \mathcal{R}^{\pi}(T; \bv, 0, \Feature)}{\DimFeature \sqrt{\Lambda T}} \geq 
        \frac{c_{12}}{120\sqrt{2c_{13} K }} 
        := \LowerC{2}. 
    \end{aligned}
\end{equation}
\begin{remark}\label{remark:c_9_2}
    When \(\plow = \Omega (\log K)\) and \(\phigh = \Omega (\log K)\), we have 
    \(\LowerC{2} = \Omega(\log K)\).

    When \(\plow = \Omega (1)\) and \(\phigh = \Omega (1)\), we have 
    \(\LowerC{2} = \Omega(\frac{1}{{K}})\).

\end{remark}

\subsubsection{Proof of Lemma \ref{lem:catalog-feasible-phi}}

Set $k:=\bar K=\lfloor(d+1)/4\rfloor$ and $\alpha:=k/d$.
We use two standard ingredients.

\emph{(i) Entropy upper bound for binomials.}
For any integers $0\le k\le d$,
\begin{equation}\label{eq:entropy-bound-app}
\binom{d}{k}\ \le\ \exp\!\big(d\,H(k/d)\big),
\end{equation}
which follows from $\sum_{j=0}^{d}\binom{d}{j}p^{\,j}(1-p)^{d-j}=1$ by taking $p=k/d$.

\emph{(ii) Supporting-line bound at $p=\tfrac14$.}
Since $H$ is concave on $[0,1]$ and $H'(p)=\log\!\big(\tfrac{1-p}{p}\big)$,
\begin{equation}\label{eq:supporting-line-app}
H(x)\ \le\ H\!\left(\tfrac14\right)+\big(x-\tfrac14\big)\log 3\qquad(x\in[0,1]).
\end{equation}
Because $k=\lfloor(d+1)/4\rfloor$, we have $\alpha=\tfrac{k}{d}\le \tfrac14+\tfrac{1}{4d}$; hence by \eqref{eq:supporting-line-app},
\begin{equation}\label{eq:H-alpha-app}
H(\alpha)\ \le\ H\!\left(\tfrac14\right)+\frac{\log 3}{4d}.
\end{equation}

Combining \eqref{eq:entropy-bound-app} and \eqref{eq:H-alpha-app} gives
\begin{equation}\label{eq:binom-master-app}
\binom{d}{k}
\ \le\ 
\exp\!\Big(d\,H(\tfrac14)+\tfrac14\log 3\Big)
\ =\ e^{\,dH(1/4)}\,3^{1/4}.
\end{equation}

By the explicit choice \eqref{eq:explicit-d-kbar},
$d \;\le\; \frac{\log\!\big((N-\DimFeature)/K\big)-\tfrac14\log 3}{H(1/4)}$, i.e.,
$e^{\,dH(1/4)}\,3^{1/4} \;\le\; \frac{N-\DimFeature}{K}$.
Applying this to \eqref{eq:binom-master-app} yields
\[
K\binom{d}{\bar K} \;\le\; N-\DimFeature,
\]
hence $K\binom{d}{\bar K}+\DimFeature-d\le N$, as claimed. The identity
$\bar K=\lfloor(d+1)/4\rfloor$ holds by construction in \eqref{eq:explicit-d-kbar}.

Finally, note that the condition \eqref{eq:phi-condition} implies
\(
\log\!\big((N-\DimFeature)/K\big)-\tfrac14\log 3 \ge 4\,H(1/4),
\)
so the right-hand side of \eqref{eq:explicit-d-kbar} is at least $4$; thus $d\ge 4$ (in particular $d>3$), as required for the instance.

\qed

\subsubsection{Proof of Lemma \ref{lm:regret_lower_bound_tilde2}}
Because all products in \(\widetilde{S}_n\) have the same feature vector \(\Feature_{\widetilde{U}_n}\), the optimal price for each product in \(\widetilde{S}_n\) is identical.
Let this price be \(p^*_{\widetilde{U}_n}\) and the corresponding price vector be \(\bp^*_{\widetilde{U}_n}\).
Then
\begin{equation*}
\begin{aligned}
&\, r\bigl(S^*(\bv_{W}, \Feature), \bp^*(\bv_{W}, \Feature)\bigr) - \sup_{\bp}\, r\bigl(\widetilde{S}_n, \bp\bigr)\\
\ge\;& r\bigl(S^*(\bv_{W}, \Feature), \bp^*_{\widetilde{U}_n}\bigr) - r\bigl(\widetilde{S}_n, \bp^*_{\widetilde{U}_n}\bigr)
=\;
\frac{K\,p^*_{\widetilde{U}_n}\,\exp(\bv_W^{\top}\Feature_{W}-p^*_{\widetilde{U}_n})}
{K\,\exp(\bv_W^{\top}\Feature_{W}-p^*_{\widetilde{U}_n})+1}
-
\frac{K\,p^*_{\widetilde{U}_n}\,\exp(\bv_W^{\top}\Feature_{\widetilde{U}_n}-p^*_{\widetilde{U}_n})}
{K\,\exp(\bv_W^{\top}\Feature_{\widetilde{U}_n}-p^*_{\widetilde{U}_n})+1}\\
=\;&
\frac{K\,p^*_{\widetilde{U}_n}\,\exp(-p^*_{\widetilde{U}_n})
\bigl(\exp(\bv_W^{\top}\Feature_{W})-\exp(\bv_W^{\top}\Feature_{\widetilde{U}_n})\bigr)}
{\bigl(K\,\exp(\bv_W^{\top}\Feature_{W}-p^*_{\widetilde{U}_n})+1\bigr)
 \bigl(K\,\exp(\bv_W^{\top}\Feature_{\widetilde{U}_n}-p^*_{\widetilde{U}_n})+1\bigr)}\\[2pt]
\overset{(a)}{\ge}\;&
\frac{K\,\bar p_l\,\exp(-\bar p_h)}
{\bigl(K\,\exp(\bar v-\bar p_l)+1\bigr)^2}
\bigl(\exp(\bv_W^{\top}\Feature_{W})-\exp(\bv_W^{\top}\Feature_{\widetilde{U}_n})\bigr)\\
\ge\;&
\frac{K\,\bar p_l\,\exp(\bv_W^{\top}\Feature_{\widetilde{U}_n}-\bar p_h)}
{\bigl(K\,\exp(\bar v-\bar p_l)+1\bigr)^2}
\left(
\exp\!\left(
\frac{\epsilon}{\sqrt{\bar K}}
\bigl(\bar K-|\widetilde U_n\cap W|\bigr)
\right)
-1
\right)\\
\ge\;&
\frac{K\,\bar p_l\,\exp(-\bar p_h)}
{\bigl(K\,\exp(\bar v-\bar p_l)+1\bigr)^2}
\cdot
\frac{\epsilon}{\sqrt{\bar K}}
\bigl(\bar K-|\widetilde U_n\cap W|\bigr),
\end{aligned}
\end{equation*}
where \((a)\) uses the facts that \(\bar p_l \le p^*_{\widetilde{U}_n}\le \bar p_h\) and \(\bv_W^{\top}\Feature_{\cdot}\le \bar v\).
We also used that
\(\bv_W^{\top}\Feature_{W}-\bv_W^{\top}\Feature_{\widetilde{U}_n}
=\frac{\epsilon}{\sqrt{\bar K}}\bigl(\bar K-|\widetilde U_n\cap W|\bigr)\ge 0\)
and \(\exp(x)-1\ge x\) for all \(x\ge 0\).
Setting
\[
c_{12}\;:=\;\frac{K\,\bar p_l\,\exp(-\bar p_h)}
{\bigl(K\,\exp(\bar v-\bar p_l)+1\bigr)^2},
\]
we obtain the stated inequality.
\qed

\subsubsection{Proof of Lemma \ref{lm:bound_kl_div2}}

\begin{equation*}
    \begin{aligned}
        \mathrm{KL}
        \left(
            \mathbb{P}_W 
            \left(
                \cdot \mid S_n, \bp_n
            \right) 
            \|  
            \mathbb{P}_{W \cup \{i\}}
            \left(
            \cdot \mid S_n, \bp_n
            \right)
        \right)
        =&\sum_{j \in S_n \cup\{0\}}
        q(j, S_n, \bp_n, \Feature_n; \bv_{W} )
        \log 
            \frac{q(j, S_n, \bp_n, \Feature_n; \bv_{W} )}{q(j, S_n, \bp_n, \Feature_n; \bv_{W \cup \{i\}} ) }
             \\
        \leq & 
        \sum_j q(j, S_n, \bp_n, \Feature_n; \bv_{W} )
        \frac{q(j, S_n, \bp_n, \Feature_n; \bv_{W} )-q(j, S_n, \bp_n, \Feature_n; \bv_{W \cup \{i\}} )}
        {q(j, S_n, \bp_n, \Feature_n; \bv_{W \cup \{i\}} )} \\
        =& \sum_j 
        \frac{\left|q(j, S_n, \bp_n, \Feature_n; \bv_{W} )-q(j, S_n, \bp_n, \Feature_n; \bv_{W \cup \{i\}})\right|^2}
        {q(j, S_n, \bp_n, \Feature_n; \bv_{W \cup \{i\}})}.
    \end{aligned}
\end{equation*}
where the inequality holds because $\log (1+y) \leq y$ for all $y>-1$. Because $q(j, S_n, \bp_n, \Feature_n; \bv_{W \cup \{i\}}) \geq 
        e^{\phigh}(1 + K e^{\epsilon-\plow})$ for all $j \in S_n \cup\{0\}$. Thus, the inequality above is reduced to
\begin{equation}\label{eq:kl_bound2}
    \begin{aligned}
        &\mathrm{KL}
            \left(
                \mathbb{P}_W 
                \left(
                    \cdot \mid S_n, \bp_n
                \right) 
                \|  
                \mathbb{P}_{W \cup \{i\}}
                \left(
                \cdot \mid S_n, \bp_n
                \right)
            \right)
        \leq
        e^{\phigh}(1 + K e^{\epsilon-\plow})
         \sum_{j \in S_n \cup\{0\}}
        \left|q(j, S_n, \bp_n, \Feature_n; \bv_{W} )-q(j, S_n, \bp_n, \Feature_n; \bv_{W \cup \{i\}})\right|^2.
    \end{aligned}
\end{equation}
We next upper bound $\left|q(j, S_n, \bp_n, \Feature_n; \bv_{W} )-q(j, S_n, \bp_n, \Feature_n; \bv_{W \cup \{i\}})\right|$ separately. 

First consider $j=0$. We have

\begin{equation*}
    \begin{aligned}
        &\left|
            q(j, S_n, \bp_n, \Feature_n; \bv_{W} )-q(j, S_n, \bp_n, \Feature_n; \bv_{W \cup \{i\}})
        \right| \\
        =& 
        \left|
            \frac{1}{1+\sum_{k \in S_n} \exp \left(\Feature_{kn}^{\top} \bv_W - \bp_k\right)}
            -\frac{1}{1+\sum_{k \in S_n} \exp \left(\Feature_{kn}^{\top} \bv_{W \cup\{i\}} - \bp_k \right)}
        \right| \\
        = &       
        \left|
            \frac{\sum_{k \in S_n} \left[\exp\left(\Feature_{kn}^{\top} \bv_{W \cup\{i\}} - \bp_k \right) -  \exp \left(\Feature_{kn}^{\top} \bv_W - \bp_k\right)\right]}
            {\left(
                1+\sum_{k \in S_n} \exp \left(\Feature_{kn}^{\top} \bv_W - \bp_k\right)
            \right) 
            \left(
                1+\sum_{k \in S_n} \exp \left(\Feature_{kn}^{\top} \bv_{W \cup\{i\}} - \bp_k \right)
            \right)}
        \right| 
        \\
        \leq& 
        \frac{1}{(1+K\exp(-\phigh))^2} \cdot \exp(\epsilon -\plow)
        \sum_{k \in S_n}
        \left|
            \Feature_{kn}^{\top}\left(\bv_W-\bv_{W \cup\{i\}}\right)
        \right| 
        \leq
        \frac{\exp(\epsilon-\plow) \epsilon}
        {\sqrt{\bar{K}}(1+K\exp(-\phigh))^2}  \sum_{\Feature_U \in S_n} \mathbf{1}\{i \in U\}.
\end{aligned}
\end{equation*}

Here, we use \(\exp(-\phigh) \leq \exp \left(\bz_{kn}^{\top} \bv_W - \bp_k\right) \leq \exp(\epsilon - \plow) \), \( \exp(-\phigh) \leq \exp \left(\bz_{kn}^{\top} \bv_{W \cup\{i\}} - \bp_k \right) \leq \exp(\epsilon - \plow)\), and \(|\exp(a) -  \exp (b)|\leq \exp(\max\{a, b\}) |a - b|\).

For $j>0$ corresponding to \(\Feature_{jn} = \Feature_U\) and $i \notin U$, we have
\begin{equation*}
    \begin{aligned}
        &\left|
            q(j, S_n, \bp_n, \Feature_n; \bv_{W} )-q(j, S_n, \bp_n, \Feature_n; \bv_{W \cup \{i\}})
        \right| \\
        =& 
        \left|
            \frac{\exp \left(\Feature_{jn}^{\top} \bv_W - \bp_j\right)}
            {1+\sum_{k \in S_n} \exp \left(\Feature_{kn}^{\top} \bv_W - \bp_k\right)}
            -\frac{\exp \left(\Feature_{jn}^{\top} \bv_{W \cup\{i\}} - \bp_j \right)}
            {1+\sum_{k \in S_n} \exp \left(\Feature_{kn}^{\top} \bv_{W \cup\{i\}} - \bp_k \right)}
        \right| \\
        =&  
        \exp \left(\Feature_{jn}^{\top} \bv_W - \bp_j\right)
        \left|
            \frac{1}{1+\sum_{k \in S_n} \exp \left(\Feature_{kn}^{\top} \bv_W - \bp_k\right)}
            -\frac{1}{1+\sum_{k \in S_n} \exp \left(\Feature_{kn}^{\top} \bv_{W \cup\{i\}} - \bp_k \right)}
        \right|\\
        \leq&  
        \exp(\epsilon-\plow)
        \frac{\exp(\epsilon-\plow) \epsilon}
        {\sqrt{\bar{K}}(1+K\exp(-\phigh))^2} \sum_{\Feature_U \in S_n} \mathbf{1}\{i \in U\}
        = 
        \frac{\exp(2\epsilon-2\plow) \epsilon}
        {\sqrt{\bar{K}}(1+K\exp(-\phigh))^2} \sum_{\Feature_U \in S_n} \mathbf{1}\{i \in U\}.
\end{aligned}
\end{equation*}

Here the  inequality holds because $\exp \left(\Feature_{jn}^{\top} \bv_W - \bp_j\right)=\exp \left(\Feature_{jn}^{\top} \bv_{W\cup\{i\}} - \bp_j\right) \leq \exp(\epsilon-\plow)$, since $i \notin U$.    Moreover, the number of indices \(j\) satisfying this condition is 
\[
K - \sum_{\Feature_U \in S_n} \mathbf{1}\{i \in U\}.
\]

For $j>0$ corresponding to \(\Feature_{jn} = \Feature_U\) and $i \in U$, we have

\begin{equation*}
    \begin{aligned}
        &\left|
            q(j, S_n, \bp_n, \Feature_n; \bv_{W} )-q(j, S_n, \bp_n, \Feature_n; \bv_{W \cup \{i\}})
        \right| \\
        =& 
        \left|
            \frac{\exp \left(\Feature_{jn}^{\top} \bv_W - \bp_j\right)}
            {1+\sum_{k \in S_n} \exp \left(\Feature_{kn}^{\top} \bv_W - \bp_k\right)}
            -\frac{\exp \left(\Feature_{jn}^{\top} \bv_{W \cup\{i\}} - \bp_j \right)}
            {1+\sum_{k \in S_n} \exp \left(\Feature_{kn}^{\top} \bv_{W \cup\{i\}} - \bp_k \right)}
        \right| \\
        \leq &  
        \exp \left(\Feature_{jn}^{\top} \bv_W - \bp_j\right)
        \left|
            \frac{1}{1+\sum_{k \in S_n} \exp \left(\Feature_{kn}^{\top} \bv_W - \bp_k\right)}
            -\frac{1}{1+\sum_{k \in S_n} \exp \left(\Feature_{kn}^{\top} \bv_{W \cup\{i\}} - \bp_k \right)}
        \right|\\
        &+
        \left|
            \exp \left(\Feature_{jn}^{\top} \bv_W - \bp_j\right) 
            -\exp \left(\Feature_{jn}^{\top} \bv_{W \cup \{i\}} - \bp_j\right) 
        \right|
        \frac{1}{1+\sum_{k \in S_n} \exp \left(\Feature_{kn}^{\top} \bv_{W \cup\{i\}} - \bp_k \right)}
        \\
        \leq&  
        \frac{\exp(2\epsilon-2\plow) \epsilon}
        {\sqrt{\bar{K}}(1+K\exp(-\phigh))^2}  \sum_{\Feature_U \in S_n} \mathbf{1}\{i \in U\}
        + \frac{\exp(\epsilon-\plow) \epsilon}{\sqrt{\bar{K}}}\cdot \frac{1}{1+K \exp(-\phigh)}\\
        \leq & 
        \frac{\exp(2\epsilon-2\plow) \epsilon}
        {\sqrt{\bar{K}}(1+K\exp(-\phigh))^2}  \sum_{\Feature_U \in S_n} \mathbf{1}\{i \in U\}
        + 
        \frac{\exp(\epsilon-\plow) \epsilon}
        {\sqrt{\bar{K}}(1+K \exp(-\phigh))}.
    \end{aligned}
\end{equation*} 
Besides,  the number of indices \(j\) satisfying this condition is \(\sum_{\Feature_U \in S_n} \mathbf{1}\{i \in U\}\).

Combining all upper bounds on $\left|
            q(j, S_n, \bp_n, \Feature_n; \bv_{W} )-q(j, S_n, \bp_n, \Feature_n; \bv_{W \cup \{i\}})
        \right|$ and \Cref{eq:kl_bound2}, we have
\begin{equation*}
    \begin{aligned}
    &\mathrm{KL}
    \left(
        \mathbb{P}_W 
        \left(
            \cdot \mid S_n, \bp_n
        \right) 
        \|  
        \mathbb{P}_{W \cup \{i\}}
        \left(
        \cdot \mid S_n, \bp_n
        \right)
    \right) \\
    \leq &  \exp(\phigh)(1 + {K} \exp(\epsilon-\plow))\epsilon^2/\bar{K}
     \bigg[  \frac{ \exp(2\epsilon-2\plow) }
        {(1+K\exp(-\phigh))^4}  \left(\sum_{\Feature_U \in S_n} \mathbf{1}\{i \in U\}\right)^2\\
        & + \frac{ \exp(4\epsilon-4\plow)}
        {(1+K\exp(-\phigh))^4} \left(\sum_{\Feature_U \in S_n} \mathbf{1}\{i \in U\}\right)^2
        \left(\bar{K} - \sum_{\Feature_U \in S_n} \mathbf{1}\{i \in U\}\right)
    \\
     & +  \frac{8 \exp(4\epsilon-4\plow)}
        {(1+K\exp(-\phigh))^4} \left(\sum_{\Feature_U \in S_n} \mathbf{1}\{i \in U\}\right)^2
        \sum_{\Feature_U \in S_n} \mathbf{1}\{i \in U\}
        + \frac{2\exp(2\epsilon-2\plow)}
        {(1+K \exp(-\phigh))^2} \sum_{\Feature_U \in S_n} \mathbf{1}\{i \in U\} \bigg]\\
    \leq & 
    \frac{c_{13} \epsilon^2}{\bar{K}}  \sum_{\Feature_U \in S_n} \mathbf{1}\{i \in U\},
    \end{aligned}
\end{equation*}

where  we use \(\sum_{\Feature_U \in S_n} \mathbf{1}\{i \in U\} \leq \bar{K}\), \(\epsilon \leq 1\) and 
\begin{equation} \label{df:c_13}
    c_{13} = 
    e^{\phigh}(1 + K e^{\epsilon-\plow})
    \left[
        \frac{2 e^{4-4\plow}\bar{K}}
        {(1+Ke^{-\phigh})^4} 
        +
        \frac{ e^{2-2\plow}\bar{K} }
        {(1+Ke^{-\phigh})^4} 
        +\frac{2e^{2-2\plow}}
        {(1+K e^{-\phigh} ))^2}
    \right].
\end{equation}

Finally, summing over \(n=1,\ldots,\lceil \Lambda T\rceil\) and applying the chain rule for KL divergence (together with the tower property) yields
\begin{equation*}
\begin{aligned}
&\mathrm{KL}\!\bigl(\mathbb{P}_{W}\,\|\,\mathbb{P}_{W\cup\{i\}}\bigr)
= \mathbb{E}_{W}\!\left[
\log \frac{\prod_{n=1}^{\lceil \Lambda T \rceil}\mathbb{P}_{W}\!\bigl(j_n \mid H'_{n-1}\bigr)}
               {\prod_{n=1}^{\lceil \Lambda T \rceil}\mathbb{P}_{W\cup\{i\}}\!\bigl(j_n \mid H'_{n-1}\bigr)}
\right] 
= \mathbb{E}_{W}\!\left[
\sum_{n=1}^{\lceil \Lambda T \rceil}
\log \frac{\mathbb{P}_{W}\!\bigl(j_n \mid H'_{n-1}\bigr)}
          {\mathbb{P}_{W\cup\{i\}}\!\bigl(j_n \mid H'_{n-1}\bigr)}
\right] \\[4pt]
&= \sum_{n=1}^{\lceil \Lambda T \rceil}
\mathbb{E}_{W}\!\left[
\log \frac{\mathbb{P}_{W}\!\bigl(j_n \mid H'_{n-1}\bigr)}
          {\mathbb{P}_{W\cup\{i\}}\!\bigl(j_n \mid H'_{n-1}\bigr)}
\right] 
= \sum_{n=1}^{\lceil \Lambda T \rceil}
\mathbb{E}_{W}\!\left[
\mathrm{KL}\!\left(
\mathbb{P}_{W}(j_n  \mid H'_{n-1})
\,\big\|\,
\mathbb{P}_{W\cup\{i\}}(j_n  \mid H'_{n-1})
\right)
\right] \\[4pt]
&= \mathbb{E}_{W}\!\left[
\sum_{n=1}^{\lceil \Lambda T \rceil}
\mathrm{KL}\!\left(
\mathbb{P}_{W}(j_n \mid S_n,\bp_n)
\,\big\|\,
\mathbb{P}_{W\cup\{i\}}(j_n  \mid S_n,\bp_n)
\right)
\right] \;\le\;
\mathbb{E}_W\!\left[\sum_{n=1}^{\lceil \Lambda T\rceil}\ 
c_{13}\,\epsilon^2\, \sum_{\Feature_U \in S_n} \mathbf{1}\{i \in U\}\right]
=
c_{13}\,\epsilon^2\,\mathbb{E}_W[N_i].    
\end{aligned}
\end{equation*}

\subsection{Case II}
\label{sec:case-ii}
\subsubsection{Intermediate Problem}

Consider the case $\bv = 0$,  the problem of interest reduces to the following:
\begin{equation*}
    \inf_{\pi}\sup_{\btheta, \bv}\mathcal{R}^{\pi}(T; \bv, \btheta, \Feature) 
    \geq\inf_{\pi} \sup_{ \btheta}\mathcal{R}^{\pi}(T; 0, \btheta, \Feature) .
\end{equation*}

In this setting, the problem can be formulated as follows:

\begin{instance}(Assortment-dependent arrivals with MNL demand)\label{inst:assort_dep}
The seller has $N$ distinct products, and  each product $i \in [N]$ is associated with a fixed feature vector $\Feature_i \in \mathbb{R}^\DimRank$.  For each period $t\in [T]$, the seller chooses an assortment $S_t \subseteq [N]$ with $|S_t| = K$ and sets prices of products in the assortment. For simplicity, we use a price vector 
$\bp_t = (p_{1t}, p_{2t}, \ldots, p_{Nt}) \in [\plow, \phigh]^{N}$ to represent the pricing decision (and set prices corresponding to the out-of-assortment products to be $\phigh$). Customer arrivals in period \(t\) follows Poisson distribution with mean
$\Lambda_t \;=\; \Lambda \exp\!\bigl(\btheta^{\top}\bx(S_t)\bigr)$,
where \(\Lambda>0\) is known, \(\btheta\in\mathbb{R}^\DimRank\) is an unknown parameter, and \(\bx(S_t)\in\mathbb{R}^\DimRank\) is a known function of the assortment.
Under the MNL model (price-only utilities), a customer arriving in period \(t\) chooses \(i\in S_t\) with probability
$\frac{\exp(-p_{it})}{\,1+\sum_{j\in S_t}\exp(-p_{jt})\,}$,
\text{and makes no purchase with the remaining probability.}
\end{instance}



Let \(S^*(\theta)\in\arg\max_{|S|=K}\exp\!\bigl(\btheta^{*\top}\bx(S)\bigr)\) and define
\(\Lambda^* := \Lambda \exp\!\bigl(\btheta^{*\top}\bx(S^*)\bigr)\).
Note that the per-customer revenue only depends on prices. Therefore, $r(S_t,\bp_t ) \le \sup_{S , \bp} r(S,\bp) = \sup_{\bp} r(S_t,\bp):= r^* $, which further gives 
\begin{equation}
\label{ineq:regret_lambda}
    \mathcal{R}^{\pi}(T; 0, \btheta, \Feature) =  \sum_{t = 1}^{T} \Lambda^* r^* - \sum_{t = 1}^{T} \E \left( \Lambda_t  r(S_t,\bp_t) \right) \geq \sum_{t = 1}^{T} (\Lambda^* -  \Exp\Lambda_t) r^*.
\end{equation}



By the definition of per-customer reward, we have
 \begin{equation}\label{eq:r_low_bound}
     r^* = \sup_{\bp \in \PriceSet} \sum_{i \in S_t} \frac{\exp(-p_{it})}{\,1+\sum_{j\in S_t}\exp(-p_{jt})\,} \geq \frac{K\plow\exp(-\phigh)}{1 + K\exp(-\plow)}.
 \end{equation}


Plugging in the lower bound of $r^*$ back to Inequality \eqref{ineq:regret_lambda} gives
\begin{equation}
     \mathcal{R}^{\pi}(T; 0, \btheta, \Feature)  \ge \sum_{t = 1}^{T} (\Lambda^* -  \Exp\Lambda_t) \frac{K\plow\exp(-\phigh)}{1 + K\exp(-\plow)}.
\end{equation}

The problem then reduces to a problem with constant per-customer-revenue but assortment-dependent arrivals.
\begin{instance}(Assortment-dependent arrivals with constant per-customer revenue)\label{inst:assort_dep_r}
We consider \(N\) products over \(T\) periods. In each period \(t\), the seller selects an assortment \(S_t\subseteq[N]\) with \(|S_t|= K\).
Customer arrivals in period \(t\) follows Poisson distribution with mean
$\Lambda_t \;=\; \Lambda \exp\!\bigl(\btheta^{*\top}\bx(S_t)\bigr)$,
where \(\Lambda>0\) is known, \(\btheta^*\in\mathbb{R}^\DimRank\) is unknown, and \(\bx(S_t)\in\mathbb{R}^\DimRank\) is a known function of the assortment.

The expected revenue for each customer is $r:=  \dfrac{K\plow\exp(-\phigh)}{1 + K\exp(-\plow)}$.


\end{instance}


\paragraph{Adversarial construction.}
\begin{instance}[Adversarial construction]\label{df:worst_case_example_3}
Suppose $\min\{\DimRank - 2, N \}\ge K, \Lambda \geq 1 $.

Let \(\bar{K}, \DimRank \in \mathbb{Z}_{++}\) satisfy
$\bar{K} \;=\; \min\!\Bigl\{\Bigl\lfloor\frac{\DimRank - K + 1}{3}\Bigr\rfloor,\; K\Bigr\}$,
$d \;=\; \DimRank - K + \bar{K}$,
and fix a small parameter \(\epsilon \in \bigl(0,\,1/\sqrt{{K}}\bigr)\).
For each \(W \subseteq [d]\) with \(|W|=\bar{K}\), define \(\btheta_W \in \mathbb{R}^{\DimRank}\) coordinatewise by
\[
[\btheta_W]_i \;=\;
\begin{cases}
\epsilon, & i \in W,\\[2pt]
0, & i \in [d]\setminus W,\\[2pt]
\epsilon, & i \in \{d+1,\ldots,\DimRank\}.
\end{cases}
\]
Clearly, $\|\btheta_W\|_2 \le 1$, satisfying the boundedness required in \Cref{assump:theta}. 
Collect these in
\[
\mathbf{\Theta}
\;=\; \{\btheta_W : W \in \mathcal{W}_{\bar{K}}\}
\;=\; \{\btheta_W : W \subseteq [d],\, |W|=\bar{K}\},
\]
where \(\mathcal{W}_{\bar{K}}\) is the family of all \(\bar{K}\)-subsets of \([d]\).
Finally, define the assortment-dependent vector \(\bx(S_t)\in\mathbb{R}^{\DimRank}\) by
\[
[\bx(S_t)]_i \;=\;
\begin{cases}
\dfrac{\bar{x}}{\sqrt{K}}, & i \in S_t \cap [\DimRank],\\
0, & i \notin S_t \cap [\DimRank].
\end{cases}
\]
Since $\|\bx\|_2 \le \bar{x}$, the constructed   assortment-dependent vector \(\bx\)  satisfy the boundedness condition in \Cref{assump:x}.

\end{instance}

We use \(\mathbb{E}_W\) and \(\mathbb{P}_W\) to represent the law parameterized by \(\btheta_W\) and the 
policy \(\pi\), respectively. The following lemma establishes a lower bound for 
\( \Lambda^* - \Lambda_t\) by comparing \(S_t\) with \(W\).

\begin{lemma}\label{lm:lambda_lower_bound_tilde}
    We have 
    \[
      \Lambda^* - \Lambda_t
      \;\geq\;         
      \frac{\Lambda \bar{x}\epsilon }{\sqrt{K}}
            \left(\bar{K}- \bigl|S_t \cap W\bigr|\right).
    \]
\end{lemma}

Next, we establish a lower bound on the cumulative regret. Define \(\widetilde{N}_i := \sum_{t=1}^{T} \mathbf{1}\{i \in S_t\}\). By \Cref{lm:lambda_lower_bound_tilde}, it follows that, $\forall\,W \in \mathcal{W}_{\bar K}$,
\begin{equation*}
    \begin{aligned} 
        &\mathcal{R}^{\pi}(T; 0, \btheta, \Feature) 
        \geq
        \mathbb{E}_W \sum_{t=1}^{T} 
        \Bigl(
            \Lambda^* - \Lambda_t
        \Bigr) r 
        \geq \mathbb{E}_W \sum_{t=1}^{T}   
      \frac{\Lambda \bar{x}\epsilon }{\sqrt{K}}
       \left(\bar{K}- \bigl|S_t \cap W\bigr|\right) r
        = \;          
      \frac{\Lambda \bar{x}\epsilon }{\sqrt{K}}
        \Bigl(\bar{K} T - \sum_{i \in W} \mathbb{E}_W[\widetilde{N}_i]\Bigr) r.
    \end{aligned}
\end{equation*}

Denote \(\mathcal{W}_{\bar{K}}^{(i)} := \{W \in \mathcal{W}_{\bar{K}} : i \in W\}\) and \(\mathcal{W}_{\bar{K}-1} := \{W \subseteq[\DimRank] : |W|= \bar{K}- 1\}\).
Let \(P_{\btheta}\) be the uniform prior on \(\{\btheta_W:\, W\in\mathcal W_{\bar K}\}\); that is, sample \(W\sim{\rm Unif}(\mathcal W_{\bar K})\) and set \(\btheta=\btheta_W\). Therefore,

\begin{equation*}
    \begin{aligned}
        \E_{\btheta \sim P_{\btheta} } 
        \Bigl[ 
            \mathcal{R}^{\pi}(T; 0, \btheta, \Feature) 
        \Bigr] 
        \geq &      
      \frac{\Lambda \bar{x}\epsilon }{\sqrt{K}}
        \biggl(
            \bar{K} T - 
                \frac{1}{|\mathcal{W}_{\bar{K}}|} 
                \sum_{W \in \mathcal{W}_{\bar{K}}}\sum_{i\in W}
                \mathbb{E}_W[\widetilde{N}_i]
        \biggr) r 
        =    
      \frac{\Lambda \bar{x}\epsilon }{\sqrt{K}}
        \biggl(
            \bar{K} T - 
                \frac{1}{|\mathcal{W}_{\bar{K}}|} 
                \sum_{i=1}^d \sum_{W \in \mathcal{W}_{\bar{K}}^{(i)}} \mathbb{E}_W[\widetilde{N}_i]
        \biggr) r \\
        = &           
      \frac{\Lambda \bar{x}\epsilon }{\sqrt{K}}
        \biggl(
            \bar{K} T -
                \frac{1}{|\mathcal{W}_{\bar{K}}|}
                \sum_{W \in \mathcal{W}_{\bar{K}-1}} \sum_{i \notin W} \mathbb{E}_{W \cup \{i\}}[\widetilde{N}_i]
        \biggr) r  \\
        \geq&         
      \frac{\Lambda \bar{x}\epsilon }{\sqrt{K}}
        \biggl(
            \bar{K} T -
                \frac{|\mathcal{W}_{\bar{K}-1}|}{|\mathcal{W}_{\bar{K}}|} 
                \max_{W \in \mathcal{W}_{\bar{K} - 1}} 
                \sum_{i \notin W} \mathbb{E}_{W \cup \{i\}}[\widetilde{N}_i] 
        \biggr)r\\
        =& 
      \frac{\Lambda \bar{x}\epsilon }{\sqrt{K}}
        \biggl(
            \bar{K} T  -
                \frac{|\mathcal{W}_{\bar{K}-1}|}{|\mathcal{W}_{\bar{K}}|} 
                \max_{W \in \mathcal{W}_{\bar{K}-1}} 
                \sum_{i \notin W}
                \left[
                \mathbb{E}_{W \cup \{i\}}[\widetilde{N}_i]
                - \mathbb{E}_W[\widetilde{N}_i] 
                + \mathbb{E}_W[\widetilde{N}_i] 
                \right]
        \biggr) r.
    \end{aligned}
\end{equation*}

Since \(\sum_{i \notin W} \mathbb{E}_W[\widetilde{N}_i] \leq \sum_{i=1}^\DimRank \mathbb{E}_W[\widetilde{N}_i] \leq \bar{K} T \) and \(\frac{|\mathcal{W}_{\bar{K}-1}|}{|\mathcal{W}_{\bar{K}}|} =\frac{\binom{\DimRank}{\bar{K} - 1}}{\binom{\DimRank}{\bar{K}}} =\frac{\bar{K}}{(\DimRank - \bar{K} + 1)} \leq 1/3,\) we finally get
\begin{equation*}
    \begin{aligned}
        &\E_{\btheta \sim P_{\btheta} } 
        \Bigl[ 
            \mathcal{R}^{\pi}(T; 0, \btheta, \Feature) 
        \Bigr]
        \geq 
      \frac{\Lambda \bar{x}\epsilon }{\sqrt{K}}
        \Bigl(
            \frac{2}{3}
            \bar{K} T - 
            \frac{\bar{K}}{d - \bar{K} + 1}
            \max_{W \in \mathcal{W}_{\bar{K}-1}}
            \sum_{i \notin W}
            \bigl|
                \mathbb{E}_{W \cup \{i\}}[\widetilde{N}_i] - \mathbb{E}_W[\widetilde{N}_i]
            \bigr|
        \Bigr) r.
    \end{aligned}
\end{equation*}

\paragraph{Pinsker's inequality.}
Finally, we focus on upper bounding 
\(\bigl|\mathbb{E}_{W \cup\{i\}}\bigl[\widetilde{N}_i\bigr]-\mathbb{E}_W\bigl[\widetilde{N}_i\bigr]\bigr|\) 
for any \(W \in \mathcal{W}_{\bar K-1}\):
\begin{equation*}
    \begin{aligned}
        \Bigl|
            \mathbb{E}_W
            \bigl[
                \widetilde{N}_i
            \bigr] 
            - 
            \mathbb{E}_{W \cup \{i\}}
            \bigl[
                \widetilde{N}_i
            \bigr]
        \Bigr| 
        \leq&
        \sum_{j=0}^{T} 
        j \,\Bigl|
            \mathbb{P}_W
            \bigl[
                \widetilde{N}_i=j
            \bigr]
            -
            \mathbb{P}_{W \cup \{i\}}
            \bigl[
                \widetilde{N}_i=j
            \bigr]
        \Bigr|
        \leq 
        T\cdot 
        \sum_{j=0}^{T}
        \Bigl|
            \mathbb{P}_W
            \bigl[
                \widetilde{N}_i=j
            \bigr]
            -
            \mathbb{P}_{W \cup \{i\}}
            \bigl[
                \widetilde{N}_i=j
            \bigr]
        \Bigr|\\
        \leq& 2T\cdot 
        \bigl\|
            \mathbb{P}_W - \mathbb{P}_{W \cup \{i\}}
        \bigr\|_{\mathrm{TV}} 
        \;\leq\;
        T\cdot 
        \sqrt{
            2 \mathrm{KL}(\mathbb{P}_W \,\|\, \mathbb{P}_{W \cup \{i\}})
        }.
    \end{aligned}
\end{equation*}
where \(\|\mathbb{P}_W - \mathbb{P}_{W \cup \{i\}}\|_{\mathrm{TV}}\) is the total variation distance and \(\mathrm{KL}(\mathbb{P}_W \| \mathbb{P}_{W \cup \{i\}})\) is the Kullback-Leibler divergence.

The following lemma bounds the KL divergence.
\begin{lemma}\label{lm:bound_kl_div_poi}
For any \(W \in \mathcal{W}_{\bar{K} - 1}\) and \(i \in[\DimRank ]\),
\[
\mathrm{KL}(\mathbb{P}_W \,\|\, \mathbb{P}_{W \cup \{i\}}) 
\;\leq\; \exp(2\bar{x}) \Lambda \,\cdot\, \frac{\bar x^2}{K}\, \mathbb{E}_W[\widetilde{N}_i] \,\cdot\, \epsilon^2.
\]
\end{lemma}

Using \Cref{lm:bound_kl_div_poi} and the bound above, we obtain
\[
\E_{\btheta \sim P_{\btheta} } 
\Bigl[ 
    \mathcal{R}^{\pi}(T; 0, \btheta, \Feature) 
\Bigr] 
\;\geq\;
      \frac{\Lambda \bar{x}\epsilon }{\sqrt{K}}
\Bigl(
    \frac{2\bar{K}T}{3}\;-\;
    \frac{\bar{K} T}{\DimRank  - \bar{K} + 1}\sum_{i=1}^\DimRank 
    \sqrt{ 2\exp(2\bar{x})\Lambda \,\frac{\bar x^2}{K}\, \mathbb{E}_W[\widetilde{N}_i]  \,\epsilon^2}
\Bigr) r.
\]
Applying the Cauchy-Schwarz inequality and since \(\sum_{i=1}^\DimRank  \mathbb{E}_W[\widetilde{N}_i]  \leq {K}\,T\),
\[
\sum_{i=1}^\DimRank 
\sqrt{
    2e^2 \Lambda \,\frac{\bar x^2}{K}\, \mathbb{E}_W[\widetilde{N}_i] \,\epsilon^2
    } 
\;\;\leq\;\;
\sqrt{ 2\exp(2\bar{x}) \Lambda }\,\frac{\bar x}{\sqrt K}\,
\sqrt{
    \DimRank \sum_{i=1}^\DimRank  \mathbb{E}_W[\widetilde{N}_i]
}\,\epsilon \leq\;
\sqrt{ 2\exp(2\bar{x}) \DimRank  \Lambda T}\,{\bar x}\,
\,\epsilon.
\]
Thus, we obtain
\begin{equation*}
\begin{aligned}
    \E_{\btheta \sim P_{\btheta} } \mathcal{R}^{\pi}(T; 0, \btheta, \Feature) 
    \geq &\;         
      \frac{\Lambda \bar{x}\epsilon }{\sqrt{K}}
    \Bigl(
        \frac{2\bar{K}T}{3}\;-\;
        \frac{\bar{K} T}{\DimRank  - \bar{K} + 1}\,
        \sqrt{
            2\exp(2\bar{x}) \bar{x}^2\DimRank  \Lambda T \,\epsilon^2
        }
    \Bigr) r.
\end{aligned}
\end{equation*}
\begin{itemize}
    \item If 
        $\Lambda T \ \ge\ \frac{2(\DimRank - K + 1)K}{9\,\exp(2\bar{x})\,\bar x^{2}}$,
setting 
    $\epsilon = \frac{\DimRank  - \bar{K} + 1}{3 \exp(\bar{x}) \bar{x}} \sqrt{\frac{2}{\DimRank \Lambda T}}$,
we have \(\epsilon \in \bigl(0,\,1/\sqrt{K}\bigr]\). Thus, we obtain

\begin{equation*}
\begin{aligned}
    \E_{\btheta \sim P_{\btheta} } \mathcal{R}^{\pi}(T; 0, \btheta, \Feature) 
    \geq &
    \frac{(\DimRank  - \bar{K} + 1)\bar{K} }{9\sqrt{2}\exp(\bar{x})\bar{x} \sqrt{\DimRank  K}} \sqrt{\Lambda T} r \geq 
    \frac{\sqrt{\DimRank }\bar{K} }{12\sqrt{2}\exp(\bar{x})\bar{x} \sqrt{K}} \sqrt{\Lambda T} r \\
    \geq & 
    \frac{1 }{12\sqrt{2}\exp(\bar{x})\bar{x} } \sqrt{\Lambda T} r \min\{\sqrt{\DimRank K}, \sqrt{\frac{4(\DimRank - K - 1)}{3K}} \frac{(\DimRank - K - 1)}{3}\} .
\end{aligned}
\end{equation*}
\item If 
        $\Lambda T \ <\ \frac{\DimRank - K + 1}{12\,\exp(2\bar{x})\,\bar x^{2}}\;K$,
setting $\epsilon = 1/\sqrt{K}$, we have 
\begin{equation*}
    \begin{aligned}
        & \E_{\btheta \sim P_{\btheta} } \mathcal{R}^{\pi}(T; 0, \btheta, \Feature) \geq  
        \frac{\Lambda \bar{x}\epsilon}{\sqrt{K}}
        \frac{\bar{K}T}{3}
        \geq 
        \frac{ \bar{x}}{K}
        \min\!\Bigl\{\Bigl\lfloor\frac{\DimRank - K + 1}{3}\Bigr\rfloor,\; K\Bigr\} 
        \frac{\Lambda T}{3} .
    \end{aligned}
\end{equation*}
\end{itemize}

Combining two cases, we have 
\begin{equation*}
    \begin{aligned}
        \E_{\btheta \sim P_{\btheta} } \mathcal{R}^{\pi}(T; 0, \btheta, \Feature) \geq  c_{9,6} \sqrt{\Lambda T}  
        \min\!\Bigl\{\Bigl\lfloor\frac{\DimRank - K + 1}{3}\Bigr\rfloor,\; K\Bigr\} ,
    \end{aligned}
\end{equation*}
for some positive constant $c_{9,3}$ only depends on $ K, \plow, \phigh, \bar{x}, \DimRank$.

Besides, if \(K \leq \dfrac{\DimFeature + 1}{4}\), we have \(d = \DimFeature, \bar{K} = K \) and
\begin{equation}
    \begin{aligned}
        \liminf_{T \to \infty} \dfrac{ \E_{\btheta \sim P_{\btheta} } \mathcal{R}^{\pi}(T; 0, \btheta, \Feature)}{\sqrt{\DimRank  \Lambda T}} \geq 
        \frac{K\sqrt{K}\plow\exp(-\phigh)}{12\sqrt{2}\exp(\bar{x})\bar{x}(1 + K\exp(-\plow))}
        := \LowerC{3}.
    \end{aligned}
\end{equation}
\begin{remark}\label{remark:c_9_3}
    When \(\plow = \Omega (\log K)\) and \(\phigh = \Omega (\log K)\), we have 
    \(\LowerC{3} = \Omega(\frac{\sqrt{K}\log K}{\exp(\bar{x}) \bar{x}})\). 

    When \(\plow = \Omega (1)\) and \(\phigh = \Omega (1)\), we have 
    \(\LowerC{3} = \Omega(\frac{\sqrt{K}}{\exp(\bar{x}) \bar{x}})\). 
    
\end{remark}

\subsubsection{Proof of Lemma \ref{lm:lambda_lower_bound_tilde}}
We have
\begin{equation*}
    \begin{aligned}
            \Lambda^* - \Lambda_t
            =  \Lambda (\exp( \btheta_W^\top \bx(W)  ) - \exp(\btheta_{W}^\top \bx(S_t)) )
            \geq   \Lambda ( \btheta_W^\top \bx(W)  
            - \btheta_{W}^\top \bx(S_t)) 
            = \frac{\Lambda \bar{x}\epsilon }{\sqrt{K}}
            \left(\bar{K}- \bigl|S_t \cap W\bigr|\right).
    \end{aligned}
\end{equation*}

\subsubsection{Proof of Lemma \ref{lm:bound_kl_div_poi}}
Note that for two Poisson distribution \(P\left(\lambda_1\right), P\left(\lambda_2\right)\), we have 
\begin{equation*}
    \begin{aligned}
        D_{\mathrm{KL}}\left(P\left(\lambda_1\right) \| P\left(\lambda_2\right)\right)
        =
        \lambda_1 \log \left(\frac{\lambda_1}{\lambda_2}\right)-\lambda_1+\lambda_2 
        \leq  \lambda_1 \frac{\lambda_1 - \lambda_2}{\lambda_2} -\lambda_1+\lambda_2 
        =\frac{(\lambda_1 - \lambda_2)^2}{\lambda_2} .
    \end{aligned}
\end{equation*}
Note that 
\begin{equation*}
    \begin{aligned}
        \Lambda
        \exp\left( 
            \btheta_{W \cup \{i\}}^\top \bx(S_t)) 
        \right) 
        - \Lambda 
        \exp\left(
            \btheta_W^\top \bx(S_t)
        \right)  
        \leq  & 
          \exp(\bar{x}) \Lambda \epsilon \frac{\bar x}{\sqrt K}
           \mathbf{1}\{i \in S_t\},\\
        \Lambda
        \exp\left( 
            \btheta_{W \cup \{i\}}^\top \bx(S_t) 
        \right) \geq & \Lambda.
    \end{aligned}
\end{equation*}
Substituting \(\lambda_1, \lambda_2\), we have 
\begin{equation*}
    \begin{aligned}
        \mathrm{KL}\!\left(
        \mathbb{P}_{W}(\cdot \mid S_t,\bp_t)
        \,\big\|\,
        \mathbb{P}_{W\cup\{i\}}(\cdot \mid S_t,\bp_t)
        \right)
        \;\leq\;&
        \frac{\left( \exp(\bar{x}) \Lambda \frac{\bar x}{\sqrt K}\epsilon 
         \mathbf{1}\{i \in S_t\}\right)^2}{\Lambda} 
         = 
        \exp(2\bar{x}) \Lambda \frac{\bar x^2}{ K}\epsilon^2   \mathbf{1}\{i \in S_t\}.
    \end{aligned}
\end{equation*}

Finally, summing over \(t=1,\ldots, T \) and applying the chain rule for KL divergence (together with the tower property) yields
\begin{equation*}
\begin{aligned}
\mathrm{KL}\!\bigl(\mathbb{P}_{W}\,\|\,\mathbb{P}_{W\cup\{i\}}\bigr)
&= \mathbb{E}_{W}\!\left[
\log \frac{\prod_{t=1}^{T}\mathbb{P}_{W}\!\bigl(\cdot\mid H_{t-1}\bigr)}
               {\prod_{t=1}^{T}\mathbb{P}_{W\cup\{i\}}\!\bigl(\cdot \mid H_{t-1}\bigr)}
\right] 
= \mathbb{E}_{W}\!\left[
\sum_{t=1}^{T}
\log \frac{\mathbb{P}_{W}\!\bigl(\cdot \mid H_{t-1}\bigr)}
          {\mathbb{P}_{W\cup\{i\}}\!\bigl(\cdot \mid H_{t-1}\bigr)}
\right] \\[4pt]
&= \sum_{t=1}^{T}
\mathbb{E}_{W}\!\left[
\log \frac{\mathbb{P}_{W}\!\bigl(\cdot \mid H_{t-1}\bigr)}
          {\mathbb{P}_{W\cup\{i\}}\!\bigl(\cdot \mid H_{t-1}\bigr)}
\right] 
= \sum_{t=1}^{T}
\mathbb{E}_{W}\!\left[
\mathrm{KL}\!\left(
\mathbb{P}_{W}(\cdot\mid H_{t-1})
\,\big\|\,
\mathbb{P}_{W\cup\{i\}}(\cdot\mid H_{t-1})
\right)
\right] \\[4pt]
&= \mathbb{E}_{W}\!\left[
\sum_{t=1}^{T}
\mathrm{KL}\!\left(
\mathbb{P}_{W}(\cdot \mid S_t,\bp_t)
\,\big\|\,
\mathbb{P}_{W\cup\{i\}}(\cdot \mid S_t,\bp_t)
\right)
\right] \\[4pt]
&\;\le\;
\mathbb{E}_W\!\left[\sum_{t=1}^{T}\ \exp(2\bar{x}) \Lambda \frac{\bar x^2}{ K}\epsilon^2 \mathbf{1}\{i \in S_t\}\right]
=
\exp(2\bar{x}) \Lambda \frac{\bar x^2}{ K}\epsilon^2\,\mathbb{E}_W[N_i].    
\end{aligned}
\end{equation*}
\subsection{Conclusion (aggregation over the three instances)}

From the three instances established above, we obtain the following three claims.

\begin{enumerate}
\item
If $\min\{\DimFeature-2,\,N\}\ge K$ and $\Lambda\ge 1$, then there exists a corresponding problem instance such that
\begin{equation*}
    \inf_{\pi}\sup_{\bv,\btheta,\bz}
        \mathcal{R}^{\pi}(T;\bv,\btheta,\bz)
    \ \ge\
     c_{9, 4} \min\!\left\{1,\ \sqrt{\frac{\bar v^{2}}{\DimFeature}}\right\} \sqrt{\Lambda T}.
\end{equation*}

\item
If $\Lambda\ge 1$ and $\log\!\frac{N-\DimFeature}{K}\ge 8\log 2-\frac{11}{4}\log 3$, then there exists a corresponding problem instance such that
\begin{equation*}
    \inf_{\pi}\sup_{\bv,\btheta,\bz}
        \mathcal{R}^{\pi}(T;\bv,\btheta,\bz)
    \ \ge\
     c_{9, 5}\sqrt{\log\!\left(\frac{N-\DimFeature}{K}\right)}
     \min\left\{\frac{\bar{v}}{\sqrt{\DimFeature}},\, 1\right\}\sqrt{\Lambda T}.
\end{equation*}

\item
If $\min\{\DimRank-2,\,N\}\ge K$ and $\Lambda\ge 1$, then there exists a corresponding problem instance such that
\begin{equation*}
    \inf_{\pi}\sup_{\bv,\btheta,\bz}
        \mathcal{R}^{\pi}(T;\bv,\btheta,\bz)
    \ \ge\
     c_{9,6}\min\!\Bigl\{\Bigl\lfloor\frac{\DimRank - K + 1}{3}\Bigr\rfloor,\; K\Bigr\}\sqrt{\Lambda T}.
\end{equation*}
\end{enumerate}

Besides, define the following events (conditions):
$\mathcal{E}_4 := \left\{\min\{\DimFeature-2,\,N\}\ge K,\ \Lambda\ge 1\right\},
\mathcal{E}_5 := \left\{\Lambda\ge 1,\ \log\!\frac{N-\DimFeature}{K}\ge 8\log 2-\frac{11}{4}\log 3\right\},
\mathcal{E}_6 := \left\{\min\{\DimRank-2,\,N\}\ge K,\ \Lambda\ge 1\right\}.$

Consequently, if at least one of $\mathcal{E}_4,\mathcal{E}_5,\mathcal{E}_6$ holds, then there exists a problem instance such that
\begin{equation*}
\inf_{\pi}\sup_{\bv,\btheta,\bz}
    \mathcal{R}^{\pi}(T;\bv,\btheta,\bz)
\ \ge\
c_9\,\sqrt{\Lambda T},
\end{equation*}
where we set
\begin{equation*}
\begin{aligned}
c_9
\;:=\;
\max\Bigg\{
\mathbf{1}\{\mathcal{E}_4\}\,c_{9,4}\min\!\left\{1,\ \sqrt{\bar v^{2}/\DimFeature}\right\},\ 
\mathbf{1}\{\mathcal{E}_5\}\,c_{9,5}\sqrt{\log\!\left(\frac{N-\DimFeature}{K}\right)}
\min\!\left\{\frac{\bar v}{\sqrt{\DimFeature}},\,1\right\},\ \\
\mathbf{1}\{\mathcal{E}_6\}\,c_{9,6}\min\!\Bigl\{\Bigl\lfloor\frac{\DimRank-K+1}{3}\Bigr\rfloor,\ K\Bigr\}
\Bigg\}.
\end{aligned}
\end{equation*}
By construction, $c_9>0$ depends only on $\DimFeature$, $\DimRank$, $K$, $\plow$, $\phigh$, and $\bar x$.



Besides, if we denote 
\begin{equation*}
    \begin{aligned}
        \mathcal{E}_1 = 
        \left\{K \leq \dfrac{\DimFeature + 1}{4}\right\},\quad 
        \mathcal{E}_2 = \left\{\DimFeature \leq \left\lfloor \frac{\log\!\big((N-\DimFeature)/K\big)-\tfrac14\log 3}{\,H(1/4)\,}\right\rfloor\right\}, \quad
        \mathcal{E}_3 = \left\{K \leq \dfrac{\DimFeature + 1}{4}\right\}.
    \end{aligned}
\end{equation*}

We have 
\begin{equation*}
\begin{aligned}
    &\liminf_{T\to \infty} \frac{\mathcal{R}^{\pi}(T;\bv,\btheta,\bz)}{\sqrt{\Lambda T}}\\ 
    =& \max\left\{
      \mathbf{1}\{\mathcal{E}_1\}\mathbf{1}\{\mathcal{E}_4\}   \frac{c_{10}\sqrt{\DimFeature K}}{120\sqrt{2c_{11}}},
      \mathbf{1}\{\mathcal{E}_2\}\mathbf{1}\{\mathcal{E}_5\}   \frac{c_{12}\DimFeature}{120\sqrt{2c_{13}}},
      \mathbf{1}\{\mathcal{E}_3\}\mathbf{1}\{\mathcal{E}_6\}     \frac{K\plow\exp(-\phigh)}{12\sqrt{2}\exp(\bar{x})\bar{x}(1 + K\exp(-\plow))}
    \right\}\\
    =&  \max\left\{
      \mathbf{1}\{\mathcal{E}_1\}\mathbf{1}\{\mathcal{E}_4\} \LowerC{1},
      \mathbf{1}\{\mathcal{E}_2\}\mathbf{1}\{\mathcal{E}_5\}  \LowerC{2},
      \mathbf{1}\{\mathcal{E}_3\}\mathbf{1}\{\mathcal{E}_6\}  \LowerC{3}
    \right\}.
\end{aligned}
\end{equation*}

\end{document}